\documentclass[10pt]{article}
\usepackage{blindtext}
\usepackage{enumerate}
\usepackage[OT1]{fontenc}
\usepackage{caption}
\usepackage[usenames,dvipsnames]{xcolor}
\usepackage[cal=boondox]{mathalfa}
\usepackage{geometry}
\usepackage{dsfont}
\usepackage{pgfplots}
\usepackage{smile}
\usepackage{multirow}
\usepackage{rotating}

\usepackage{esvect}
\usepackage[table]{xcolor}
\usepackage{ltablex}
\usepackage{tabularx}
\usepackage{wrapfig}
\usepackage{enumitem}
\usepackage{tikz}
\usetikzlibrary{patterns}
\usetikzlibrary{arrows.meta, positioning, calc,fit}
\usepackage{makecell}
\usepackage{booktabs}
\usepackage{siunitx,colortbl}
\definecolor{SDEblue}{RGB}{28 58 88}
\definecolor{cc1}{rgb}{1.0, 0.44, 0.37}
\definecolor{cc2}{rgb}{0.0, 0.2, 0.6}
\definecolor{cc3}{RGB}{255, 191, 0}
\definecolor{cc4}{RGB}{0, 128, 128}
\usepackage[colorlinks,
            linkcolor=black,
            anchorcolor=blue,
            citecolor=black
            ]{hyperref}
 
\usepackage{algorithm}
\usepackage{algorithmic}
\usepackage{bbm}
\usepackage{acronym}
\usepackage{tcolorbox}
\usepackage{hhline}
\setlist[enumerate]{leftmargin=*}
\definecolor{lightroyalblue}{HTML}{F6F8FD} 
\definecolor{royalblue}{HTML}{4169E1}
\definecolor{lighterblue}{HTML}{f2fafd}  
\newtcolorbox{abox}{colback=lightroyalblue,colframe=black}
\definecolor{LightCyan}{rgb}{.9, .95, 1.}
\definecolor{rowblue}{RGB}{232,241,255}

\usepackage{hyperref}
\usepackage[capitalize]{cleveref}
\usepackage[protrusion=false,expansion=true]{microtype}

\def\##1\#{\begin{align}#1\end{align}}
\def\$#1\${\begin{align*}#1\end{align*}}

\usepackage{arydshln}

\newcommand{\RM}[1]{\left(\romannumeral#1\right)}
\newcommand{\URM}[1]{\uppercase\expandafter{\romannumeral#1}}

\newcommand{\epsFN}{\epsilon_{\mathrm{FN}}}
\newcommand{\epsI}{\epsilon_{1}}
\newcommand{\epsII}{\epsilon_{2}}

\newcommand{\Es}{\mathcal{E}^{*}_{\attn}}

\newcommand{\pit}{P_{\mathrm{in}}}
\newcommand{\Nit}{N_{\mathrm{in}}}
\newcommand{\esb}{\mathcal{E}^{*}_{\mathrm{P}}}
\newcommand{\Lba}{L}
\newcommand{\Uba}{U}

\newcommand{\textbrm}[1]{\textcolor{blue}{\textbf{\textup{#1}}}}

\acrodef{rope}[RoPE]{Rotary Position Embedding}
\acrodef{icl}[ICL]{In-context Learning}
\acrodef{mha}[MHA]{Multi-Head Attention}
\acrodef{llm}[LLM]{Large Language Model}
\acrodef{ood}[OOD]{Out-Of-Distribution}
\acrodef{sdh}[SDH]{Slash-Dominant Head}
\acrodef{svd}[SVD]{Singular-Value Decomposition}
\acrodef{pe}[PE]{Position Embedding}

\newcommand{\Attn}{\mathcal{S}}
\newcommand{\attn}{S}
\newcommand{\softmax}{\mathrm{softmax}}

\newcommand{\xbq}{\xb_{\mathrm{q}}}
\newcommand{\yq}{\hat{y}_{\mathrm{q}}}
\newcommand{\Eq}{E_{\mathrm{q}}}
\newcommand{\InP}{\mathrm{InP}}

\newcommand{\RMat}[3]{{R}_{#1,#2}}
\newcommand{\RoPE}[3]{\Re_{#1,#2}}

\newcommand{\CSA}{\mathrm{CSA}}

\theoremstyle{plain}
\newtheorem{theorem}{Theorem}[section]
\newtheorem{lemma}[theorem]{Lemma}

\newtheorem{corollary}[theorem]{Corollary}
\newtheorem{assumption}[theorem]{Assumption}
\newtheorem{definition}[theorem]{Definition}

\ifdefined\proof
    \renewenvironment{proof}{\par\noindent{\bf Proof\ }}{\hfill$\blacksquare$\\[2mm]}
\else
    \newenvironment{proof}{\par\noindent{\bf Proof\ }}{\hfill$\blacksquare$\\[2mm]}
\fi

\newtheorem{hypothesis}{Induction Hypothesis}[section]

\title{Demystifying the Slash Pattern in Attention: The Role of RoPE}

\author{Yuan Cheng\textsuperscript{1,$\ast$}\quad Fengzhuo Zhang\textsuperscript{1,$\ast$,$\dagger$}\quad Yunlong Hou\textsuperscript{1,$\ast$}\quad Cunxiao Du\textsuperscript{2} \\\quad Chao Du\textsuperscript{2}\quad Tianyu Pang\textsuperscript{2}\quad Aixin Sun\textsuperscript{3} \quad Zhuoran Yang\textsuperscript{4} 
}

\date{}

\begin{document}
    \maketitle
    \begingroup
    \renewcommand\thefootnote{}
    \footnotetext{\textsuperscript{1}National University of Singapore, \textsuperscript{2}Sea AI Lab, \textsuperscript{3}Nanyang Technological University, \textsuperscript{4}Yale University}.
    \footnotetext{Email:\{yuan.cheng,fzzhang,yhou\}@u.nus.edu, cnsdunm@gmail.com, \{duchao, tianyupang\}@sea.com, axsun@ntu.edu.sg, zhuoran.yang@yale.edu}
\footnotetext{$\ast$ Equal contribution,$\dagger$ Project Lead.}
    \endgroup
     \vspace{-15pt}
\begin{figure}[!h]
\centering
     
\begin{minipage}{0.88\linewidth}
     \centering
     \includegraphics[width=\linewidth]{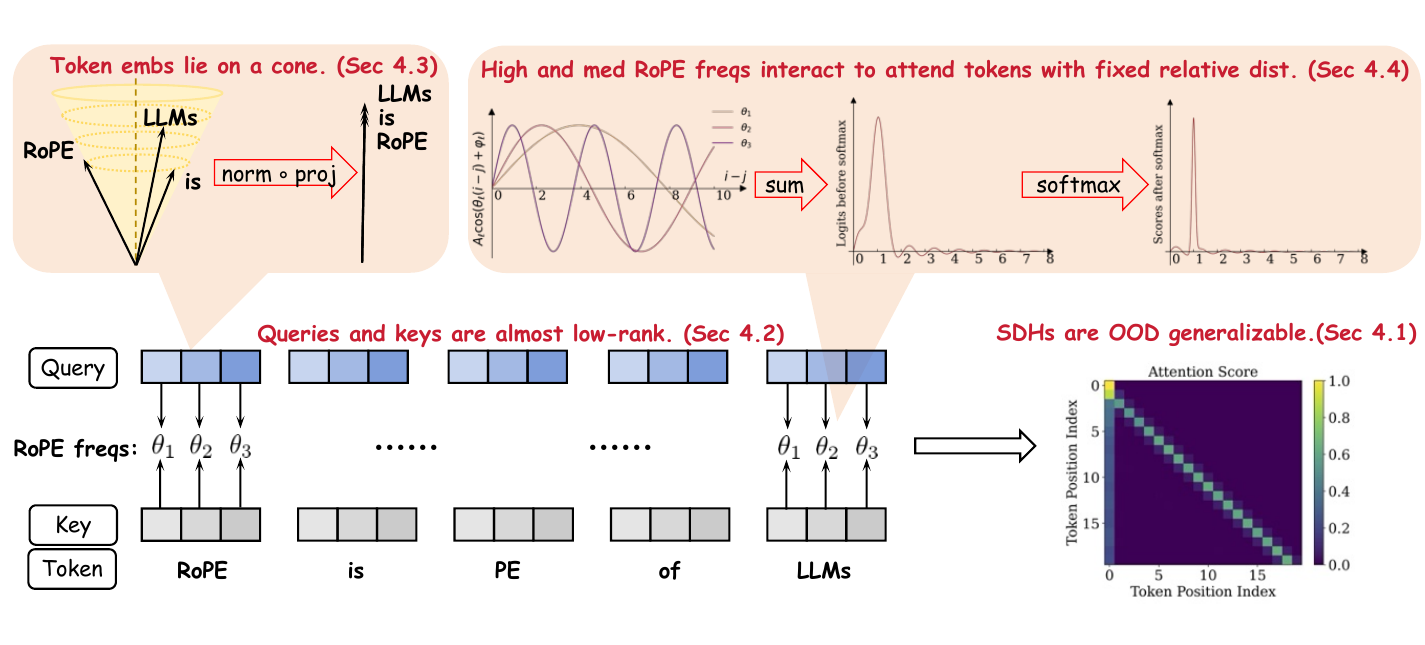}
     \caption{Illustration of the emergence of \acp{sdh}. Attention scores are determined by pre-PE {queries, keys}, and {\ac{rope}} (left bottom). Because token embeddings lie approximately on a cone, queries/keys are almost rank-one, and nearly identical across tokens (left top), so \ac{rope} primarily governs variation of attention scores across tokens. Then \ac{rope}'s high- and medium-frequency components interact constructively at specific lags, producing the attention score peaks at offset $\Delta$ (right top). As a result, \acp{sdh} emerge and are \ac{ood} generalizable (right bottom).}
     \label{fig:Teaser}
     \end{minipage}
 \end{figure}
\begin{abstract}
Large Language Models (LLMs) often exhibit slash attention patterns, where attention scores concentrate along the $\Delta$-th sub-diagonal for some offset $\Delta$. These patterns play a key role in passing information across tokens. But why do they emerge? In this paper, we demystify the emergence of these Slash-Dominant Heads (SDHs) from both empirical and theoretical perspectives. First, by analyzing open-source LLMs, we find that SDHs are intrinsic to models and generalize to out-of-distribution prompts. To explain the intrinsic emergence, we analyze the queries, keys, and Rotary Position Embedding (RoPE), which jointly determine attention scores. Our empirical analysis reveals two characteristic conditions of SDHs: (1) Queries and keys are almost rank-one, and (2) RoPE is dominated by medium- and high-frequency components. Under these conditions, queries and keys are nearly identical across tokens, and interactions between medium- and high-frequency components of RoPE give rise to SDHs. Beyond empirical evidence, we theoretically show that these conditions are sufficient to ensure the emergence of SDHs by formalizing them as our modeling assumptions. Particularly, we analyze the training dynamics of a shallow Transformer equipped with RoPE under these conditions, and prove that models trained via gradient descent exhibit SDHs. The SDHs generalize to out-of-distribution prompts.
\end{abstract}

\section{Introduction}\label{sec:intro}
Large Language Models (LLMs)\acused{llm} have demonstrated remarkable capabilities across a wide range of domains, including natural language processing, reasoning, and planning~\citep{brown2020language}. Given a prompt that contains a question, an \ac{llm} can interactively generate a coherent and contextually appropriate answer. A crucial ingredient behind this ability is the model's capability to pass information across different tokens in the sequence, most notably, from the prompt tokens to the answer tokens. Understanding how information propagation is implemented within \acp{llm} is therefore an important question in its mechanistic interpretability. Since modern \acp{llm} are built on the \emph{Transformer} architecture, where the communication between tokens is achieved primarily by the self-attention mechanism, the model's information-passing behavior is closely linked to specific structural patterns in its attention scores.

Prior work has identified several characteristic attention score patterns, including antidiagonal, block-sparse, vertical and slash patterns~\citep{jiang2024minference,xu2025xattention}. Among these, the \emph{slash pattern}, where the attention score concentrates along the $\Delta$-th sub-diagonal of the attention score matrix for some offset $\Delta \in \mathbb{N}$, is particularly intriguing. We refer to attention heads exhibiting slash patterns as \acp{sdh}. \acp{sdh} and their slash patterns play important algorithmic roles in \acp{llm}. For example, they enable \ac{icl} via the induction head circuit~\citep{elhage2021mathematical,olsson2022context},  which is a special case of an SDH with $\Delta = 1$. Concretely, the induction head circuit (Figure \ref{fig:example of induction head}) consists of a forwarding head and a feature-matching head. In the forwarding head, attention scores concentrate along the first sub-diagonal, so each token attends primarily to its immediate predecessor (prefix), effectively forwarding semantic features from the prefix to the current token. In addition, another line of work leverages slash patterns to help accelerate long-context inference~\citep{jiang2024minference,xu2025xattention,zhao2025paroattention,DBLP:conf/iclr/LaiLLMZ25,li2025mtraining}.

More generally, \acp{sdh} with diverse values of $\Delta$ are prevalent in modern open-source LLMs (see \Cref{fig:qwen_small_longb}). These \acp{sdh} enable a token at position $i$ to attend directly to the token at position $i-\Delta$, thereby passing information from earlier tokens to later ones. Their widespread presence and functional importance naturally motivate our central research question:

\begin{quote}
\centering
\emph{\color{cc1} How do pretrained \acp{llm} implement \acp{sdh} using their transformer architectures?}
\end{quote}


\begin{wrapfigure}{r}{0.6\textwidth}
    \centering
    \includegraphics[width=0.9\linewidth]{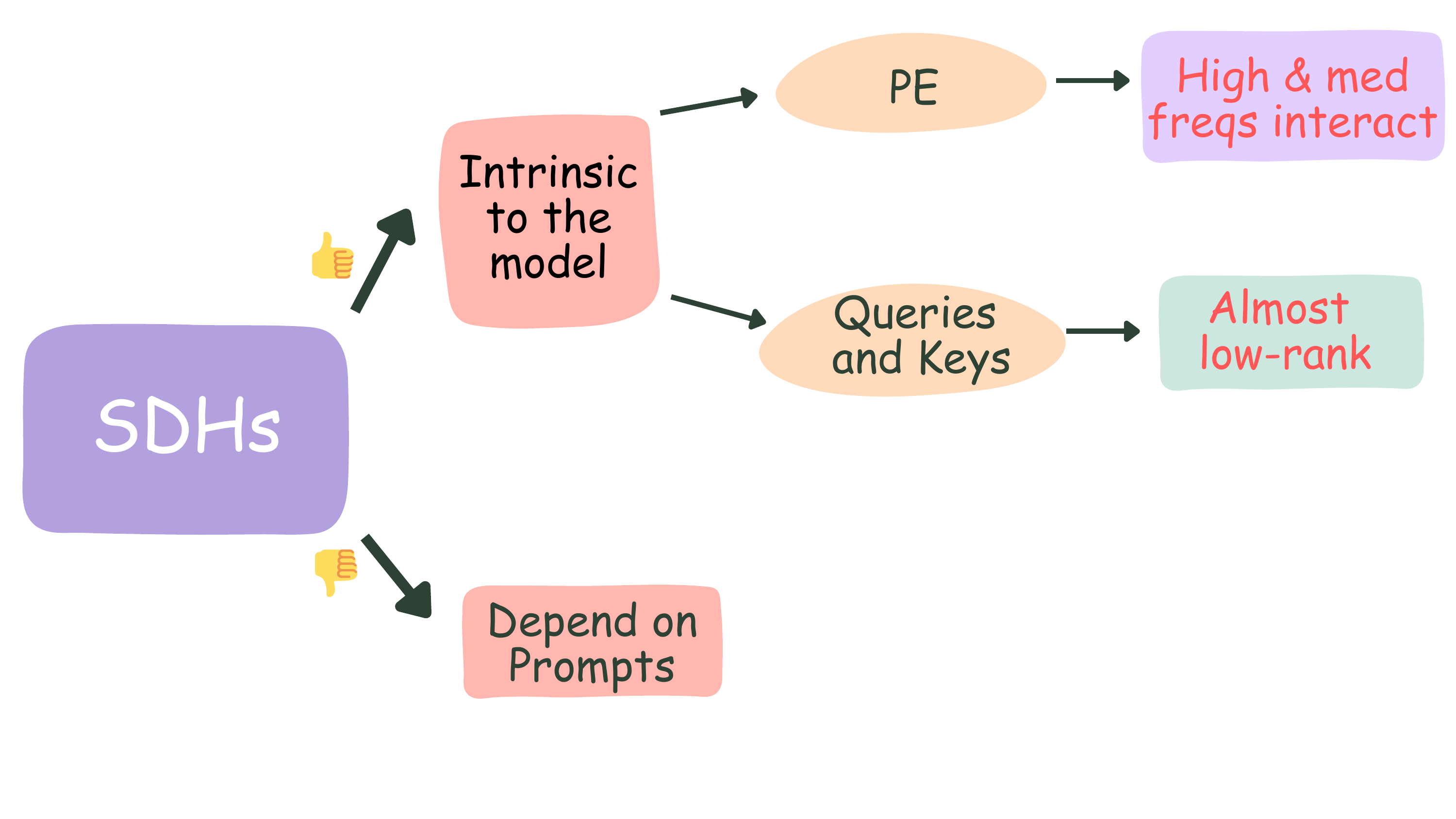}
    \parbox{0.95\linewidth}{\caption{Mind map of our empirical studies. }\label{fig:mind map}}
    \vspace{-22pt}
    \end{wrapfigure}
    
We tackle this question via both thorough empirical studies and rigorous theoretical analysis. In a nutshell, we find that the {\color{cc2} the emergence of  SDHs is intrinsic to the transformer architecture itself, and mainly attributed to the low-rankness of query and key matrices, and \ac{pe}, particularly, \ac{rope}~\citep{su2024roformer}}. We show this via thorough empirical studies and rigorous theoretical analysis.

\vspace{5pt}

\noindent \textbf{Empirical Studies} (\Cref{sec: emp}). We focus on open-source \ac{llm}s such as Gemma-7B, Qwen2.5-7B-Instruct, and Llama3-8B-Instruct~\citep{team2024gemma,grattafiori2024llama,yang2025qwen25}, which are decoder-only transformers with \ac{rope}. 
Our empirical studies consist of three parts, detailed below. For clarity, their overall structure is illustrated in the mind map shown in \Cref{fig:mind map}.

\begin{figure}[t]
\centering
\subfigure[Average attention score matrix of $\rmL 18\rmH 7$.]{\includegraphics[width=0.29\textwidth]{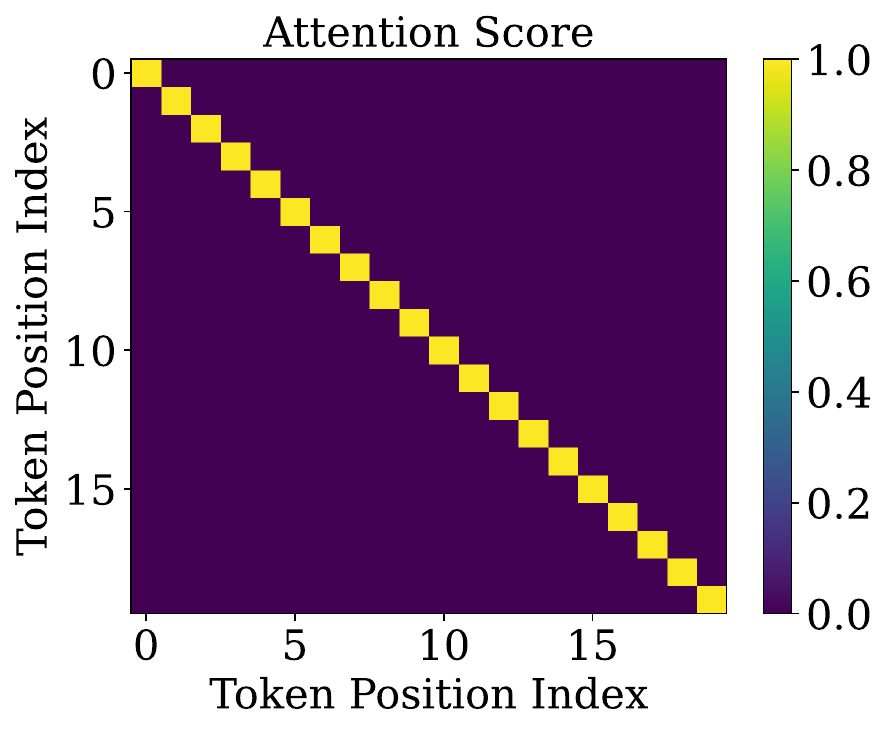}\label{fig:qwen_l18h7}}
\hspace{1.3em}
\subfigure[Average attention score matrix of $\rmL 21\rmH 15$.]{\includegraphics[width=0.29\textwidth]{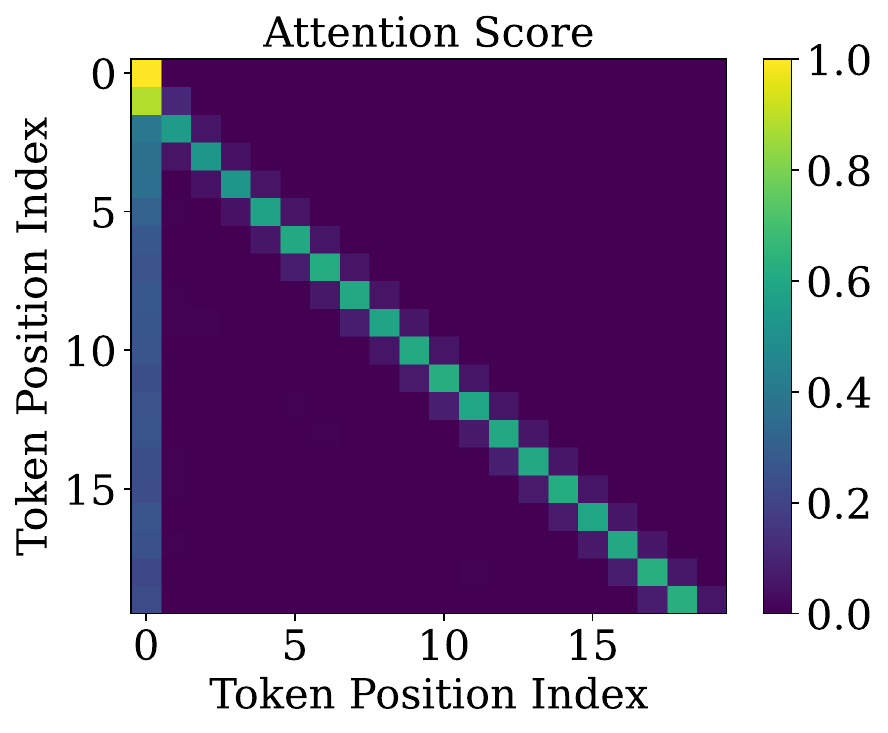}}
\hspace{1.3em}
\subfigure[Average attention score matrix of $\rmL 1\rmH 3$.]{\includegraphics[width=0.29\textwidth]{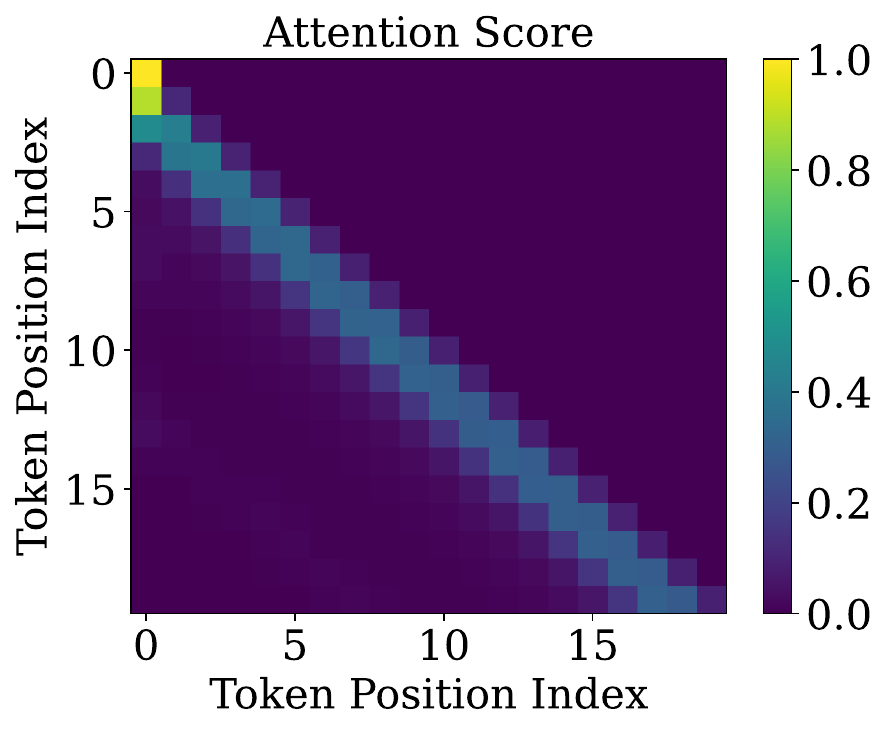}\label{fig:qwen_l1h4}}

\subfigure[Average attention score matrix of $\rmL 0\rmH 7$.]{\includegraphics[width=0.325\textwidth]{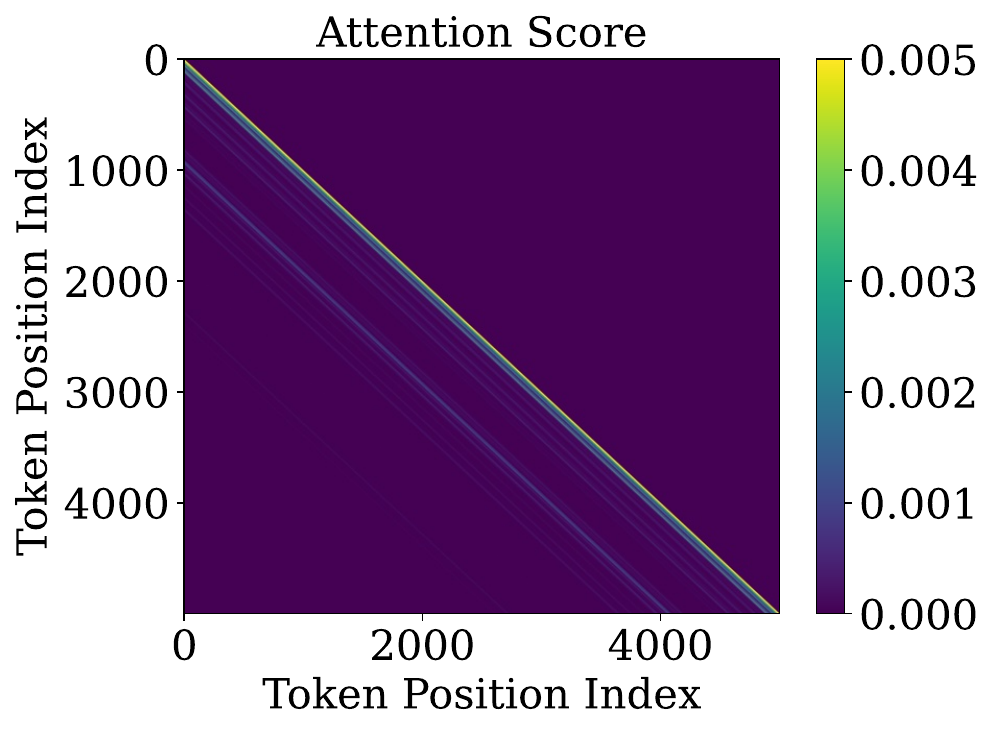}\label{fig:qwen_l0h7}}
\subfigure[Average attention score matrix of $\rmL 1\rmH 11$.]{\includegraphics[width=0.325\textwidth]{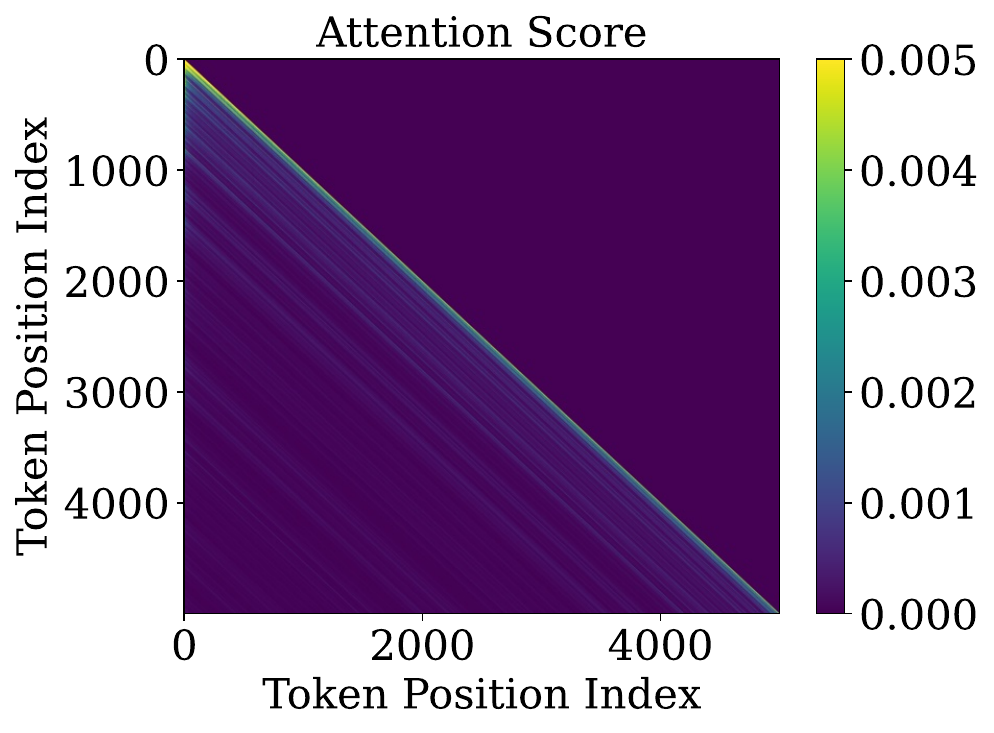}}
\subfigure[Average attention score matrix of $\rmL 2\rmH 6$.]{\includegraphics[width=0.325\textwidth]{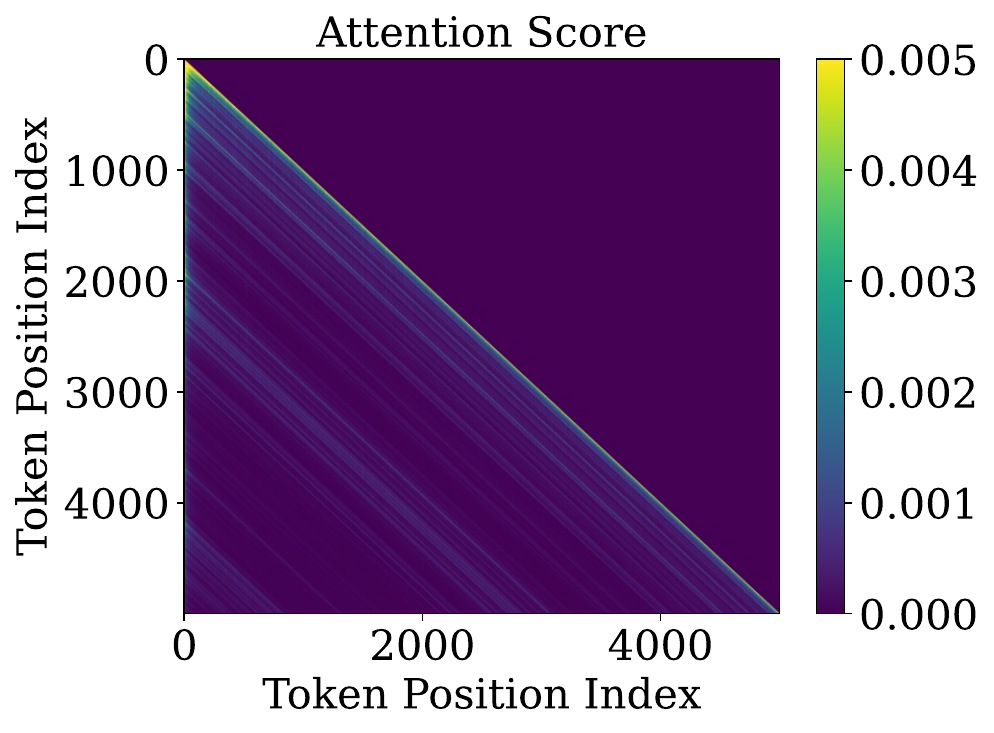}\label{fig:qwen_l2h6}}
\caption{Average of attention score matrices in Qwen2.5-7B-Instruct with prompts from LongBench. We denote the $a$-th head at $b$-th layer as $\rmL b \rmH a$ in this paper. In panels (a)–(c), attention concentrates on the sub-diagonals with small offsets $0,1$ and $2$, respectively. In panels (d)–(f), it also concentrates on sub-diagonals with large offsets exceeding $500$.}
\label{fig:qwen_small_longb}
\vspace{-12pt}
\end{figure}

\begin{itemize}[wide, labelindent=0pt]
    \setlength{\parskip}{0pt} 
    \setlength{\topsep}{0pt}
    \item [(i)] ({\color{cc1}SDHs are Intrinsic}) First, we study \emph{\color{cc1} whether \acp{sdh} depend on the input prompt or are intrinsic to the model itself.} To this end, we compare the attention scores of the in-distribution prompts to those generated by randomly generated prompts. We observe that the slash dominance pattern persists even in the \ac{ood} setting. This implies that \acp{sdh} are largely \textit{\color{cc2} independent} of the specific input prompt, and they arise mainly due to an \emph{\color{cc2} intrinsic algorithmic} mechanism of the transformer model.
\end{itemize}

Note that attention scores are determined by the interplay of the queries and keys, and the rotational matrices of  \ac{rope}. We proceed to examine how each of these components contributes to the emergence of  \acp{sdh}.

\begin{itemize}[wide, labelindent=0pt]
    \setlength{\parskip}{0pt} 
    \setlength{\topsep}{0pt}
    \item [(ii)] ({\color{cc1}Role of 
    Queries and Keys}) Perhaps surprisingly, our experiments reveal that the \emph{\color{cc2} pre-\ac{pe} queries and keys are low-rank}, and particularly, \emph{\color{cc2} almost rank-one}. As a result, for any token, the pre-\ac{pe} queries and keys of an SDH point in the same direction, which implies that the semantic contents of queries and keys contribute little to differentiating attention scores.
    Consequently, 
    the variations in attention scores across tokens has to  be \emph{\color{cc2} driven almost entirely by  \ac{rope}}. Furthermore, we show that these rank-one queries and keys appear because (a) token embeddings lie on a cone and (b) the weight matrices $W_K$ and $W_Q$ project these embeddings to the main axis of the cone.

    \item [(iii)] ({\color{cc1}Role of 
    RoPE}) As a result of \ac{rope} and rank-one queries and keys,  for any $i,j\in\NN$, the pre-softmax {attention logit} from position $i$ attending to position $j$ admits a Fourier-like decomposition. Specifically, we have 
\begin{align}\label{eq:attn_logit_intro}
    \mathtt{AttnLogit}(i,j) = \sum\nolimits_{l=1}^{d/2} A_{l} \cdot \cos \bigl (\theta_l \cdot (i-j) + \varphi_{l}\bigr), 
    \end{align}
where $d$ is the hidden dimension, $\{ \theta_{l}\}_{l=1}^{d/2}$ are the \ac{rope} frequencies. Here $\{ A_{l}\}_{l=1}^{d/2} $ and $\{ \varphi_{l} \}_{l=1}^{d/2}$ are amplitudes and phases, which are almost invariant across tokens thanks to the rank-one structure of the pre-\ac{pe} queries/keys. Consequently, the slash pattern, peaks of the attention logits in \eqref{eq:attn_logit_intro} when $i-j = \Delta$, 
is entirely determined by the frequencies $\{ \theta_{l}\}_{l=1}^{d/2}$. 
For a more fine-grained understanding, we compare the contributions of individual frequency components in the attention logits. 
We observe that \emph{\color{cc2} high- and medium-frequency} components play a dominant role in forming slash patterns, whereas low frequencies contribute little.


\end{itemize}

We illustrate the mechanism of \ac{sdh} in \Cref{fig:Teaser} and summarize the main takeways of the empirical studies as follows.

\begin{abox} 
    \looseness -1 \textbf{\color{cc1}Takeaways of Empirical Experiments:} \acp{sdh} are an intrinsic architectural effect of the transformer models. They emerge from an almost rank-one structure of the pre-\ac{pe} queries and keys, which suppresses the variations of semantic information across tokens. Moreover, the high- and medium-frequencies of \ac{rope} interact constructively at a specific offset $\Delta \in \mathbb{N}$, producing attention-score peaks at $\Delta$. 

\end{abox}
\vspace{5pt}
\noindent \textbf{Theoretical Analysis} (\Cref{sec: theory}). To complement the empirical studies, we further show that \acp{sdh} emerge from gradient-based training under conditions found by the empirical experiments. In particular, we adopt a shallow attention-only transformer equipped with \ac{rope}, which is trained on \ac{icl} regression tasks.  
We assume that the token embeddings lie on a cone, which is a premise for the queries and keys being almost rank-one, and is empirically validated by pretrained \acp{llm}.
Moreover, we introduce a \emph{\color{cc2} slash-dominance frequency condition on \ac{rope}} that quantitatively characterizes the behavior of \ac{rope} frequencies. Under these two conditions, we prove that \acp{sdh} emerge from gradient-based training and characterize the full training dynamics. These results provide theoretical guarantees for the hypothesized mechanisms of \acp{sdh} found by experiments. 

\begin{abox} 
    \looseness -1 \textbf{\color{cc1} Takeaways of Theoretical Analysis:} 
    When \ac{rope} satisfies a slash-dominance frequency condition and the token embeddings lie on a cone, models trained under these conditions are proven to exhibit \acp{sdh} and generalize effectively to \ac{ood} input, theoretically demonstrating the \emph{sufficiency} of these conditions for the emergence of \acp{sdh}.

\end{abox}

{\noindent \bf Roadmap.} The rest of the paper is organized as follows: \Cref{sec:related work} reviews related works, \Cref{sec:prelim} introduces the background of transformer architecture. \Cref{sec: emp} presents our empirical studies for small $\Delta$, while \Cref{sec: theory} develops theoretical results in the same small $\Delta$ regime. \Cref{sec:long_range_sdh} then extend these results to large $\Delta$.  Finally, \Cref{sec:discussion} discusses potential applications and extensions of our results. Detailed experiment results are provided in Appendices \ref{app:exp_details} to \ref{app:head_result}. Proof sketches and detailed proofs are provided in Appendices \ref{app: proof-st1} to \ref{app: proof-converge}.

\section{Related Works}\label{sec:related work}
\textbf{Rotary Position Embeddings (RoPE).} Proposed by \citet{su2024roformer}, \ac{rope} introduced multiplicative ``rotations'' on queries and keys so attention implicitly depends on relative positions, and became the default PE in many LLMs~\citep{yang2025qwen25,yang2024context,team2024gemma,grattafiori2024llama}. Following \ac{rope}, a series of works extended the context length window of the pretrained models by modifying the base frequency of \ac{rope}~\citep{xiong2023effective,roziere2023code,xiong2025dope}, interpolating the position indexes and frequencies~\citep{chen2023extending,peng2023yarn,ding2024longrope,zhong2024understanding}, and controlling the \ac{rope} feature gaps~\citep{wang2024resonance}. Another line of work investigated the mechanism behind \ac{rope}. \citet{barbero2024round} showed that rather than simply inducing attention decay with distance, models like Gemma-7B used RoPE’s high frequencies to build robust positional attention heads and low frequencies to carry semantic information, and further proposed a modified variant (p-RoPE) that improves performance. \citet{xiong2025dope} proposes DoPE, a training-free denoising method for RoPE that suppresses noisy positional frequency components in attention maps via truncated matrix entropy, improving long-context length extrapolation. In addition, RoPE has been shown to induce various distinctive structural patterns in LLMs. \citet{DBLP:conf/icml/JinMXSTD0Z25} empirically shows that RoPE-induced “massive values” in the queries and keys of self-attention emerge consistently across layers and play a critical role in contextual knowledge understanding in LLMs, with perturbing these values significantly degrading performance on context-dependent tasks. 

\vspace{3pt}
\begin{wrapfigure}{r}{0.5\textwidth}
    \centering
    \includegraphics[width=1.0\linewidth]{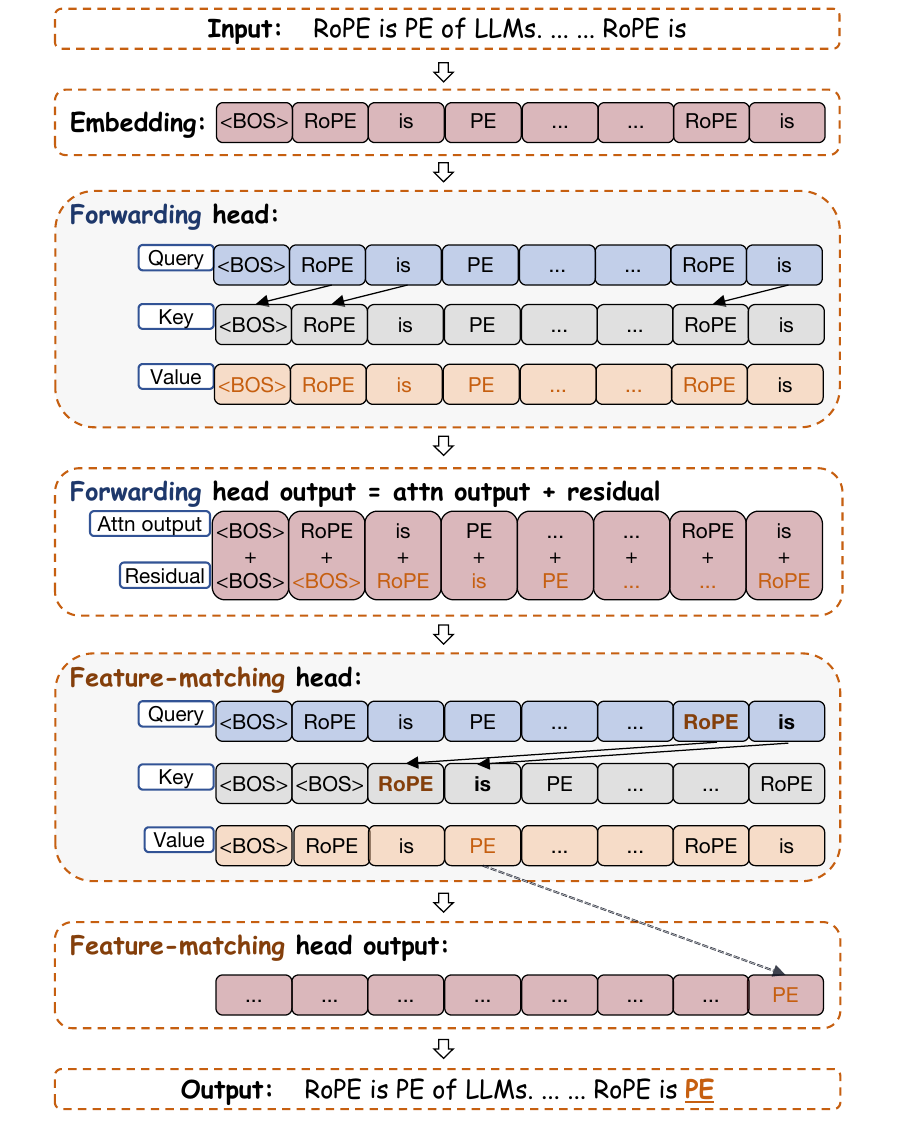}
    \parbox{0.95\linewidth}{\caption{
    The forwarding head forwards semantic information from the prefix to the current token. In the feature-matching head, the target question token  (``is") matches tokens whose prefixes exhibit similar semantics to the target, enabling the model to generate the correct continuation ``PE".}\label{fig:example of induction head}}
    \vspace{-18pt}
\end{wrapfigure}

\noindent \textbf{Induction Head Circuit.} First introduced by \citet{elhage2021mathematical,olsson2022context}, \emph{induction head circuit} is a specialized cascade of attention heads in transformer models that play a central role in \emph{in-context learning}. The induction head circuit typically involves two attention heads, a forwarding head and a feature-matching head, working in tandem through a forward-then-match process to transmit semantics and complete an answer, as in Figure \ref{fig:example of induction head}. To examine its emergence closely, subsequent empirical work has explored controlled synthetic settings~\citep{reddy2023mechanistic, bietti2023birth, edelman2024evolution,singh2024needs}. In particular, \citet{edelman2024evolution} designed a Markov-chain–based in-context learning task and showed that transformers develop “statistical induction heads” through a multi-phase training. \citet{bietti2023birth} revealed a two-phase learning process: the rapid acquisition of global bigram statistics, followed by the slower development of induction heads. From a theoretical perspective, \citet{nichani2024transformers, chen2024unveiling, edelman2024evolution,ekbote2025one} investigated induction heads under underlying causal structures such as trees or $n$-gram Markov chains, \cite{barbero2024round}and \citet{wang2024transformers} demonstrated that two-layer transformers can efficiently represent both standard and generalized induction head mechanisms. However, all these works neglected the contribution of \ac{rope} to the induction head. Concretely, they employed either vanilla one-hot \ac{pe}~\citep{nichani2024transformers} or ALiBi~\citep{wang2024transformers} and its variants~\citep{chen2024unveiling,ekbote2025one}. As a result, their assumptions and conclusions did not fully reflect the empirical behavior of many real-world \acp{llm}. \citet{xie2021explanation,zhang2025and} considered in-context learning from a different Bayesian
perspective. 


We provide additional discussion of related work in Appendix~\ref{app: add_rel}.

\section{Preliminaries}\label{sec:prelim}
In this section, we introduce the mathematical details of the Causal Self-Attention (CSA) model with \ac{rope}, which serves as the building block of the LLMs studied in this paper.  

\noindent\textbf{Transformer Architecture.} We consider decoder-only transformer models, which are composed of three main parts: a token embedding module, a stack of identical transformer blocks, and a final output layer (usually a softmax layer for language modeling). The model takes a sequence of tokens as input and maps it to a sequence of vectors via the token embedding module, known as token embeddings. These embeddings are then processed by the transformer blocks. 
Each transformer block processes a sequence of vectors and outputs a sequence of vectors of the same dimension. Concretely, the input to the first transformer block is the token embeddings, and the input to any subsequent block is the output of the preceding block. The output of the final block is then passed to the output layer to produce the final result, such as predicting the next token in the sequence.

Throughout this paper, we let $P$ denote a prompt, which is a sequence of tokens of length $N = N(P)$. We let $d $ denote the hidden dimension of the transformer model. That is, each transformer block maps a sequence of vectors of dimension $d$ to a sequence of vectors of dimension $d$.

%

\vspace{5pt} 

\noindent\textbf{Rotary Position Embedding (RoPE).} 
For any angle $\phi$, a rotation by angle $\phi$  in $\RR^2$ is represented by the unitary matrix
\begin{align}
    \rho(\phi) =
    \begin{bmatrix}
    \cos \phi & -\sin \phi \\
    \sin \phi & \cos \phi
    \end{bmatrix} \in \RR^{2\times 2}.
\end{align}
Without loss of generality, we assume that $d$ is even to simplify the presentation.
\ac{rope} encodes the \emph{absolute positional information} into rotation matrices involving a sequence of pre-specified set of decreasing frequencies, denoted by $\bvartheta=(\theta_1,\cdots,\theta_{{d}/{2}})$. 
Here, $\theta_\ell$ is the $\ell$-th frequency of $\bvartheta$ for each $\ell \in [d/2]$. 
Specifically, the information of each token position $i \in [N]$ is embedded into a matrix 
\begin{align}\label{Eq: rope}
  \RMat{\bvartheta}{i}{d}=\mathrm{diag}\left(\rho(i\cdot \theta_1),\cdots,\rho(i\cdot \theta_{d/2})\right) \in \RR^{d\times d},
\end{align}
where $\mathrm{diag}(\cdot)$ denotes the diagonal matrix with the given entries on the diagonal.  
Thus, intuitively, \ac{rope} proposes to represent each position $i$ as rotations with angles $i\cdot \theta_1, \cdots, i\cdot \theta_{d/2}$.
Practically, these frequencies are chosen as an exponentially decreasing sequence, with high frequencies (small $\ell$, large $\theta_\ell$) corresponding capturing local syntactic structures and low frequencies (large $\ell$, small $\theta_\ell$) capturing long-range dependencies.
A classic choice of $\bvartheta$, proposed by \citet{su2024roformer}, is $\theta_\ell = 10000^{-2\ell/d}$ for $\ell \in [d/2]$. In this paper, we consider a general $\bvartheta$ and will investigate the conditions under which it performs well. 

To simplify the notation, for any  $\bvartheta$ and $i \in [N]$, we let $ \RoPE{\bvartheta}{i}{d}: \RR^d \to \RR^d$ denote the \ac{rope} operator with frequency sequence $\bvartheta$ applied to a vector at position $i$.
Then, for any vector $\vb \in \RR^{d}$, we write 
\begin{align}
   \RoPE{\bvartheta}{i}{d}(\vb)
  =\RMat{\bvartheta}{i}{d}\vb
 =
 \begin{bmatrix}
 \rho(i\theta_1) & & \\
  & \ddots & \\
  & &\rho(i\theta_{d/2}) \\
 \end{bmatrix}
 \begin{bmatrix}
 \vb_{1:2} \\
 \vdots \\
 \vb_{d-1:d} \\
 \end{bmatrix}
 =
 \begin{bmatrix}
 \rho(i\theta_1) \vb_{1:2} \\
 \vdots \\
  \rho(i\theta_{d/2})\vb_{d-1:d} \\
 \end{bmatrix},
 \nonumber
\end{align} 
where we let $\vb_{j:k}$ denote the sub-vector of $\vb$ from the $j$-th to the $k$-th components, for all $j,k \in [d]$.
In other words, after applying \ac{rope}, the $(2\ell-1)$-th and $2\ell$-th components of the vector $\vb$ are rotated by an angle of $i\cdot \theta_\ell$, for all $\ell \in [d/2]$.
In the sequel, we omit the subscripts $\bvartheta$ and $i$  in $\RoPE{\bvartheta}{i}{d}$ when there is no ambiguity, and with slight abuse of notation, we allow the operator $\Re$ to act on a row vector and, row-wise, on a matrix.

\vspace{5pt} 
\noindent\textbf{Causal Self-Attention with RoPE.} 
A key component of each transformer block is the Causal Self-Attention (CSA) mechanism. 
We focus on the CSA that uses \ac{rope} to encode positional information. In the following, we provide the mathematical details of the CSA with RoPE.

The trainable parameters of the CSA are the attention weight matrices $W_Q, W_K, W_V \in \RR^{d \times d_1}$. 
Here $d_1 = d / \mathtt{num\_heads} < d$, where $\mathtt{num\_heads}$ stands for the number of attention heads. 
For any prompt $P$ with $N$ tokens, we let $H :=H(P)= (\hb_1,\ldots,\hb_N)^\top \in \RR^{N \times d}$ denote the hidden states of the prompt. Here, $H$ is either the token embeddings or the output of the preceding transformer blocks. 
CSA takes $H$ as the input and maps it to a sequence of $N$ vectors of dimension $d$. 
We define the CSA output in the following four steps. 
\vspace{-2pt}
\begin{itemize}[wide, labelindent=0pt]
    \setlength{\parskip}{0pt} 
    \setlength{\topsep}{0pt}
\item [(i)] (Computing Queries, Keys, and Values.) First, the CSA maps hidden states $H$  to a sequence of queries, keys, and values, which are denoted by  $Q=H W_Q=(\qb_1,\cdots,\qb_N)^\top$, $K=H W_K=(\kb_1,\cdots,\kb_N)^\top$, and $V=H W_V=(\vb_1,\cdots,\vb_N)^\top$, respectively\footnote{Different from the standard formulation, Qwen2.5 family introduces a special modification where the query vector $\qb_i$ (similarly for $\kb_i$) includes an additional bias term $\mathbf{b}_Q$, such that $\qb_i = W_Q^{\top}\hb_i + \mathbf{b}_Q$.}. 
\item [(ii)] (Applying \ac{rope} to Queries and Keys.) Then, we apply \ac{rope} to the query matrix $Q$ and key matrix $K$. Specifically, for each $\qb_i, \kb_i \in \RR^d$ at position $i \in [N]$, we get the rotated vectors as $\tilde\qb_i = \RoPE{\bvartheta}{i}{d}(\qb_i)$ and $\tilde\kb_i = \RoPE{\bvartheta}{i}{d}(\kb_i)$. We denote $\widetilde{Q}=(\tilde\qb_1,\cdots,\tilde\qb_N)^\top$ and $\widetilde{K}=(\tilde\kb_1,\cdots,\tilde\kb_N)^\top$.

\item [(iii)] (Computing Attention Scores using $\widetilde{Q}$ and $\widetilde{K}$.) Then we compute the attention scores $S$ using $\widetilde{Q}$ and $\widetilde{K}$, according to  
\begin{align}
\label{eq:attention_scores}
\attn=\attn(P):=\mathrm{softmax} (\frak{M}(\widetilde{Q} {\widetilde{K}}^\top)) \in \RR^{N \times N}.
\end{align}
Here, $\frak{M}$ is a causal mask operator such that for any matrix $A$, 
 $\frak{M}$$(A)_{i,j}$ is $A_{i,j}$ when $i \geq j$ and $-\infty$ otherwise.
Moreover,  $\mathrm{softmax}(\cdot)$ is the softmax operator, which is applied in a row-wise manner. 
For any row vector $\ub \in \RR^{1 \times N}$, the $i$-th component of $\mathrm{softmax}(\ub)$ is given by $\mathrm{softmax}(\ub)_i=e^{\ub_i}/\sum_{j=1}^N e^{\ub_j}$.
Each entry of $\frak{M}(\widetilde{Q}\widetilde{K}^\top)$ is called a \emph{logit}. By the definition of \ac{rope} in \eqref{Eq: rope} and the causal mask $\frak{M}$, the $(i,j)$-th logit is  given by ${\tilde\qb_i}^\top\tilde\kb_j=\qb_i\RMat{\bvartheta}{j-i}{d}\kb_j$ when $i\geq j$ and $-\infty$ otherwise.

\item [(iv)] Finally, the CSA output is computed as a weighted sum of the values, where the weights are the attention scores $S$ in \eqref{eq:attention_scores}. Specifically, the CSA output is given by
\begin{align}
  \CSA\big(H;W_{\{Q,K,V\}},\bvartheta\big) = \attn(P) V =  \mathrm{softmax} (\frak{M}(\widetilde Q {\widetilde K}^{\top}))H W_V \in \RR^{N \times d}. \label{Eq: CSA}
  \end{align}
\end{itemize}

By now, we have defined CSA on the hidden states \(H \in \mathbb{R}^{N\times d}\) of a given prompt $P$. From the next section, we consider that the prompt is \emph{randomly} drawn either from the pretraining prompt distribution $\cD$, on which the pretrained models are trained, or from an explicitly specified distribution $\cD'$, whose support may extend beyond that of the pre-training distribution $\cD$, i.e., it may generate out-of-distribution inputs. In that case, for a statistic of interest $f(P)$, we analyze its expected behavior $\bbE_{P\sim \bar{\cD}}[f(P)]$, where $\bar{\cD}\in\{\cD,\cD'\}$. In particular, we will study the statistical properties of CSA and the attention scores in \Cref{sec: emp,sec: theory,sec:long_range_sdh}.

\section{Empirical Study of Slash-Dominance of Attention}\label{sec: emp}

In this section, we empirically study the emergence of slash-dominance of attention in \ac{llm}s,
which is a special attention pattern observed in various \ac{llm}s. 
As shown in \Cref{fig:qwen_small_longb}, on certain heads, the attention scores have relatively high magnitude on certain sub-diagonals with various values of offsets $\Delta \in \mathbb{N}$. We refer to such a sub-diagonal winning pattern as \emph{slash-dominance}.

To systemically examine the emergence of the slash-dominant pattern, we first define slash-dominance with mathematical rigor, then report the \acp{sdh} found on various \ac{llm}s empirically. Finally, we 
investigate why these special heads emerge. In particular, we focus on three models: Gemma-7B, Llama3-8B-Instruct, and Qwen2.5-7B-Instruct, all of which adopt \ac{rope} as their \ac{pe}.

\vspace{5pt} 
\noindent\textbf{Mathematical Definition of Slash-Dominance.} 
Recall that the attention score of an attention head is computed as in \eqref{eq:attention_scores}, which involves the queries, keys, and \ac{rope}. 
Slash-dominance appears when this attention score matrix has high values on a certain sub-diagonal line with offset $\Delta$.

\begin{definition}[$(\kappa, \Delta)$-Slash-Dominance]\label{def:slash_dom}
Let $\cD$ be the distribution of prompts. 
For any given lag  $\Delta \in \mathbb{N}$ and threshold $\kappa \in [0,1]$,  an attention head is said to be $\kappa$ slash-dominant at lag  $\Delta$, if
\begin{align} \label{eq:def_slash_dom}
 \mathrm{average~slash~score}:=   \bbE_{P\sim \cD}\bigg[ \frac{1}{N(P)-\Delta} \sum\nolimits_{i=\Delta+1}^{N(P)}\attn_{i,i-\Delta}(P)\bigg] \geq \kappa. 
\end{align}
Here, $N(P)$ is the length of the prompt $P$, $\attn(P)$ is the attention score matrix on this head for the prompt $P$, and $\attn_{i,j}(P)$ is the $(i,j)$ component in the attention score matrix $\attn(P)$, denoting the attention score from position $i$ to $j$. 
\end{definition}

We highlight that both $N(P)$ and $\attn(P)$  depend on the prompt $P$, which is randomly sampled according to the pretraining data distribution $\cD$. In the following, we refer to the left-hand side of \eqref{eq:def_slash_dom} as the \emph{average slash score} at lag $\Delta$. 
Intuitively, it measures the average attention paid to tokens that are $\Delta$ positions before the current token. In the attention score matrix, these connections form a slash line parallel to the main diagonal, which motivates the name \emph{slash-dominance}.

\begin{wrapfigure}{r}{0.27\textwidth}
 \vspace{-0.2cm}\includegraphics[width=\linewidth]{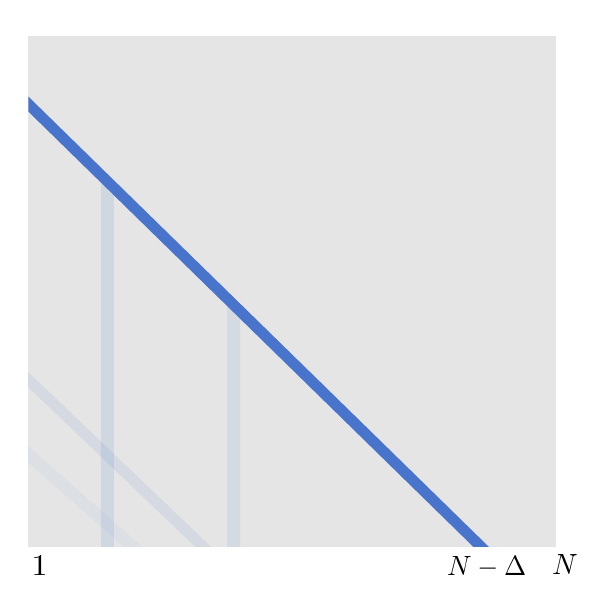}
   \caption{Illustration of Slash-Dominance at $\Delta$.}
   \label{fig:Slash-Dominance}
   \vspace{-0.4cm}
\end{wrapfigure}
As shown in this definition, an attention head is an \ac{sdh} if \emph{the average attention score on some sub-diagonal line with offset $\Delta$ is larger than a threshold $\kappa$}. 
It only concerns a single slash line with lag $\Delta$ as in \Cref{fig:Slash-Dominance}. This definition includes the induction head circuits as a special case. 
In particular, the forwarding head of an induction head circuit always copies the information from the previous token, which corresponds to a specialized \ac{sdh} with $\Delta=1$ and $\kappa \approx 1$~\citep{olsson2022context}.
Furthermore, the choice of $\kappa$ in \eqref{eq:def_slash_dom} is crucial. Setting it too low will identify too many heads as \acp{sdh} and thus not specific enough, while setting it too high will identify too few heads and thus not sensitive enough.

\vspace{3pt}
\noindent\textbf{Empirical Identification of \acp{sdh}.}  We identify \acp{sdh} in three \acp{llm}: Qwen2.5-7B-Instruct~\citep{yang2025qwen25}, Llama3-8B-Instruct~\citep{grattafiori2024llama}, and Gemma-7B~\citep{team2024gemma}. To ensure that the identified \acp{sdh} are prominent, we set $\kappa = 0.1$ with a fixed context length of $6000$. We let $\mathcal{D}$ denote the distribution of prompts in the \texttt{LongBenchV2} benchmark~\citep{bai2025longbenchv2}, which provides a heterogeneous long-context workload and yields stable cross-model statistics. To approximate the expectation in~\eqref{eq:def_slash_dom}, we sample $500$ prompts from \texttt{LongBenchV2} and truncate each to $6000$ tokens, ensuring a consistent context length across models and avoiding context-extension artifacts. Following prior observations that early tokens (including BOS and warm-up positions) exhibit disproportionately large attention scores~\citep{xiao2023efficient}, we exclude attention entries involving positions $1$–$4$ when computing $S_{i,i-\Delta}$ to prevent these anomalous values from dominating the slash-dominance statistics.

We find that all three models exhibit \acp{sdh} with $\Delta < 5$, reflecting short-range dependencies in attention. In Fig.~\ref{fig:qwen_small_longb}(a)–(c), we plot the attention maps of the identified \acp{sdh} in Qwen2.5-7B-Instruct with $\Delta = 0,1,2$, respectively. In these heads, tokens at relative distances of 0, 1, or 2 from any target token receive prominent attention. We further remark that when a smaller value of $\kappa$ is used, as in Fig.~\ref{fig:qwen_small_longb}(d)–(f), slash dominance also arises at large offsets $\Delta > 1000$,  with much smaller average slash scores on the order of $10^{-3}$. Accordingly, in this section, we study the characteristics and mechanisms of \acp{sdh} with small offsets, i.e., $\Delta < 5$, while the analysis of long-range \acp{sdh} is deferred to \Cref{sec:long_range_sdh}. Additional experimental details are provided in Appendix~\ref{app:exp_details}, and the full list of identified \acp{sdh} is provided in Appendix~\ref{app:head_result}.

\subsection{OOD Generalization of \acp{sdh}}\label{sec:ood}

To understand the mechanism of \acp{sdh}, we begin by posing the following question:
\begin{quote}
\centering
\emph{\color{cc1} Are \acp{sdh} sensitive to prompt distribution or intrinsic to the model itself?} 
\end{quote}

We answer this question by examining whether \acp{sdh} generalize to arbitrary \ac{ood} contexts. 
When we change the prompt distribution from $\cD$ to another distribution $\cD'$, if the condition in \eqref{eq:def_slash_dom} still holds for the same threshold $\kappa$, then we say that the \acp{sdh} are \ac{ood} generalizable. This means that the identified \acp{sdh} are independent of the choice of the prompt distribution, and thus intrinsic to the \ac{llm} model itself. In other words, regardless of the prompted tokens,  the heads identified in Appendix~\ref{app:head_result} consistently apply the same mechanism to form the slash pattern.

To test OOD generalization, we evaluate the average slash scores defined in \eqref{eq:def_slash_dom}  with a special prompt distribution $\cD'$, where each token of the prompt is sampled i.i.d. from the uniform distribution over the alphabet. We sample $500$ OOD prompts, each with $6000$ random tokens. We use these OOD prompts to compute the average slash scores as given in \eqref{eq:def_slash_dom} for the identified \acp{sdh}. 

\begin{figure}[ht]
\centering
\subfigure[Average attention score matrix of $\rmL 18\rmH 7$.]{\includegraphics[width=0.29\textwidth]{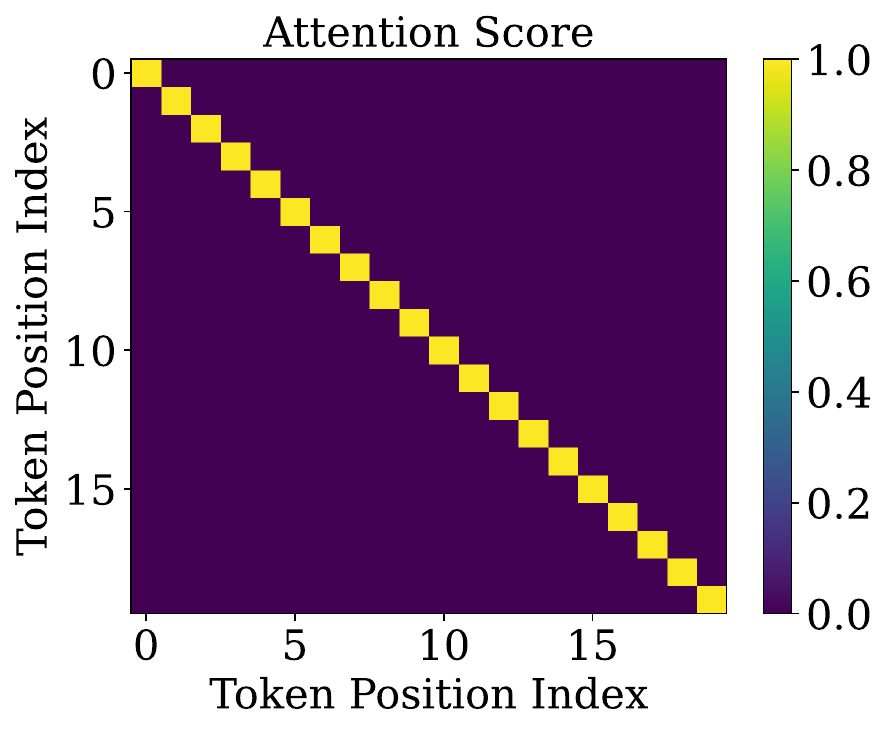}}
\hspace{1.3em}
\subfigure[Average attention score matrix of $\rmL 21\rmH 15$.]{\includegraphics[width=0.29\textwidth]{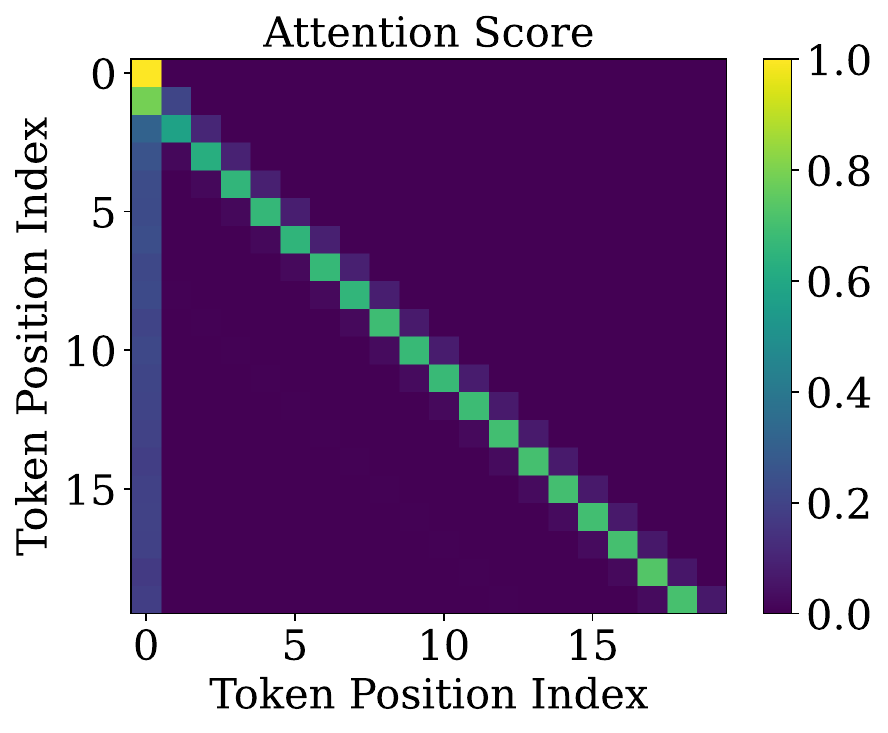}}
\hspace{1.3em}
\subfigure[Average attention score matrix of $\rmL 1\rmH 3$.]{\includegraphics[width=0.29\textwidth]{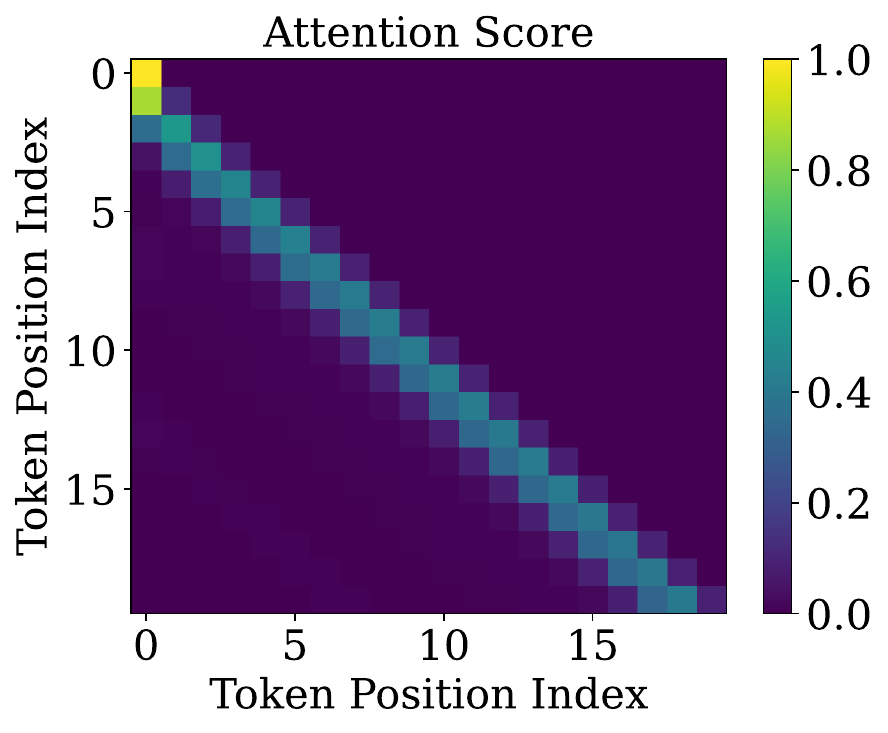}}


\caption{Average attention score matrices of \acp{sdh} with \emph{small} $\Delta$ in Qwen2.5-7B-Instruct with prompts whose tokens are i.i.d.\ samples from the uniform distribution over the alphabet.}
\label{fig:qwen_small_ood}
\end{figure}

\vspace{5pt}
{\noindent \bf \acp{sdh} can generalize to OOD prompts.} 
We plot average attention scores of identified \acp{sdh} under OOD prompts in Figure~\ref{fig:qwen_small_ood}. These \ac{sdh}s are from  Qwen2.5-7B-Instruct and include $\rmL18\rmH7$, $\rmL21\rmH5$, and $\rmL1\rmH3$ for $\Delta=0,1,2$, and they are the same \acp{sdh} reported in Figure~\ref{fig:qwen_small_longb}.
As shown in Figure~\ref{fig:qwen_small_ood}, {\color{cc2} the same slash pattern persists under OOD prompts}. We observe the same results for Llama3-8B-Instruct and Gemma-7B, whose details are deferred to Appendix \ref{app:head_result}.
Furthermore, we quantitatively assess  OOD generalization by {\color{cc2} computing the average slash scores for all the identified \acp{sdh} with in-distribution and OOD prompts, and comparing the ratios.} 
That is, we compute 
$$
\textrm{ratio} = \frac{\textrm{average slash score with OOD prompts} }{ \textrm{average slash score with LongBench prompts}}
$$
\begin{wrapfigure}{r}{0.28\textwidth}
\vspace{-12pt}
\includegraphics[width=\linewidth]{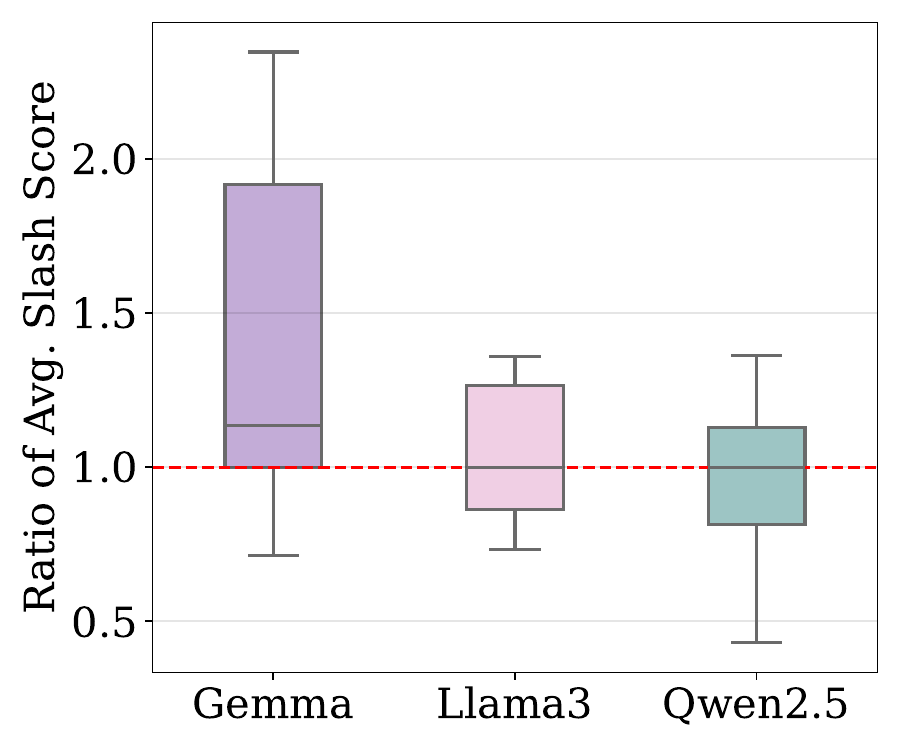}
   \caption{Ratios of average slash scores with \ac{ood} and in-distribution prompts for \acp{sdh}.}
   \vspace{-0.2cm}
    \label{fig:ood_avg_small}
    \vspace{-2pt}
\end{wrapfigure}
for all the identified \acp{sdh}, whose complete list is in Appendix~\ref{app:head_result}.
We show the box plot of these ratios in Figure~\ref{fig:ood_avg_small}. 
We observe that the average slash scores under \ac{ood} prompts are generally higher, or at least comparable to those under in-distribution prompts. 
Detailed head-level results for the \ac{ood} case are provided in Table~\ref{table:head_eg_small_full} in Appendix~\ref{app:head_result}. 
Furthermore, for the identified \ac{sdh}s, as long as the average slash scores on \ac{ood} prompts remain high, even if they differ from the in-distribution scores, these heads can still retrieve and forward the semantics of preceding tokens to the current token. 
Indeed, as shown in the box plot, most of the average slash scores on \ac{ood} prompts are larger than half of those on the in-distribution prompts and are therefore
at least $0.05$ (recall that we set $\kappa = 0.1$ when identifying
\acp{sdh}). As a result, most of the identified \acp{sdh} remain valid when $\mathcal D$ is replaced by \ac{ood} prompts and $\kappa$ is reduced
to $0.05$. 
Therefore, the slash-dominance pattern generalizes beyond the pretraining distribution. Emergence of the slash-dominance pattern is {\color{cc2} not relevant to the semantic meaning of the prompts}, but is {\color{cc2}intrinsic to the model architecture}. The main finding is summarized as follows. 

\begin{abox} 
    \looseness -1 \textbf{\hypertarget{takeaway1}{\color{cc1}Takeaway 1}:} 
    The \acp{sdh} identified by taking $0\leq \Delta\leq 4$ and $\kappa = 0.1$ in \eqref{eq:def_slash_dom} are able to generalize to \ac{ood} prompts. As a result, the slash-dominance behavior is intrinsic to the model architecture and not due to the semantic meaning of the prompts. 
\end{abox}

\subsection{Approximate Low-Rankness of pre-\ac{pe} Queries and Keys}\label{sec:lowrankQK}

Since the emergence of \acp{sdh} is intrinsic to the model, it is natural to ask how the transformer architecture contributes to it. 
Note that attention scores of CSA with \ac{rope} are determined by the pre-\ac{pe} queries and keys, and \ac{rope}, as shown in \eqref{eq:attention_scores}. 
In the following, we examine these components in detail. 
We first focus on  the {pre-\ac{pe} queries and keys} and study the following question:
\begin{quote}
\centering
\emph{\color{cc1} How do the pre-\ac{pe} queries and keys contribute to the emergence of \acp{sdh}?}
\end{quote}
We show that, interestingly, on \acp{sdh}, the pre-\ac{pe} queries and keys make almost no contribution to differentiating attention scores. In particular, as we will show in the following, on \acp{sdh}, the pre-\ac{pe} queries and keys are almost low-rank, particularly rank-one. This means that the semantic contents of the queries and keys are nearly identical across tokens and thus contribute little to the emergence of \acp{sdh}.

\begin{figure}[htb]
\centering
\subfigure[Hidden state $H$ of $\rmL 18\rmH 7$ after PCA.]{\includegraphics[width=0.3\textwidth]{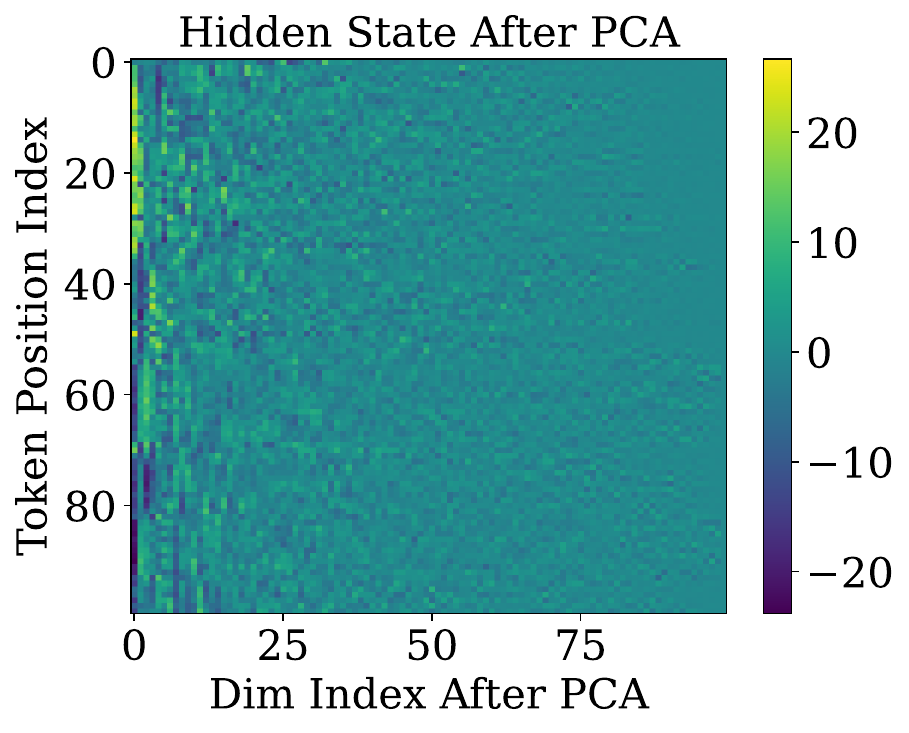}}
\hspace{1.3em}
\subfigure[Queries $Q$ of $\rmL 18\rmH 7$.]{\includegraphics[width=0.3\textwidth]{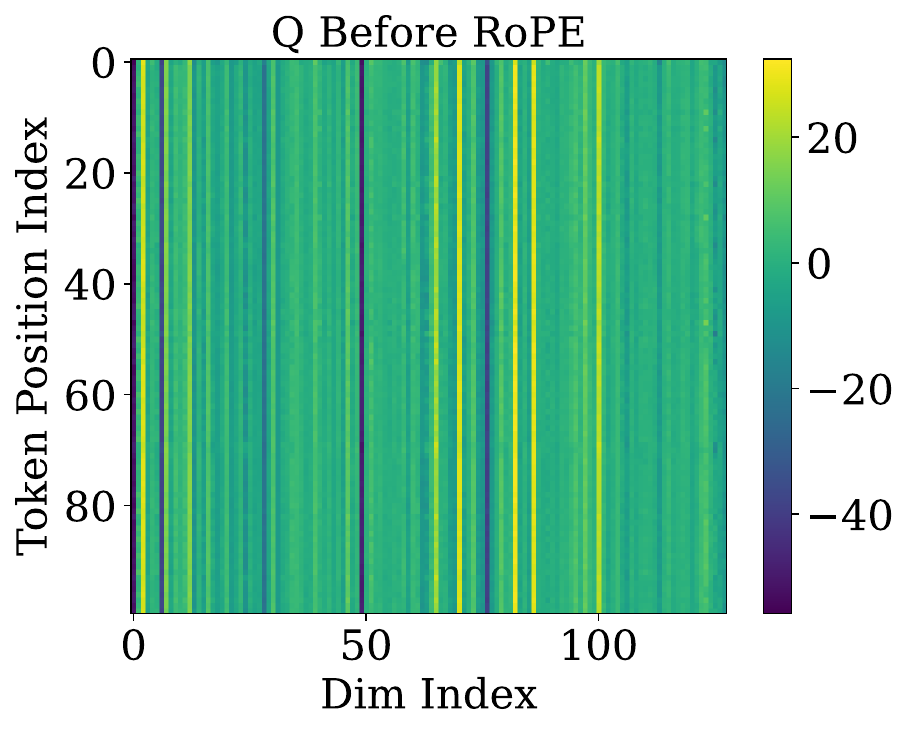}}
\hspace{1.3em}
\subfigure[Keys $K$ of $\rmL 18\rmH 7$.]{\includegraphics[width=0.29\textwidth]{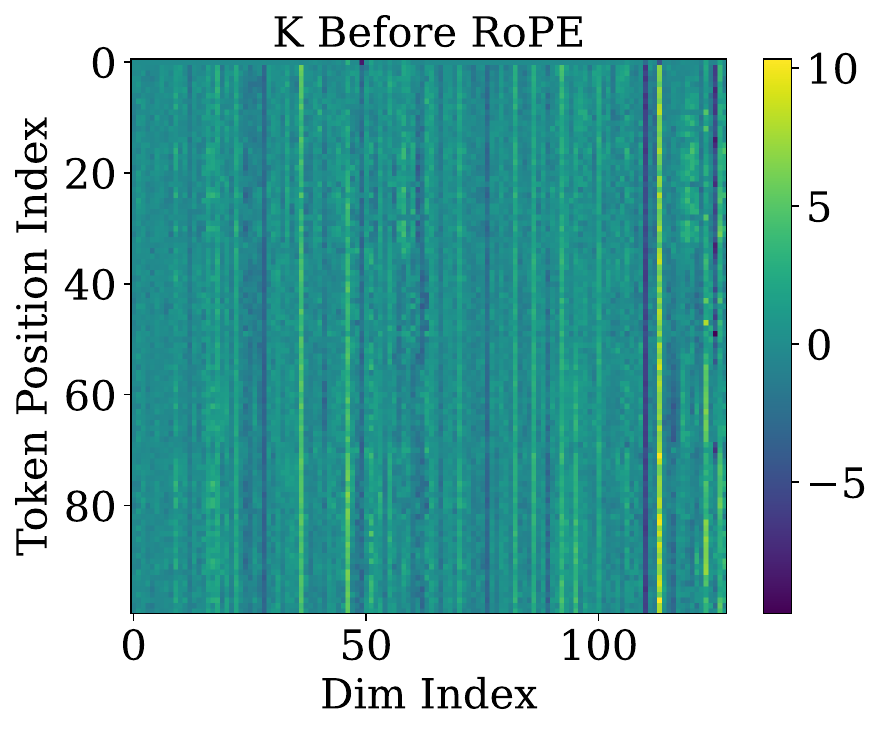}}

\subfigure[Hidden state $H$ of $\rmL 21\rmH 15$ after PCA.]{\includegraphics[width=0.3\textwidth]{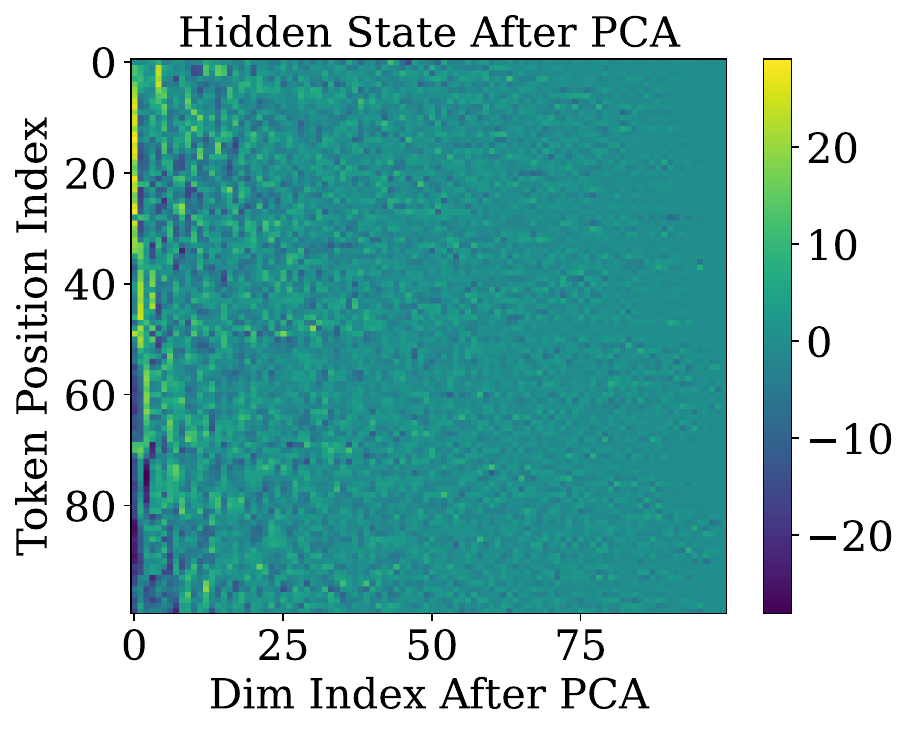}}
\hspace{1.3em}
\subfigure[Queries $Q$ of $\rmL 21\rmH 15$.]{\includegraphics[width=0.3\textwidth]{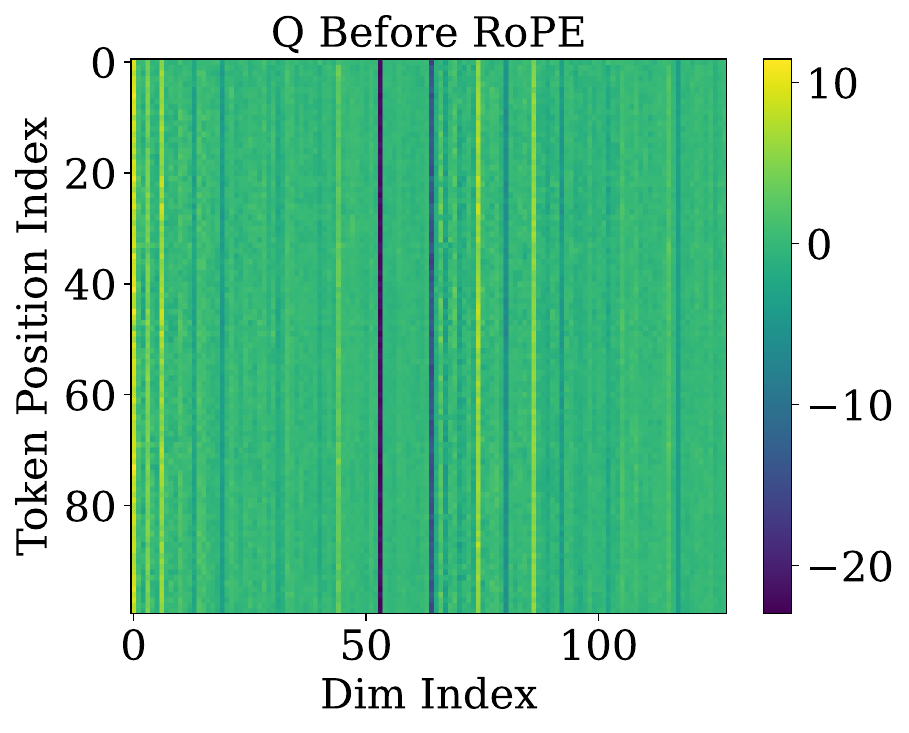}}
\hspace{1.3em}
\subfigure[Keys $K$ of $\rmL 21\rmH 15$.]{\includegraphics[width=0.29\textwidth]{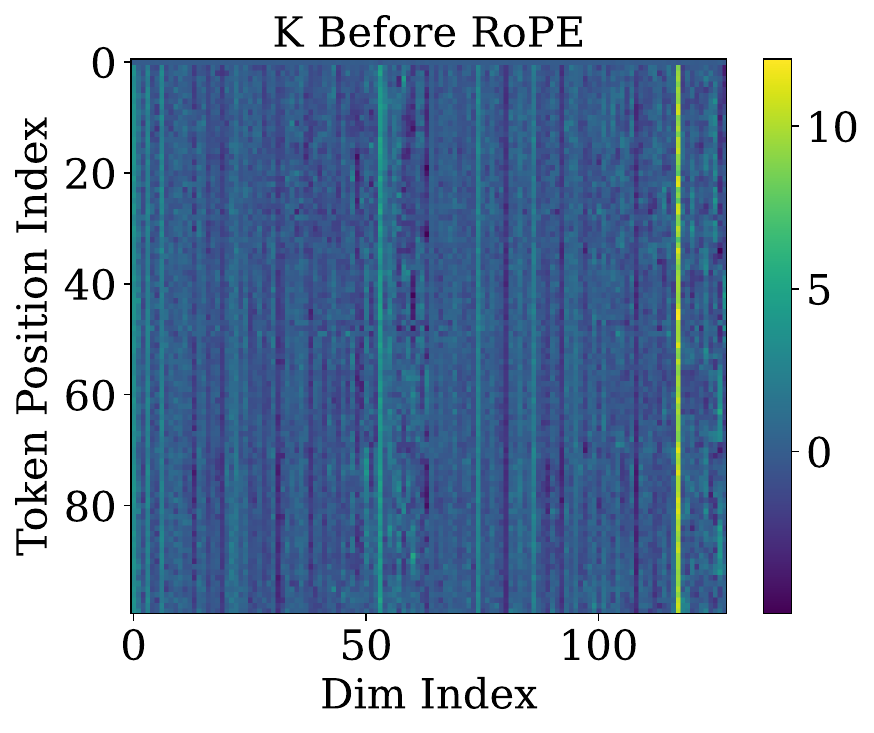}}
\caption{These figures show the queries and keys before the \ac{rope} implementation, as well as the hidden states for Qwen2.5-7B-Instruct. The queries and keys each have dimension $128$, and the hidden states are reduced to $128$ dimensions for demonstration.}
\label{fig:qwen_qk_small}
\end{figure}

To motivate the exploration, we visualize two \acp{sdh} of Qwen2.5-7B-Instruct: $\rmL 18\rmH 7$ with $\Delta=0$ and $\rmL 21\rmH 15$ with $\Delta=1$ in Figure~\ref{fig:qwen_qk_small}. As shown in the right two columns, the pre-\ac{pe} queries and keys are highly similar across tokens, implying that the query and key matrices are {\color{cc2} almost low-rank, particularly rank-one}.  In contrast, the hidden states of different tokens exhibit substantial variation. Additional visualizations for more heads and models are provided in \Cref{fig:gemma_qk,fig:llama_qk} in Appendix~\ref{app:head_result}, which consistently exhibit the same structural pattern as Figure~\ref{fig:qwen_qk_small}. This preliminary observation motivates us to examine the low-rankness of the pre-\ac{pe} queries and keys in detail.

\vspace{5pt} 
\noindent\textbf{Assessing Low-Rankness of a Matrix.} We adopt the following procedure to assess the low-rankness of a matrix $X\in\bbR^{N\times d}$, which contains $ d$-dimensional features of $N$ tokens. 
We perform \ac{svd} of $X$ and obtain the $d$ singular values $\sigma_1 \geq \sigma_2 \geq \ldots \geq \sigma_d$. Then we measure the spectral decay of $X$ using the following two quantities:
\begin{align} \label{eq:lowrank_metrics}
    r_{\ell}(X) = \frac{\sigma_{\ell}^2}{\sum\nolimits_{i=1}^{d}\sigma_i^2}, \text{ and}\quad R_{\tau}(X) = \min \Big\{\ell\in[d]\,\Big|\, \sum\nolimits_{i=1}^{\ell}r_i(X)\geq {\tau}\Big\}\qquad \text{for}~~ \tau\in[0,1].
\end{align}
Here, $r_{\ell}(X)$ measures the proportion of power captured by the $\ell$-th singular value in terms of the squared singular values. 
Moreover, $R_{\tau}(X)$  is the minimal subspace dimension required to account for $\tau$ of the matrix's total power, which can be viewed as the effective rank. 
In the following, we use {\color{cc2} $r_1(X)$} ($r_\ell (X)$ with $\ell=1$)  to measure how close a matrix is to being rank-one. In particular, if $X$ is close to being rank-one,  $r_1(X)$ should be close to one. 
Similarly, we use $R_{\tau}(X)$ with $\tau \approx 1$ to measure how close a matrix is to being low-rank. 
For example, if $X$ is close to a rank-$r$ matrix for some $r < d$, then $R_{\tau}(X) = r$  for some $\tau$ close to one. In the sequel, we set $\tau = 0.95$ in experiments. 
When the meaning of $r_1(X)$ and $R_{\tau}(X)$ is clear from the context, we write $r_1$ and $R_{\tau}$ for brevity.

\begin{figure}[ht]
    \centering
    \subfigure[Average $r_1(\uparrow)$ of all heads and \ac{sdh}s with small $\Delta$.]{\includegraphics[width=0.48\linewidth]{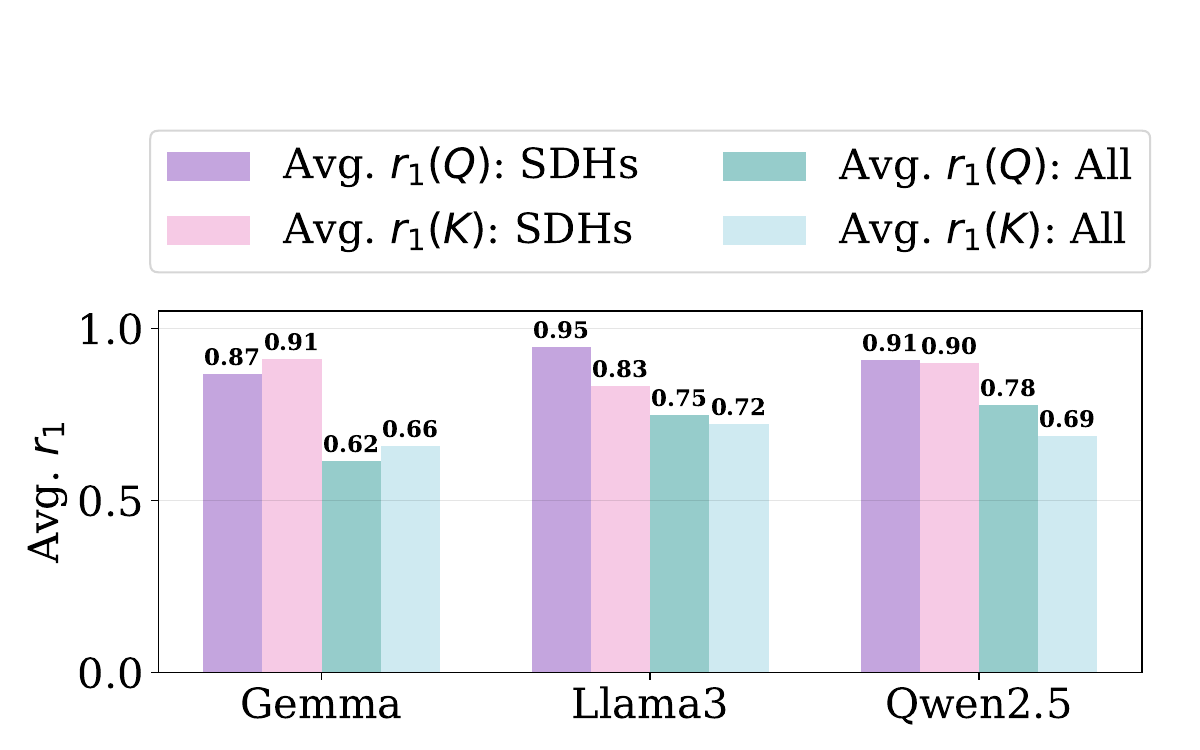}
    \label{fig:r1_com_allhead_small}}
    \hspace{0.01em}
    \subfigure[Average $R_{0.95}(\downarrow)$ of all heads and \ac{sdh}s with small $\Delta$.]{\includegraphics[width=0.48\linewidth]{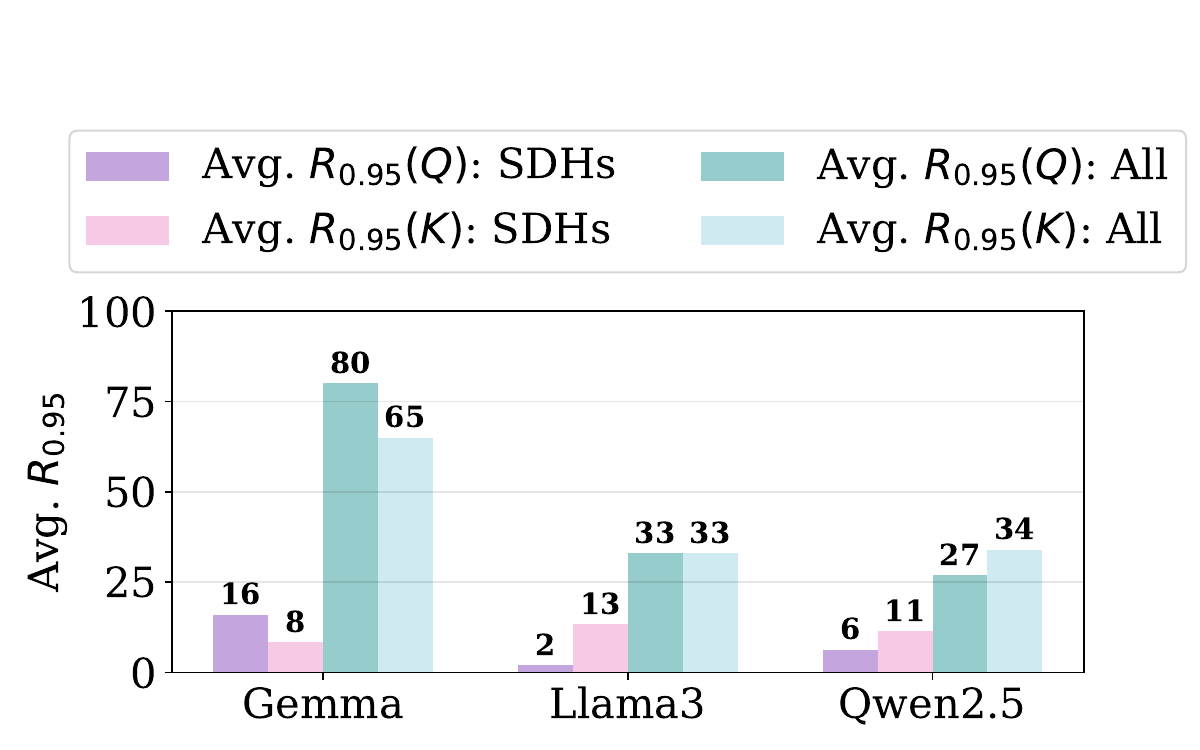}
    \label{fig:R_com_allhead_small}}
    \caption{Comparison of average $r_1$, $R_{0.95}$ of \ac{sdh}s with those of all heads. For \emph{small} $\Delta$, at least one of   $Q$  or $K$ exhibits \emph{substantially lower average ranks} for \acp{sdh} compared to all heads.
}
    \label{fig:avg_low_rank}
\end{figure}

\noindent\textbf{Pre-\ac{pe} Queries and Keys are Almost Rank-One.} 
According to \eqref{eq:lowrank_metrics}, we compute $r_1$ and $R_{0.95}$ for the query and key matrices $Q$ and $K$ on each attention head. In Figure~\ref{fig:avg_low_rank}, we respectively report the average values of these metrics over all heads and over \acp{sdh} only. Head-wise values are provided in Table~\ref{table:head_eg_small_full} of Appendix~\ref{app:head_result}. Particularly, Figure~\ref{fig:r1_com_allhead_small} shows that, for each model, {\color{cc2}at least one of  $r_1(Q)$ and $r_1(K)$ strictly exceeds $0.9$ on average over \acp{sdh} only}, and these values are strictly lower on the other heads.  Figure~\ref{fig:R_com_allhead_small} shows that the effective ranks of $Q$ and $K$ are low only on the \acp{sdh}, while these matrices on the other heads have much higher ranks. These results show that the pre-\ac{pe} queries and keys have substantially lower ranks on \acp{sdh}, and at least one of the queries and keys is almost rank-one. 
This rank-one observation is consistent with Figure~\ref{fig:qwen_qk_small}-(b) and (c). 
We highlight this finding as follows.  

\begin{abox} 
    \looseness -1 \textbf{\hypertarget{takeaway2}{\color{cc1}Takeaway 2}:} On the \acp{sdh} identified by taking $0\leq \Delta\leq 4$ and $\kappa = 0.1$ in \eqref{eq:def_slash_dom}, the pre-\ac{pe} queries $Q$  and keys $K$ all have low effective ranks, measured via $R_{0.95}$ defined in \eqref{eq:lowrank_metrics}. Moreover, at least one of $Q$ and $K$ is close to being rank-one.
\end{abox}

\subsection{How to Get Approximately Rank-One Queries and Keys?} \label{sec:rank_one_qk}
A natural follow-up question is:
\begin{quote}
\centering
\emph{\color{cc1} How \acp{sdh} achieve the approximately rank-one  queries and keys, and what the shared low-dimensional space induced by $Q$ and $K$ represents?}
\end{quote}

\begin{figure}[ht]
    \centering
    \subfigure[$r_1(\uparrow)$  related to queries.]{\includegraphics[width=0.48\linewidth]{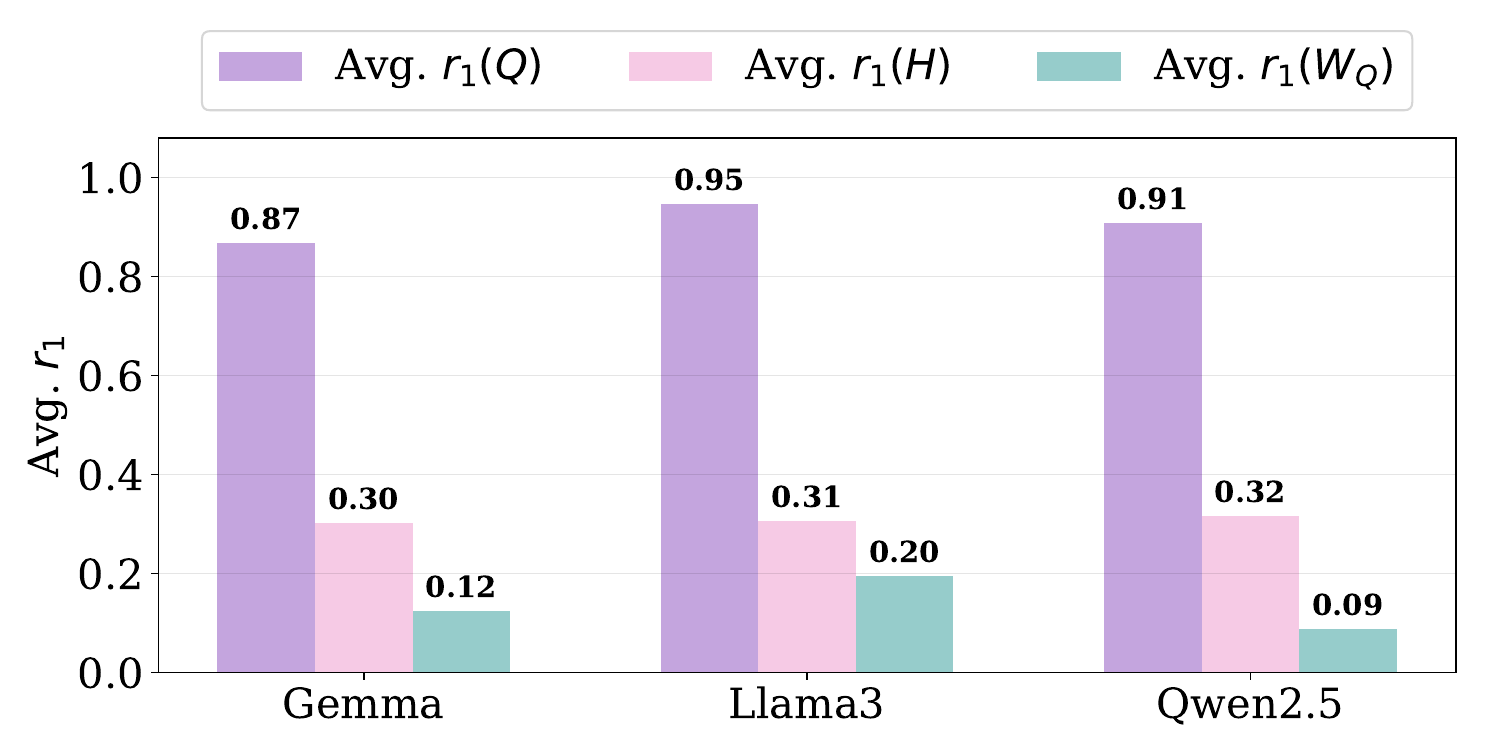}\label{fig:small-rWQ-low}}
\hspace{0.01em}
\subfigure[$r_1(\uparrow)$ related to keys.]{\includegraphics[width=0.48\linewidth]{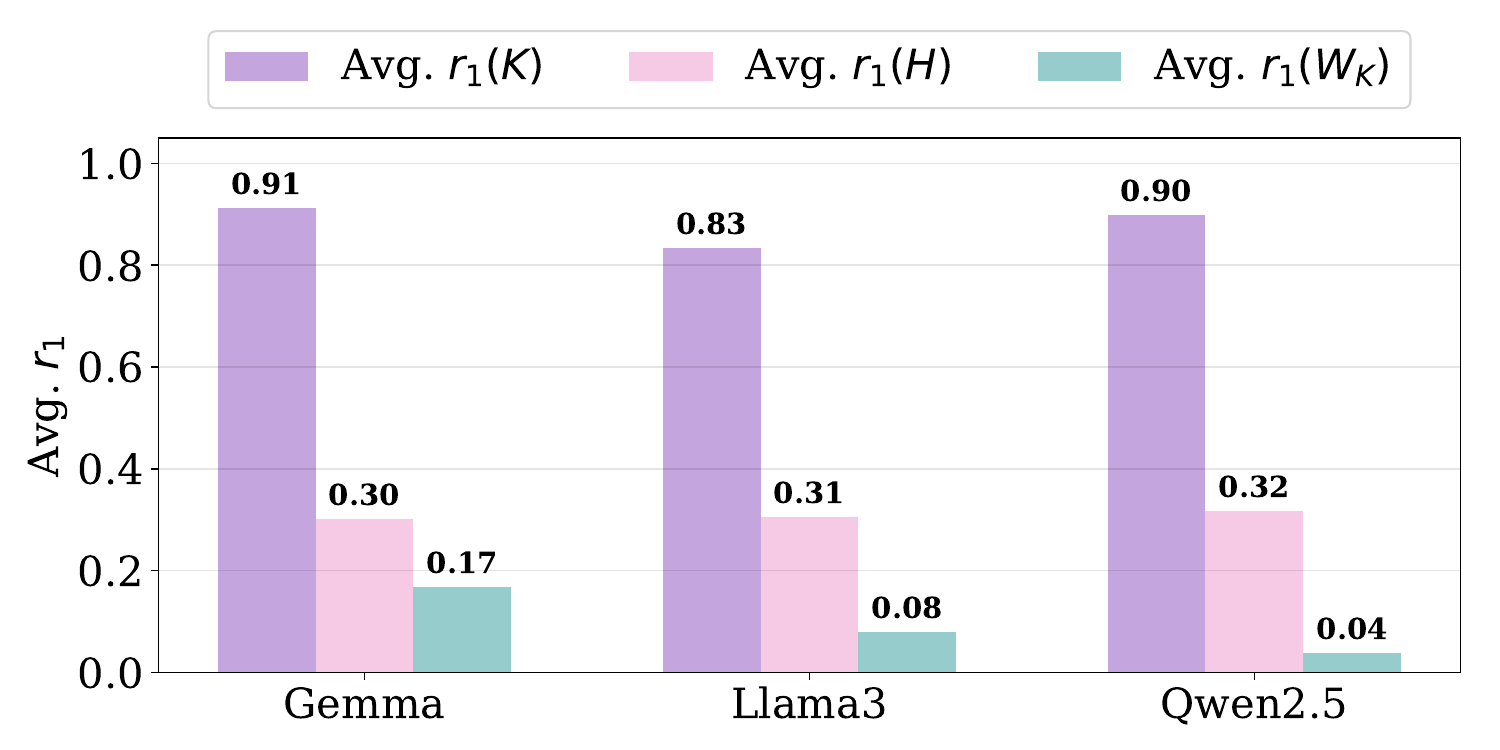}\label{fig:small-rWK-low}}
\vspace{-0.5em}
\subfigure[$R_{0.95}(\downarrow)$ related to queries.]{\includegraphics[width=0.48\linewidth]{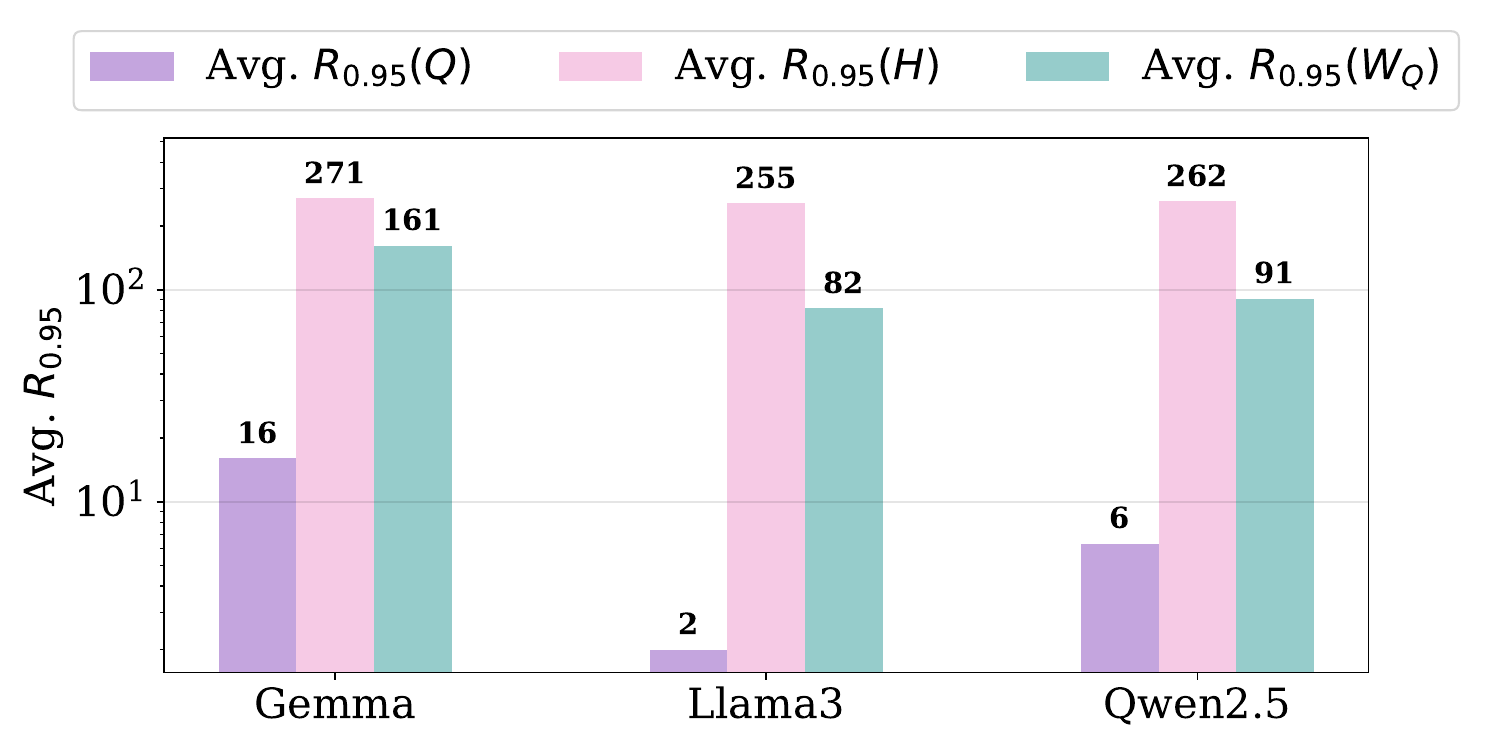}\label{fig:small-RWQ-low}}
\hspace{0.01em}
\subfigure[$R_{0.95}(\downarrow)$ related to keys.]{\includegraphics[width=0.48\linewidth]{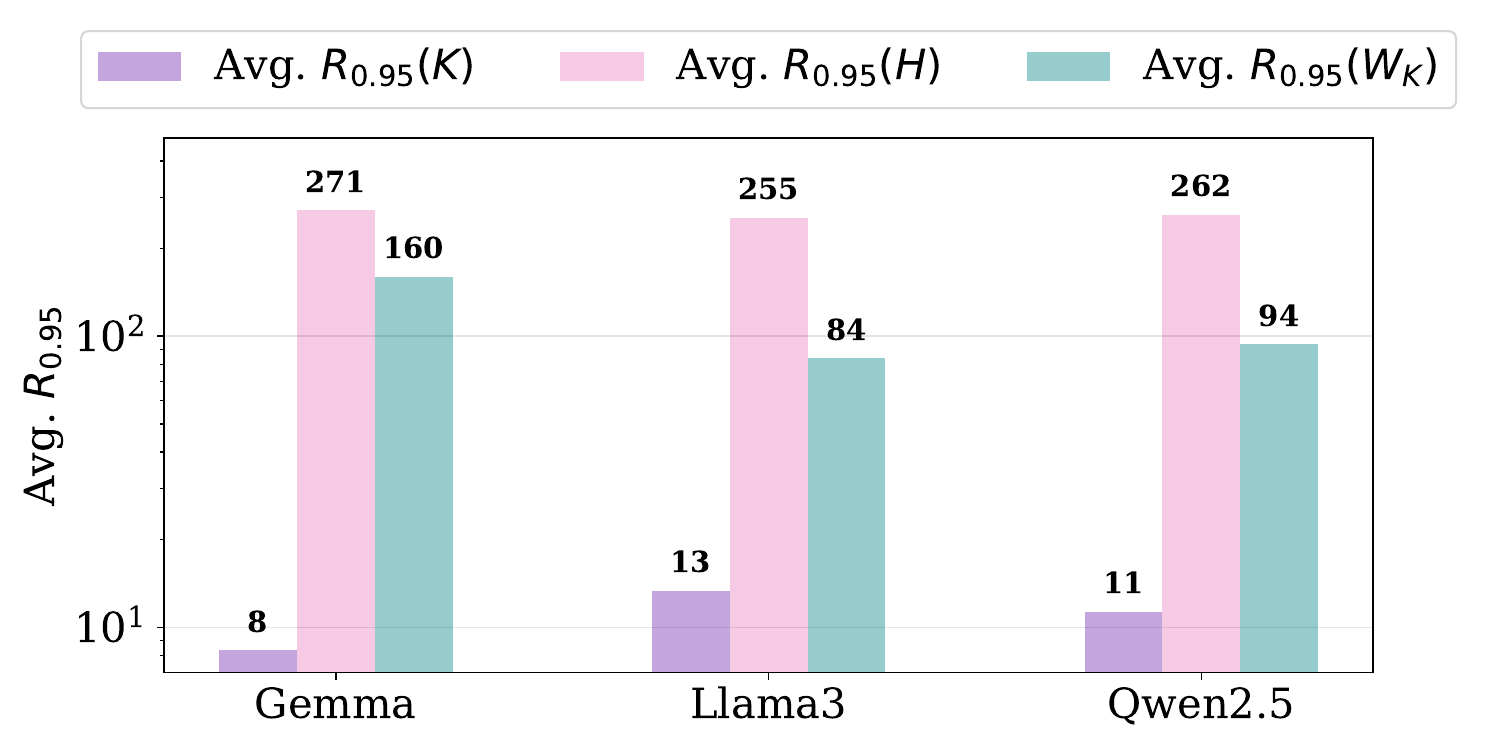}\label{fig:small-RWK-low}}
\caption{Comparison of average $r_1$ and $R_{0.95}$ for query/key matrices relative to the hidden states and their corresponding weight matrices, for \ac{sdh}s with \emph{small} $\Delta$. Both hidden states $H$ and the weight matrices $W_Q$, $W_K$ exhibit much higher rank than the resulting pre-PE queries and keys.}
\label{fig:rank_examine_small}
\end{figure}

Since the queries and keys are obtained 
by $Q = HW_Q$ and $K = HW_K$, where $H$ consists of the hidden states, and $W_Q$ and  $W_K$ are the weight matrices, a natural conjecture is that either $H$ or one of  $W_Q$ and  $W_K$ is low-rank. To test this hypothesis, 
we compute the average values of $r_1$ and $R_{0.95}$ for $H$, $W_Q$ (resp. $W_K$), and $Q$ (resp. $K$) over the identified \acp{sdh}. These values are plotted in Figure~\ref{fig:rank_examine_small}. This figure clearly shows that the effective ranks of $H$, $W_Q$, and $W_K$ are much higher than those of $Q$ and $K$, and thus {\color{cc2}clearly refutes this hypothesis}. Therefore, {\color{cc2}low-rankness of $Q$ and $K$ must arise from the interaction between the hidden states $H$ and weight matrices $W_Q$ and $W_K$}.

\vspace{5pt}
{\noindent \bf Analysis of Token Embeddings.} In general, analyzing the hidden states $H$ is difficult, since at each layer they are formed through complex transformations of all preceding layers and vary with the prefix sequence due to prior attention. To simplify, we focus on the $0$-th layer, where $H$ is the input to the first transformer layer, i.e., the token embeddings. 
In this case, we set the prompt $P$ as the concatenation of all tokens, whose length is the size of the alphabet. Let $H = H(P)$ denote the token embeddings of this prompt, whose $i$-th vector is denoted by  $\hb_i \in \RR^d$ for each $i$. With these hidden states, in the first transformer layer, the $i$-th query of CSA is given by $  \qb_i = W_Q^\top  \hb_i$ for  Gemma-7B and Llama3-8B-Instruct, and $ \qb_i =  W_Q^{\top}\hb_i + \mathbf{b}_Q$ for Qwen2.5-7B-Instruct. Recall that Qwen2.5-7B-Instruct has additional bias terms. We have similar expressions for the $i$-th key $\kb_i$ using $\hb_i$, $W_K$ (and $\mathbf{b}_K$).

\vspace{5pt}

{\noindent \bf Metrics for Subspace Alignment.} 
Similar to $r_{\ell}$ and $R_{\tau}$ defined in \eqref{eq:lowrank_metrics}, we introduce metrics that characterize how well a vector $\mathbf{x} \in \RR^d$ aligns with the leading subspaces of $W_{Q}$ and $W_K$. 
Recall that $W_{Q} \in \RR^{d\times d_1}$ where $d$ is the latent dimension of the model and $d _1 = d/ \mathtt{num\_heads}$ with $\mathtt{num\_heads}$ being the number of attention heads. 
Let $W_Q = \sum_{\ell =1}^{d_1 }\sigma_\ell \mathbf{v}_\ell \mathbf{u}_\ell^{\top}$ denote the  \ac{svd} of $W_Q$, where $\{\sigma_{\ell}\}_{\ell=1}^{d_1}  $ are the sigular values in descending order. 
Note that $\{ \vb_{\ell} \}_{\ell=1}^{d_1}  \subseteq \RR^d$. 
For any $\ell \in [d_1 ]$, we let $\ell(\xb)$ denote the index of $\ell$-th larget element among $\{ (\sigma_i\cdot \vb_{i}^\top \mathbf{x} )^2 \}_{i = 1}^{d_1 }$. Then we define 
\begin{align}\label{equ:lowrank_computation}
    \tilde{r}_\ell (\mathbf{x},W_Q) := \frac{\big(\sigma_{\ell(\mathbf{x})} \mathbf{v}_{\ell(\mathbf{x})}^{\top} \mathbf{x}\big)^2}{\sum_{i=1}^{d}\big(\sigma_i \vb_i^{\top} \mathbf{x}\big)^2}, \quad\text{ and}\quad \tilde{R}_{\tau}(\mathbf{x},W_Q) := \min \bigg\{\ell \in [d_1 ]  \,\bigg|\, \sum\nolimits_{i=1}^{\ell}\tilde{r}_i (\mathbf{x},W_Q)\geq \tau\bigg\}, 
\end{align}
where $\tau \in [0, 1]$. 
We can define $\tilde r_{\ell}(\xb , W_{K})$ and $\tilde R_{\tau } (\xb, W_{K})$ in a similar way. 
Intuitively, $\tilde r_{\ell} (\xb, W_Q)$ characterizes how well $\xb$ aligns with the $\ell(\xb)$-th singular vector of $W_Q$, and $\tilde R_{\tau} (\xb, W_{Q})$ is small when $\xb$ is only spanned by a limited number of vectors among $\{ \vb_{i}\}_{i=1}^{d_1}$. 
This quantity enables us to quantify how transformed hidden states allocate energy across parameter-induced subspaces and, in turn, assess the coupling between hidden states and the attention weights. 
In the following, for brevity, we write $\tilde{r}_{\ell} (Q,W_Q)$ (resp. $\tilde{R}_{\tau}(Q,W_Q)$) as $\tilde{r}_{\ell} (Q)$ (resp. $\tilde{R}_{\tau}(Q)$), and analogously for $K$. 

\begin{figure}[ht]
    \centering
    \subfigure[$\tilde{r}_1(\uparrow)$ and ${r}_1(\uparrow)$ related to queries.]{\includegraphics[width=0.48\textwidth]{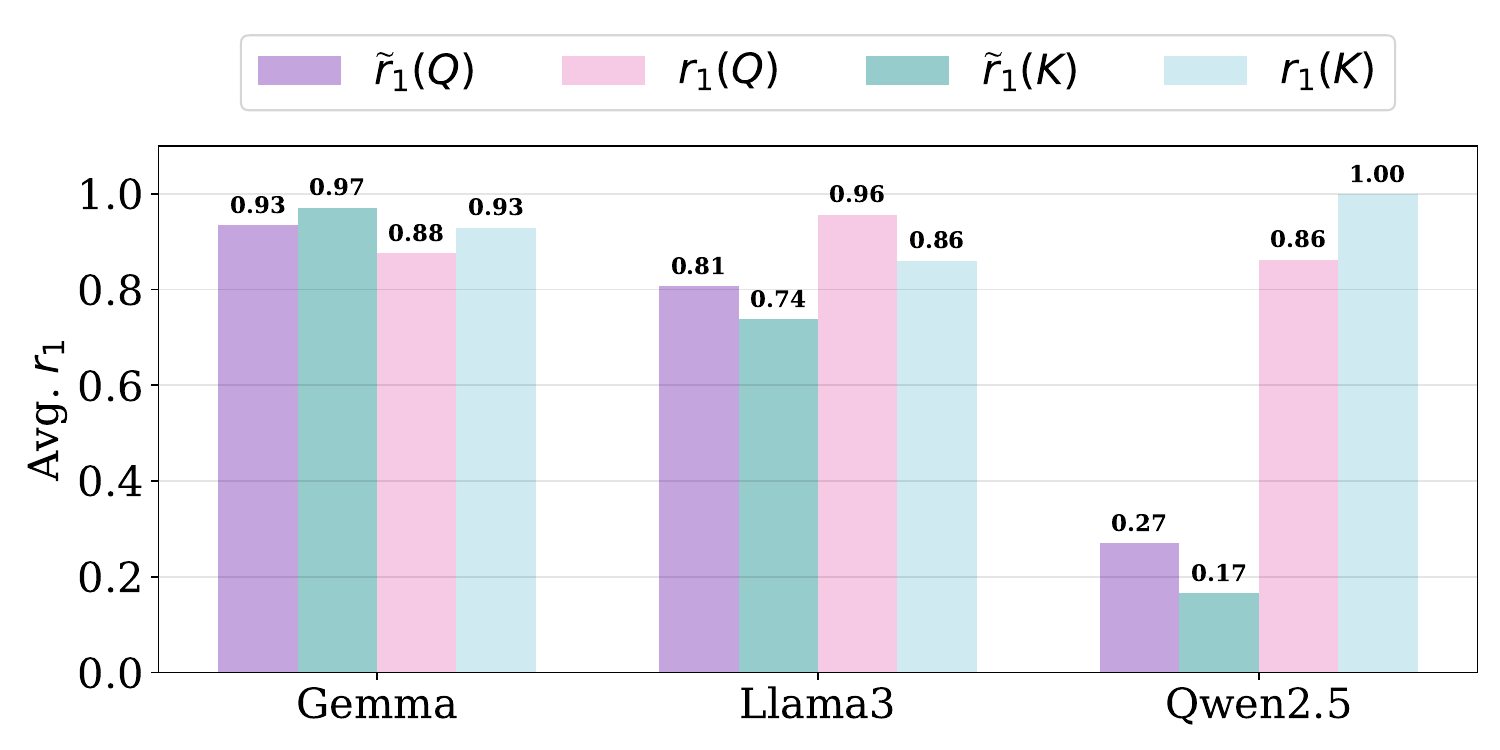}}
    \subfigure[$R_{0.95}(\downarrow)$ and $\widetilde{R}_{0.95}(\downarrow)$ related to queries.]{\includegraphics[width=0.48\textwidth]{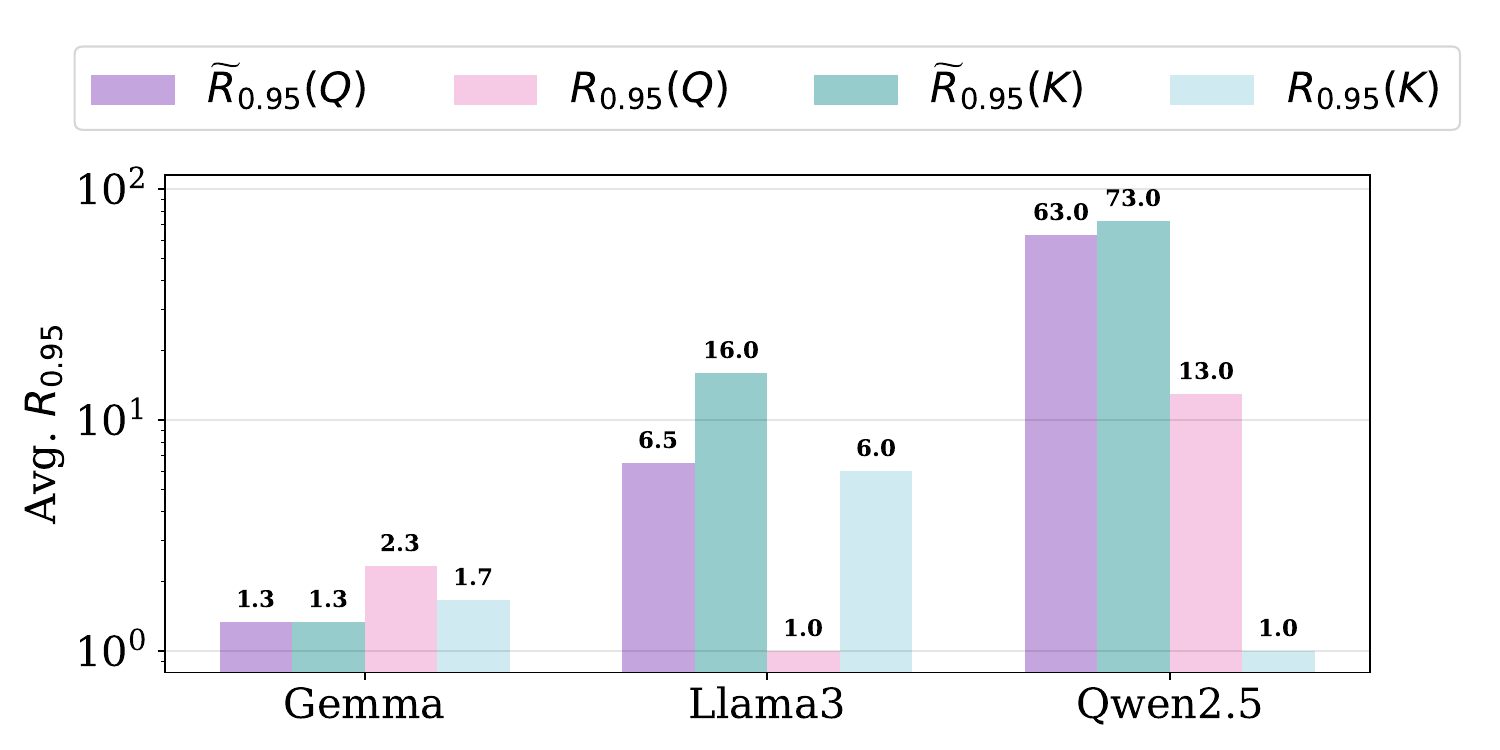}}
    \caption{Comparison of average $r_1$, $R_{0.95}$, $\tilde{r}_1$, and $\tilde{R}_{0.95}$ for query/key matrices relative to the 0-th layer hidden states (token embeddings) for \ac{sdh}s with \emph{small} $\Delta$ in Gemma-7B, Llama3-8B-Instruct, and Qwen2.5-7B-Instruct. }
    \label{fig:til_R_r_small}
\end{figure}

\vspace{5pt}
{\noindent \bf Demystifying Why $Q$ and $K$ are Approximately Rank-One in Zeroth Layer.}
We plot the average values of $\tilde r_1$ and $\tilde R_{0.95}$ in Figures~\ref{fig:til_R_r_small}, and compare them with $r_{1}$ and $R_{0.95}$ respectively.  These average values are computed over all the tokens and the layer zero \acp{sdh}. 
This figure shows that the behaviors of $\tilde{r}_1$ and $\tilde{R}_{0.95}$ closely match those of $r_1$ and $R_{0.95}$ in Gemma-7B and Llama3-8B-Instruct, but differ in Qwen2.5-7B-Instruct.
This confirms that in Gemma-7B and Llama3-8B-Instruct,
{\color{cc2} each token embedding concentrates most of its energy approximately in a one-dimensional principal subspace of $W_Q$ and $W_K$.} That is, each token embedding aligns well with low-rank (and almost rank-one) subspaces of $W_K$ and $W_Q$. 

It is natural to ask  {\color{cc1} whether there exists a special one-dimensional principal subspace of each $W_Q$ or $W_K$ that aligns with all the token embeddings}. To examine this, we define the 
dominating singular vector of each $W_{Q}$ as 
\begin{align*} 
   \{ \vb_{Q}, \ub_{Q} \}  = \{\vb_{\ell^*}, \ub_{\ell^*}\} , \qquad \textrm{where}~~ \ell^{*} = \argmax\nolimits_{\ell\in[d]}\sum\nolimits_{\mathbf{x}\in\mathcal{X}}\big(\sigma_\ell \mathbf{v}_\ell^{\top} \mathrm{RMSN}(\mathbf{x})\big)^2. 
\end{align*}
Here, 
$\mathcal{X}$ is the set of token embeddings, and $\mathrm{RMSN}(\cdot)$ denotes root-mean-square (RMS) normalization \citep{zhang2019root},  
which is implemented before the attention modules in the pretrained models. 
We define $\{ \vb_K, \ub_K\}$ similarly for each $W_K$. 
Note that each head has a different $\{ W_K, W_Q\}$ and thus it is possible to have different dominating singular vectors. 
We examine how well the token embeddings align with these vectors by computing the relative variation (RV), which is defined as 
\begin{align*}
    \text{RV}(\mathbf{v})=\frac{\text{std}\big(\{ \mathbf{v}^{\top}\text{RMS}(\mathbf{x}) \,|\, \mathbf{x}\in\mathcal{X} \}\big)}{\big|\text{mean}\big(\{ \mathbf{v}^{\top}\text{RMS}(\mathbf{x}) \,|\, \mathbf{x}\in\mathcal{X} \}\big)\big|}.
\end{align*}
Here, std$(\cdot)$ and mean$(\cdot)$ denote the standard deviation and mean of a set, respectively. 
We report the average values of $\mathrm{RV}(\mathbf{v}_Q)$ and $\mathrm{RV}(\mathbf{v}_K)$ over the attention heads in layer zero in \Cref{fig:embedding_RV}. The detailed values of each head are deferred
to Table~\ref{table:angle_qk_full} in Appendix~\ref{app:tables_in_L0}. As a baseline, we also report $\mathrm{RV}(\mathbf{v}_{\text{rand}})$ where $\mathbf{v}_{\text{rand}}$ is drawn from the multivariate normal distribution $\mathcal{N}({\mathbf{0}}, I_d)$ on $\RR^{d}$.
This figure shows that $\mathbf{v}_{Q}^{\top}\mathrm{RMS}(\mathbf{x})$ and $\mathbf{v}_{K}^{\top}\mathrm{RMS}(\mathbf{x})$ are  highly concentrated for Gemma-7B and moderately concentrated for Llama3-8B-Instruct. 
In other words, for each $\xb$, the component perpendicular to $\mathbf{v}_{Q}$ or $\mathbf{v}_K$ has a small magnitude, compared with the parallel component. 
This implies that all token embeddings align well with $\mathbf{v}_{Q}$ and $\mathbf{v}_{K}$ on each head. Thus, {\color{cc2} for Gemma-7B and Llama3-8B-Instruct, in the zeroth layer, $Q$ (resp. $K$) is approximately rank-one because all the token embeddings lie approximately on a cone centered around $\vb_{Q}$ (resp. $\vb_K$). 
}

For Qwen2.5-7B-Instruct, we additionally calculate the norms of bias parameters $\|\mathbf{b}_Q\|$ (resp. $\|\mathbf{b}_K\|$) and the average values of $\|W_Q^{\top}\hb_i\|$ (resp. $\|W_K^{\top}\hb_i\|$) over the alphabet, which are reported in \Cref{fig:L0_norm_Qwen}. Detailed values of each \acp{sdh} are deferred to Table~\ref{table:sigma_qk_partial} in Appendix~\ref{app:tables_in_L0}. 
This figure shows that the average norm of hidden states is much larger than that of bias parameters, suggesting that the approximate rank-one structure of $Q$ and $K$ is primarily driven by overly large bias parameters.
\begin{figure}[t]
    \begin{minipage}[t]{0.48\textwidth}
        \centering
        \includegraphics[width=\linewidth]{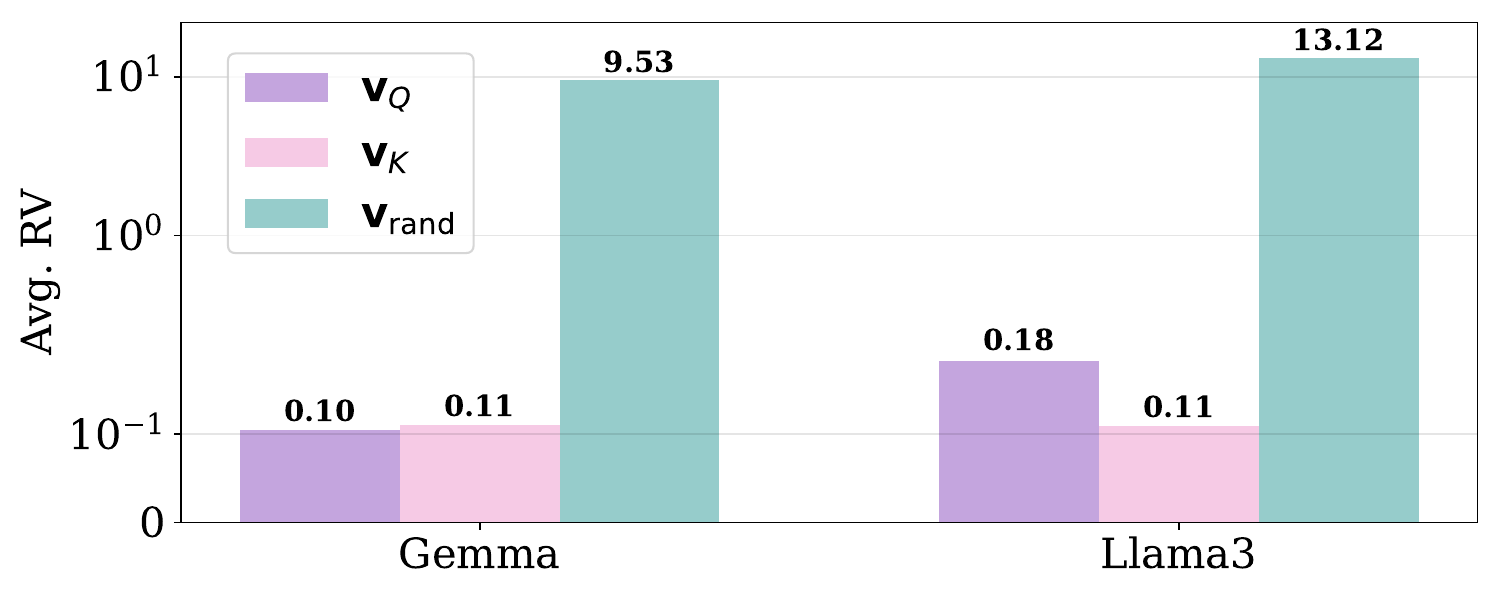}
        \caption{Average of the relative variation of token embeddings projected onto the dominant subspace of the \acp{sdh} in the 0-th layer of Gemma-7B and Llama3-8B-Instruct.}
        \label{fig:embedding_RV}
    \end{minipage}
    \hfill
    \begin{minipage}[t]{0.48\textwidth}
        \centering
        \includegraphics[width=\linewidth]{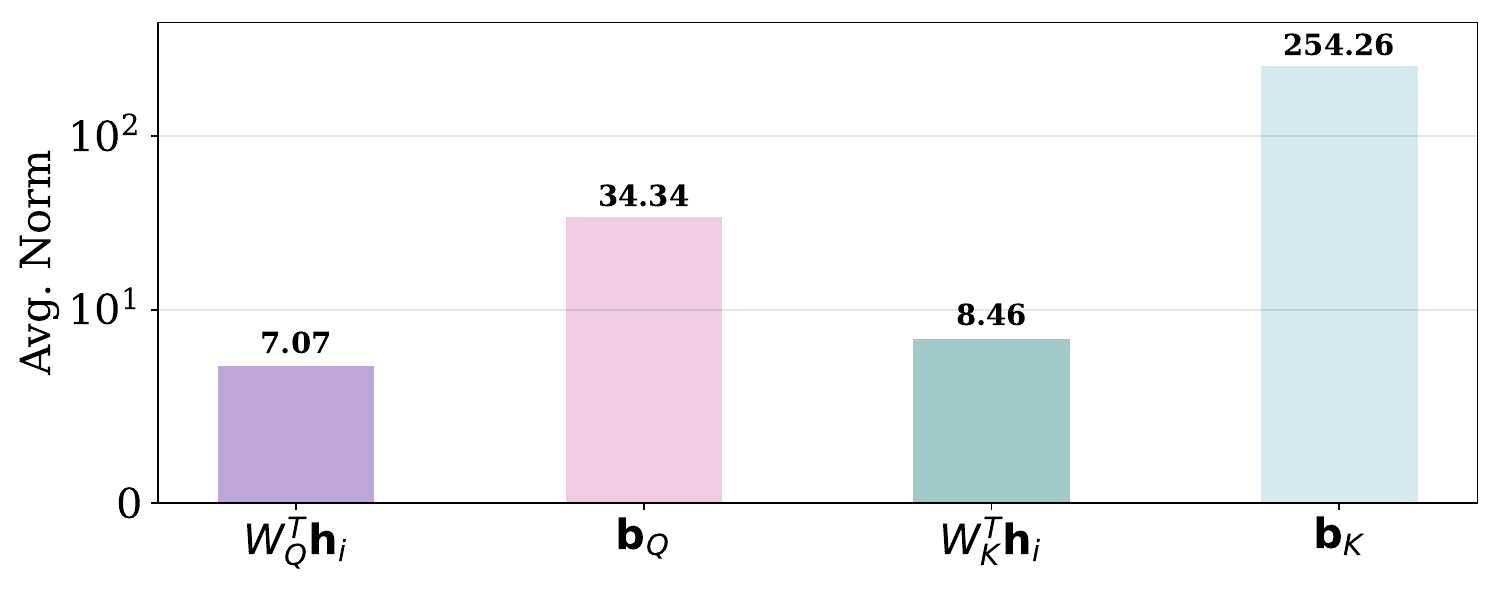}
        \caption{The average norms of $\|W_Q^{\top}\hb_i\|$ and $\|W_K^{\top}\hb_i\|$ over the alphabet and \ac{sdh}, together with the average norms of $\|\mathbf{b}_Q\|$ and $\|\mathbf{b}_K\|$ over \ac{sdh} in the 0-th layer of Qwen2.5-7B-Instruct.}
        \label{fig:L0_norm_Qwen}
    \end{minipage}
\end{figure}

Combining all these facts, we have a complete picture of why $Q$ and $K$ are approximately rank-one in the zeroth layer for the  three models. The conclusion is summarized as follows.

\begin{abox} 
    \looseness -1 \textbf{\hypertarget{takeaway3}{\color{cc1}Takeaway 3}:} In the zeroth layer,  $Q$ or $K$ of \ac{sdh}s are approximately rank-one because of one of the following two reasons: (i) the overly large bias parameter $\bb_{Q}$ or $\bb_{K}$ (Qwen2.5-7B-Instruct) or (ii) the fact that token embeddings are approximately on a cone $\mathcal{C}(\vb,c)=\{\xb \in \RR^d\,|\, \vb^{\top}\xb/\|\xb\|=c\}$ for some axis vector $\vb$ and scalar $c$ (Gemma-7B and Llama3-8B-Instruct).
\end{abox}

In this subsection, we show that at least one of pre-\ac{pe} queries $Q$ and keys $K$ of \ac{sdh}s is approximately low-rank and, in particular, rank-one, as indicated by their high $r_1$ values. This implies that {\color{cc2}queries and keys across different tokens are largely similar}.
These shared query and key vectors are then transformed by \ac{rope}, which introduces the position information by rotating different two-dimensional sub-vectors in queries and keys with different frequencies, as in \eqref{Eq: rope}. 
The inner product between the \ac{rope}-transformed queries and keys form the attention logit matrix. 
Therefore, {\color{cc2}in the \acp{sdh}, the slash pattern, which involves difference in attention scores across different tokens, is primarily determined by the position information introduced by \ac{rope}}. In the following subsection, we aim to answer the question:
\begin{quote}
\centering
    \emph{\color{cc1} How do \ac{rope} and its frequencies contribute to \acp{sdh}?}
\end{quote}

\subsection{Collaboration of Frequencies in \ac{rope} Determines Slash Pattern}\label{sec: rope&sladom}

Recall that we use $\widetilde{\qb}_i$ and $\widetilde{\kb}_j$ to denote the \ac{rope}-transformed queries and keys, respectively. (See \Cref{sec:prelim} for the definition of CSA with \ac{rope}.) 
We first examine how 
 \ac{rope} gives rise to \acp{sdh} when pre-PE queries and keys are approximately rank-one. 
To this end, we decompose each entry of  the attention logit matrix according to the used frequencies in \ac{rope} as
\begin{align}
     \widetilde{\qb}_i^\top \widetilde{\kb}_j:= \sum\nolimits_{\ell=1}^{d/2} \big[R_{\bvartheta,i}\qb_i\big]_{2\ell-1:2\ell}^{\top}\big[R_{\bvartheta,j}\kb_j\big]_{2\ell-1:2\ell}= \sum\nolimits_{\ell=1}^{d/2} \InP(i,j,\ell),\label{Eq: logit decom}
\end{align}
where we define $\InP(i,j,\ell)=\big[R_{\bvartheta,i}\qb_i\big]_{2\ell-1:2\ell}^{\top}\big[R_{\bvartheta,j}\kb_j\big]_{2\ell-1:2\ell}$.
Here, \ac{rope} multiplies the input with the matrix $\RMat{\bvartheta}{i}{d}$ (defined in \eqref{Eq: rope}), which is block-diagonal and consists of a sequence of $2\times 2$ rotation matrices. Each $2\times 2$ rotation matrix rotates a two-dimensional sub-vector of $\qb_i$ by an angle $i\cdot \theta_\ell$.
In this way, the sum above can be expressed as a Fourier-like form with the inner product 
\begin{align} 
\label{eq:logit_decomposition}
\InP(i,j,\ell)=A_{\ell} \cdot \cos\bigl(\theta_{\ell}\cdot (i-j)+\varphi_{\ell}\bigr), 
\end{align}
\begin{wrapfigure}{r}{0.44\textwidth}
\includegraphics[width=\linewidth]{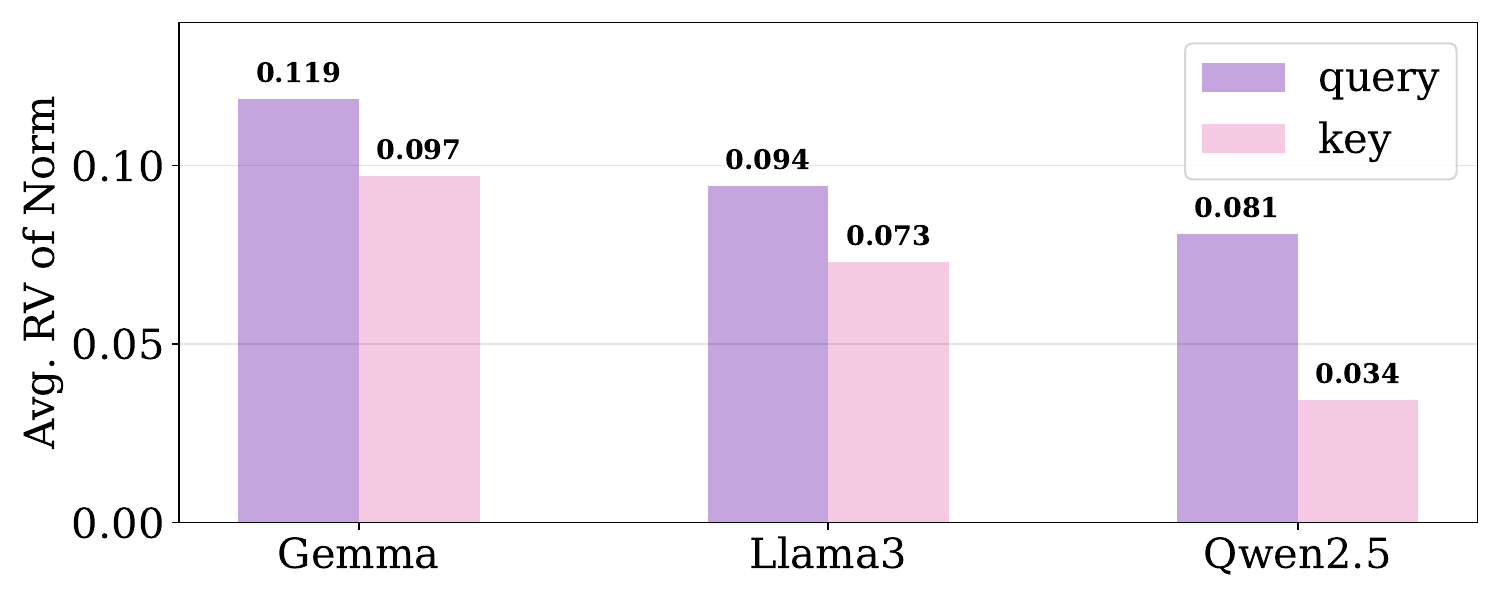}
        \caption{Average of the relative variation
        of the norm of the query vector $\|\qb_i\|$ and key $\|\kb_i\|$ for small $\Delta$ in Gemma-7B, Llama3-8B-Instruct and Qwen2.5-7B-Instruct.}
        \label{fig:QKNorm_RV_small}
        \vspace{-0.4cm}
\end{wrapfigure}
where $A_{\ell}$ and $\varphi_{\ell}$ are almost independent of $i$ and $j$, and $\theta_{\ell}$ comes from \ac{rope}. 
Since $Q$ and $K$ are approximately rank-one, the direction of $\qb_i$ is approximately the same for all $i$ (similar for $\kb_j$). We further compute the relative variation of the norms of query and key vectors and report their averages over tokens and \acp{sdh} in \Cref{fig:QKNorm_RV_small}. As shown in the figure, the relative
variation is at most $0.12$, indicating that the norms of $\qb_i$ (and similarly $\kb_j$) are nearly constant. Consequently, both the directions and magnitudes of the queries and keys
are approximately invariant across token pairs. Thus, $A_\ell$ and $\varphi_\ell$ are approximately identical across token pairs, as they only depend on queries and keys.
From \eqref{eq:logit_decomposition}, it follows that for any token pair $(i,j)$, the corresponding logit, and hence the attention
score, depends only on the relative distance $i-j$, which leads to the slash pattern.
\vspace{3pt}

\noindent \textbf{Contributions of \ac{rope} Frequencies.}
Next, we examine how the different \ac{rope} frequencies collectively contributed to the slash pattern according to \eqref{Eq: logit decom}. 
Note that $\InP(i,j,\ell)$ only involves the influence of the $\ell$-th frequency on the logit from position $i$ to $j$. 
We visualize $\InP(i,j,\ell)$ with $i$ fixed to $i=100$ and $j$ varying within  $[1, i]$ in Figure~\ref{fig:qwen_prod}, which is based on  Qwen2.5-7B-Instruct. Here, each column corresponds to the influence of one frequency on different token positions. 
We observe clear periodic patterns within columns, which reflect the rotational effect of RoPE on queries and keys. 
The magnitude of influence is proportional to the variation within each column. Qualitatively, high- and medium-frequency components ($\ell\in [1, 42]$) exhibit greater variation than low-frequency components ($\ell\in [43, 64]$).
\vspace{-5pt}
\begin{figure}[ht]
\centering
\subfigure[$\InP(100,j,\ell)$ of $\rmL 18\rmH 7$.]{\includegraphics[width=0.32\textwidth]{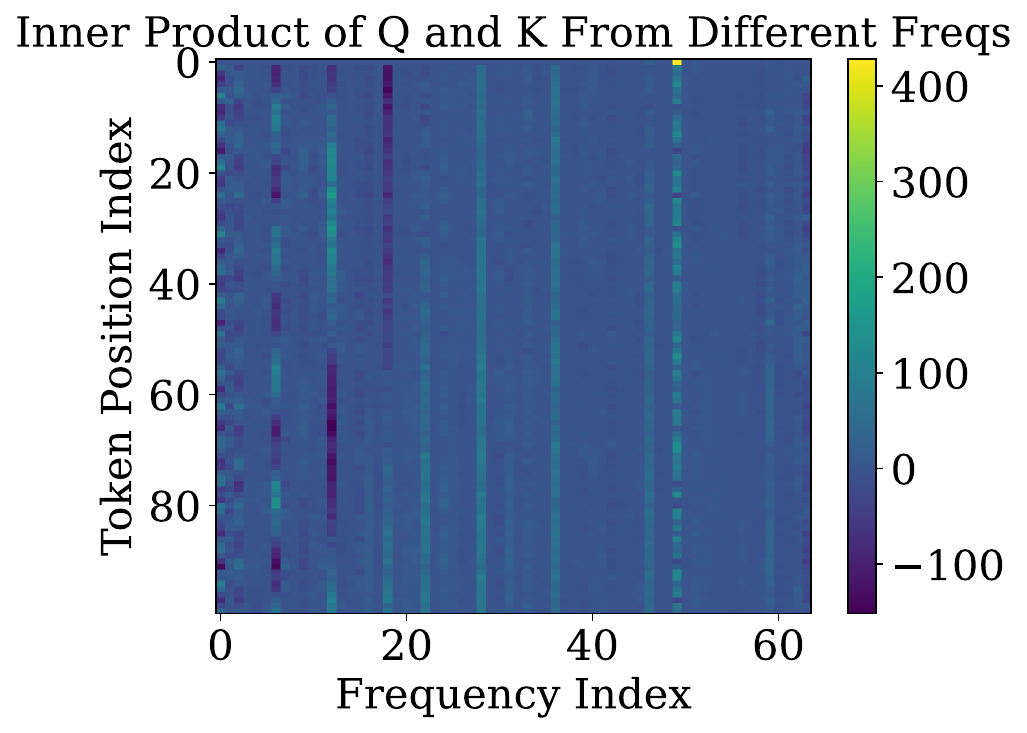}}
\hspace{0.01em}
\subfigure[$\InP(100,j,\ell)$ of $\rmL 21\rmH 15$.]{\includegraphics[width=0.32\textwidth]{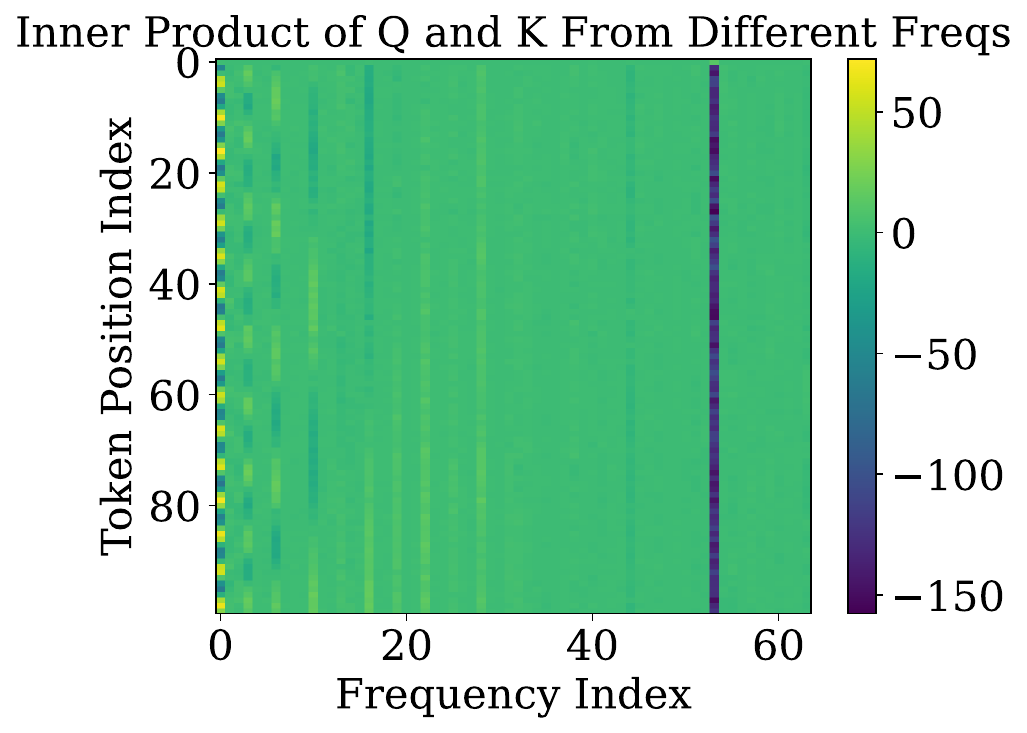}}
\hspace{0.01em}
\subfigure[$\InP(100,j,\ell)$ of $\rmL 1\rmH 3$.]{\includegraphics[width=0.32\textwidth]{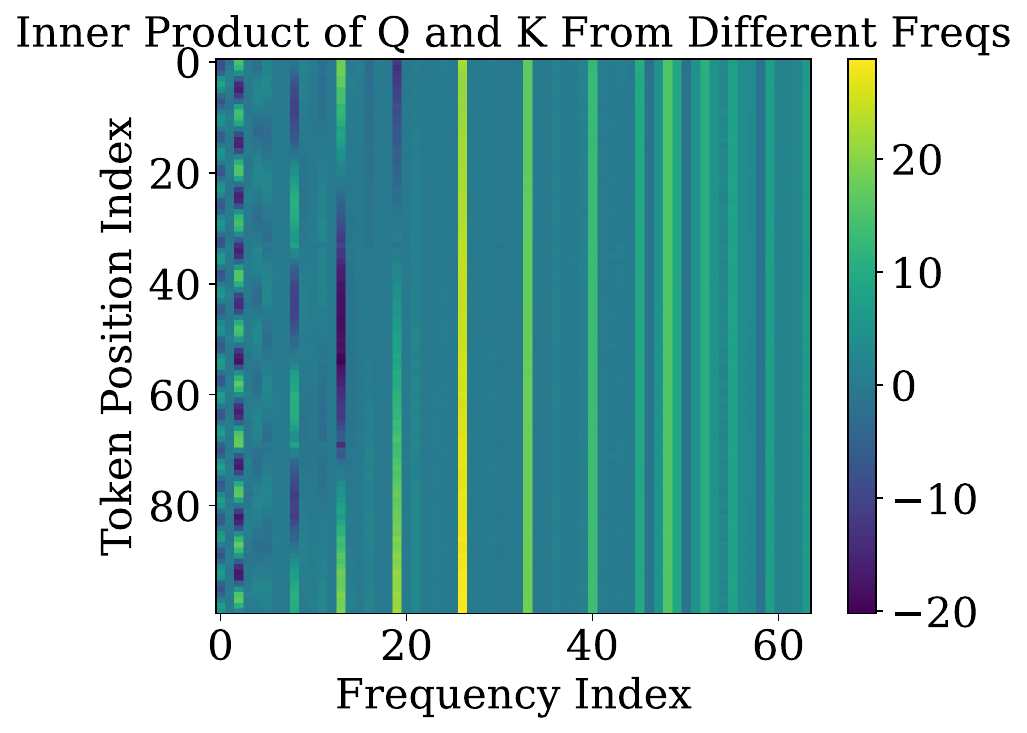}}
\vspace{-5pt}

\caption{Heatmaps of $\InP(i,j,\ell)$ on various \acp{sdh} of Qwen2.5-7B-Instruct. We set $i = 100$ and let $j$ vary. Qwen2.5-7B-Instruct has $64$ frequencies ($1 \leq \ell \leq 64$). The variation within each column shows the influence of each frequency. Qualitatively, high- and medium-frequency components $\ell\in [1, 42]$ exhibit greater variation.}
\label{fig:qwen_prod}
\end{figure}

\begin{wrapfigure}{r}{0.46\textwidth}
\includegraphics[width=\linewidth]{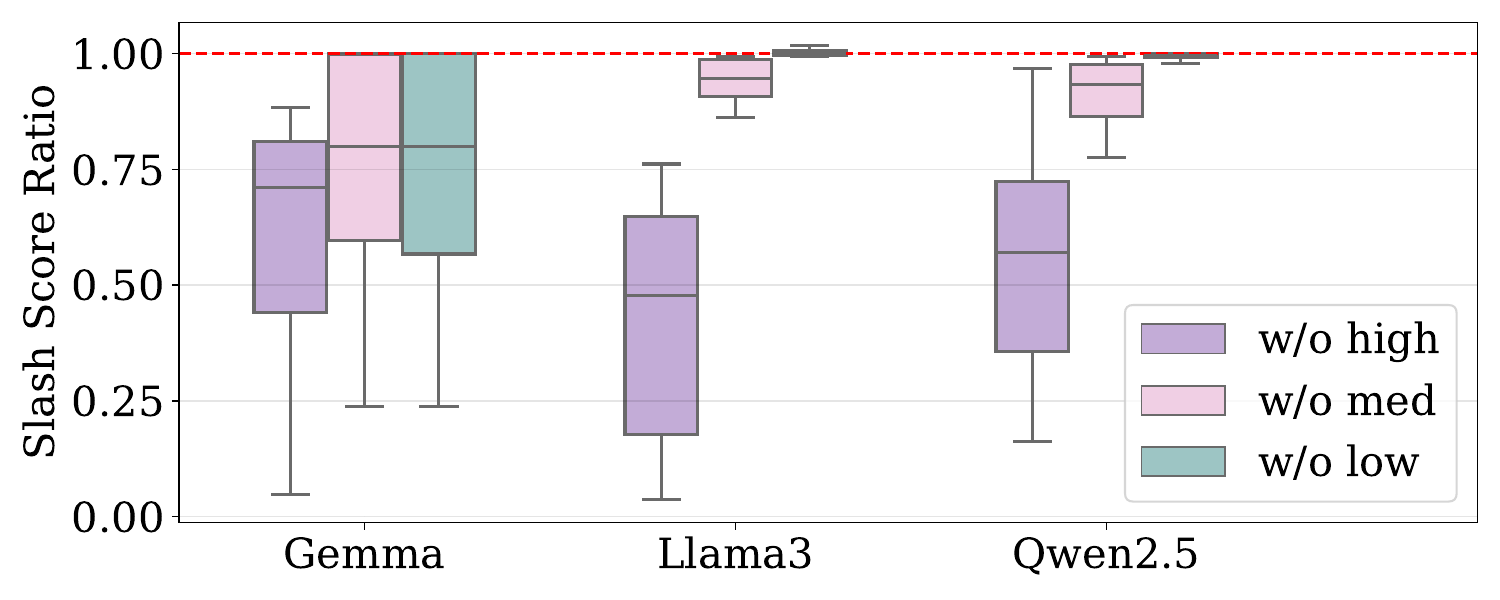}
    
        \caption{The figure quantifies the effect of low-, medium-, and high-frequency components on \acp{sdh} by reporting, for each band, the ratio of the average slash score after removing that band to the original average score.}
        \label{fig:removeFrequency_ratio_small}
\end{wrapfigure}
For a more quantitative characterization, we evaluate slash-dominance after selectively removing certain frequencies when computing the attention score. Specifically, we split the $64$ frequencies into three groups --- high ($\ell\in [1, 21]$), medium ($\ell\in [22, 42]$), and low ($\ell\in [43, 64]$). 
For each group, we remove six frequencies uniformly in each interval, e.g., $\ell=1,5,9,13,17,21$ for $\ell\in [1, 21]$ ($10\%$ of the total), whose corresponding sub-vectors are left unrotated, and compute the attention logits still as in \eqref{Eq: logit decom}. Using these modified attention logits, we compute attention scores using the softmax function and average slash score as Definition \ref{def:slash_dom}. 
We compute the ratio of the new score to the original score for each identified \acp{sdh}, and show the box plot in Figure~\ref{fig:removeFrequency_ratio_small}. 
We observe that, when removing high frequencies, the average slash score decreases in all three models. Removing the medium frequencies leads to a mild decrease, while the low frequencies yield the least drop. 
Therefore, the high frequencies are the most critical for \acp{sdh}, medium frequencies have a moderate impact, and low frequencies contribute the least.

\begin{abox} 
    \looseness -1 \textbf{\hypertarget{takeaway4}{\color{cc1}Takeaway 4}:} 
    {The slash pattern appears because the attention logits approximately admit a Fourier-like decomposition as in \eqref{Eq: logit decom} and \eqref{eq:logit_decomposition}.}
    The high- and medium-frequency components in \ac{rope} are more important than the low-frequency components for \acp{sdh}. 
\end{abox}

\vspace{2pt}
In summary, in this section, we demystify the emergence of \acp{sdh} via empirical studies. Our main findings are listed as in Takeaways \hyperlink{takeaway1}{1}-\hyperlink{takeaway4}{4}. In particular, our reasoning logic goes as follows:
\begin{abox} 
    \looseness -1 
\begin{itemize}
\item [(i)] First, we mathematically characterize the slash pattern in Definition \ref{def:slash_dom} and use that to identify \acp{sdh} empirically on three \ac{llm}s.
\item [(ii)] We show that these identified \acp{sdh} remain valid on \ac{ood} prompts, which implies that slash patterns are intrinsic to the model architectures. 
\item [(iii)] We show that the pre-\ac{pe} queries $Q$ and keys $K$ on the \acp{sdh} are approximately low rank. This implies that the pre-\ac{pe} queries and keys of the same \ac{sdh} approximately point to the same direction across different tokens. This further means that \ac{rope} is the only architectural component of the \ac{llm} that contributes to the across-token differences in attention scores of the slash pattern. 
\item [(iv)] By writing the attention logits into a Fourier-like decomposition, we show that the high- and medium-frequencies of \ac{rope} play a more important role in forming the slash patterns. 
\end{itemize}
\end{abox}

Furthermore, in Section \ref{sec:rank_one_qk}, by focusing on the token embeddings, we investigate why the queries and keys are approximately rank-one. The key observation is that the token embeddings lie approximately on cones, which are then transformed into almost the same direction by $W_Q$ or $W_K$.

\section{Theoretical Study of Shallow Transformers}\label{sec: theory}
In this section, we provide theoretical support for the empirical findings in the small $\Delta$ regime in Section \ref{sec: emp}. 
We show that under two conditions similar to those identified by experiments (Takeaways~\hyperlink{takeaway2}{2} to \hyperlink{takeaway4}{4}): {\color{cc2}(i) token embeddings lie approximately on a cone} and {\color{cc2}(ii) \ac{rope} is dominated by medium- and high-frequency components}, \acp{sdh} provably emerge. 
Moreover, the  {\color{cc2}learned \acp{sdh} are provably \ac{ood} generalizable} (Takeaway~\hyperlink{takeaway1}{1}).
In particular, we focus on a simple but fundamental \ac{icl} setting, and train a shallow transformer with \ac{rope} via Gradient Descent (GD). 
We show that the transformer learns a two-layer induction-head circuit, and the first-layer attention head is an \ac{sdh}.

The theoretical results are organized as follows. We present the \ac{icl} data model in Section~\ref{subsec: data model} and then propose a slash-dominance frequency
condition that quantitatively characterizes the \ac{rope} frequency interactions in Section~\ref{subsec: freq con}. It can be viewed as a mathematically rigorous formulation of Takeaway~\hyperlink{takeaway4}{4}. We present the transformer architecture and training algorithm in Section~\ref{subsec: network arch} and \ref{subsec: training alg}. Finally, we present the main theoretical results in Section~\ref{subsec: theory result}.

\subsection{Data Model}\label{subsec: data model}
We consider a regression task under the ICL framework, where the \ac{sdh} plays an important role. Such a setting is commonly adopted in theoretical studies of ICL and induction heads~\citep{zhang2024trained,huang2024context,chen2024training,yang2024context,zhang2024context}. The goal of in-context regression is to correctly learn a predictor $\yq:=\widehat{y}(\xbq)$ that approximates the true response
$y = f(\xbq)$ for a question $\xbq \sim \cD_{\mathcal{X}}$ and a target
function $f \sim \cD_{\mathcal{F}}$, based on a prompt $P$ containing demonstration examples with the same function $f$. 
Specifically, the data are sampled as follows. We first draw a function $f \sim \cD_{\mathcal{F}}$. Next, we independently sample a collection of $\Nit$ inputs $\xb_1, \ldots, \xb_{\Nit}$ together with a question $\xbq$, all from $\cD_{\mathcal{X}}$. For each $\xb_i$, we define the label as $y_{i}=f(\xb_i)$.
Then we construct the prompt $P=(\xb_1,y_1,\cdots,\xb_{\Nit},y_{\Nit},\xbq)$, which has a fixed length of $N=2\Nit+1$.
 
\vspace{2pt}
\noindent\textbf{Task Distribution.}
We consider a specialized \emph{linear regression} family: $\mathcal{F}=\big\{f: \cX\rightarrow \mathbb{R} 
 \mid f(\xb)=\langle \wb,  \xb \rangle \text{ with } \wb \in 
 \mathbb{R}^{d_\cX}, \|\wb\|_2 \leq \sqrt{d_\cX}, \cX \subseteq \mathbb{R}^{d_\cX} \big\}$, which is widely adopted in recent studies for in-context learning~\citep{zhang2024trained,huang2024context}. As a result, the distribution $\cD_{\cF}$ is induced by the distribution of the random weight vector $\wb$, denoted by $\cD_{\Omega}$. For simplicity, we consider a normalized case where $\EE[\wb]=0$ and $\Var[\wb]=I_{d_\cX}$. 
 
\vspace{2pt}
\noindent\textbf{Data Distribution.}
We consider the \emph{feature learning} setting, where $\mathcal{X}=\left\{\vb_1,\ldots,\vb_K\right\} \subseteq \RR^{d_\cX}$ are a finite set of $K$ orthogonal and normalized features. Here, with slight abuse of notation, we use $K$ to denote the number of features. Each data point $\xb$ is sampled from $\cX$ with the probability ${p_k}$ for sampling $\vb_k$, where $p_k\in (0,1)$ for $k\in[K]$ and $\sum_{k\in[K]}p_k=1$. Such a data model has been widely employed in the theoretical studies of deep learning, including ensemble methods~\citep{allen2020towards}, and in-context learning~\citep{huang2024context}. 
For simplicity, we consider the balanced case where $p_k = {1}/{K}$ for all $k \in [K]$. The extension to the imbalanced case, where $p_k$ takes non-uniform values, can be derived similarly to \citet{huang2024context}.

\subsection{Slash-Dominance Frequency Condition}\label{subsec: freq con}
From observation in \Cref{sec: rope&sladom}, we find that \ac{rope} and, especially, its frequencies play a crucial role in the slash-dominant pattern. We introduce a {\color{cc2}quantitative slash-dominance frequency
condition} that precisely characterizes frequency interaction in this section. Because \ac{rope} operations depend on the token embeddings, we first introduce the token embedding of prompts as follows.
 
\paragraph{Token Embedding.}
For each $i \in [\Nit]$, we consider embedding information of input $\xb_i$ at position $2i-1$ and its label $y_i$ at position $2i$ into orthogonal subspaces. Concretely, we consider $E_{2i-1}=[\cbb^\top \;  \xb_i^\top \;  0 \; 0]^\top$, and $E_{2i}=[ \cbb^\top \;  0 \; 1 \; y_i]^\top \in \RR^{d}$, where $\cbb \in \RR^{d_\cbb},\|\cbb\|_2=1$ represents the semantically independent information and is identical across all token embeddings, and $d=d_\cbb+d_\cX+2$. In addition, $\Eq=[\cbb^\top\; \xbq^\top\;0\;0]^\top$. As a result, the token embedding $E$ is defined as follows.
\begin{align*}
E=E(P)=
\begin{bmatrix}
     E_1 & \cdots & E_{2{\Nit}} & \Eq 
\end{bmatrix}^\top \in \RR^{N \times d }.  
\end{align*}
 In the definition above, all token embeddings lie on a cone whose axis is given by $[ \cbb^\top \;  0 \; 0 \; 0]^\top$. 
 This is motivated by Takeaway~\hyperlink{takeaway3}{3}, and to further distinguish the semantically independent information (the cone axis) from the semantically dependent information, we embed the semantically dependent and independent components in orthogonal coordinate subspaces. For brevity, we further denote the semantically dependent subspace $E^{\xb,y}=E_{:,d_\cbb+1:d}\in \RR^{N\times (d_\cX+2)}$ and $E^y=E_{:,d} \in \RR^{N}$, 
 which are the last $d_\cX+2$ columns and the last column of $E$, respectively.

Note that \ac{rope} has $d/2$ frequencies. In the following, we state the frequency assumptions for the cone component (the first $d_\cbb$ dimensions) and the semantic component (the last $d-d_\cbb$ dimensions), respectively.

\begin{assumption}[Slash-Dominance Frequency sequence]
\label{Assp 1: Frequency seq}
Let $N$ be the fixed length of the prompt $P$ and $d$ be the hidden dimension of embeddings. Denote the Kronecker delta function by $\delta_0(x):=\mathbbm{1}\{x=0\}$. We assume that   frequency sequences $\bvartheta=(\theta_1,\cdots,\theta_{d/2})$ satisfy the following:  
\vspace{1pt}
\begin{enumerate}[topsep=0pt, left=0pt, label={\arabic*.}]
    \item \textbf{The last ${(d-d_\cbb)}/{2}$ low frequencies are small enough.} For any ${d_\cbb}/{2}+1 \leq s \leq {d}/{2}$, $\theta_s\leq O({N^{-\alpha}})$ is a low frequency with $\alpha \geq 2$. 
    \item \textbf{Sinusoal components of the first ${d_\cbb}/{2}$ frequencies approximate the pulse.} There exists constants $L,C_1 \geq 0, C_2 \in \RR$, and a small enough noise   tolerance $\epsFN \leq {L C_1}/{N}$ such that for any $x \in \ZZ$  and $ |x| \leq N$, we have
\begin{align}
    \big|\sum\nolimits_{s=1}^{d_\cbb/2} \cos(\theta_s x) - C_1\cdot \delta_0(x)-C_2\big| \leq \epsFN. \label{Eq: pulse}
\end{align} 
\end{enumerate}
\end{assumption}

We assume that the \emph{last ${(d-d_\cbb)}/{2}$ frequencies}, corresponding to  semantically dependent content $\xb,y$, are all low. This is aligned with the observation in \cite{barbero2024round} that low frequencies act primarily on semantic content.
Conversely, the \emph{first $d_\cbb /2$ frequencies} correspond
to the semantically independent content $\cbb$. This assumption states that if the selected frequency components have identical amplitudes $1$, then the resulting sum of sinusoidal components approximates a pulse, which is represented by the Kronecker delta function $\delta_0(x)$ as in \eqref{Eq: pulse}. Here, the uniform amplitudes are intrinsic to the Fourier expansion of the approximation target $\delta_0(x)$, which distributes equal amplitude across all basis frequencies.

After training, however, the relative contributions of frequencies need not remain uniform, as their effective amplitudes are determined by the learned weight matrices and token embeddings. Assumption \ref{Assp 1: Frequency seq} further verifies that not all frequencies in \ac{rope} are required for the \acp{sdh}. In particular, the low frequencies do not affect the validity of Assumption~\ref{Assp 1: Frequency seq}, consistent with Takeaway~\hyperlink{takeaway4}{4}. On the contrary, if medium-frequencies or high-frequencies are set to zero, Assumption \ref{Assp 1: Frequency seq} may be violated, and slash-dominance no longer holds. As shown in \Cref{fig:removeFrequency_ratio_small}, setting part of the medium- or high-frequencies to zero substantially reduces the average slash scores. 

\begin{figure}[t]
\centering
\subfigure[Ideal first ${d_\cbb}/{2}$ frequencies.]{\includegraphics[width=0.42\textwidth]{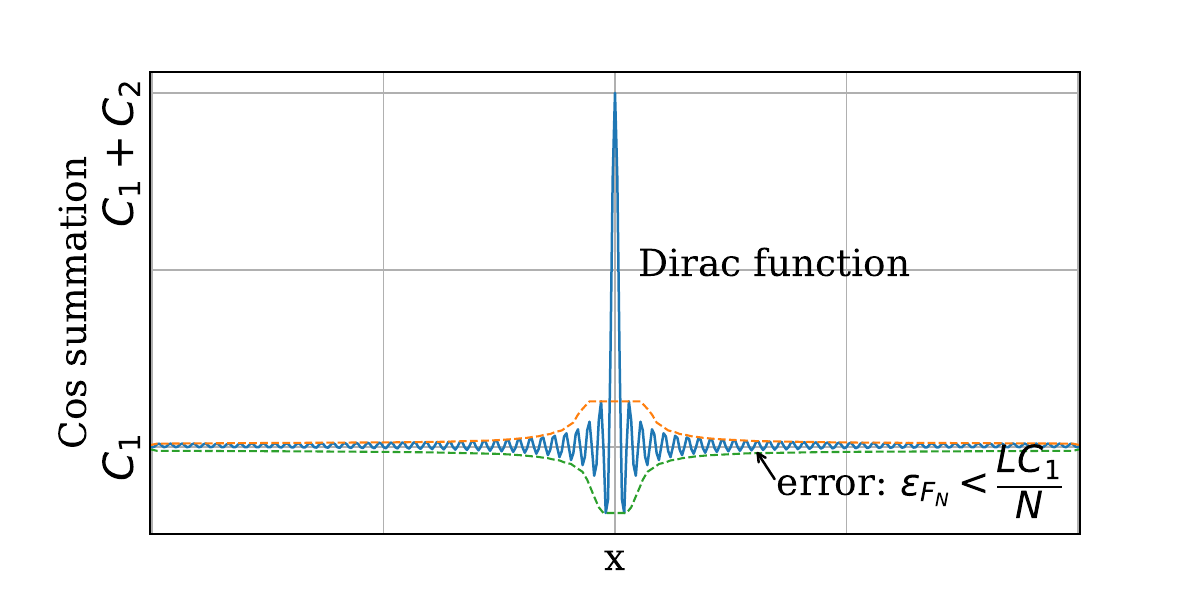}}\label{fig: Freq illustration}
\quad
\subfigure[Gemma-7B: $\rmL 0\rmH 7$.]{\includegraphics[width=0.42\textwidth]{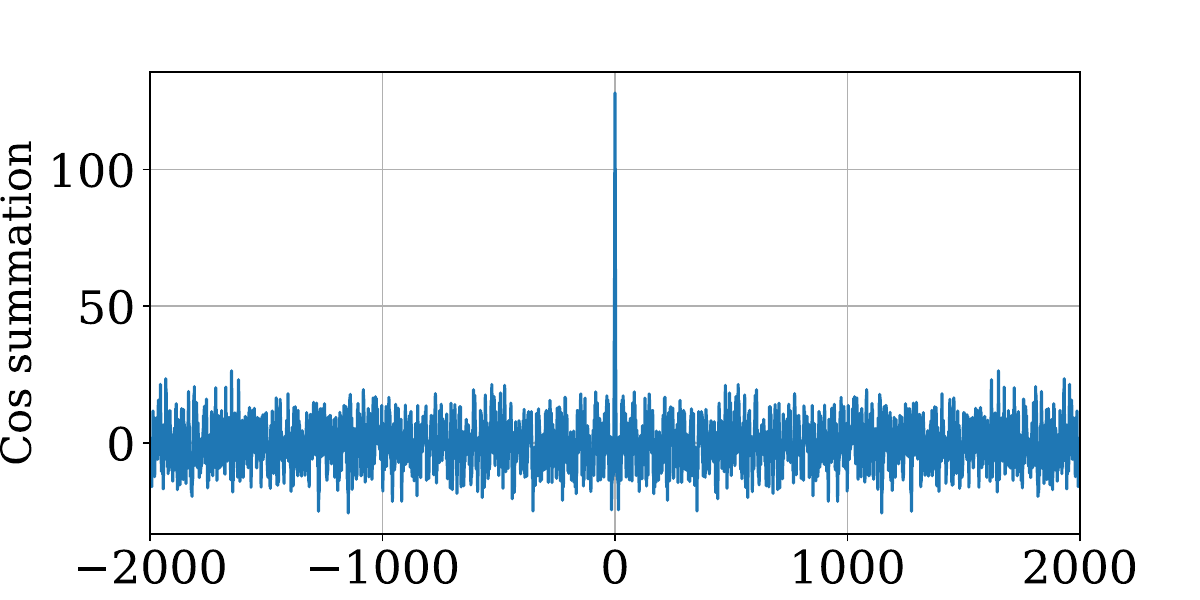}}

\subfigure[Llama3-8B-Instruct: $\rmL 0\rmH 2$.]{\includegraphics[width=0.42\textwidth]{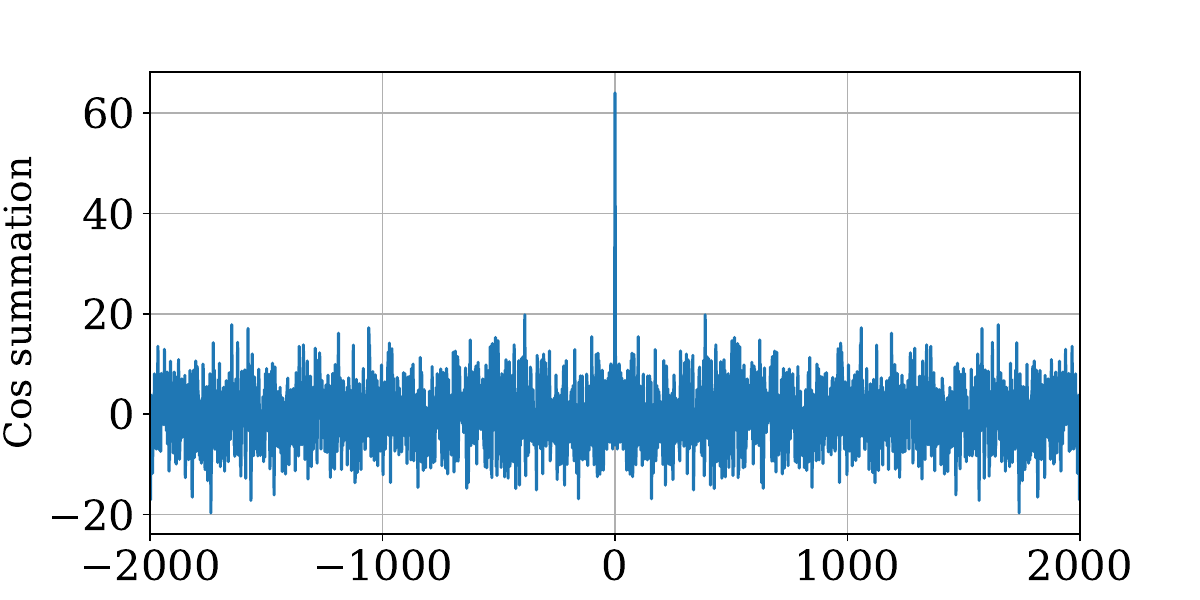}}
\quad
\subfigure[Qwen2.5-7B-Instruct: $\rmL 13\rmH 13$.]{\includegraphics[width=0.42\textwidth]{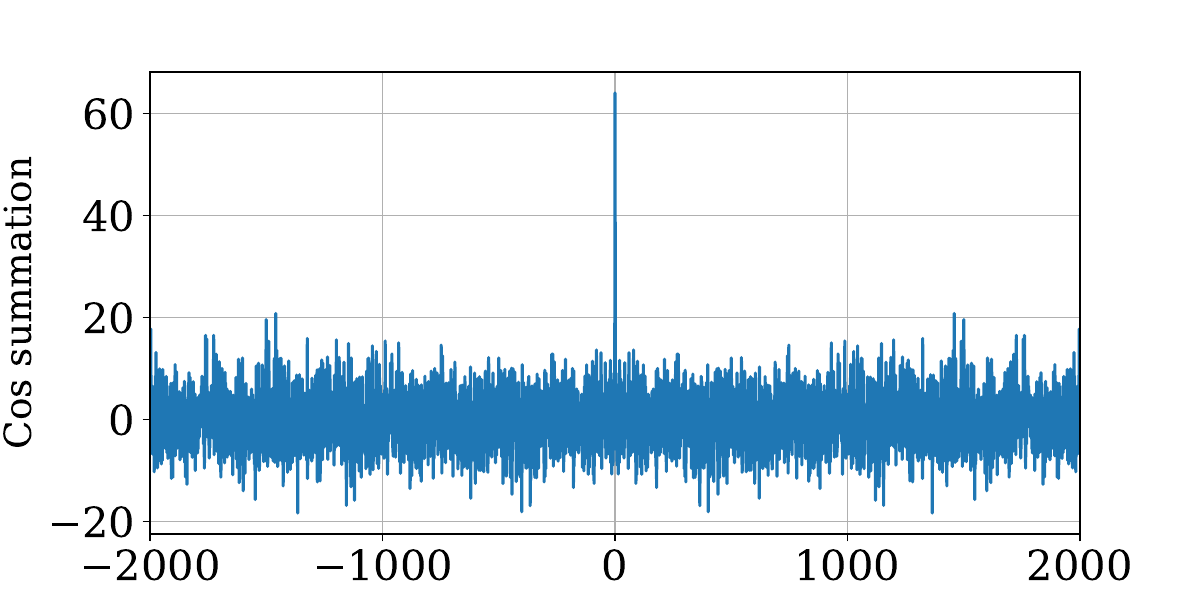}
}
\caption{Summed cosine functions of frequencies. Panel (a) is the ideal case using the first ${d_\cbb}/2$ frequencies in Assumption \ref{Assp 1: Frequency seq}; Panel (b)--(d) use the \emph{active} frequencies in practice (defined as those with average $\InP(i,j,l)$ greater than one-tenth of the largest $\InP(i,j,l)$). Results (b)--(d) for Gemma-7B, Llama3-8B-Instruct, and Qwen2.5-7B-Instruct show a pattern consistent with (a).}
\label{fig: cos sum in prac}
\vspace{-0.2in}
\end{figure}
\vspace{3pt}
\noindent\textbf{Empirical Validation of Assumption \ref{Assp 1: Frequency seq}.}
We remark that the slash-dominance frequency condition is satisfied by pretrained models. Specifically, we first identify the \emph{active frequencies}, defined as those whose average $\InP(i,j,l)$ exceeds one-tenth of the maximum. These active frequencies contribute primarily to the attention scores. We plot their sinusoidal sum given on the left-hand side of \eqref{Eq: pulse}. As shown in \Cref{fig: cos sum in prac}, the active frequencies employed in open-source LLMs (e.g., Gemma-7B, Llama3-8B-Instruct, Qwen2.5-7B-Instruct) satisfy this condition.

\vspace{-2pt}
\begin{abox} 
\vspace{-2pt}
    \looseness -1 \textbf{\color{cc1}Takeaway 5:} Assumption~\ref{Assp 1: Frequency seq} is empirically validated and serves as a sufficient condition on the \ac{rope} frequencies for inducing \acp{sdh}.
\vspace{-2pt}
\end{abox}
\vspace{-2pt}

\subsection{Network Architecture}\label{subsec: network arch}
We next introduce the network structure, which takes the token embedding $E$ as input and outputs the desirable prediction ${\yq}$.

\noindent\textbf{Reduced Two-layer single head Disentangled Transformer.} We consider a \emph{two-layer single head transformer}, and for simplicity in analyzing the residual stream and output of each CSA layer separately, we adopt the \emph{disentangled transformer}  (see Definition~\ref{def:disentangled} and \eqref{Eq: reduce W_q}), which feeds the concatenation, rather than the sum, of the residual and CSA layer output as the input into the subsequent layer. Notably, the \emph{disentangled transformer} is equivalent to a standard decoder-based, attention-only transformer~\citep{nichani2024transformers}, and thus its use entails no loss of generality.

\begin{definition}[Two-layer Disentangled Transformer]\label{def:disentangled}
Let $E \in \RR^{N\times d_0}$ with $d_0=d$ be the input token embedding and $\bvartheta^{(0)}=\bvartheta$ denote the initial \ac{rope} frequency sequence. Define the hidden dimensions of two subsequent layers as $d_1=2d_0$ and $d_2=2d_1$. For each layer $\ell=1,2$,
let $W_{\{Q,K,V\}}^{(\ell)}:=\{W_Q^{(\ell)},W_K^{(\ell)},W_V^{(\ell)}\} \subset \RR^{d_{\ell-1}\times d_{\ell-1}}$ be weight matrices, and let
$ W_O\in\mathbb{R}^{d_2\times d_{\mathrm{out}}}$ be the output matrix. Set all parameters $
\theta=\big\{W_{\{Q,K,V\}}^{(\ell)}\big\}_{\ell=1}^{2}$$\bigcup\,\{W_O\}.$ The frequency sequence is updated at each layer by $\bvartheta^{(\ell)}=({\bvartheta^{(\ell-1)}, \bvartheta^{(\ell-1)}})$.  
Then the hidden states $H^{(\ell)}\in\mathbb{R}^{N\times d_\ell}$ at each layer $\ell=0,1,2$ and the final output of ${\mathrm{TF}}_{\theta}$ are defined as follows:
\begin{align*}
&H^{(0)} = E,\\
&H^{(\ell)} = \big[
  H^{(\ell-1)},\,
  \CSA\big(H^{(\ell-1)};W_{\{Q,K,V\}}^{(\ell)},\bvartheta^{(\ell)}\big)\big]
, \quad\ell=1,2\\
&{{\mathrm{TF}}}_{\theta}(E) = H^{(2)}  W_O.
\end{align*}
We abbreviate $\CSA\big(H^{(\ell-1)};W_{\{Q,K,V\}}^{(\ell)},\bvartheta^{(\ell)}\big)$ as $\CSA\big(H^{(\ell-1)}\big)$ whenever the context is clear. Moreover, as in \eqref{eq:attention_scores}, we denote Layers 1 and 2 attention scores by $S^{(1)}(E;\theta)$ and $S^{(2)}(E;\theta)$, respectively.
\end{definition}

   Before formally giving the expression of the transformer output, we first simplify the weight matrices $
\theta=\big\{W_{\{Q,K,V\}}^{(\ell)}\big\}_{\ell=1,2}$$\bigcup\,\{W_O\}$ for ease of analysis. A visualization of the simplification is provided in \Cref{fig:layer1weight,fig:layer2weight} in Appendix~\ref{app: Notation}.  

First, for $\ell=1,2$, we reduce $W_{K,Q}^{(\ell)}$ to sparse matrices as follows. 
\begin{align}
    W_Q^{(1)}=\begin{bmatrix}
          \widetilde{W}_Q^{(1)}     & 0_{d_\cbb\times(d_\cX+2)}\\
          0_{(d_\cX+2) \times d_\cbb} & 0_{(d_\cX+2) \times (d_\cX+2)}
        \end{bmatrix}, \quad
        W_Q^{(2)}=
        \left[\begin{array}{c:c}   
            \begin{array}{cc}
                 0_{d_\cbb\times d_\cbb}& 0_{d_\cbb\times (d_\cX+2)} \\
                 0_{(d_\cX+2)\times d_\cbb} & \widetilde{W}_Q^{(2)}
            \end{array}& 0_{d\times d}\\
            \hdashline
            0_{d\times d} & 0_{d\times d}
        \end{array}\right]
        ,\label{Eq: reduce W_q}
\end{align}
where the block $\widetilde{W}_Q^{(1)} \in \RR^{d_\cbb \times d_\cbb}$ is the only non-zero block in ${W}_Q^{(1)}$. This reduction follows from Takeaway \hyperlink{takeaway2}{2} and \hyperlink{takeaway3}{3}. Concretely, in this simplified model, query $Q$ exhibits almost rank-one and lies within the subspace generated by the cone axis $[ \cbb^\top \;  0 \; 0 \; 0]^\top$. Consequently, the first-layer weights need only take nonzero values on the corresponding subspace to map the token embeddings distributed on a cone onto the cone axis. In addition, the block $\widetilde{W}_Q^{(2)} \in \RR^{(d_\cX+2) \times (d_\cX+2)}$ is the only non-zero block in ${W}_Q^{(2)}$. The sparsity of ${W}_Q^{(2)}$ follows from the structure of the disentangled transformer, and the second layer focuses on the semantically dependent part of the input $H^{(1)}$. 

 In addition, since only the relative correlation $W_{Q}^{(\ell)}$ and $W_K^{(\ell)}$ matters when calculating attention scores, we only make $W_{Q}^{(\ell)}$ \emph{trainable} while keeping $W_{K}^{(\ell)}$ \emph{fixed} during training as follows.
\begin{align}
W_K^{(1)}=\begin{bmatrix}
          \widetilde{W}_K^{(1)}     & 0_{d_\cbb\times(d_\cX+2)}\\
          0_{(d_\cX+2) \times d_\cbb} & 0_{(d_\cX+2) \times (d_\cX+2)}
        \end{bmatrix},
        W_K^{(2)}=\begin{bmatrix}   
            0_{d\times d}& 0_{d\times d}\\
            I_{d} & 0_{d\times d} 
        \end{bmatrix}, 
        \label{Eq: reduce W_k}
    \end{align}
    \vspace{-2pt}
where we assume that $\widetilde{W}_K^{(1)}$ is fixed and satisfies $ (\widetilde{W}_K^{(1)})^\top\cbb=(1,0,\cdots,1,0)^\top$ $:=\tilde\cbb \in \RR^{d_\cbb}$. 

 Finally, we specify the value and output matrices as 
\begin{align}
        W_V^{(1)}=I_{d}, W_V^{(2)}=I_{2d},W_O=[0_{d\times d}\;\:0_{d\times d}\;\:I_{d}\;\:0_{d\times d}]^\top. \label{Eq: reduce W_v}
\end{align}  
\vspace{-2pt}
As a result, the trainable weights reduce to $\tilde\theta=\{\widetilde{W}_Q^{(1)},\widetilde{W}_Q^{(2)}\}$.   

\vspace{2pt}
\noindent\textbf{Output of the Reduced Model.} With reduced trainable $\tilde\theta$ defined above, and since we are only interested in the predicting $\xbq$ at position $N$, the predictor is given by $\yq(E;\tilde{\theta})={{\mathrm{TF}}}_{\tilde{\theta}}(E)_{(N,d)}$. The explicit form of this predictor is obtained by plugging in the fixed parameters in \eqref{Eq: reduce W_k} and \eqref{Eq: reduce W_v}, and is provided in \eqref{Eq: reduced model} in Appendix~\ref{app: Notation}.

\subsection{Training Settings}\label{subsec: training alg}
\paragraph{Loss Function.} We choose the standard squared loss used in a linear regression task as follows. 
\begin{align}
    L(\tilde\theta)=\frac{1}{2}\EE_{\wb,P}\left[(\yq(E(P) ;\tilde\theta)-\langle\wb,\xbq\rangle)^2\right], \label{Eq: loss function}
\end{align}
where the expectation is taken with respect to the joint distribution of the task $\wb$ and the prompt~$P$.
\vspace{4pt}
\noindent 
\begin{minipage}{0.45\linewidth}
\paragraph{Training Algorithm.} As shown in \Cref{alg:meta}, we train the reduced two-layer disentangled transformer with two-stage gradient descent on the squared loss. This two-stage procedure is motivated by the empirical observation that different layers tend to converge in distinct phases~\citep{bietti2023birth}. The reduced weights are initialized as $\widetilde{W}_Q^{(1)}=0_{d \times d}$, $\widetilde{W}_Q^{(2)}= I_{d}$. In the first stage, $\widetilde{W}_Q^{(2)}(0)= I_{d}$ is fixed, and gradient descent operates on $\widetilde{W}_Q^{(1)}(0)= I_{d}$ with learning rate $\eta_1$ until timestep $\tau_1+1$, and outputs $\widetilde{W}_Q^{(1)}(\tau_1+1)$. Then, in the second stage, the weight $\widetilde{W}_Q^{(1)}(\tau_1+1)$ is fixed, and gradient descent operates on $\widetilde{W}_Q^{(2)}(0)= I_{d}$ with learning rate $\eta_2$ until timestep $\tau_1+\tau_2+1$. 

\end{minipage}\hfill
\begin{minipage}{0.52\linewidth}
\begin{algorithm}[H]
   \caption{Two-stage Training Algorithm}
   \label{alg:meta}
\begin{algorithmic}[1]
   \STATE {\bfseries Input:} learning rate $\eta_1,\eta_2 >0 $; iteration $\tau_1,\tau_2$.
   \STATE Initialize $\widetilde{W}_Q^{(1)}(0)=0_{d \times d}$, $\widetilde{W}_Q^{(2)}(0)= I_{d}$. 
   \STATE {\color{blue}// Stage I}
   \FOR{$t=1$ {\bfseries to} $\tau_1+1$}
   \STATE $\widetilde{W}_Q^{(1)}{(t)} = \widetilde{W}_Q^{(1)}{(t-1)} - \eta_1 \nabla_{\widetilde{W}_Q^{(1)}} L(\tilde\theta^{(t-1)})$.
   \STATE $\tilde\theta{(t)}=(\widetilde{W}_Q^{(1)}{(t)},\widetilde{W}_Q^{(2)}(0))$.
   \ENDFOR
   \STATE {\color{blue}// Stage II}
   \FOR{$t=\tau_1+1$ {\bfseries to} $\tau_1+\tau_2+1$}
   \STATE $\widetilde{W}_Q^{(2)}{(t)} = \widetilde{W}_Q^{(2)}{(t-1)} - \eta_2 \nabla_{\widetilde{W}_Q^{(2)}} L(\tilde\theta^{(t-1)})$.
   \STATE $\tilde\theta{(t)}=(\widetilde{W}_Q^{(1)}{(\tau_1+1)},\widetilde{W}_Q^{(2)}(t))$.
   \ENDFOR
   \STATE {\bfseries Output:} $\hat{\theta}=\tilde\theta{(\tau_1+\tau_2+1)}$.
\end{algorithmic}
\end{algorithm}
\end{minipage}

\subsection{Main Theoretical Result}\label{subsec: theory result}
In this section, we characterize the convergence of \Cref{alg:meta}, and analyze its performance in Stages~I and~II. Our theoretical regression setting follows a general induction head circuit mechanism as introduced in \Cref{fig:example of induction head} in \Cref{sec:related work}. To output an accurate prediction, (i) after Stage I, the first layer focuses on the prefix positions and thereby forms an \ac{sdh}, enabling each $y_i$ to integrate information from its associated prefix $\xb_i$; (ii) After Stage~II, the learned second layer identifies the $y_i$ whose prefix matches the question token, i.e., $\xb_i=\xbq$, and outputs the corresponding value. 

To formalize these behaviors, we introduce notations for tracking the attention scores during training. Under the reduced model with parameters $\tilde{\theta}$, we denote the first layer attention score from the position $i$ to the position $j$ by $\attn^{(1)}_{i,j}(E;\tilde{\theta})$, and the second layer attention score from the question token to the position $i$ by $\attn^{(2)}_{i}(E;\widetilde{\theta})$. 
To further characterize Layer 2 behavior, for any $k\in[K]$, we define the aggregated attention score from the question token to the feature $k$ by
\begin{align}
    \Attn^{(2)}_{k}(E;\tilde{\theta}):=\sum\nolimits_{i\in\cV_k(P)}\attn_{2i}^{(2)}(E;\tilde{\theta}),\label{Attn_k}
\end{align}
\vspace{-2pt}
where $\cV_k:=\cV_k(P)\subset[\Nit]$ is the feature $k$ index set, such that $\xb_i=\vb_k$ for $i\in\cV_k(P)$. 
 For simplicity, given the parameters $\widetilde{\theta}{(t)}$ at step $t$, we abbreviate $\attn^{(1)}_{i,j}(E;\widetilde{\theta}{(t)})$, $\Attn^{(2)}_{k}(E;\widetilde{\theta}{(t)})$ and $\yq(E;\widetilde{\theta}(t))$ as $\attn_{i,j}^{(1)}{(t)}$, $\Attn^{(2)}_{k}{(t)}$, and $\yq(t)$, respectively.
 
These quantities allow us to monitor how the network forms an \ac{sdh} in Layer~1 and outputs an accurate prediction in Layer~2. Specifically, two attention patterns are desired: (i) In the first layer, an \ac{sdh} requires $\attn^{(1)}_{i,i-1}(\tau_1+1) \approx 1$. (ii) In the second layer, when $\xbq=\vb_k$, we similarly require $\Attn_k^{(2)}{(\tau_1+\tau_2+1)} \approx 1$. Under these conditions, the model outputs an accurate prediction:
  \begin{align}
      \yq{(\tau_1+\tau_2+1)}
      & \textstyle =\sum_{i\in[\Nit]} \attn_{2i}^{(2)}{(\tau_1+\tau_2+1)}\cdot {y}_{i}\approx \langle \wb,\vb_k\rangle. \label{Eq: output alt}
  \end{align}
 We formally state our main theorem below, which characterizes the convergence of loss and the dynamics of these attention scores. 
\begin{theorem}[Training Dynamics]\label{Thm: stage1 & 2}Suppose Assumption \ref{Assp 1: Frequency seq} holds. Consider $N \gg \mathrm{poly}(K) \gg \mathrm{polylog}(N)$, $N \gg d$, and $p_k=1/K$ for any $k \in [K]$. We set $\epsI=O({N^{-\frac{1}{2}}})$ and $\epsII=O({N^{-\frac{1}{4}}})\bigcap \Omega(\epsI)$ as the attention concentration errors of the first and second layer, respectively. Then the following holds for the training of \Cref{alg:meta},
    
    \noindent \textbrm{Stage I: \ac{sdh} Emergence in First Layer.} In Stage I with $\tau_1$ at most $$O\bigg(\frac{ K N\log (N)}{ \eta_1 C_1}+\frac{ K\log(K^{-1}\epsI^{-1})}{ \epsI \eta_1 C_1}\bigg),$$ the attention head of the first layer is trained to be $1-\epsI$ slash-dominant at $1$. Formally, for any token at position $i \in [N]$, it almost attends to the immediately previous token, i.e., $1- \epsI \leq \attn^{(1)}_{i,i-1}(\tau_1+1)$.  

    \noindent \textbrm{Stage II: Feature Matching in Second Layer.} In Stage II with $\tau_2$ at most $$O\bigg(\frac{ K^2 \log(K)}{\eta_2}+\frac{K\log(K\epsII^{-1/2})}{\eta_2\epsII}\bigg),$$ the attention head of second layer has a \textbf{concentrated} attention score on the same feature with the question token. Formally, if $\xbq=\vb_k$, then with high probability, the second layer almost attends to input tokens featuring $\vb_k$: $(1-\Attn_k^{(2)}(\tau_1+\tau_2+1))^2 \leq O(\epsII)$.
    
    \noindent \textbrm{End: The Loss Convergence.} At the end of Stage II, the loss converges. Formally, $L(\hat{\theta}) \leq \epsII$.
 \end{theorem}
 The above theorem shows that if the frequencies satisfy Assumptions \ref{Assp 1: Frequency seq}, the training of Algorithm \ref{alg:meta} converges to the minimum of loss via GD, with polynomial time efficiency, and proceeds in two distinct stages:
(i) The first layer captures positional dependencies and rapidly learns a slash-dominant attention pattern with $\Delta=1$, given a sequence of frequencies satisfying Assumption \ref{Assp 1: Frequency seq}, and
(ii) The second layer captures semantic and feature dependencies and achieves feature matching. Notably, the error in the second layer, $\epsII$, is larger than that in the first layer, $\epsI$, because feature matching relies on accurate prefix recognition by the \ac{sdh} formed in the first layer. Detailed proofs of \Cref{Thm: stage1 & 2} are provided in Appendices~\ref{app: proof-st1} to \ref{app: proof-converge}. 

\paragraph{Extension to \acp{sdh} with General offsets $\Delta$.} \Cref{Thm: stage1 & 2} focused on the specific case of small offset $\Delta = 1$. This is because, in our prompt, each $y_i$ is paired with its desired prefix $\xb_i$ at a distance of $1$. This setup requires the \ac{sdh} to attend to the immediately preceding token. However, the proof is not restricted to this particular offset. In the general case, if the data were constructed such that the distance between $y_i$ and $\xb_i$ were an arbitrary $\Delta$, the same reasoning would demonstrate the emergence of an \ac{sdh} at offset $\Delta$.
\paragraph{OOD Generalization.}   Given the learned model $\hat{\theta}$ and a test prompt for a linear task $\wb$ (possibly outside the support of $\cD_{\Omega}$), let the question token be $\xbq = \vb_k$. By Stage II attention concentration, Eq.~\eqref{Eq: output alt} shows that, with high probability, the question prediction $\yq$ is $\textstyle  \Attn^{(2)}_{k}(\tau_1+\tau_2+1)\langle \wb,\vb_k\rangle+\sum_{m\not=k}\Attn^{(2)}_{m}(\tau_1+\tau_2+1)\langle \wb,\vb_m\rangle\approx\langle \wb,\vb_k\rangle.$ This showcases the remarkable OOD generalization capability of SDHs.
\begin{abox} 
    \looseness -1 \textbf{\color{cc1}Takeaway 6:}  \Cref{Thm: stage1 & 2} shows that, under the structural conditions in Takeaways~\hyperlink{takeaway2}{2}-\hyperlink{takeaway4}{4} and the slash-dominance frequency condition Assumption \ref{Assp 1: Frequency seq}, a shallow transformer learns an \ac{sdh} that are \ac{ood} generalizable to tasks.
\end{abox}

We further remark that the convergence rate in \Cref{Thm: stage1 & 2} does not depend on the two parameters $\alpha$ and $d_\cbb$ in Assumption \ref{Assp 1: Frequency seq}, as $\alpha \ge 2$ and $N \gg d \gg d_\cbb$, making the terms involving $\alpha$ and $d_\cbb$ a higher order term, and asymptotically negligible in the final convergence bound.

We finally comment on why GD performs well despite the nonconvex loss in 
\eqref{Eq: loss function}. The key is that the optimization landscape is highly 
structured, containing \emph{many} sub–global-optimal parameter configurations. 
In particular, for both Layer~1 and Layer~2, any configuration that produces 
sufficiently concentrated attention scores, as characterized in the two stages of 
Theorem~\ref{Thm: stage1 & 2}, already lies in such a sub–global-optimal region. Owing to the monotonicity and saturation 
properties of the softmax function, once a logit at a given position becomes 
sufficiently larger than the others, the attention scores concentrate almost 
entirely on that position. As a result, GD does not require a carefully tuned descent direction 
or precise learning-rate scheduling. The gradients naturally guide the 
parameters toward these large regions of concentrated-attention solutions. Indeed, our convergence rate bounds show that in both Stage~I and Stage~II, sufficiently large learning 
rates $\eta_1$ and $\eta_2$ move the parameters into their 
sub–global-optimal regions in essentially \emph{one step}. In much deeper real-world \acp{llm}, interactions across 
many layers introduce considerably more complex dynamics, and the behavior of GD 
may not be nearly as simple or predictable. Nevertheless, the emergence of 
\acp{sdh} is still expected.


\section{Extensions to \acp{sdh} with Large $\Delta$}\label{sec:long_range_sdh}
In this section, we show that the definition of \acp{sdh} for small $\Delta$ introduced in \Cref{def:slash_dom} can be naturally extended to the large $\Delta$ (e.g., $\Delta > 500$) regime, by appropriately adjusting the threshold $\kappa$. This complements our findings for small $\Delta \in \{0,1, \ldots, 4\}$ in \Cref{sec: emp}. Intriguingly, most of the empirical findings hold similarly to small $\Delta$, which demonstrates the robustness of our insights.  

\begin{figure}[ht]
\centering

\subfigure[Average attention score matrix of $\rmL 0\rmH 7$.]{\includegraphics[width=0.325\textwidth]{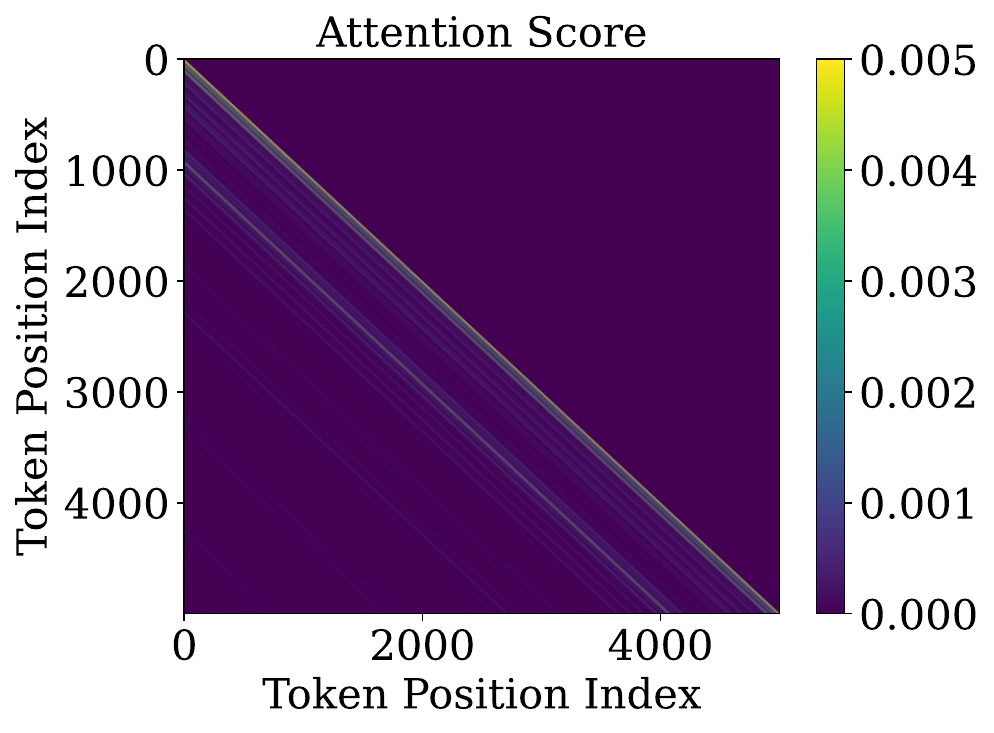}}
\subfigure[Average attention score matrix of $\rmL 1\rmH 11$.]{\includegraphics[width=0.325\textwidth]{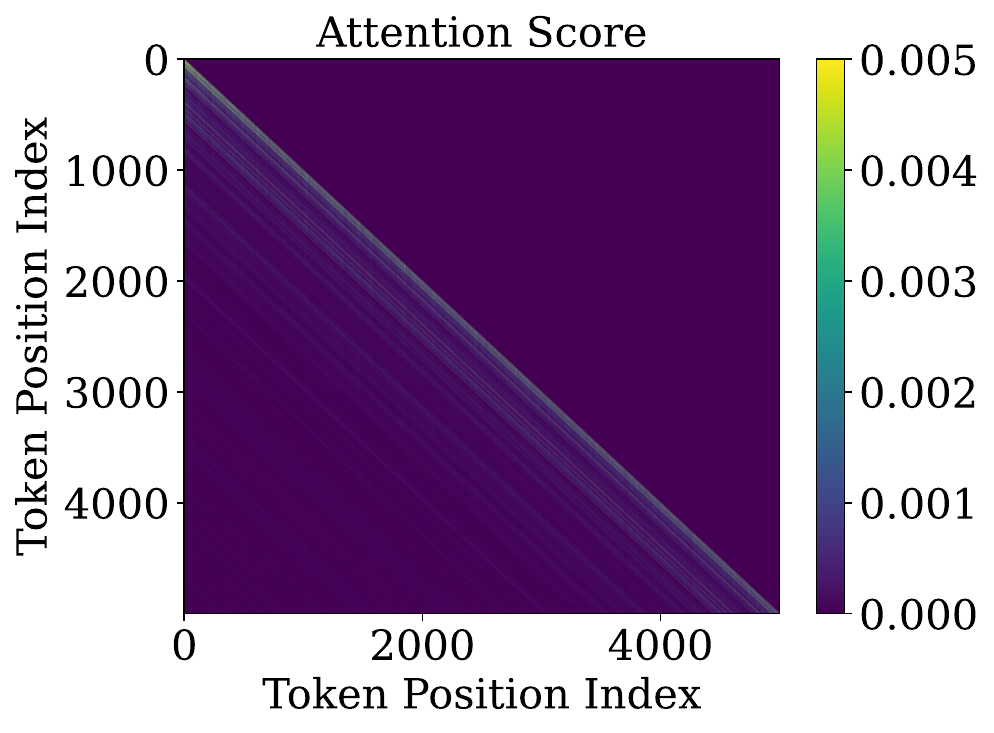}}
\subfigure[Average attention score matrix of $\rmL 2\rmH 6$.]{\includegraphics[width=0.325\textwidth]{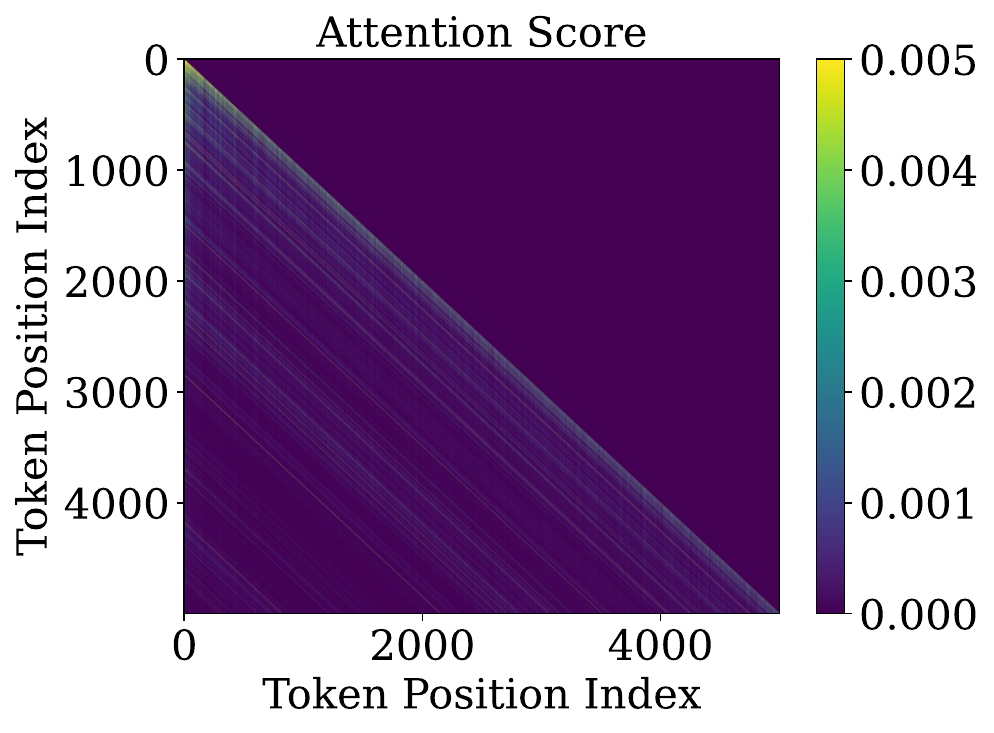}}

\caption{Average attention score matrices  of \acp{sdh} with \emph{large} $\Delta$ in Qwen2.5-7B-Instruct with prompts whose tokens are i.i.d.\ samples from the uniform distribution over the alphabet.}
\label{fig:qwen_large_ood}
\end{figure}

\paragraph{Definition of \acp{sdh} with Large $\Delta$.} Definition~\ref{def:slash_dom} identifies slash patterns in attention score matrices by examining the average slash scores. When $\kappa$ is chosen sufficiently large relative to the context length (e.g., $\kappa = 0.1$ for a context length of $6000$ in Section~\ref{sec: emp}), this definition reliably detects slash patterns with small $\Delta$. As illustrated in Figure.~\ref{fig:qwen_small_longb}(d)–(f), however, there also exist slash patterns with large $\Delta$. These long-range slash patterns have much smaller average slash scores, on the order of $10^{-3}$. However, directly setting $\kappa = 10^{-3}$ leads to many spurious
detections. For example, it would label offsets $\Delta < 10$ as
slash-dominant for almost all heads. This high false-positive rate arises from the locality of natural language, which induces irregularly high attention values to nearby tokens rather than consistent slash patterns. To mitigate this locality effect, we restrict the offset range to $500 \leq \Delta \leq 5000$ and set $\kappa = 10^{-3}$. The goal of this section is to verify that the properties of \acp{sdh} with small $\Delta$ generalize to a \emph{subset} of \acp{sdh} with large $\Delta$. Empirically, with the hyperparameters $500 \leq \Delta \leq 5000$ and $\kappa = 10^{-3}$, we identify \acp{sdh} (with large $\Delta$) only in Llama3-8B-Instruct and Qwen2.5-7B-Instruct since the context length of Gemma-7B is insufficient for this setting. Accordingly, we report experimental results only for these two models in this section. The detailed lists of selected heads are provided in Appendix~\ref{app:head_result}. A refined definition and a more comprehensive investigation of \acp{sdh} with large $\Delta$ are left to future work.


\begin{wrapfigure}{r}{0.31\textwidth}
 \vspace{-0.2cm}\includegraphics[width=\linewidth]{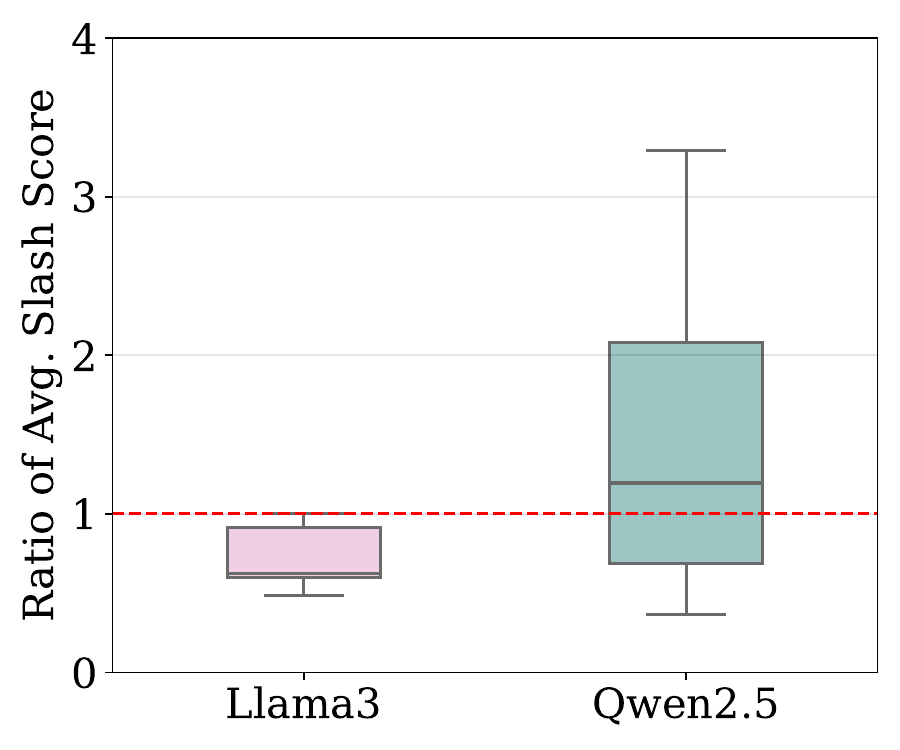}
   \caption{Ratio of average slash scores with \ac{ood} and in-distribution prompts (large $\Delta$).}
    \label{fig:ood_avg_large}
    \vspace{-0.2cm}
\end{wrapfigure}

\paragraph{OOD Generalization of \acp{sdh}.} Similar to \Cref{sec:ood}, we will show that \acp{sdh} with large $\Delta$ is OOD generalizable and intrinsic to the models. We test \ac{ood} generalization by evaluating the average slash scores under a special \ac{ood} prompt distribution $\cD'$, in which every token is sampled independently from a uniform alphabet. We plot the average attention score matrices of identified \acp{sdh} under OOD prompts in Figure~\ref{fig:qwen_large_ood}. These \acp{sdh} are from Qwen2.5-7B-Instruct and include $\rmL0\rmH7$, $\rmL1\rmH11$ and $\rmL2\rmH6$ for $\Delta=937,804,1934$, which are the same \acp{sdh} reported in Figure~\ref{fig:qwen_small_longb}. Figure~\ref{fig:qwen_large_ood} shows that {\color{cc2} the same slash pattern persists under OOD prompts}. Same results are also observed for Llama3-8B-Instruct (see Appendix \ref{app:head_result}).

We further quantify \ac{ood} generalization by {\color{cc2} computing the average slash scores for the identified \acp{sdh} with in-distribution and OOD prompts, and comparing the ratios.}  The resulting box plot of these ratios is shown in Figure~\ref{fig:ood_avg_large}. The average slash scores under \ac{ood} prompts are generally higher, or at least comparable to those obtained from in-distribution prompts. Detailed head-level results for the \ac{ood} case are provided in \Cref{table:head_eg_large_full} in Appendix~\ref{app:head_result}. 

Hence, our Takeaway~\hyperlink{takeaway1}{1} generalizes to large $\Delta$. The emergence of the slash-dominance pattern is {\color{cc2} not relevant to the semantic meaning of the prompts}, but is {\color{cc2}intrinsic to the model architecture}. 

\begin{figure}[ht]
\centering
\subfigure[Hidden state $H$ of $\rmL 0\rmH 7$ after PCA.]{\includegraphics[width=0.31\textwidth]{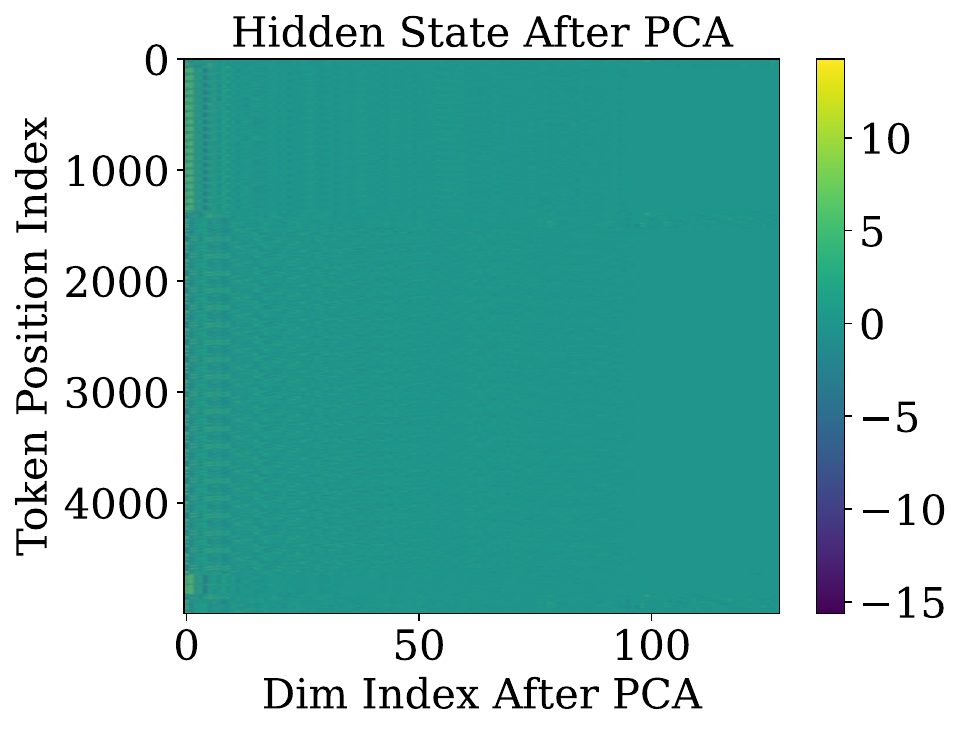}}
\hspace{0.02em}
\subfigure[Queries $Q$ of $\rmL 0\rmH 7$.]{\includegraphics[width=0.31\textwidth]{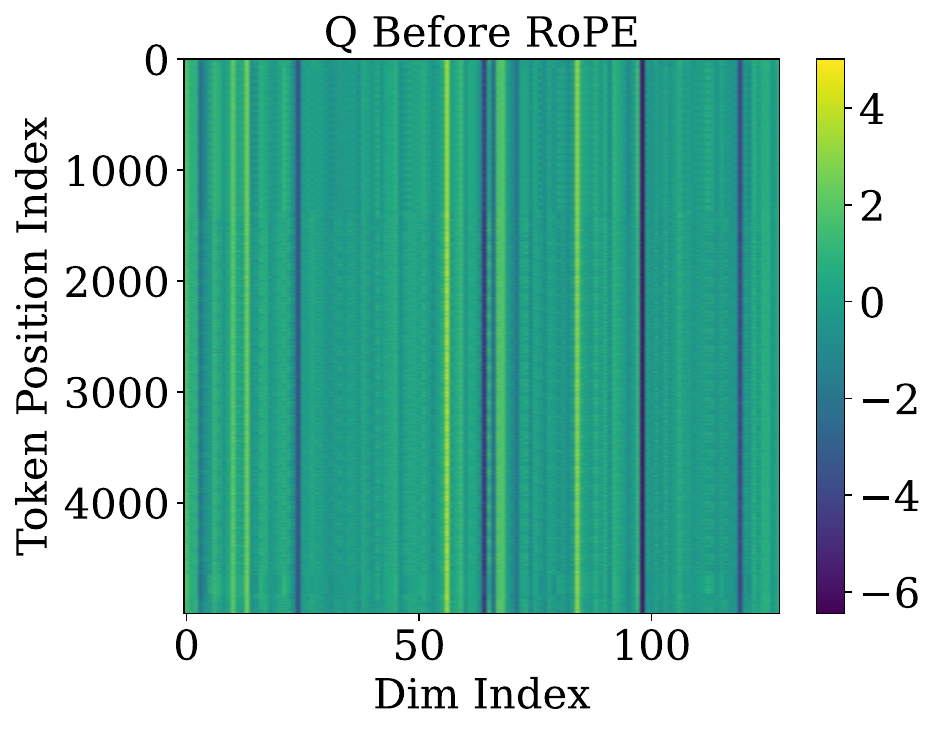}}
\hspace{0.02em}
\subfigure[Keys $K$ of $\rmL 0\rmH 7$.]{\includegraphics[width=0.31\textwidth]{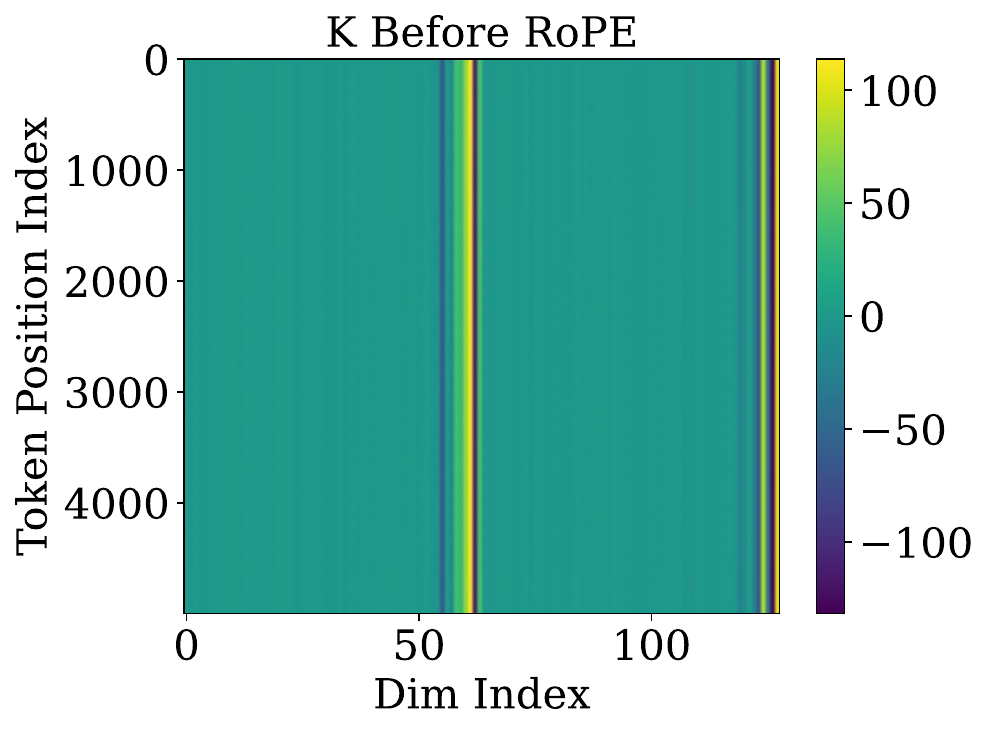}}
\caption{These figures show the queries and keys before the \ac{rope} implementation, as well as the hidden states for \acp{sdh} with large $\Delta$ in Qwen2.5-7B-Instruct. The queries and keys each have dimension $128$, and the hidden states are reduced to $128$ dimensions for demonstration.}
\label{fig:qwen_qk_large}
\end{figure}
\vspace{-3pt}
\paragraph{Approximate Low-Rankness of pre-\ac{pe} Queries and Keys.} Similar to \Cref{sec:lowrankQK}, we show the approximate low-rankness of pre-\ac{pe} queries and keys. We first visualize a \ac{sdh} of Qwen2.5-7B-Instruct $\rmL 0\rmH 7$ in Figure~\ref{fig:qwen_qk_large}. This \ac{sdh} demonstrates slash dominance across multiple large offsets, all with $\Delta \ge 900$. As shown in Figure~\ref{fig:qwen_qk_large} (b) and (c), the pre-\ac{pe} queries and keys are highly similar across tokens, implying that the query and key matrices are {\color{cc2} almost low-rank, particularly rank-one}. Additional visualizations for more heads and models are provided in \Cref{fig:qwen_qk} in Appendix~\ref{app:head_result}, which consistently exhibit the same structural pattern. 
\vspace{-3pt}
\begin{figure}[h]
    \centering
    \subfigure[Average $r_1(\uparrow)$ of all heads and \ac{sdh}s with large $\Delta$.]{\includegraphics[width=0.48\linewidth]{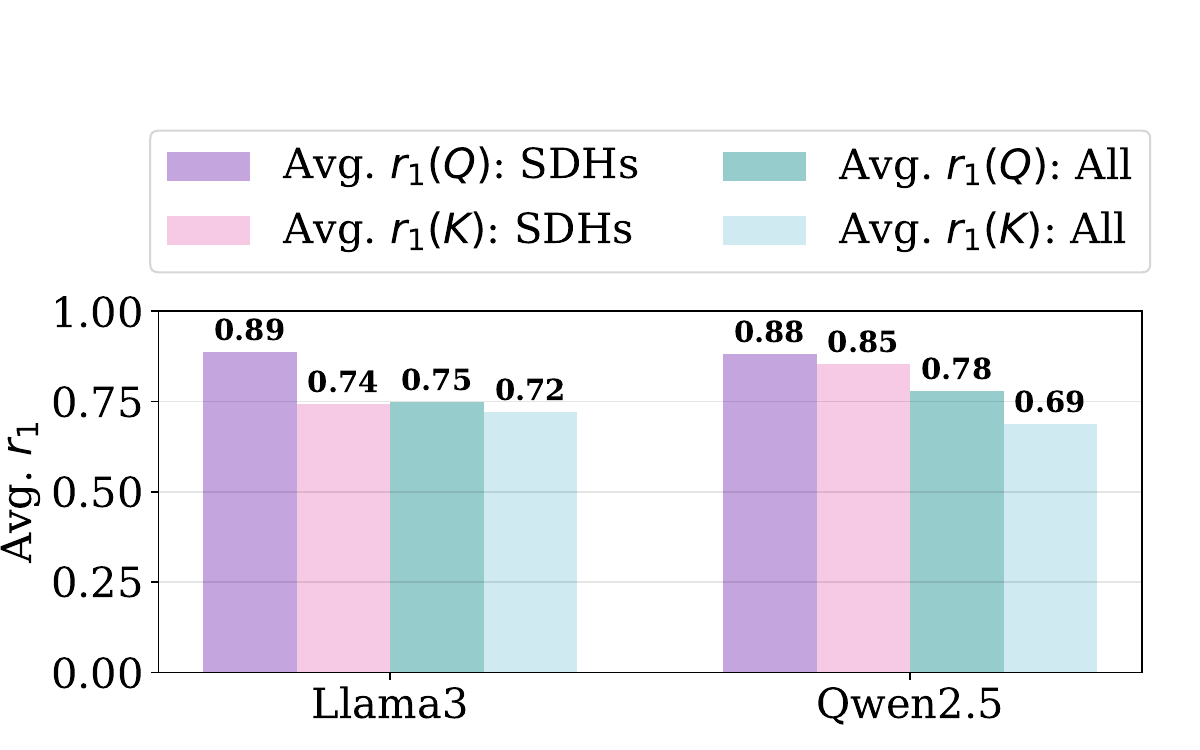}
    \label{fig:r1_com_allhead_large}}
    \hspace{0.01em}
    \subfigure[Average $R_{0.95}(\downarrow)$ of all heads and \ac{sdh}s with large $\Delta$.]{\includegraphics[width=0.48\linewidth]{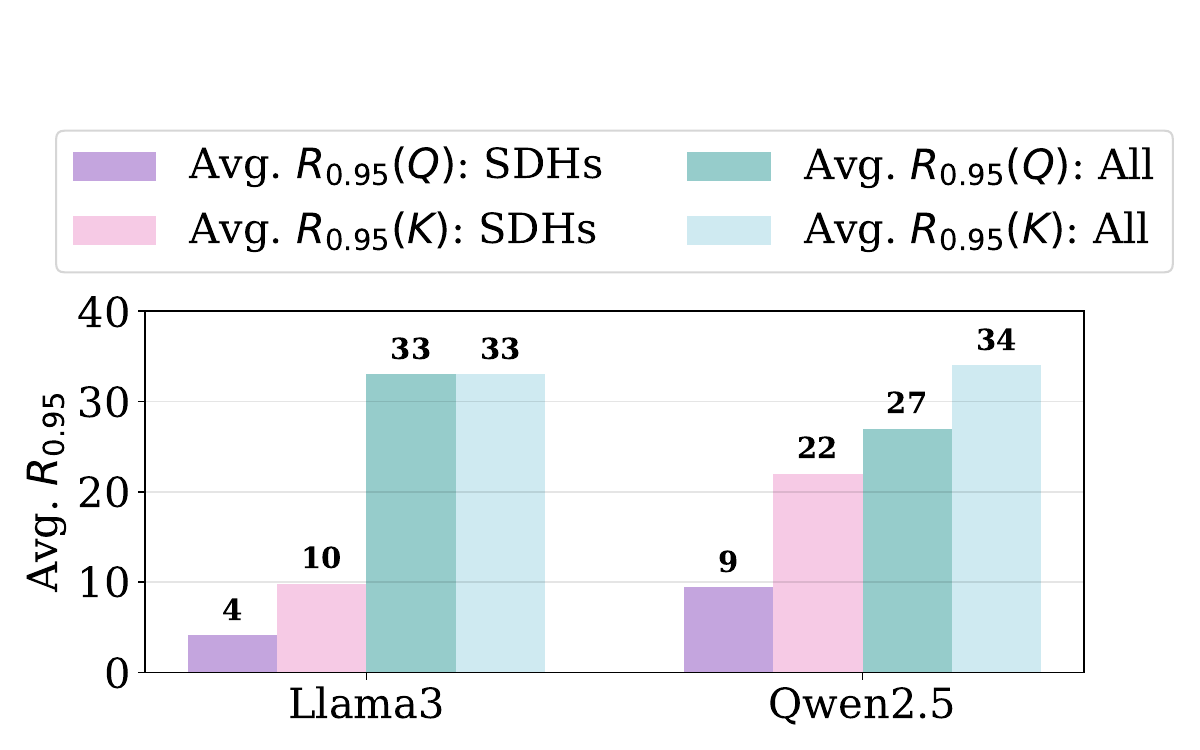}
    \label{fig:R_com_allhead_large}}
    \caption{Comparison of average $r_1$, $R_{0.95}$ of \ac{sdh}s with those of all heads. For large $\Delta$, at least one of the $Q$- or $K$ exhibits \emph{substantially lower average ranks} for \acp{sdh} compared to all heads.
}
    \label{fig:avg_low_rank_large}
\end{figure}

We then quantify the low-rankness of the pre-\ac{pe} queries and keys in detail by $r_1$ and $Q_{0.95}$ according to \eqref{eq:lowrank_metrics} in Figure~\ref{fig:avg_low_rank_large}. In particular, Figure~\ref{fig:r1_com_allhead_large} shows that, for each model, {\color{cc2}at least one of  $r_1(Q)$ and $r_1(K)$ strictly exceeds $0.88$ on average over \acp{sdh} only}, and these values are strictly lower on the other heads. Figure~\ref{fig:R_com_allhead_large} shows that the effective ranks of $Q$ and $K$ are low only on the \acp{sdh}, while these matrices on the other heads have much higher ranks. 

Hence, our Takeaway~\hyperlink{takeaway2}{2} generalizes to the large $\Delta$ regime. The pre-PE queries $Q$ and keys $K$ all have low effective ranks. Moreover, at least one of $Q$ and $K$ is close to being rank-one.

\begin{figure}[h]
    \centering
    \subfigure[$r_1(\uparrow) $ related to queries.]{\includegraphics[width=0.48\linewidth]{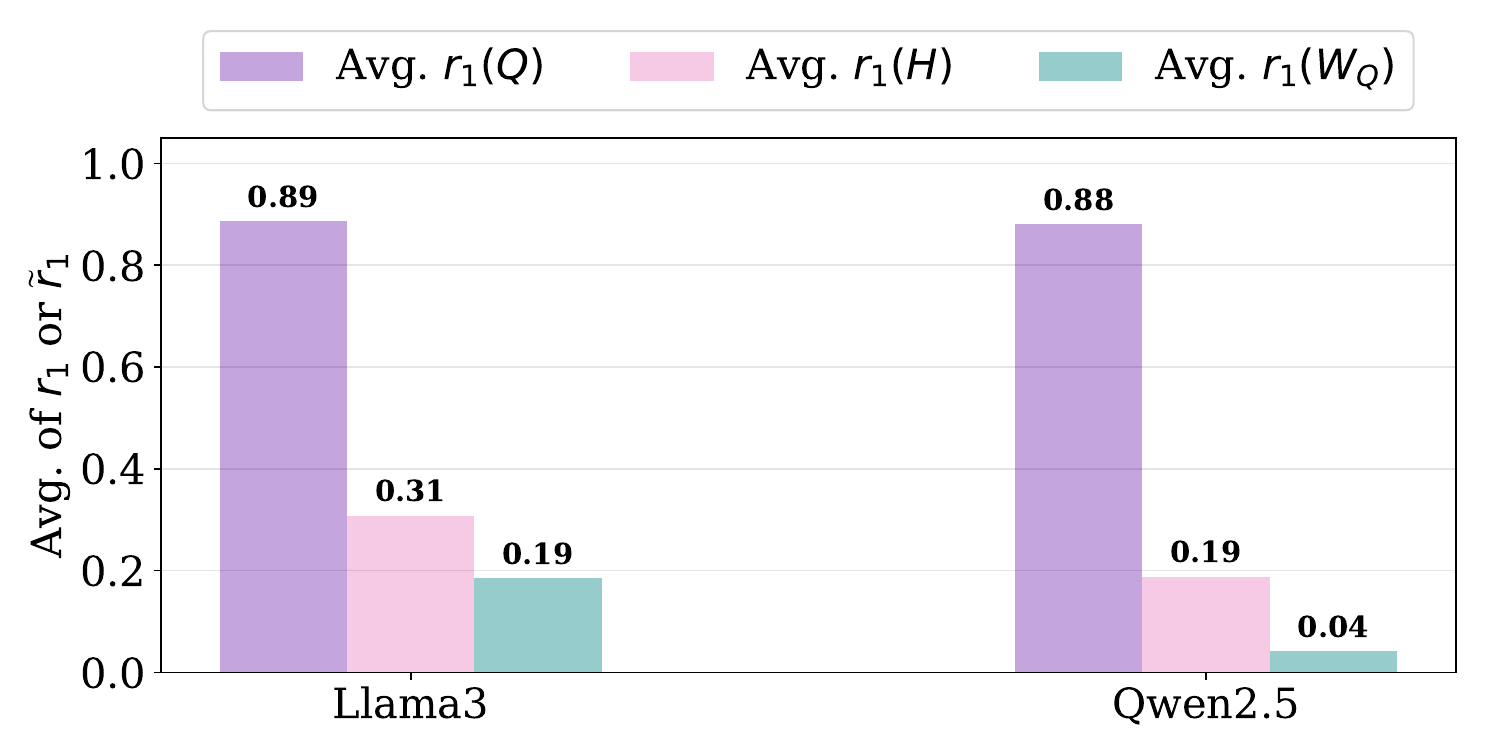}\label{fig:large-rWQ-low}}
\hspace{0.01em}
\subfigure[$r_1(\uparrow) $ related to keys.]{\includegraphics[width=0.48\linewidth]{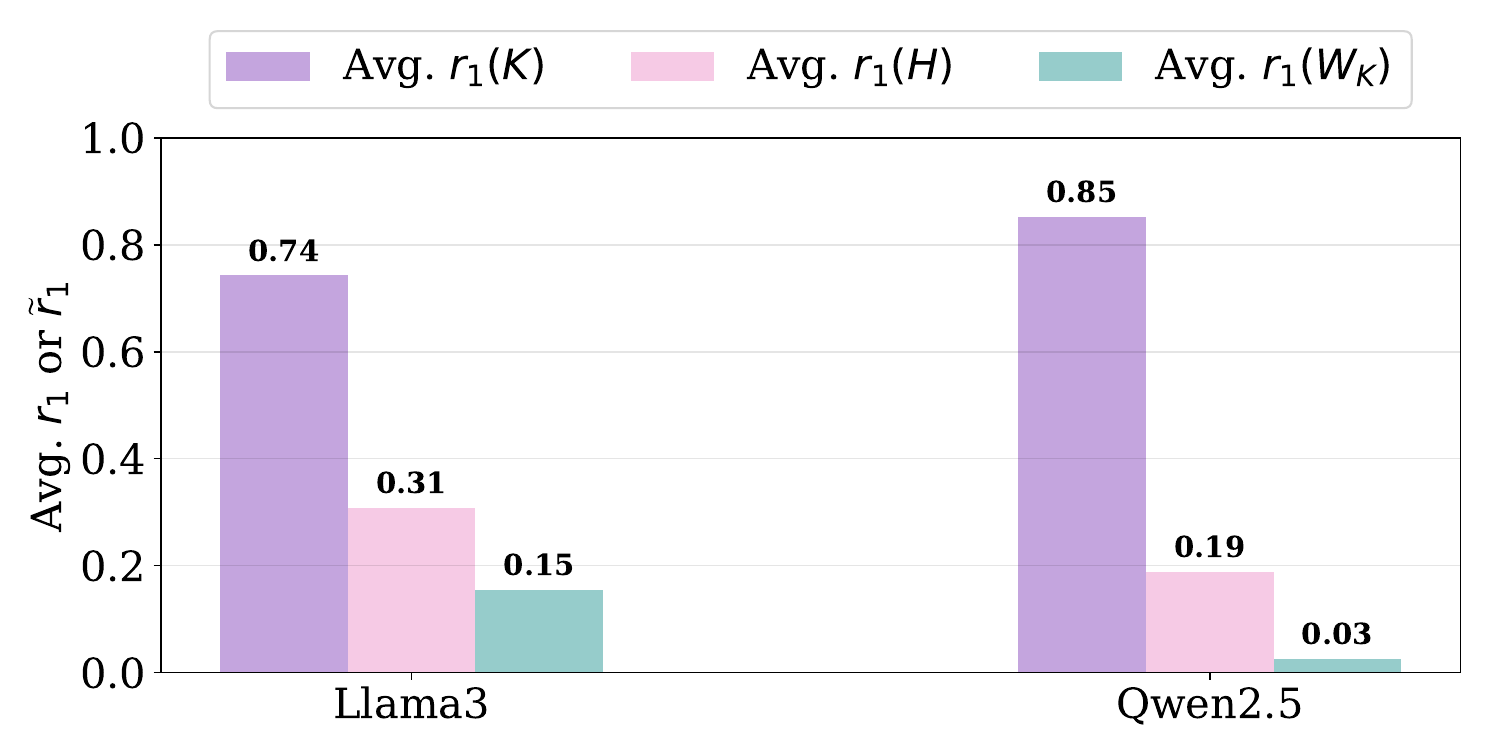}\label{fig:large-rWK-low}}

    \subfigure[$R_{0.95}(\downarrow)$ related to queries.]{\includegraphics[width=0.45\linewidth]{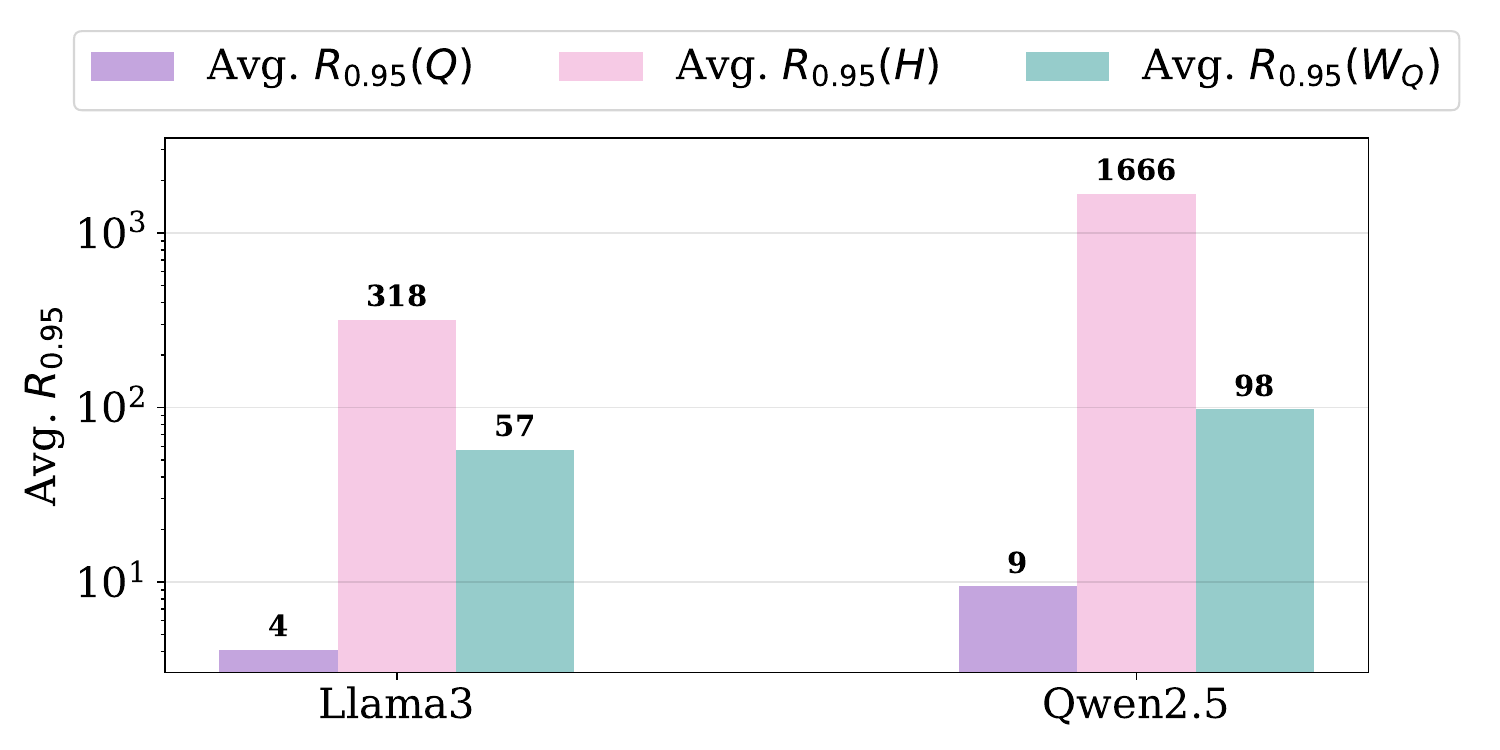}\label{fig:large-RWQ-low}}
\hspace{0.01em}
\subfigure[$R_{0.95}(\downarrow)$ related to keys.]{\includegraphics[width=0.45\linewidth]{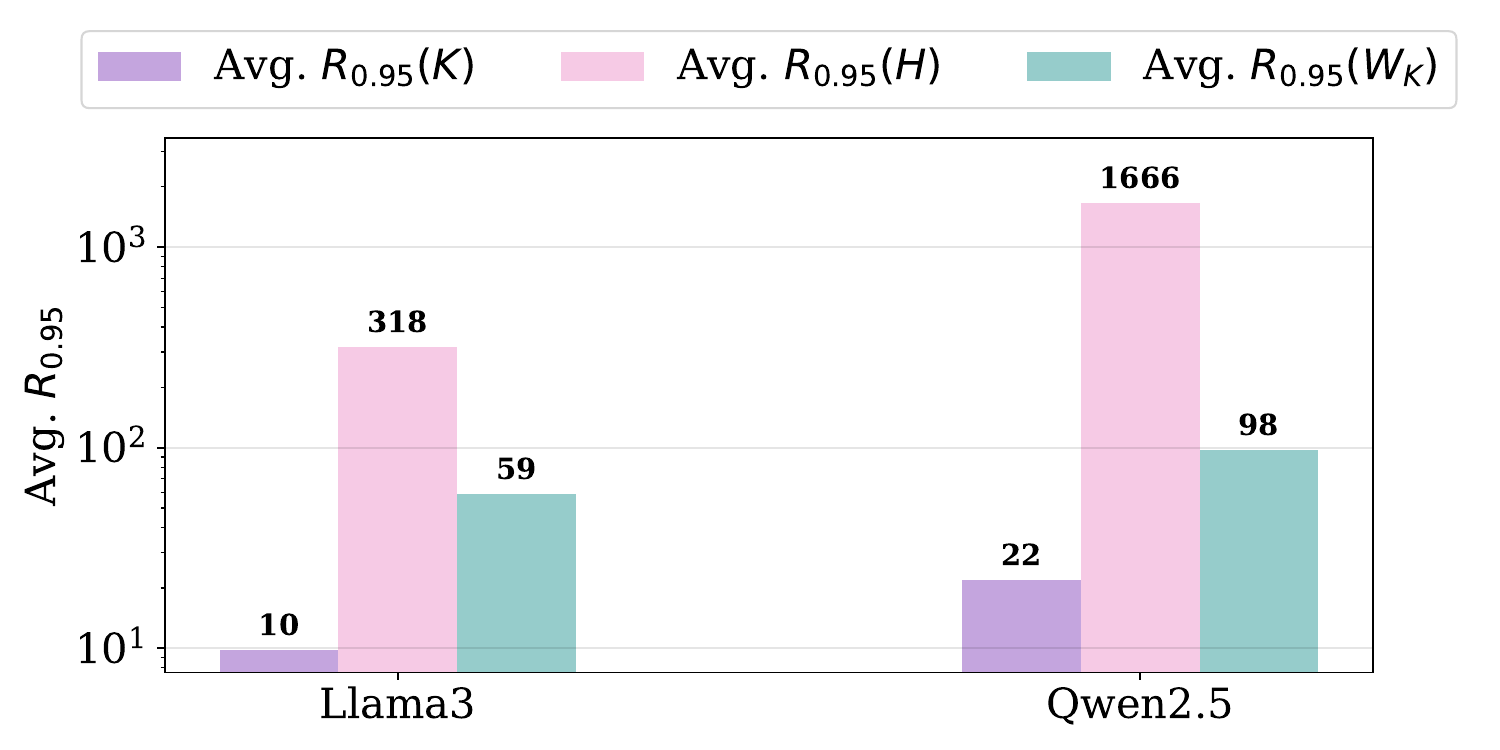}\label{fig:large-RWK-low}}
\caption{Comparison of average $r_1$ and $R_{0.95}$ for query/key matrices relative to the hidden states and their corresponding weight matrices, for \ac{sdh}s with \emph{large} $\Delta$. Both hidden states $H$ and the weight matrices $W_Q$, $W_K$ exhibit much higher rank than the resulting pre-PE queries and keys, indicating that the low-rankness arises from the interaction between $H$ and the weights $W_Q$, $W_K$.}
\label{fig:rank_examine_large}
\vspace{-0.5em}
\end{figure}

\paragraph{How to Get Approximately Rank-One Queries and Keys?} Similar to \Cref{sec:rank_one_qk}, we next answer how \acp{sdh} achieve the approximately rank-one queries and keys. 
As in Figures~\ref{fig:rank_examine_large}, we compute the average values of $r_1$ and $R_{0.95}$ for $H$, $W_Q$ (resp. $W_K$), and $Q$ (resp. $K$) and average these values over the identified \acp{sdh}. Figures~\ref{fig:rank_examine_large} clearly shows that the effective ranks of $H$, $W_Q$, and $W_K$ are much higher than those of $Q$ and $K$. Thus, {\color{cc2}low-rankness of $Q$ and $K$ must arise from the interaction between the hidden states $H$ and weight matrices $W_Q$ and $W_K$}. 

\begin{figure}[ht]
    \centering
    \subfigure[$\tilde{r}_1(\uparrow)$ and ${r}_1(\uparrow)$ related to queries.]{\includegraphics[width=0.48\textwidth]{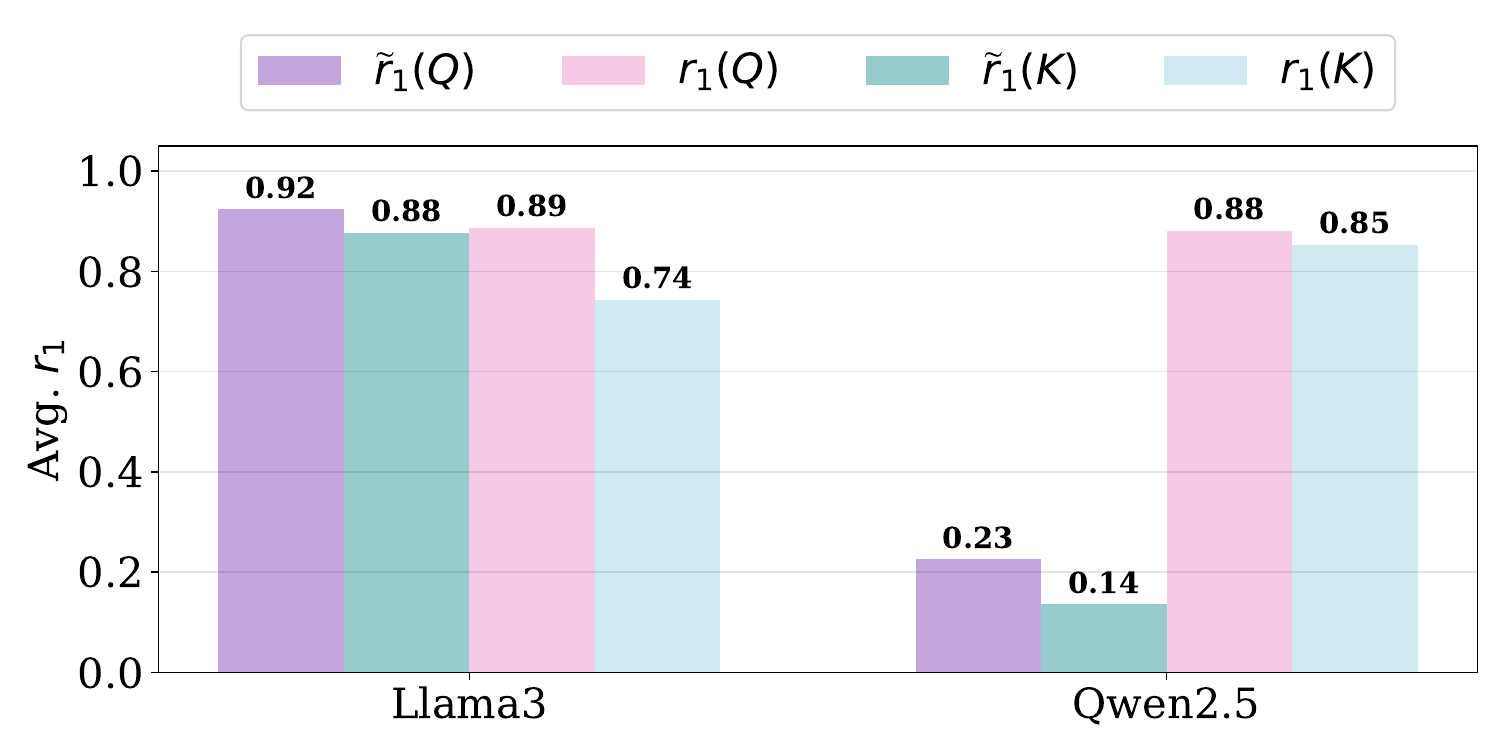}}
    \subfigure[$R_{0.95}(\downarrow)$ and $\widetilde{R}_{0.95}(\downarrow)$ related to queries.]{\includegraphics[width=0.48\textwidth]{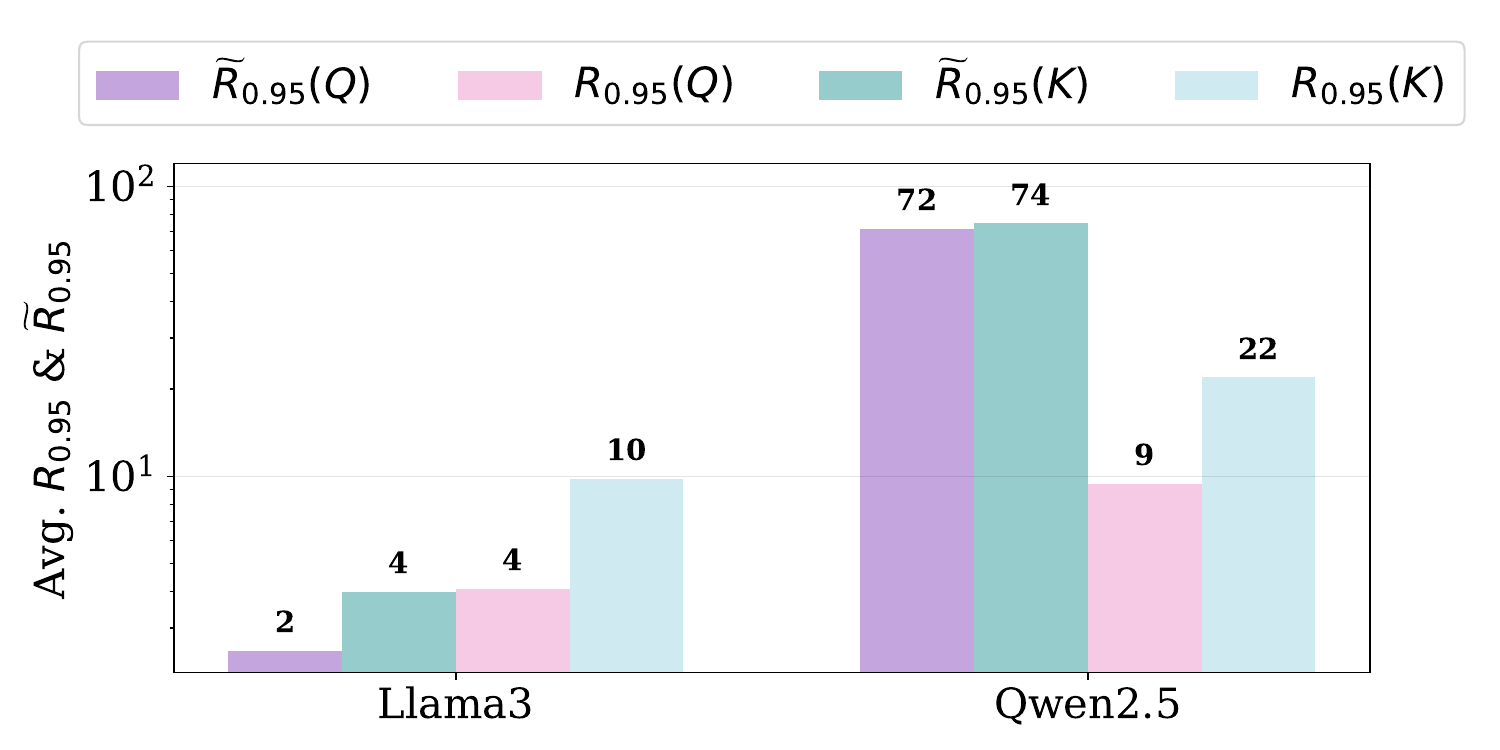}}
    \caption{Comparison of average $r_1$, $R_{0.95}$, $\tilde{r}_1$, and $\tilde{R}_{0.95}$ for query/key matrices relative to the 0-th layer hidden states (token embeddings) for \ac{sdh}s with \emph{large} $\Delta$ in Gemma-7B, Llama3-8B-Instruct, and Qwen2.5-7B-Instruct. }
    \label{fig:til_R_r_large}
\end{figure}

We further plot the average values of $\tilde r_1$ and $\tilde R_{0.95}$, and compare them with $r_{1}$ and $R_{0.95}$ respectively in Figures~\ref{fig:til_R_r_large}. These average values are computed over all the tokens and the \acp{sdh} in Layer 0. 
Figures~\ref{fig:til_R_r_large} shows that the behaviors of $\tilde{r}_1(Q)$ and $\tilde{R}_{0.95}(Q)$ closely match those of $r_1(Q)$ and $R_{0.95}(Q)$ in Llama3-8B-Instruct, but differ in Qwen2.5-7B-Instruct.
This confirms that in Gemma-7B and Llama3-8B-Instruct,
{\color{cc2} each token embedding concentrates most of its energy approximately in a one-dimensional principal subspace of $W_Q$ and $W_K$.} That is, each token embedding aligns well with low-rank (and almost rank-one) subspaces of $W_K$ and $W_Q$. Hence, our Takeaway~\hyperlink{takeaway3}{3} generalizes to large $\Delta$.

\begin{figure}[ht]
\centering
\subfigure[$\InP(5000,j,l)$ of $\rmL 0\rmH 7$.]{\includegraphics[width=0.32\textwidth]{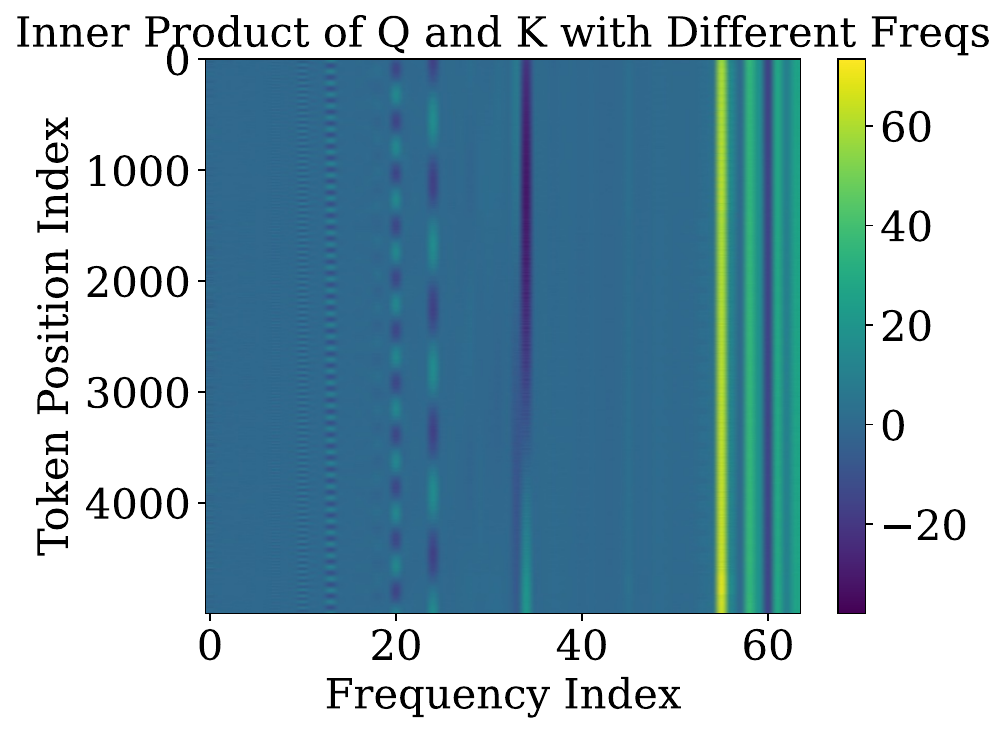}}
\hspace{0.01em}
\subfigure[$\InP(5000,j,l)$ of $\rmL 1\rmH 11$.]{\includegraphics[width=0.32\textwidth]{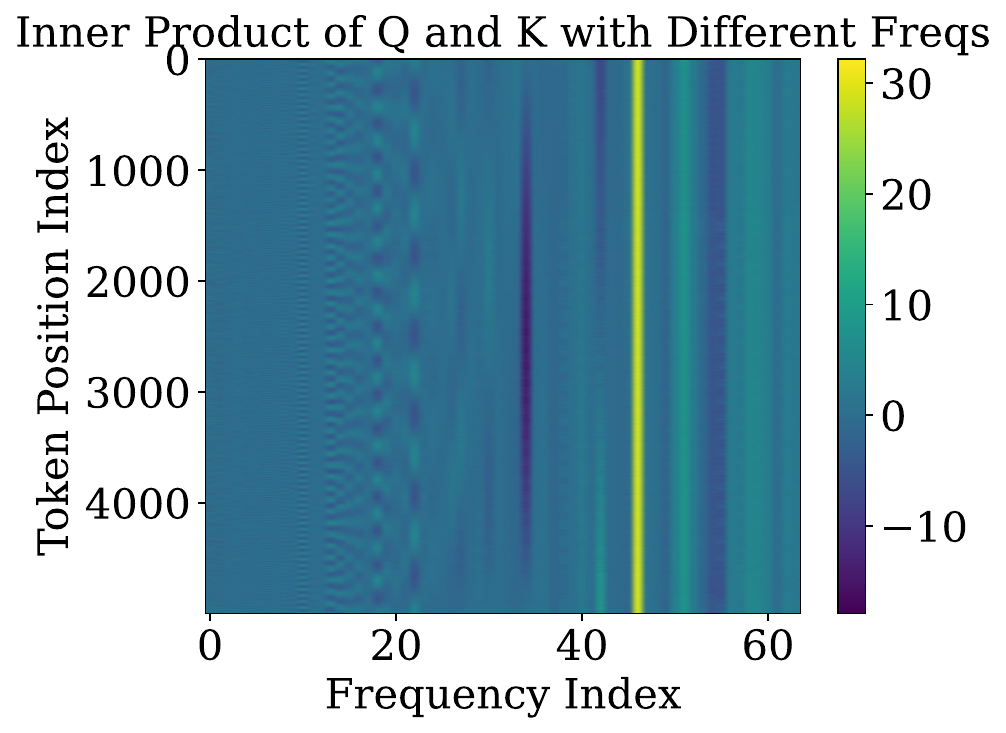}}
\hspace{0.01em}
\subfigure[$\InP(5000,j,l)$ of $\rmL 2\rmH 6$.]{\includegraphics[width=0.32\textwidth]{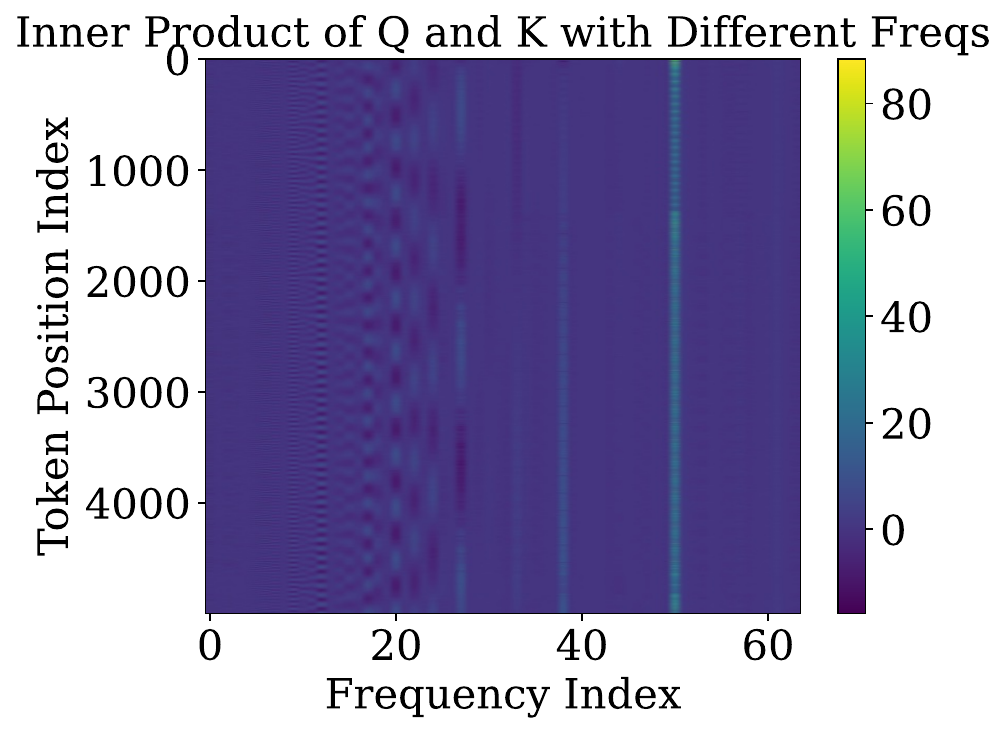}}

\caption{Heatmaps of $\InP(i,j,\ell)$ on various \acp{sdh} of Qwen2.5-7B-Instruct. where we set $i=5000$ and let $j$ vary. Qwen2.5-7B-Instruct has $64$ frequencies ($1 \leq \ell \leq 64$). The variation within each column shows the influence of each frequency. Qualitatively, high- and medium-frequency components $\ell\in [1, 42]$ exhibit greater variation.}
\label{fig:qwen_prod_large}
\end{figure}

\paragraph{Collaboration of Frequencies in RoPE Determines Slash Pattern.} Similar to \Cref{sec: rope&sladom}, given that pre-PE queries and keys are approximately rank-one, we proceed to examine how the different \ac{rope} frequencies collectively contributed to the slash pattern according to \eqref{Eq: logit decom}. We compute the relative variation of the norms of the query and
key vectors and plot their averages over tokens and SDHs in Figure~\ref{fig:QKNorm_RV_large}. As shown in the figure, the relative variation is at most $0.093$, indicating that the norms of $\qb_i$ (and similarly $\kb_j$) are nearly constant. As a result, $\qb_i$ is approximately the same for all $i$, and similarly for $\kb_j$. This implies that the amplitudes $A_\ell$ and initial phases $\varphi_\ell$ in \eqref{eq:logit_decomposition} are approximately identical across token pairs, which mirrors the behavior observed in the small-$\Delta$ regime and explains how RoPE gives rise to SDHs in the large $\Delta$ setting.

\begin{figure}[ht]
       \begin{minipage}[t]{0.48\textwidth}
        \centering
        \includegraphics[width=\linewidth]{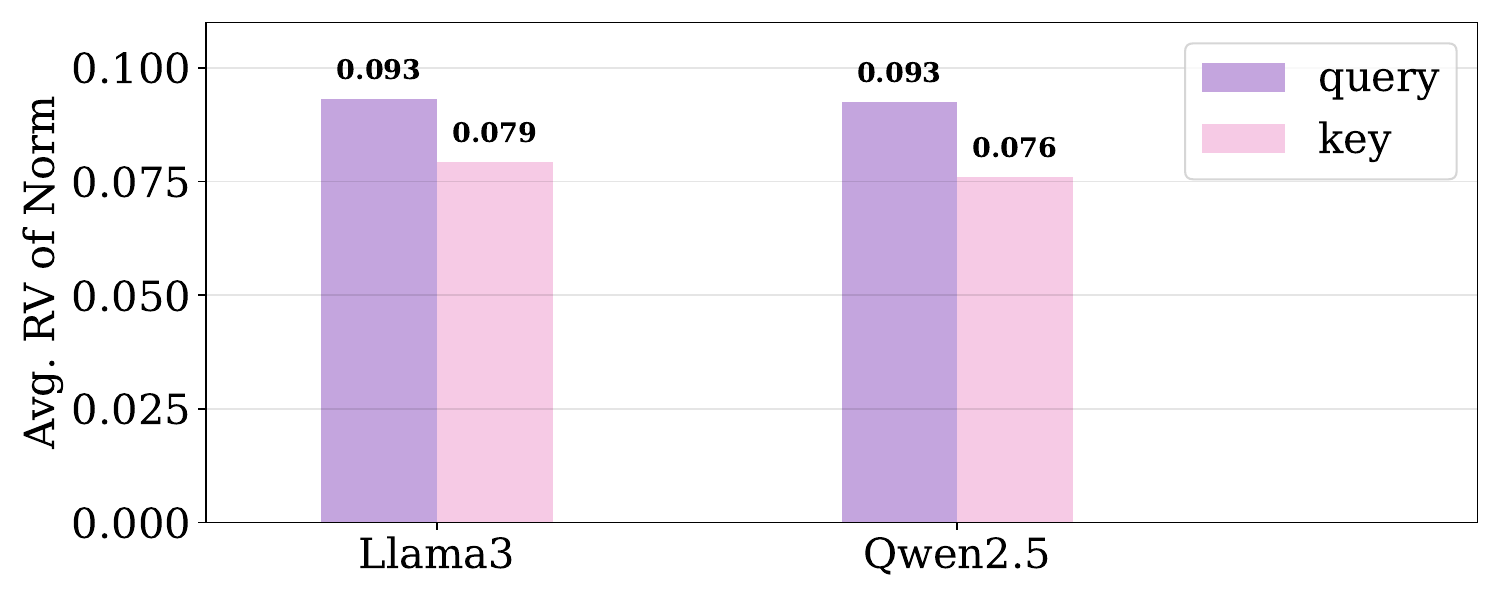}
        \caption{Average of the relative variation
        of the norm of the query vector $\|\qb_i\|$ and key $\|\kb_i\|$ for large $\Delta$ in Llama3-8B-Instruct and Qwen2.5-7B-Instruct.}
        \label{fig:QKNorm_RV_large}
    \end{minipage}
    \hfill
    \begin{minipage}[t]{0.48\textwidth}
        \centering
        \includegraphics[width=\linewidth]{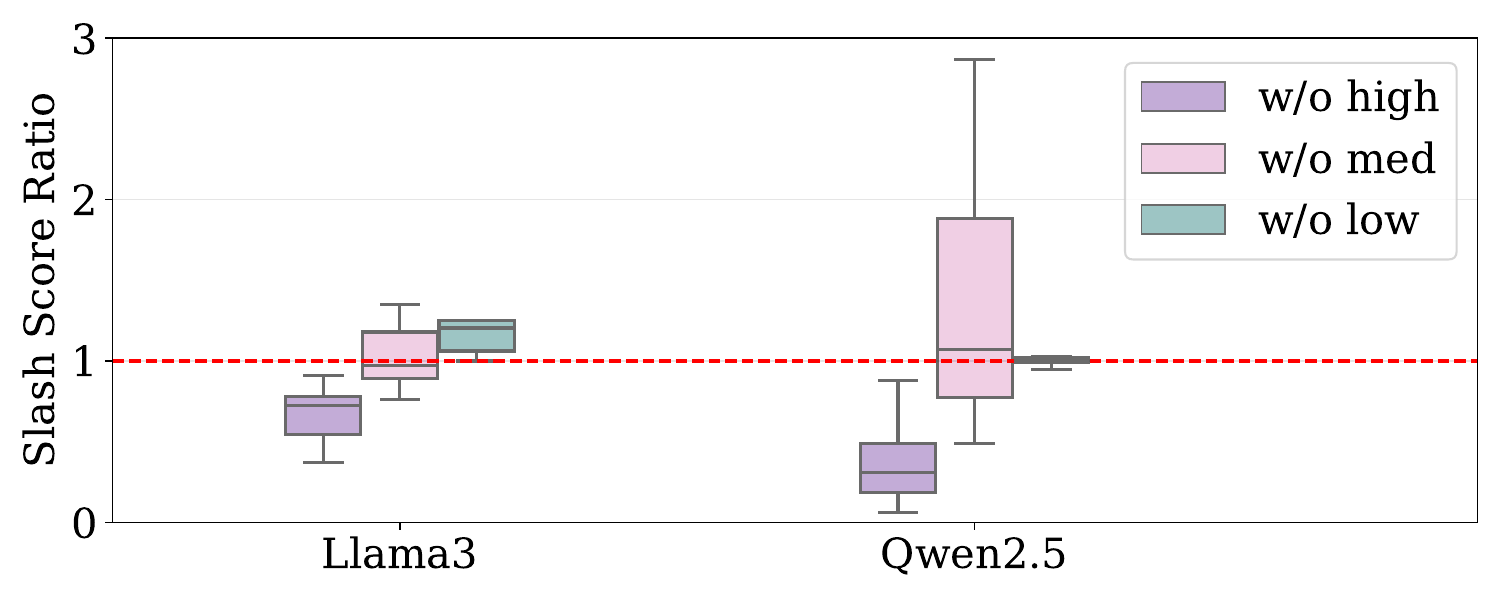}
    
        \caption{The figure quantifies the effect of low-, medium-, and high-frequency components on \acp{sdh} by reporting, for each band, the ratio of the average slash score after removing that band to the original average score.}
        \label{fig:removeFrequency_ratio_large}
    \end{minipage}
\end{figure}

To further characterize the collaboration of frequencies, we then visualize $\InP(i,j,\ell)$ with $i$ fixed to $i=5000$ and $j$ varying within $[1, 5000]$ in Figure~\ref{fig:qwen_prod_large}, based on Qwen2.5-7B-Instruct.
We observe that high- and medium-frequency components ($\ell\in [1, 42]$) exhibit greater variation than low-frequency components ($\ell\in [43, 64]$). For a more quantitative characterization, we evaluate slash-dominance after selectively removing certain frequencies when computing the attention score using the same criteria as the small $\Delta$ regime. We compute the ratios of the new score to the original score for each identified \acp{sdh} in Llama3-8B-Instruct and Qwen2.5-7B-Instruct, and show the box plot in Figure~\ref{fig:removeFrequency_ratio_large}. We observe that, when removing high frequencies, the average slash score decreases in both models. Removing the medium frequencies leads to a mild change, or even no decrease, while the low frequencies yield the least drop. 

Hence, our Takeaway~\hyperlink{takeaway4}{4} generalizes to large $\Delta$. The high- and medium-frequency components in \ac{rope} are more important than the low-frequency components for \acp{sdh}.

\section{Discussion}\label{sec:discussion}
In the preceding sections, we investigated the mechanistic interpretability of \acp{sdh}, showing that this phenomenon arises as an intrinsic algorithmic effect contributed by \ac{rope}. Our empirical and theoretical findings in \Cref{sec: emp,sec: theory,sec:long_range_sdh}, provide new insights into how \ac{rope} interacts with attention mechanisms. In this discussion part, we first outline two potential directions where our findings could be applied to improve the efficiency and efficacy of the current model architectures. Then we discuss the implications of our results for the models with other kinds of \ac{pe}.

\paragraph{Parameter Efficiency of Queries and Keys.} Motivated by the observed \emph{approximate low-rankness}, we suggest that the parameter matrices $W_Q$ and $W_K$ of some heads can be constrained to low rank. This could reduce the number of parameters and save computational resources during both training and inference, without significant loss in performance. We take an initial step to verify this. In \Cref{tab:parameter_compression} in Appendix~\ref{app:exp_parameter_compression}, we compress the query and key matrices of pretrained Qwen2.5-7B-Instruct and Llama3-8B-Instruct, which yields notable parameter reduction while preserving accuracy. A promising direction for future work is to explicitly enforce low-rank constraints during training to see whether we can achieve further efficiency gains.

\paragraph{Effective Length Generalization.}
Our result reveals that high- and medium-frequency components of \ac{rope} predominantly contribute to slash-dominant behavior, whereas low frequencies play a negligible role. Based on this observation, one potential application is to edit or reweight the low-frequency components of \ac{rope} to enhance length generalization in \acp{llm}. We leave the systematic exploration of this idea to future experimental work.

\paragraph{Implications for Other Kinds of \ac{pe}s.} 
Our analysis focuses on models that use \ac{rope} (e.g., the Gemma, Llama, and Qwen families). Other positional-encoding schemes, \textsc{Alibi}~\citep{press2021train}, \textsc{NoPE}~\citep{kazemnejad2023impact}, and sinusoidal positional embeddings~\citep{vaswani2017attention}, can exhibit behavior similar to or distinct from \ac{rope}. For example, \textsc{Alibi} adds a relative bias directly to query–key inner products; it is length-extrapolative by design and may therefore support \ac{ood} generalization. In contrast, \textsc{NoPE} (no explicit positional signal) and sinusoidal embeddings (often injected only at the input layer) may render \acp{sdh} more context dependent. A detailed study of \acp{sdh} under these alternatives is left to future work.

\section{Conclusion}
In this paper, we investigated the mechanism underlying \acp{sdh}. Our results demonstrate that, owing to the geometric structure of token embeddings approximately on a cone, slash-dominance is an intrinsic algorithmic effect primarily driven by the medium- to high-frequency components of \ac{rope}. We further establish a slash-dominance frequency condition that is satisfied by real open-source \acp{llm}, and prove these conditions on token embeddings and frequencies are sufficient for the emergence of \acp{sdh} by analyzing the training dynamics of a shallow Transformer. Our theoretical and empirical findings provide new insights into how \ac{rope} interacts with attention mechanisms. Our work focused on the analysis of the pretrained \acp{llm} with \ac{rope}, and we leave the investigation of other kinds of \ac{pe} as future work.

\bibliography{ref}
\bibliographystyle{ims}

\newpage
\tableofcontents
\appendix
\section{Additional Related Work}\label{app: add_rel}

\noindent \textbf{Other Special Attention Patterns and Heads.} 
In addition to the slash patterns and \acp{sdh}, a line of work proposed different kinds of heads with special patterns to explain the information aggregation in transformers. For example, successor heads~\citep{gould2023successor}, content-gathering head~\citep{merullo2023circuit}, and retrieval head~\citep{wu2024retrieval}. Different from them, in our work, we also show that \acp{sdh} can be generalized to out-of-distribution cases. Moreover, by understanding certain attention score patterns of special heads~\citep{jiang2024minference,xu2025xattention,zhao2025paroattention,DBLP:conf/iclr/LaiLLMZ25,li2025mtraining}, one can improve the efficiency of \acp{llm}. Specifically, \citet{xu2025xattention} observed that the sum of the antidiagonal values in the attention score matrix serves as a strong proxy for block importance, based on which they proposed a plug-and-play framework that accelerates long-context inference in transformer models using sparse attention. \citet{li2025mtraining} proposes a distributed dynamic sparse attention framework for ultra-long context training. Their system leverages the empirical observation that many attention heads attend primarily along fixed relative offsets induced by RoPE, enabling efficient training at a million-token scale. 

\noindent\textbf{Training Dynamics of Transformer.}
\citet{huang2024context} and \citet{zhang2024trained} are the first two works to establish the in-context convergence results for the softmax attention and linear attention trained with gradient descent, respectively. Then, a series of works relaxed their assumptions and studied more general network structures. \citet{chen2024training} first explored the training dynamics of the multi-task in-context learning using multi-head transformers. \citet{yang2024context} explored in-context learning in multi-head transformers applied to nonlinear task functions. More recently, \citet{nichani2024transformers, chen2024unveiling, edelman2024evolution,cheng2025transformers} have investigated transformer learning dynamics in settings underpinned by causal structure, deepening our understanding of how such models acquire and represent inductive capabilities. \citet{wen2024sparse,kim2024cot,Liang_neurips2025_cot,Liang_icml2025_parity,huang2025transformers} investigated transformer dynamics in the chain-of-thought reasoning setting.

Another line of work studies the learning dynamics of transformers on vision tasks. For example, \citet{jelassi2022vision} analyzed the inductive biases of Vision Transformers (ViTs) with a focus on specialized positional attention. Building on this foundation, \citet{li2023theoretical} analyzed the training process of shallow ViTs in supervised classification settings and extended their study to in-context learning, under strict assumptions on network initialization. 

\section{Experimental Details}\label{app:exp_details}

In our experiments, we identify \acp{sdh} in Qwen2.5-7B-Instruct~\citep{yang2025qwen25}, Llama3-8B-Instruct~\citep{grattafiori2024llama}, and Gemma-7B~\citep{team2024gemma}. We evaluate two offset regimes that allow clean isolation of the slash pattern: (i) local offsets ($\Delta<5$), which capture short-range structural effects, and (ii) extreme long-range offsets ($500\le\Delta\le5000$), where slash lines—if present—manifest as stable diagonal structures across large positional gaps. We intentionally exclude the intermediate offsets ($5\le\Delta<500$), because attention in this range is dominated by semantic relatedness between nearby tokens. In this confounded regime, the slash-dominance score $S_{i,i-\Delta}$ reflects a multiplicative interaction between semantic dependencies and positional alignment, making it impossible to isolate the geometric slash pattern. By focusing on the short- and long-range extremes—where semantic contributions are either concentrated near the diagonal or negligible at large distances—we obtain measurements of slash-dominance that cleanly reflect positional structure.

To approximate the expectation in \eqref{eq:def_slash_dom}, we prefill each model with 500 diverse prompts from LongBenchV2~\citep{bai2025longbenchv2}, which provides heterogeneous long-context distributions and yields stable cross-model statistics. Each prompt is truncated to 6000 tokens to ensure consistent context length across models and to avoid context-extension artifacts. We collect the prefill attention matrices, which deterministically reflect each model’s attention geometry without introducing sampling noise. Following prior observations that early tokens (including BOS and warm-up positions) exhibit disproportionately large attention scores~\citep{xiao2023efficient}, we exclude attention entries involving positions $1$–$4$ from the computation of $S_{i,i-\Delta}$. Removing these anomalous early positions prevents them from dominating the slash-dominance statistics. After excluding these positions, each prompt still contributes at least $1000$ valid pairs $(i, i-\Delta)$ for the $\Delta$ ranges we study. In total, we average over more than $5\times 10^{5}$ valid samples.

For slash-dominance detection, we match the threshold $\kappa$ to the natural scale of attention in each regime. For local offsets ($\Delta<5$), we set $\kappa=0.1$, reflecting the high concentration of attention near the diagonal. However, the number of \acp{sdh} for $\Delta = 0,1,2$ is quite large, so we only report the first three heads for each $\Delta$ in the tables in Appendix~\ref{app:head_result}. For extreme long-range offsets ($\Delta\ge500$), we use $\kappa=10^{-3}$. Although small in magnitude, this threshold corresponds to a 5× enrichment over the uniform baseline $1/5000 = 2\times10^{-4}$ at context length 5000, making it a meaningful marker of strong long-range alignment. Empirically, we do not find any \acp{sdh} with large $\Delta$ in Gemma-7B. Thus, we only report results for \acp{sdh} with large $\Delta$ in Llama3 and Qwen2.5. 

For the experiments on \ac{ood} prompts, we construct prompts from i.i.d. tokens. Specifically, we independently sample each token from the uniform distribution over the model’s vocabulary. For example, for Qwen2.5-7B-Instruct, the vocabulary size is $151{,}642$, so each token is sampled with probability $1/151{,}642$. Using this procedure, we generate $500$ prompts of length $6000$ by sampling $500 \times 6000$ tokens in total.

Our experiments are conducted on NVIDIA A100 GPUs with $80$G memory.

\section{Parameter Compression Experiment}\label{app:exp_parameter_compression}
Motivated by the low-rankness of the queries and keys, we can compress the parameters of the $W_Q$ and $W_K$ matrices in the attention modules.

\paragraph{Experimental Setup.} 
As we will compute the effective singular values/subspaces of the $W_Q$ and $W_K$ over LongBench v2 dataset~\citep{bai2025longbenchv2} with heterogeneous prompts, we extend our power computation method in \eqref{equ:lowrank_computation} to the hidden states $H$ of a sequence of tokens, which is originally defined for a single token embedding $\xb$.

To take the bias $\bb$ of Qwen2.5-7B-Instruct model into account, for simplicity, we define $\sigma_0=1$ and $\vb_0 = \bb$. For other models without the bias term, we define $\sigma_0=1$ and $\vb_0 = 0$.
For each head, we first compute 
\begin{align}
    r_l(H) = \frac{\|\sigma_l \vb^\top_l H\|_{2}^2}{\sum_{i=0}^d \|\sigma_i \vb^\top_i H\|_{2}^2}
\end{align}
 for all $l\in\{0\}\cup [d]$, hidden states $H$ of prompts in LongBench v2 dataset~\citep{bai2025longbenchv2}, where $\vb_l$ is the singular vector of the weights, which is defined in \eqref{equ:lowrank_computation}. This quantity characterizes the power of this singular value $\sigma_l$ for this prompt $x$.
 We then average the power over all the prompts to get the average power 
 \begin{align}
     r_l = \frac{\sum_{H} r_l(H)}{|\mbox{LongBench v2}|}.
 \end{align}
Subsequently, we sort the singular values in the decreasing order of $\{r_l\}_{l\in[d]}$ to get $\{r_{\mathrm{sorted}(l)}\}_{l\in[d]}$, where we force $r_{\mathrm{sorted}(0)} = r_0$ to keep the bias term.
We then calculate the effective rank $R_{\mathrm{thre}}$ such that the cumulative average power hits the power threshold:
\begin{align}
    R_{\mathrm{thre}} = 
    \min\Big\{
        l \,\Big|\, \sum_{i=0}^l\sum_{j=1}^d  r_{j}\mathbbm{1}\{\mathrm{sorted}(j)=i\}>\mathrm{thre}
    \Big\}.
\end{align}
For those $W_Q$ and $W_K$ weight matrices whose effective ranks are smaller than compression threshold $\mathrm{rank}_{\mathrm{thre}}$, we adopt its low rank approximations corresponding to the singular values $\{r_l | \mathrm{sorted}(l)\leq R_{\mathrm{thre}}\}$; otherwise, we leave the weight matrices untouched.
In the end, we obtain the low-rank model.
 
We evaluate low-rank models with different energy thresholds $\mathrm{thre}$ and compression limits $\mathrm{rank}_{\mathrm{thre}}$ on LongBench~\citep{bai2024longbench}, which spans $21$ benchmarks across $6$ task types and consists of in total $4750$ prompts. For each model, we report the mean score within each task type, the overall mean across all $21$ benchmarks, and the reduction in effective parameter count computed via
$(N_{\text{total}}^{QK}-N_{\text{eff}}^{QK})/N_{\text{total}}^{QK}$, where $N_{\text{total}}^{QK}$ and $N_{\text{eff}}^{QK}$ denote, respectively, the total and effective parameter counts of all $W_Q$ and $W_K$ matrices. For a low-rank matrix ($W_Q$ or $W_K$ for an attention head) with retained rank $R_{\mathrm{thre}}$ (after truncation),
the effective parameter count for a low-rank weight matrix is $R_{\mathrm{thre}}\times(d_{\mathrm{head}}+d_{\mathrm{model}})$, instead of the original size of $d_{\mathrm{head}}\times d_{\mathrm{model}}$.

\paragraph{Experimental Result.} We present the result in Table~\ref{tab:parameter_compression}. It can be observed that we can save decent parameter count while preserving the performance of the original model. 
Since the Qwen2.5-7B-Instruct model includes an additional bias term whose norm is relatively large compared to that of the weight matrix (see Table~\ref{table:sigma_qk_partial}), we are able to discard more singular values while still maintaining performance. Specifically, we can set the power threshold $\mathrm{thre}=0.92$ without compromising its overall performance ($37.84$ v.s. $37.80$ from the baseline). It saves $6.51\%$ of the original $W_Q$ and $W_K$ weight parameters.

\begin{table}[t]
\centering
\small
\setlength{\tabcolsep}{5pt}
\renewcommand{\arraystretch}{1.15}
\newcommand{\varrow}{\(\hookrightarrow\)}


\begin{tabular}{
  l
  *{6}{S}
  S
  c
}
\toprule
\textbf{Model} &
\textbf{SdQA} &
\textbf{MdQA} &
\textbf{Sum.} &
\textbf{Few-shot} &
\textbf{Syn.} &
\textbf{Code} &
\textbf{Avg.} &
\textbf{\makecell{Reduced \\ Params}} \\
\midrule

\textbf{Llama-3-8B-Instruct}
  & 27.86 & 23.07 & 25.86 & 60.72 & 54.04 & 49.27 & 38.60 & 0 \\
\rowcolor{lightroyalblue} \varrow~$\mathtt{thre=0.98,rank_{thre}=64}$
  & 26.96 & 22.21 & 24.46 & \textbf{61.73} & 47.03 & 47.34 & 37.01 & 8.62\% \\
\rowcolor{lightroyalblue} \varrow~$\mathtt{thre=0.96,rank_{thre}=32}$
  & 26.46 & \textbf{24.10} & 25.40 & \textbf{62.59} & \textbf{57.29} & 35.99 & 38.00 & 5.90\% \\
\midrule

\textbf{Qwen2.5-7B-Instruct}
  & 18.82 & 14.65 & 22.98 & 62.50 & 63.55 & 63.65 & 37.80 & 0 \\
\rowcolor{lightroyalblue} \varrow~$\mathtt{thre=0.96,rank_{thre}=64}$
  & 18.53 & \textbf{15.59} & 22.76 & 60.85 & 61.40 & 60.88 & 36.99 & 9.54\% \\
\rowcolor{lightroyalblue} \varrow~$\mathtt{thre=0.92,rank_{thre}=32}$
  & \textbf{19.62} & \textbf{17.34} & \textbf{25.09} & 62.04 & 61.36 & 57.13 & \textbf{37.84} & 6.51\% \\
\bottomrule
\end{tabular}
\caption{\textbf{Parameter Compression Results.} The \textbf{bold} number indicates better performance than that of the baseline. The models are evaluated across $6$ task types: Single-doc QA (SdQA), Multi-doc QA (MdQA), Summarization (Sum.), Few-shot, Synthetic (Syn.) and Code. The performance of the model is maintained with decent parameter reductions.}
\label{tab:parameter_compression}
\end{table}

\section{More results of Slash-Dominant Heads}\label{app:head_result}

In this section, we present additional results on \acp{sdh} with small $\Delta$. For notational convenience, we slightly abuse $\bbE[\attn_{i,i-\Delta}]$ to represent the average slash score reported in the tables. In the following, we will abbreviate Gemma-7B, Llama3-8B-Instruct, and Qwen2.5-7B-Instruct as Gemma, Llama3, and Qwen2.5, respectively.
\subsection{More Results of Slash-dominant Heads with Small $\Delta$}\label{app:small_delta_list}
For each $\Delta$, we list at most three heads with the average slash score larger than 0.1. Due to space limitations, only a subset is shown; more slash-dominant heads with small $\Delta$ exist but are not included here.

\subsubsection{Full List of \acp{sdh} with Small $\Delta$}

We report the full list of \acp{sdh} with $\Delta \in \{ 0, 1, 2, 3, 4 \}$ for each model as follows. These \acp{sdh} are identified by using \eqref{eq:def_slash_dom} with $\kappa = 0.1$, where the average slash scores are computed over $500$ random prompts from \texttt{LongBenchV2} \citep{bai2025longbenchv2} dataset.

\vspace{5pt}
\noindent {\bf Gemma-7B Model.}

\begin{itemize}[wide, labelindent=0pt]
    \setlength{\parskip}{0pt} 
    \setlength{\topsep}{0pt}
\item For $\Delta=0$, the found \acp{sdh} are: 
\begin{itemize}[wide, labelindent=0pt]
    \setlength{\parskip}{0pt} 
    \setlength{\topsep}{0pt}
    \item [(Layers $0$-$2$)]$\rmL0\rmH2$, $\rmL0\rmH4$, $\rmL0\rmH6$, $\rmL0\rmH10$, $\rmL0\rmH15$, $\rmL1\rmH2$, $\rmL1\rmH7$, $\rmL2\rmH0$, $\rmL2\rmH3$, $\rmL2\rmH8$;
    \item [(Layers $3$-$5$)] $\rmL3\rmH1$, $\rmL3\rmH7$, $\rmL3\rmH8$, $\rmL4\rmH0$, $\rmL4\rmH9$, $\rmL4\rmH12$, $\rmL5\rmH6$, $\rmL5\rmH11$, $\rmL5\rmH13$;
    \item [(Layers $6$-$8$)]$\rmL6\rmH1$, $\rmL6\rmH3$, $\rmL6\rmH11$, $\rmL7\rmH3$, $\rmL7\rmH6$, $\rmL7\rmH11$, $\rmL7\rmH14$, $\rmL8\rmH1$, $\rmL8\rmH4$, $\rmL8\rmH5$, $\rmL8\rmH9$;
    \item [(Layers $9$-$11$)] $\rmL9\rmH5$, $\rmL9\rmH8$, $\rmL9\rmH14$, $\rmL10\rmH2$, $\rmL10\rmH5$, $\rmL10\rmH7$, $\rmL10\rmH15$, $\rmL11\rmH4$, $\rmL11\rmH7$, $\rmL11\rmH8$;  
    \item [(Layers $12$-$20$)] $\rmL12\rmH12$, $\rmL13\rmH1$, $\rmL13\rmH4$, $\rmL14\rmH14$, $\rmL16\rmH0$, $\rmL17\rmH8$, $\rmL19\rmH1$, $\rmL19\rmH15$, $\rmL20\rmH3$, $\rmL20\rmH7$; 
    \item [(Layers $21$-$25$)]  $\rmL20\rmH3$, $\rmL20\rmH7$, $\rmL21\rmH10$, $\rmL22\rmH5$, $\rmL23\rmH2$, $\rmL23\rmH14$, $\rmL24\rmH0$, $\rmL24\rmH12$, $\rmL25\rmH2$, $\rmL25\rmH4$, $\rmL25\rmH8$;
    \item [(Layer $26$)]$\rmL26\rmH0$, $\rmL26\rmH1$, $\rmL26\rmH3$, $\rmL26\rmH7$, $\rmL26\rmH9$, $\rmL26\rmH15$;
    \item [(Layer $27$)] $\rmL27\rmH0$, $\rmL27\rmH1$, $\rmL27\rmH2$, $\rmL27\rmH3$, $\rmL27\rmH4$, $\rmL27\rmH5$, $\rmL27\rmH6$, $\rmL27\rmH7$, $\rmL27\rmH8$, $\rmL27\rmH10$, $\rmL27\rmH13$, $\rmL27\rmH14$. 
\end{itemize}
\item For $\Delta=1$, the found \acp{sdh} are:
\begin{itemize}[wide, labelindent=0pt]
    \setlength{\parskip}{0pt} 
    \setlength{\topsep}{0pt}
    \item [(Layers $0$-$2$)]$\rmL0\rmH0$, $\rmL0\rmH1$, $\rmL0\rmH7$, $\rmL0\rmH8$, $\rmL0\rmH9$, $\rmL0\rmH14$, $\rmL1\rmH0$, $\rmL1\rmH7$, $\rmL1\rmH9$, $\rmL1\rmH14$, $\rmL1\rmH15$, $\rmL2\rmH2$, $\rmL2\rmH5$;
    \item [(Layers $3$-$5$)] $\rmL3\rmH2$, $\rmL3\rmH5$, $\rmL3\rmH9$, $\rmL3\rmH11$, $\rmL3\rmH14$, $\rmL3\rmH15$, $\rmL4\rmH6$, $\rmL4\rmH14$;
    \item [(Layers $6$-$18$)]$\rmL6\rmH7$, $\rmL9\rmH10$, $\rmL10\rmH10$, $\rmL11\rmH15$, $\rmL13\rmH3$, $\rmL16\rmH9$, $\rmL17\rmH0$, $\rmL17\rmH15$, $\rmL18\rmH11$;
    \item [(Layers $19$-$27$)] $\rmL19\rmH13$, $\rmL20\rmH5$, $\rmL22\rmH8$, $\rmL23\rmH3$, $\rmL23\rmH15$, $\rmL25\rmH2$, $\rmL26\rmH4$, $\rmL26\rmH11$, $\rmL27\rmH5$.
\end{itemize}
\item For $\Delta=2$, the found \acp{sdh} are: $\rmL0\rmH1$, $\rmL1\rmH7$, $\rmL2\rmH2$, $\rmL2\rmH9$, $\rmL3\rmH9$, $\rmL13\rmH3$.
\item For $\Delta=3$, the found \acp{sdh} are: $\rmL0\rmH1$, $\rmL2\rmH9$.
\item For $\Delta=4$, the found \ac{sdh} is: $\rmL0\rmH1$.
\end{itemize}

\noindent {\bf Llama3-8B-Instruct Model.}

\begin{itemize}[wide, labelindent=0pt]
    \setlength{\parskip}{0pt} 
    \setlength{\topsep}{0pt}
\item For Llama3 and $\Delta=0$, the found \acp{sdh} are:
\begin{itemize}[wide, labelindent=0pt]
    \setlength{\parskip}{0pt} 
    \setlength{\topsep}{0pt}
    \item [(Layers $0$-$8$)] $\rmL0\rmH2$, $\rmL0\rmH24$, $\rmL4\rmH5$, $\rmL4\rmH26$, $\rmL6\rmH5$, $\rmL6\rmH18$, $\rmL7\rmH0$, $\rmL8\rmH5$, $\rmL8\rmH13$, $\rmL8\rmH16$, $\rmL8\rmH23$;
    \item [(Layers $9$-$14$)] $\rmL9\rmH10$, $\rmL10\rmH0$, $\rmL10\rmH10$, $\rmL12\rmH12$, $\rmL12\rmH20$, $\rmL12\rmH30$, $\rmL13\rmH28$, $\rmL14\rmH8$, $\rmL14\rmH16$, $\rmL14\rmH19$;
    \item [(Layers $15$-$19$)]$\rmL15\rmH6$, $\rmL15\rmH10$, $\rmL15\rmH24$, $\rmL16\rmH22$, $\rmL16\rmH29$, $\rmL16\rmH30$, $\rmL17\rmH6$, $\rmL17\rmH12$, $\rmL19\rmH8$, $\rmL19\rmH11$;
    \item [(Layers $20$-$26$)] $\rmL20\rmH0$, $\rmL20\rmH3$, $\rmL21\rmH3$, $\rmL21\rmH8$, $\rmL21\rmH10$, $\rmL22\rmH9$, $\rmL22\rmH19$, $\rmL25\rmH21$, $\rmL26\rmH0$, $\rmL26\rmH1$, $\rmL26\rmH3$, $\rmL26\rmH22$;
    \item [(Layers $27$-$29$)] $\rmL27\rmH11$, $\rmL28\rmH10$, $\rmL28\rmH11$, $\rmL28\rmH17$, $\rmL28\rmH21$, $\rmL28\rmH29$, $\rmL28\rmH30$, $\rmL29\rmH8$, $\rmL29\rmH10$, $\rmL29\rmH25$;
    \item [(Layer $30$)] $\rmL30\rmH0$, $\rmL30\rmH24$, $\rmL30\rmH25$, $\rmL30\rmH26$, $\rmL30\rmH27$;
    \item [(Layer $31$)] $\rmL31\rmH0$, $\rmL31\rmH1$, $\rmL31\rmH2$, $\rmL31\rmH3$, $\rmL31\rmH5$, $\rmL31\rmH6$, $\rmL31\rmH7$, $\rmL31\rmH9$, $\rmL31\rmH11$, $\rmL31\rmH12$, $\rmL31\rmH14$, $\rmL31\rmH15$, $\rmL31\rmH18$, $\rmL31\rmH23$, $\rmL31\rmH24$, $\rmL31\rmH25$, $\rmL31\rmH26$, $\rmL31\rmH27$.
\end{itemize}

\item For Llama3 and $\Delta=1$, the found \acp{sdh} are:
\begin{itemize}[wide, labelindent=0pt]
    \setlength{\parskip}{0pt} 
    \setlength{\topsep}{0pt}
    \item [(Layers $0$-$1$)] $\rmL0\rmH0$, $\rmL0\rmH1$, $\rmL0\rmH2$, $\rmL0\rmH3$, $\rmL0\rmH6$, $\rmL0\rmH10$, $\rmL0\rmH24$, $\rmL0\rmH26$, $\rmL1\rmH16$, $\rmL1\rmH18$, $\rmL1\rmH20$;
    \item [(Layers $2$-$11$)] $\rmL4\rmH12$, $\rmL6\rmH8$, $\rmL6\rmH16$, $\rmL7\rmH1$, $\rmL7\rmH2$, $\rmL7\rmH5$, $\rmL8\rmH22$, $\rmL9\rmH10$, $\rmL9\rmH11$, $\rmL11\rmH16$, $\rmL11\rmH19$;
    \item [(Layers $12$-$17$)] $\rmL12\rmH19$, $\rmL14\rmH8$, $\rmL14\rmH9$, $\rmL14\rmH11$, $\rmL14\rmH26$, $\rmL15\rmH0$, $\rmL15\rmH7$, $\rmL16\rmH22$, $\rmL16\rmH29$, $\rmL17\rmH4$, $\rmL17\rmH5$;
    \item [(Layers $18$-$28$)] $\rmL18\rmH26$, $\rmL19\rmH8$, $\rmL21\rmH4$, $\rmL21\rmH7$, $\rmL21\rmH10$, $\rmL25\rmH20$, $\rmL25\rmH22$, $\rmL25\rmH23$, $\rmL27\rmH29$, $\rmL28\rmH26$;
    \item [(Layers $29$-$31$)] $\rmL29\rmH8$, $\rmL29\rmH26$, $\rmL29\rmH27$, $\rmL30\rmH10$, $\rmL31\rmH6$, $\rmL31\rmH15$, $\rmL31\rmH24$
\end{itemize}

\item For Llama3 and $\Delta=2$, the found \acp{sdh} are $\rmL0\rmH0$, $\rmL0\rmH2$, $\rmL1\rmH21$, $\rmL7\rmH2$, $\rmL9\rmH11$, $\rmL14\rmH8$, $\rmL14\rmH26$.

\item For Llama3 and $\Delta=3$, the found \ac{sdh} is $\rmL0\rmH0$.

\item For Llama3 and $\Delta=4$, the found \ac{sdh} is $\rmL0\rmH0$.
\end{itemize}

\noindent {\bf Qwen2.5-7B-Instruct Model.}

\begin{itemize}[wide, labelindent=0pt]
    \setlength{\parskip}{0pt} 
    \setlength{\topsep}{0pt}
\item For Qwen2.5 and $\Delta=0$, the found \acp{sdh} are:
\begin{itemize}[wide, labelindent=0pt]
    \setlength{\parskip}{0pt} 
    \setlength{\topsep}{0pt}
    \item [(Layers $0$-$17$)]$\rmL14\rmH9$, $\rmL14\rmH20$, $\rmL14\rmH25$, $\rmL15\rmH8$, $\rmL15\rmH23$, $\rmL16\rmH19$, $\rmL16\rmH26$, $\rmL17\rmH4$;
    \item [(Layers $18$-$31$)]$\rmL18\rmH0$, $\rmL18\rmH7$, $\rmL18\rmH17$, $\rmL18\rmH25$, $\rmL19\rmH0$, $\rmL19\rmH11$, $\rmL19\rmH13$, $\rmL20\rmH1$, $\rmL20\rmH2$, $\rmL20\rmH14$, $\rmL24\rmH18$, $\rmL25\rmH2$, $\rmL25\rmH3$, $\rmL25\rmH4$, $\rmL25\rmH20$.
\end{itemize}

\item For Qwen2.5 and $\Delta=1$, the found \acp{sdh} are:
\begin{itemize}[wide, labelindent=0pt]
    \setlength{\parskip}{0pt} 
    \setlength{\topsep}{0pt}
    \item [(Layer $0$)] $\rmL0\rmH6$, $\rmL0\rmH7$, $\rmL0\rmH16$, $\rmL0\rmH20$, $\rmL0\rmH23$, $\rmL0\rmH24$, $\rmL0\rmH26$;
    \item [(Layer $1$)] $\rmL1\rmH3$, $\rmL1\rmH10$, $\rmL1\rmH11$, $\rmL1\rmH12$, $\rmL1\rmH14$, $\rmL1\rmH15$, $\rmL1\rmH17$, $\rmL1\rmH18$, $\rmL1\rmH19$, $\rmL1\rmH23$, $\rmL1\rmH24$;
    \item [(Layers $2$-$3$)] $\rmL2\rmH1$, $\rmL2\rmH2$, $\rmL2\rmH5$, $\rmL2\rmH6$, $\rmL2\rmH27$, $\rmL3\rmH1$, $\rmL3\rmH3$, $\rmL3\rmH10$, $\rmL3\rmH15$, $\rmL3\rmH22$, $\rmL3\rmH24$, $\rmL3\rmH26$, $\rmL3\rmH27$;
    \item [(Layers $4$-$6$)] $\rmL4\rmH18$, $\rmL4\rmH23$, $\rmL5\rmH3$, $\rmL6\rmH1$, $\rmL6\rmH2$, $\rmL6\rmH4$, $\rmL6\rmH6$, $\rmL6\rmH15$, $\rmL6\rmH16$, $\rmL6\rmH21$, $\rmL6\rmH25$, $\rmL6\rmH26$;
    \item [(Layers $7$-$14$)] $\rmL7\rmH4$, $\rmL7\rmH5$, $\rmL10\rmH4$, $\rmL10\rmH22$, $\rmL10\rmH25$, $\rmL10\rmH27$, $\rmL11\rmH1$, $\rmL11\rmH18$, $\rmL12\rmH21$, $\rmL13\rmH8$, $\rmL13\rmH13$;
    \item [(Layers $15$-$16$)] $\rmL15\rmH5$, $\rmL15\rmH8$, $\rmL15\rmH9$, $\rmL15\rmH11$, $\rmL15\rmH13$, $\rmL15\rmH23$, $\rmL15\rmH27$, $\rmL16\rmH21$, $\rmL16\rmH23$;
    \item [(Layers $17$-$20$)] $\rmL17\rmH2$, $\rmL17\rmH3$, $\rmL17\rmH9$, $\rmL17\rmH11$, $\rmL18\rmH9$, $\rmL18\rmH20$, $\rmL19\rmH8$, $\rmL19\rmH10$, $\rmL19\rmH12$, $\rmL19\rmH13$, $\rmL20\rmH1$, $\rmL20\rmH5$;
    \item [(Layers $21$-$23$)] $\rmL21\rmH15$, $\rmL21\rmH16$, $\rmL21\rmH18$, $\rmL21\rmH19$, $\rmL21\rmH20$, $\rmL21\rmH21$, $\rmL23\rmH24$, $\rmL23\rmH26$, $\rmL23\rmH27$;
    \item [(Layers $24$-$27$)] $\rmL24\rmH0$, $\rmL24\rmH1$, $\rmL24\rmH3$, $\rmL24\rmH5$, $\rmL24\rmH6$, $\rmL24\rmH20$, $\rmL25\rmH5$, $\rmL25\rmH14$, $\rmL25\rmH18$, $\rmL25\rmH19$, $\rmL26\rmH11$, $\rmL26\rmH26$, $\rmL26\rmH27$, $\rmL27\rmH8$.
\end{itemize}

\item For Qwen2.5 and $\Delta=2$, the found \acp{sdh} are:
\begin{itemize}[wide, labelindent=0pt]
    \setlength{\parskip}{0pt} 
    \setlength{\topsep}{0pt}
    \item [(Layer $0$)] $\rmL0\rmH1$, $\rmL0\rmH5$, $\rmL0\rmH7$, $\rmL0\rmH16$, $\rmL0\rmH23$, $\rmL0\rmH24$, $\rmL0\rmH26$;
    \item [(Layer $1$)] $\rmL1\rmH3$, $\rmL1\rmH12$, $\rmL1\rmH14$, $\rmL1\rmH15$, $\rmL1\rmH16$, $\rmL1\rmH17$, $\rmL1\rmH19$, $\rmL1\rmH24$;
    \item [(Layers $2$-$7$)] $\rmL2\rmH27$, $\rmL3\rmH1$, $\rmL3\rmH9$, $\rmL3\rmH10$, $\rmL3\rmH22$, $\rmL3\rmH24$, $\rmL3\rmH25$, $\rmL3\rmH26$, $\rmL3\rmH27$, $\rmL6\rmH15$, $\rmL7\rmH4$;
    \item [(Layers $8$-$15$)] $\rmL10\rmH4$, $\rmL10\rmH22$, $\rmL13\rmH8$, $\rmL13\rmH10$, $\rmL13\rmH13$, $\rmL15\rmH5$, $\rmL15\rmH8$, $\rmL15\rmH9$, $\rmL15\rmH11$, $\rmL15\rmH12$;
    \item [(Layers $16$-$27$)] $\rmL16\rmH21$, $\rmL17\rmH2$, $\rmL18\rmH9$, $\rmL19\rmH8$, $\rmL20\rmH1$, $\rmL21\rmH16$, $\rmL21\rmH20$.
\end{itemize}

\item For Qwen2.5 and $\Delta=3$, the found \acp{sdh} are $\rmL0\rmH1$, $\rmL0\rmH5$, $\rmL1\rmH0$, $\rmL1\rmH3$, $\rmL1\rmH17$, $\rmL3\rmH9$, $\rmL6\rmH15$, $\rmL13\rmH10$.

\item For Qwen2.5 and $\Delta=4$, the found \acp{sdh} are $\rmL1\rmH0$, $\rmL1\rmH17$, $\rmL6\rmH15$.
\end{itemize}

\subsubsection{Attention Scores and Ranks of $Q$, $K$, and $H$}
In this section, we present per-head statistics for \acp{sdh} with small $\Delta$. For conciseness, we report at most three heads for each $\Delta$ in each model. Specifically, for each $\Delta$, we rank the \acp{sdh} by their average slash score and list the three heads with the largest values in the following table.
\begin{table}[H]
    \centering
    \begin{tabular}{cccccccccc}
        \hline
        Model & Head & $\tilde{r}_1(Q)$ & $\tilde{R}_{0.95}(Q)$ & $\tilde{r}_1(K)$ & $\tilde{R}_{0.95}(K)$ \\
        \hline
        Gemma & $\rmL 0\rmH 4$  & $0.895$ & $2$ & $0.941$ & $2$  \\
        Gemma & $\rmL 0\rmH 7$  & $0.953$ & $1$ & $0.978$ & $1$ \\
        Gemma & $\rmL 0\rmH 1$  & $0.954$ & $1$ & $0.994$ & $1$ \\
        \hline\hline
        Llama3 & $\rmL 0\rmH 2$  & $0.860$ & $6$ & $0.738$ & $16$  \\
        Llama3 & $\rmL 0\rmH 0$  & $0.753$ & $7$ & $0.738$ & $16$ \\
        \hline\hline
        Qwen2.5 & $\rmL 0\rmH 6$  & $0.264$ & $59$ & $0.165$ & $73$   \\
        Qwen2.5 & $\rmL 0\rmH 5$  & $0.2751$ & $67$ & $0.165$ & $73$  \\
        \hline
    \end{tabular}
    \caption{The values of $\tilde{r}_1$ and $\tilde{R}_{0.95}$ for the \acp{sdh} located in the $0$-th layer of Gemma-7B, Llama3-8B-Instruct, and Qwen2.5-7B-Instruct.}
    \label{table:sigma_qk_full}
\end{table}

\begin{table}[H]
    \setlength{\tabcolsep}{4pt}
    \renewcommand{\arraystretch}{0.9}
    \centering
    \begin{tabular}{ccccccccccccc}
        \toprule
        Model & $\Delta$ & Head & $\bbE[\attn_{i,i-\Delta}]$ & \thead{\ac{ood}\\ $\bbE[\attn_{i,i-\Delta}]$} & $r_1(Q)$ & $R_{0.95}(Q)$ & $r_1(K)$ & $R_{0.95}(K)$ & $r_1(H)$ & $R_{0.95}(H)$ \\
        \midrule
        Gemma & $0$ & $\rmL 0\rmH 4$ & $1.000$ & $1.000$ & $0.822$ & $2$ & $0.854$ & $2$ & $0.343$ & $180$ \\
        Gemma & $0$ & $\rmL 26\rmH 3$ & $1.000$ & $1.000$ & $0.775$ & $14$ & $0.726$ & $17$ & $0.283$ & $348$ \\
        Gemma & $0$ & $\rmL 27\rmH 8$ & $1.000$  & $1.000$ & $0.890$ & $3$ & $0.918$ & $2$ & $0.363$ & $241$ \\
        Gemma & $1$ & $\rmL 0\rmH 7$ & $0.904$  & $0.999$ & $0.888$ & $2$ & $0.942$ & $2$ & $0.343$ & $180$ \\
        Gemma & $1$ & $\rmL 19\rmH 13$ & $0.364$ & $0.741$ & $0.895$ & $20$ & $0.896$ & $18$ & $0.225$ & $408$  \\
        Gemma & $1$ & $\rmL 13\rmH 3$ & $0.333$ & $0.752$ & $0.885$ & $16$ & $0.914$ & $8$ & $0.274$ & $375$ \\
        Gemma & $2$ & $\rmL 0\rmH 1$ & $0.302$ & $0.352$ & $0.918$ & $3$ & $0.991$ & $1$ & $0.343$ & $180$ \\
        Gemma & $2$ & $\rmL 3\rmH 9$ & $0.148$ & $0.278$ & $0.824$ & $36$ & $0.891$ & $11$ & $0.273$ & $364$ \\
        Gemma & $2$ & $\rmL 2\rmH 2$ & $0.127$ & $0.298$ & $0.801$ & $52$ & $0.884$ & $23$ & $0.241$ & $334$ \\
        Gemma & $3$ & $\rmL 0\rmH 1$ & $0.197$ & $0.172$ & $0.918$ & $3$ & $0.991$ & $1$ & $0.343$ & $180$ \\
        Gemma & $3$ & $\rmL 2\rmH 9$ & $0.115$ & $0.144$ & $0.875$ & $18$ & $0.944$ & $2$ & $0.241$ & $334$ \\
        Gemma & $4$ & $\rmL 0\rmH 1$ & $0.101$ & $0.072$ & $0.918$ & $3$ & $0.991$ & $1$ & $0.343$ & $180$ \\ 
        \midrule
        \midrule
        Llama3 & $0$ & $\rmL 31\rmH 14$ & $0.961$ & $0.991$ & $0.876$ & $4$ & $0.801$ & $26$ & $0.243$ & $424$ \\
        Llama3 & $0$ & $\rmL 12\rmH 12$ & $0.698$ & $0.667$ & $0.946$ & $2$ & $0.759$ & $26$ & $0.265$ & $448$ \\
        Llama3 & $0$ & $\rmL 16\rmH 30$ & $0.539$ & $0.538$ & $0.892$ & $9$ & $0.848$ & $15$ & $0.254$ & $388$ \\
        Llama3 & $1$ & $\rmL 0\rmH 2$ & $0.560$ & $0.656$ & $0.957$ & $1$ & $0.860$ & $6$ & $0.310$ & $191$ \\
        Llama3 & $1$ & $\rmL 14\rmH 26$ & $0.516$ & $0.701$ & $0.965$ & $1$ & $0.788$ & $23$ & $0.271$ & $400$ \\
        Llama3 & $1$ & $\rmL 1\rmH 20$ & $0.333$ & $0.749$ & $0.978$ & $1$ & $0.838$ & $19$ & $0.388$ & $160$ \\
        Llama3 & $2$ & $\rmL 0\rmH 0$ & $0.180$ & $0.160$ & $0.956$ & $1$ & $0.860$ & $6$ & $0.310$ & $191$ \\
        Llama3 & $2$ & $\rmL 0\rmH 2$ & $0.169$ & $0.141$ & $0.957$ & $1$ & $0.860$ & $6$ & $0.310$ & $191$ \\
        Llama3 & $2$ & $\rmL 1\rmH 21$ & $0.165$ & $0.410$ & $0.977$ & $1$ & $0.838$ & $19$ & $0.388$ & $160$ \\
        Llama3 & $3$ & $\rmL 0\rmH 0$ & $0.159$  & $0.125$ & $0.957$ & $1$ & $0.860$ & $6$ & $0.310$ & $191$ \\
        Llama3 & $4$ & $\rmL 0\rmH 0$ & $0.109$  & $0.080$ & $0.957$ & $1$ & $0.860$ & $6$ & $0.310$ & $191$ \\
        \midrule
        \midrule
        Qwen2.5 & $0$ & $\rmL 18\rmH 7$ & $1.000$ & $1.000$ & $0.986$ & $1$ & $0.703$ & $26$ & $0.215$ & $419$ \\
        Qwen2.5 & $0$ & $\rmL 19\rmH 0$ & $1.000$ & $1.000$ & $0.982$ & $1$ & $0.952$ & $1$ & $0.232$ & $374$ \\
        Qwen2.5 & $0$ & $\rmL 14\rmH 25$ & $0.742$ & $0.800$ & $0.875$ & $12$ & $0.589$ & $34$ & $0.225$ & $324$ \\
        Qwen2.5 & $1$ & $\rmL 21\rmH 15$ & $0.699$ & $0.840$ & $0.963$ & $1$ & $0.712$ & $33$ & $0.235$ & $360$ \\
        Qwen2.5 & $1$ & $\rmL 13\rmH 13$ & $0.664$ & $0.904$ & $0.952$ & $1$ & $0.699$ & $32$ & $0.244$ & $366$ \\
        Qwen2.5 & $1$ & $\rmL 0\rmH 6$ & $0.655$ & $0.572$ & $0.912$ & $3$ & $0.999$ & $1$ & $0.171$ & $278$ \\
        Qwen2.5 & $2$ & $\rmL 1\rmH 3$ & $0.297$ & $0.327$ & $0.920$ & $3$ & $0.999$ & $1$ & $0.458$ & $155$ \\
        Qwen2.5 & $2$ & $\rmL 1\rmH 17$ & $0.200$ & $0.231$ & $0.869$ & $5$ & $0.975$ & $1$ & $0.458$ & $155$ \\
        Qwen2.5 & $2$ & $\rmL 0\rmH 5$ & $0.192$ & $0.100$ & $0.813$ & $23$ & $0.999$ & $1$ & $0.171$ & $278$ \\
        Qwen2.5 & $3$ & $\rmL 0\rmH 5$ & $0.161$ & $0.121$ & $0.813$ & $23$ & $0.999$ & $1$ & $0.171$ & $278$ \\
        Qwen2.5 & $3$ & $\rmL 1\rmH 17$ & $0.156$ & $0.165$ & $0.869$ & $5$ & $0.975$ & $1$ & $0.458$ & $155$ \\
        Qwen2.5 & $3$ & $\rmL 1\rmH 3$ & $0.154$ & $0.089$ & $0.920$ & $3$ & $0.999$ & $1$ & $0.458$ & $155$ \\
        Qwen2.5 & $4$ & $\rmL 1\rmH 0$ & $0.174$ & $0.205$ & $0.920$ & $3$ & $0.999$ & $1$ & $0.458$ & $155$ \\
        Qwen2.5 & $4$ & $\rmL 6\rmH 15$ & $0.114$ & $0.049$ & $0.898$ & $6$ & $0.919$ & $5$ & $0.331$ & $317$ \\
        Qwen2.5 & $4$ & $\rmL 1\rmH 17$ & $0.107$  & $0.098$ & $0.869$ & $5$ & $0.975$ & $1$ & $0.458$ & $155$ \\
        \bottomrule
    \end{tabular}
    \caption{This table lists the average attention scores of prompts in LongBench and \ac{ood} prompts, the rank information of $Q$, $K$, and $H$. }
    \label{table:head_eg_small_full}
\end{table}

\begin{table}[H]
    \centering
    \begin{tabular}{cccccc}
        \hline
        Model & Head & $r_1(W_Q)$ & $R_{0.95}(W_Q)$ & $r_1(W_K)$ & $R_{0.95}(W_K)$ \\
        \hline
        Gemma & $\rmL 0\rmH 4$  & $0.479$ & $10$ & $0.534$ & $10$  \\
        Gemma & $\rmL 26\rmH 3$  & $0.091$ & $148$ & $0.081$ & $160$   \\
        Gemma & $\rmL 27\rmH 8$  & $0.255$ & $47$ & $0.244$ & $47$   \\
        Gemma & $\rmL 0\rmH 7$  & $0.182$ & $122$ & $0.219$ & $135$   \\
        Gemma & $\rmL 19\rmH 13$  & $0.052$ & $221$ & $0.049$ & $203$    \\
        Gemma & $\rmL 13\rmH 3$  & $0.044$ & $203$ & $0.039$ & $203$   \\
        Gemma & $\rmL 0\rmH 1$  & $0.054$ & $210$ & $0.362$ & $194$   \\
        Gemma & $\rmL 3\rmH 9$  & $0.026$ & $217$ & $0.045$ & $215$   \\
        Gemma & $\rmL 2\rmH 2$  & $0.022$ & $218$ & $0.029$ & $218$   \\
        Gemma & $\rmL 2\rmH 9$  & $0.039$ & $218$ & $0.067$ & $215$   \\
        \hline\hline
        Llama3 & $\rmL 31\rmH 14$  & $0.157$ & $109$ & $0.040$ & $107$   \\
        Llama3 & $\rmL 12\rmH 12$  & $0.141$ & $101$ & $0.040$ & $99$   \\
        Llama3 & $\rmL 16\rmH 30$  & $0.075$ & $107$ & $0.051$ & $102$   \\
        Llama3 & $\rmL 0\rmH 2$  & $0.370$ & $61$ & $0.187$ & $62$   \\
        Llama3 & $\rmL 14\rmH 26$  & $0.178$ & $80$ & $0.037$ & $98$   \\
        Llama3 & $\rmL 1\rmH 20$  & $0.114$ & $69$ & $0.047$ & $71$   \\
        Llama3 & $\rmL 0\rmH 0$  & $0.417$ & $60$ & $0.187$ & $62$   \\
        Llama3 & $\rmL 1\rmH 21$  & $0.109$ & $68$ & $0.047$ & $71$   \\
        \hline\hline
        Qwen2.5 & $\rmL 18\rmH 7$  & $0.197$ & $77$ & $0.031$ & $98$   \\
        Qwen2.5 & $\rmL 19\rmH 0$  & $0.178$ & $81$ & $0.048$ & $90$   \\
        Qwen2.5 & $\rmL 14\rmH 25$  & $0.041$ & $106$ & $0.024$ & $95$   \\
        Qwen2.5 & $\rmL 21\rmH 15$  & $0.060$ & $109$ & $0.024$ & $105$   \\
        Qwen2.5 & $\rmL 13\rmH 13$  & $0.097$ & $106$ & $0.025$ & $102$   \\
        Qwen2.5 & $\rmL 0\rmH 6$  & $0.061$ & $90$ & $0.021$ & $107$   \\
        Qwen2.5 & $\rmL 1\rmH 3$  & $0.086$ & $82$ & $0.072$ & $82$   \\
        Qwen2.5 & $\rmL 1\rmH 17$  & $0.084$ & $73$ & $0.042$ & $80$   \\
        Qwen2.5 & $\rmL 0\rmH 5$  & $0.029$ & $107$ & $0.020$ & $107$   \\
        Qwen2.5 & $\rmL 1\rmH 17$  & $0.084$ & $73$ & $0.042$ & $80$   \\
        Qwen2.5 & $\rmL 1\rmH 0$  & $0.094$ & $80$ & $0.072$ & $82$   \\
        Qwen2.5 & $\rmL 6\rmH 15$  & $0.044$ & $105$ & $0.032$ & $98$   \\
        \hline
    \end{tabular}
    \caption{The table reports $r_1$ and $R_{0.95}$ for the query and key projection matrices, $W_Q$ and $W_K$, across attention heads in \acp{llm}.}
    \label{table:w_qk_rank_full}
\end{table}


\begin{table}[H]
    \centering
    \begin{tabular}{ccccccc}
        \hline
        Model & $\Delta$ & Head & $\bbE[\attn_{i,i-\Delta}]$ & \makecell{$\bbE[\attn_{i,i-\Delta}]$ w/o \\ high freqs} & \makecell{$\bbE[\attn_{i,i-\Delta}]$ w/o \\ med freqs} & \makecell{$\bbE[\attn_{i,i-\Delta}]$ w/o \\ low freqs} \\
        \hline
        Gemma & $0$ & $\rmL 0\rmH 4$ & $1.000$ & $0.884$ & $1.000$ & $1.000$ \\
        Gemma & $0$ & $\rmL 26\rmH 3$ & $1.000$ & $0.807$ & $0.999$ & $1.000$ \\
        Gemma & $0$ & $\rmL 27\rmH 8$ & $1.000$  & $0.680$ & $1.000$ & $1.000$ \\
        Gemma & $1$ & $\rmL 0\rmH 7$ & $0.904$  & $0.043$ & $0.902$ & $0.902$ \\
        Gemma & $1$ & $\rmL 19\rmH 13$ & $0.364$ & $0.282$ & $0.282$ & $0.282$ \\
        Gemma & $1$ & $\rmL 13\rmH 3$ & $0.333$ & $0.274$ & $0.274$ & $0.274$ \\
        Gemma & $2$ & $\rmL 0\rmH 1$ & $0.302$ & $0.085$ & $0.072$ & $0.072$ \\
        Gemma & $2$ & $\rmL 3\rmH 9$ & $0.148$ & $0.103$ & $0.103$ & $0.104$ \\
        Gemma & $2$ & $\rmL 2\rmH 2$ & $0.127$ & $0.045$ & $0.046$ & $0.046$ \\
        Gemma & $3$ & $\rmL 0\rmH 1$ & $0.197$ & $0.143$ & $0.125$ & $0.125$ \\
        Gemma & $3$ & $\rmL 2\rmH 9$ & $0.115$ & $0.054$ & $0.055$ & $0.042$ \\
        Gemma & $4$ & $\rmL 0\rmH 1$ & $0.101$ & $0.170$ & $0.172$ & $0.173$ \\
        \hline\hline
        Llama3 & $0$ & $\rmL 31\rmH 14$ & $0.961$ & $0.653$ & $0.908$ & $0.956$ \\
        Llama3 & $0$ & $\rmL 12\rmH 12$ & $0.698$ & $0.431$ & $0.691$ & $0.701$ \\
        Llama3 & $0$ & $\rmL 16\rmH 30$ & $0.539$ & $0.306$ & $0.493$ & $0.542$ \\
        Llama3 & $1$ & $\rmL 0\rmH 2$ & $0.560$ & $0.114$ & $0.530$ & $0.557$ \\
        Llama3 & $1$ & $\rmL 14\rmH 26$ & $0.516$ & $0.077$ & $0.509$ & $0.515$ \\
        Llama3 & $1$ & $\rmL 1\rmH 20$ & $0.333$ & $0.035$ & $0.331$ & $0.341$ \\
        Llama3 & $2$ & $\rmL 0\rmH 0$ & $0.180$ & $0.066$ & $0.155$ & $0.179$ \\
        Llama3 & $2$ & $\rmL 0\rmH 2$ & $0.169$ & $0.121$ & $0.164$ & $0.170$ \\
        Llama3 & $2$ & $\rmL 1\rmH 21$ & $0.165$ & $0.006$ & $0.164$ & $0.168$ \\
        Llama3 & $3$ & $\rmL 0\rmH 0$ & $0.159$  & $0.076$ & $0.140$ & $0.159$ \\
        Llama3 & $4$ & $\rmL 0\rmH 0$ & $0.109$  & $0.083$ & $0.098$ & $0.109$ \\
        \hline\hline
        Qwen2.5 & $0$ & $\rmL 18\rmH 7$ & $1.000$ & $0.967$ & $0.967$ & $1.000$ \\
        Qwen2.5 & $0$ & $\rmL 19\rmH 0$ & $1.000$ & $0.943$ & $0.993$ & $0.999$ \\
        Qwen2.5 & $0$ & $\rmL 14\rmH 25$ & $0.742$ & $0.473$ & $0.704$ & $0.741$ \\
        Qwen2.5 & $1$ & $\rmL 21\rmH 15$ & $0.699$ & $0.114$ & $0.653$ & $0.693$ \\
        Qwen2.5 & $1$ & $\rmL 13\rmH 13$ & $0.664$ & $0.108$ & $0.629$ & $0.654$ \\
        Qwen2.5 & $1$ & $\rmL 0\rmH 6$ & $0.655$ & $0.120$ & $0.558$ & $0.647$ \\
        Qwen2.5 & $2$ & $\rmL 1\rmH 3$ & $0.297$ & $0.137$ & $0.295$ & $0.296$ \\
        Qwen2.5 & $2$ & $\rmL 1\rmH 17$ & $0.200$ & $0.078$ & $0.175$ & $0.199$ \\
        Qwen2.5 & $2$ & $\rmL 0\rmH 5$ & $0.192$ & $0.062$ & $0.149$ & $0.188$ \\
        Qwen2.5 & $3$ & $\rmL 0\rmH 5$ & $0.161$ & $0.074$ & $0.127$ & $0.161$ \\
        Qwen2.5 & $3$ & $\rmL 1\rmH 17$ & $0.156$ & $0.089$ & $0.138$ & $0.156$ \\
        Qwen2.5 & $3$ & $\rmL 1\rmH 3$ & $0.154$ & $0.120$ & $0.152$ & $0.153$ \\
        Qwen2.5 & $4$ & $\rmL 1\rmH 0$ & $0.174$ & $0.109$ & $0.173$ & $0.174$ \\
        Qwen2.5 & $4$ & $\rmL 6\rmH 15$ & $0.114$ & $0.076$ & $0.095$ & $0.114$ \\
        Qwen2.5 & $4$ & $\rmL 1\rmH 17$ & $0.107$  & $0.094$ & $0.095$ & $0.106$ \\
        \hline
    \end{tabular}
    \caption{This table quantifies the effect of low-, medium-, and high-frequency components on \acp{sdh} by reporting, for each band, the average slash score after removing that band.}
    \label{table:head_freq_small_full}
\end{table}

\subsection{Figures of Slash-dominant Heads with Small $\Delta$ and Large Average Slash Scores}\label{app:small_delta_fig}

\subsubsection{Results of Llama3-8B-Instruct}
\begin{figure}[H]
\centering
\subfigure[Average attention score matrix of $\rmL 31\rmH 14$.]{\includegraphics[width=0.3\textwidth]{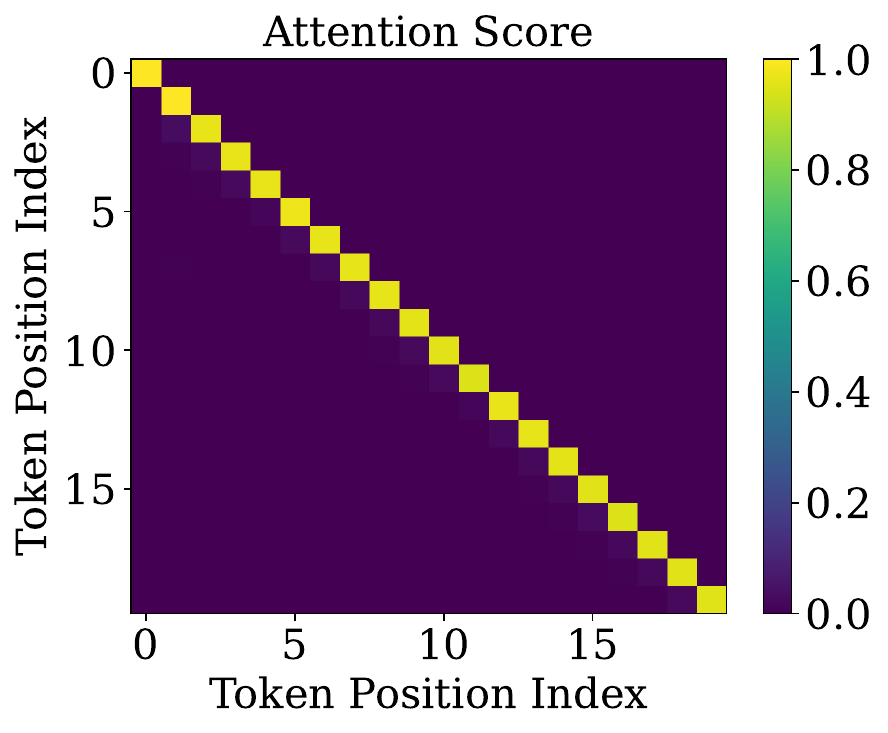}}
\quad
\subfigure[Average attention score matrix of $\rmL 0\rmH 2$.]{\includegraphics[width=0.3\textwidth]{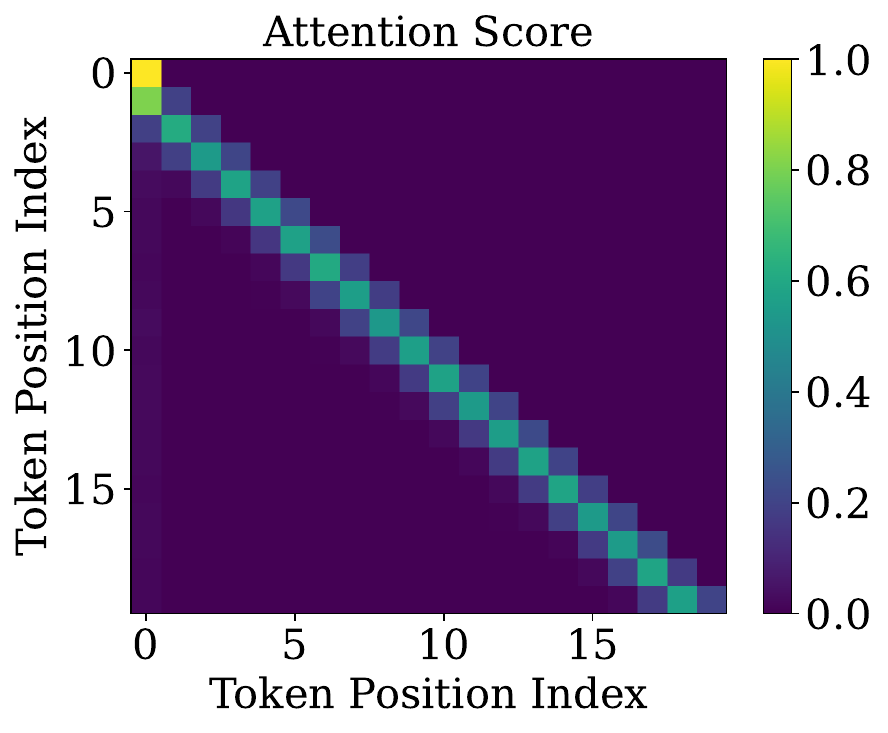}}
\quad
\subfigure[Average attention score matrix of $\rmL 0\rmH 0$.]{\includegraphics[width=0.3\textwidth]{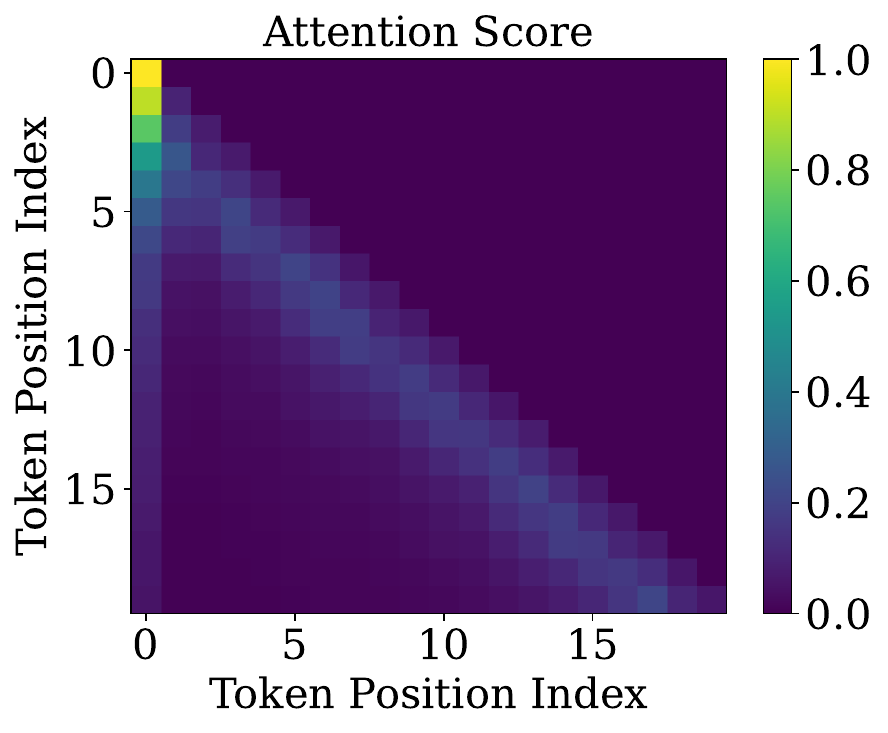}}
\caption{Average of attention score matrices in Llama3-8B-Instruct with prompts from LongBench V2.}
\label{fig:llama_small_longb}
\end{figure}

\begin{figure}[H]
\centering
\subfigure[Average attention score matrix of $\rmL 31\rmH 14$.]{\includegraphics[width=0.3\textwidth]{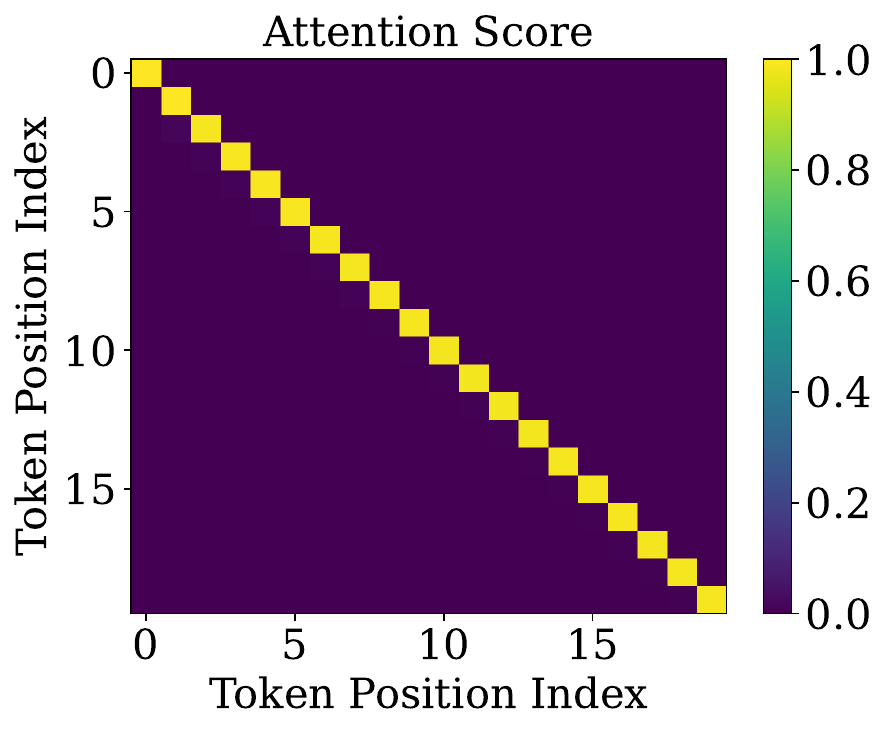}}
\quad
\subfigure[Average attention score matrix of $\rmL 0\rmH 2$.]{\includegraphics[width=0.3\textwidth]{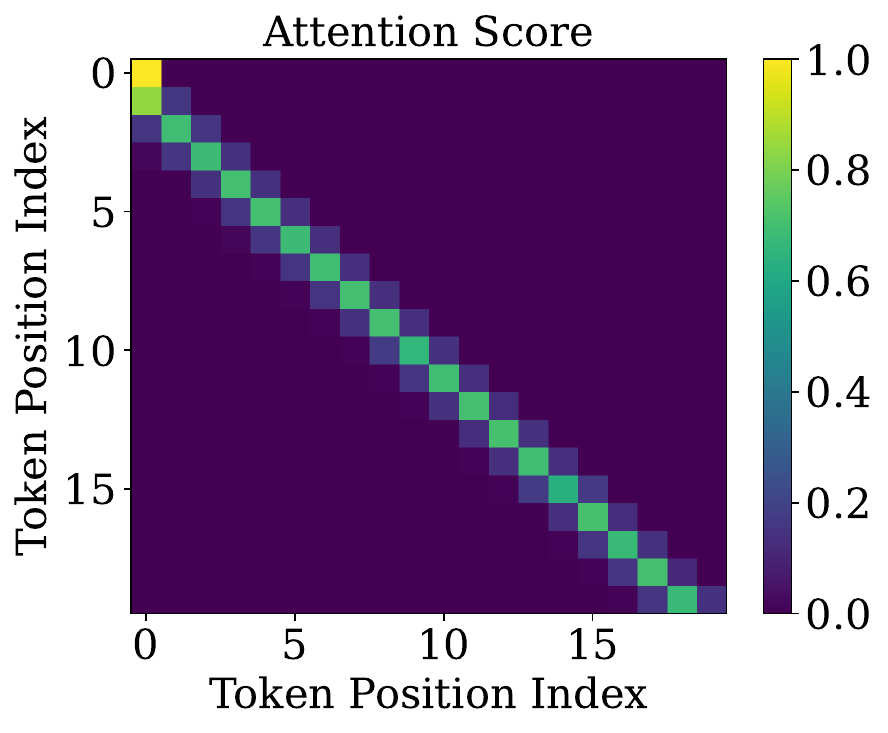}}
\quad
\subfigure[Average attention score matrix of $\rmL 0\rmH 0$.]{\includegraphics[width=0.3\textwidth]{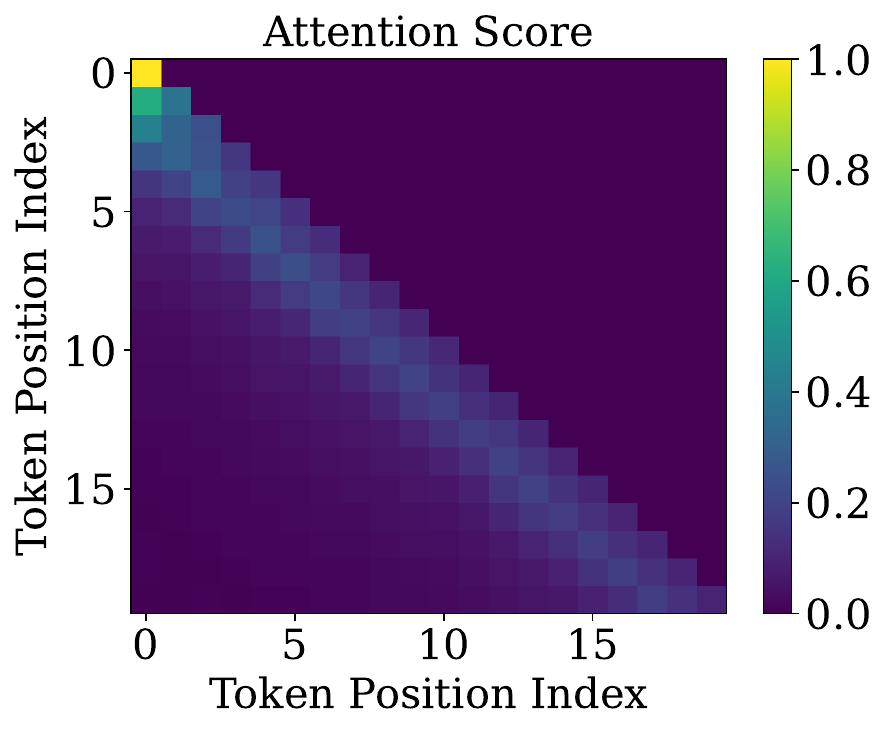}}
\caption{Average of attention score matrices in Llama3-8B-Instruct with prompts whose tokens are i.i.d.\ sampled from the uniform distribution on the alphabet.}
\label{fig:llama_small_ood}
\end{figure}

\begin{figure}[H]
\centering
\subfigure[Hidden state $H$ of $\rmL 31\rmH 14$ after PCA.]{\includegraphics[width=0.23\textwidth]{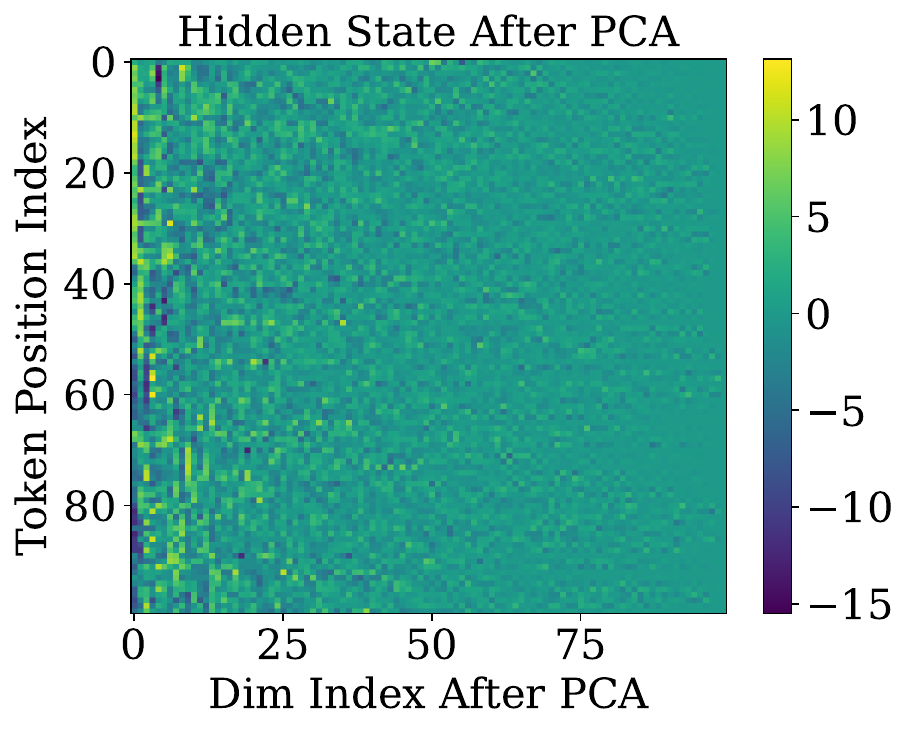}}
\hspace{0.08em}
\subfigure[Queries $Q$ of $\rmL 31\rmH 14$.]{\includegraphics[width=0.23\textwidth]{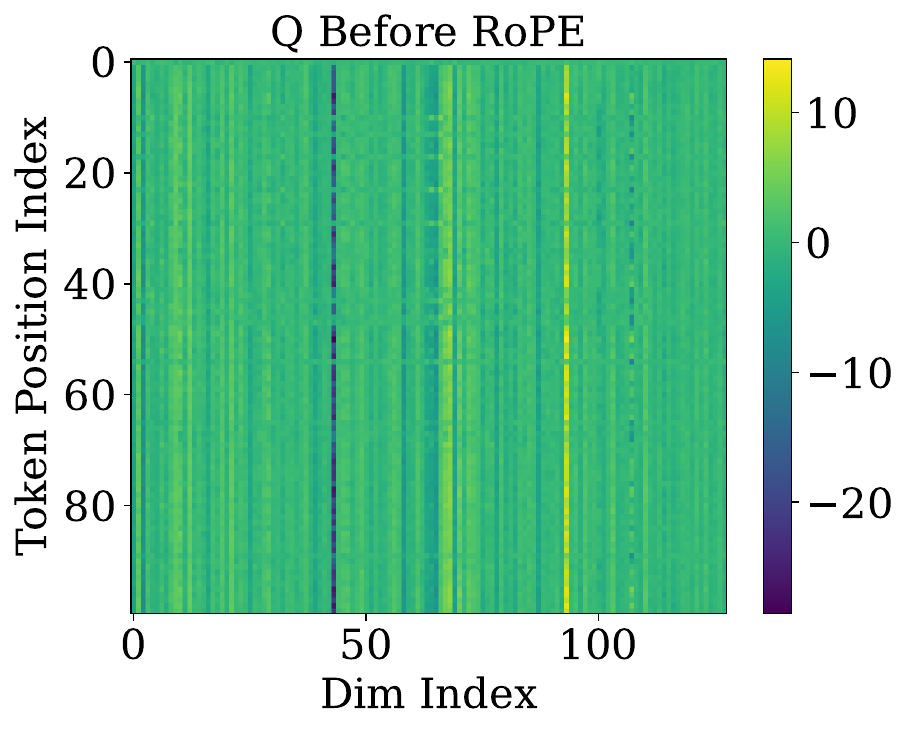}}
\hspace{0.08em}
\subfigure[Keys $K$ of $\rmL 31\rmH 14$.]{\includegraphics[width=0.23\textwidth]{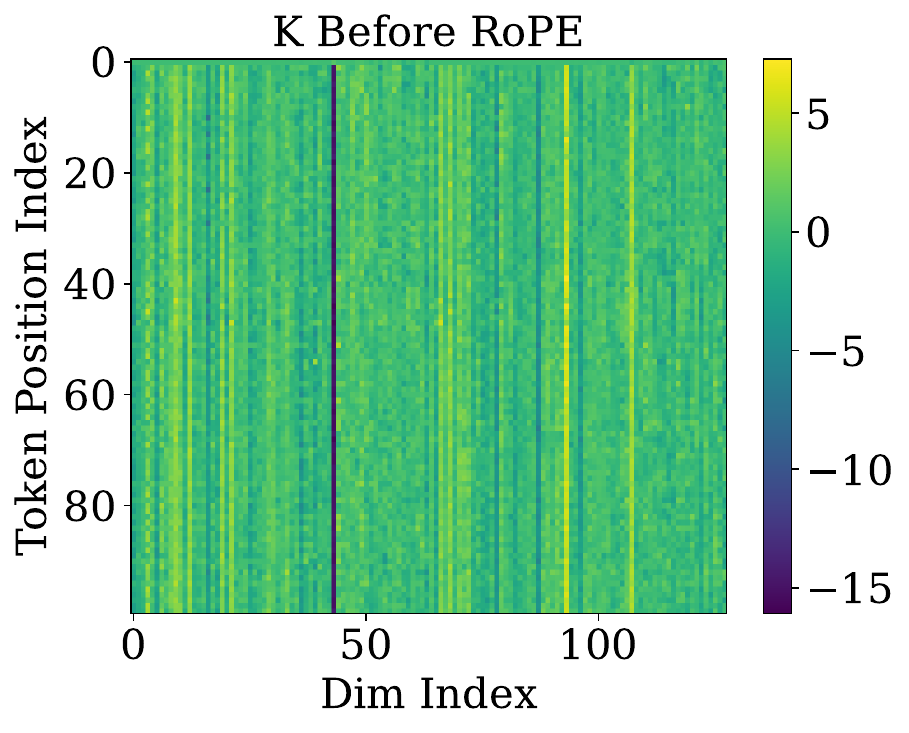}}
\hspace{0.08em}
\subfigure[$\InP(100,j,l)$ of $\rmL 31\rmH 14$.]{\includegraphics[width=0.26\textwidth]{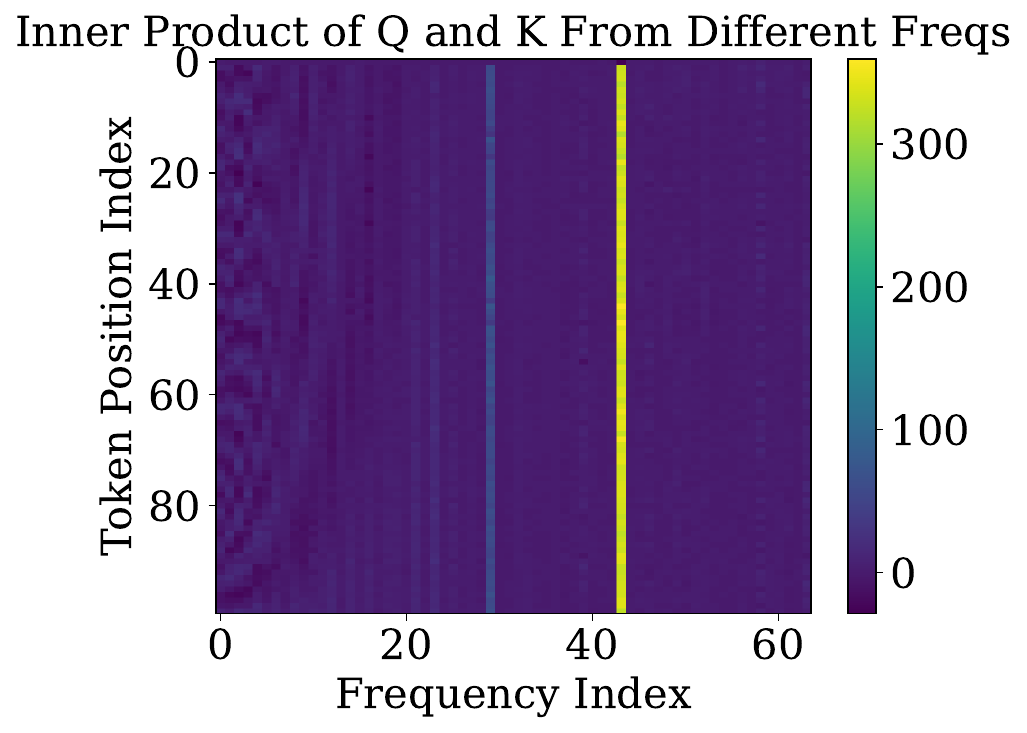}}

\subfigure[Hidden state $H$ of $\rmL 0\rmH 2$ after PCA.]{\includegraphics[width=0.22\textwidth]{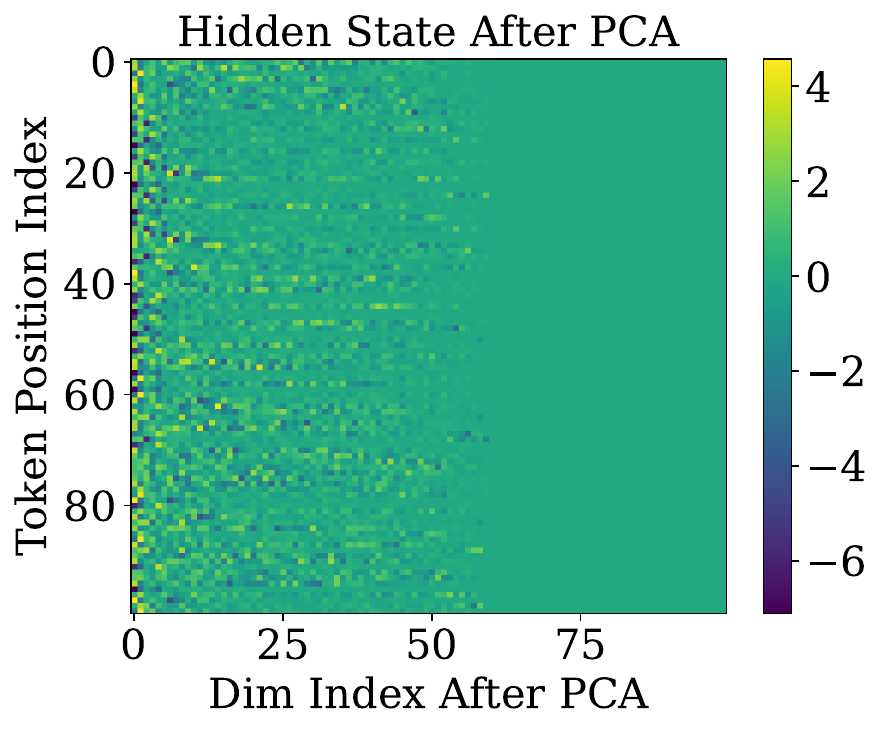}}
\hspace{0.08em}
\subfigure[Queries $Q$ of $\rmL 0\rmH 2$.]{\includegraphics[width=0.23\textwidth]{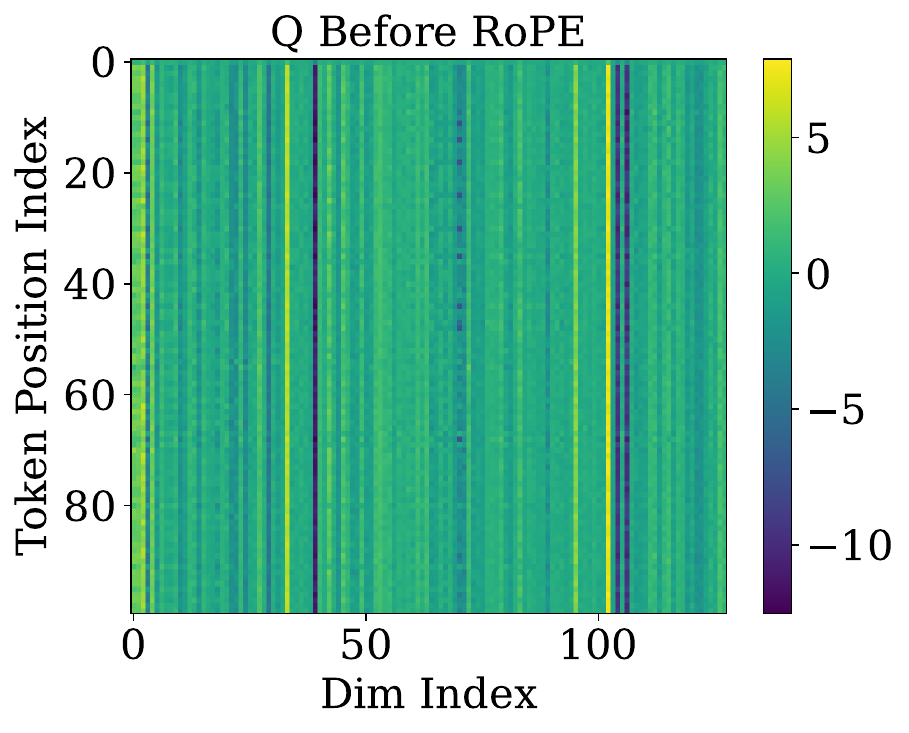}}
\hspace{0.08em}
\subfigure[Keys $K$ of $\rmL 0\rmH 2$.]{\includegraphics[width=0.23\textwidth]{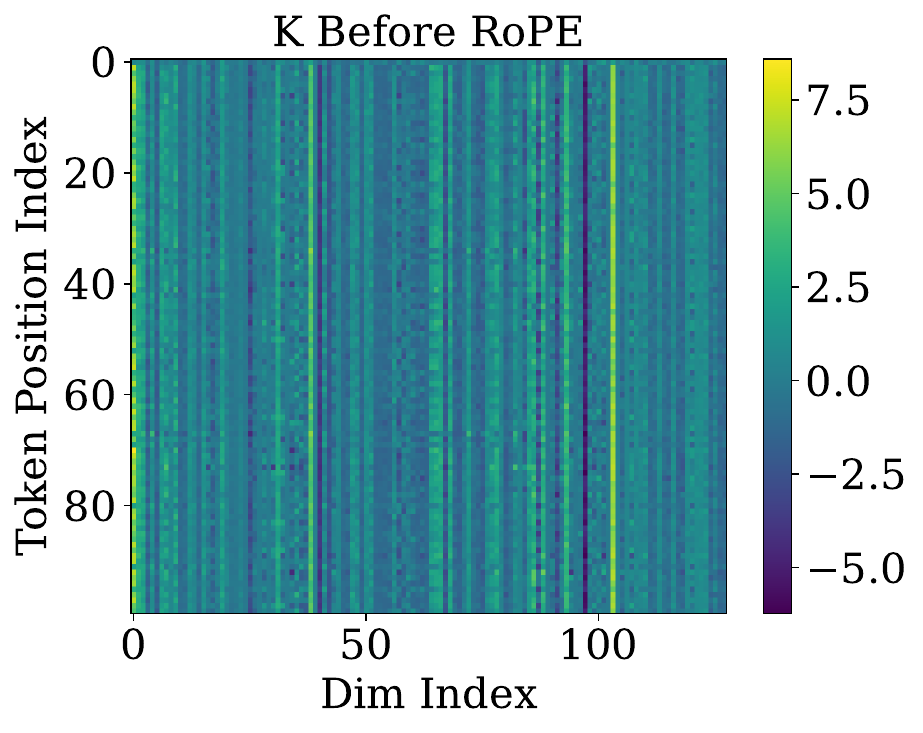}}
\hspace{0.08em}
\subfigure[$\InP(100,j,l)$ of $\rmL 0\rmH 2$.]{\includegraphics[width=0.26\textwidth]{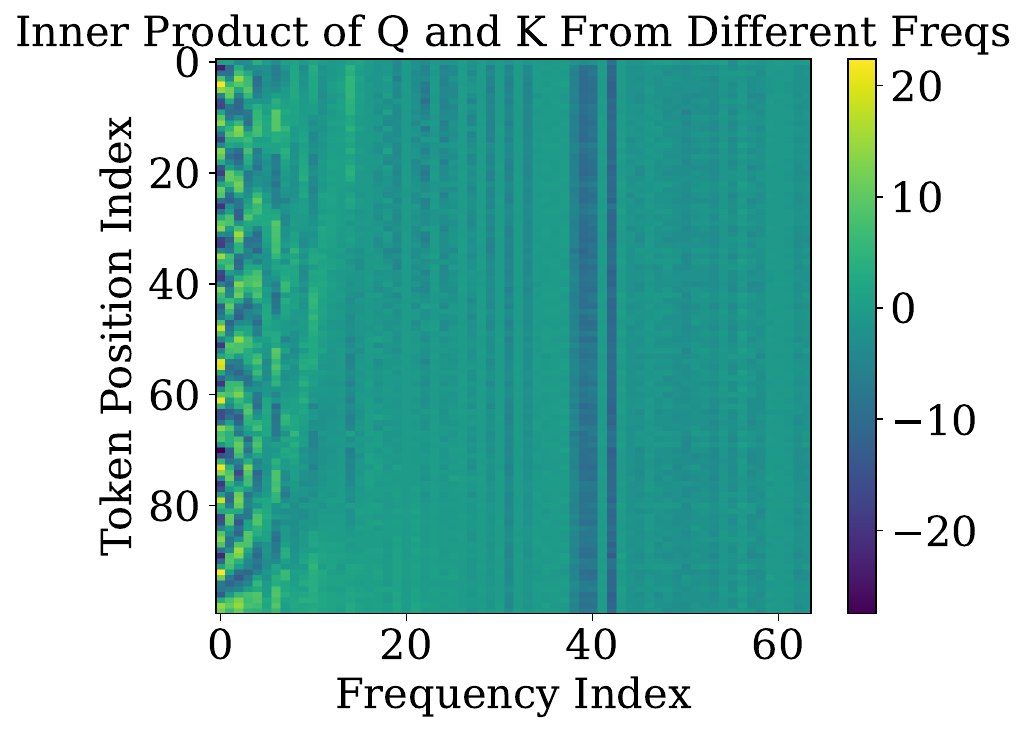}}

\subfigure[Hidden state $H$ of $\rmL 0\rmH 0$ after PCA.]{\includegraphics[width=0.22\textwidth]{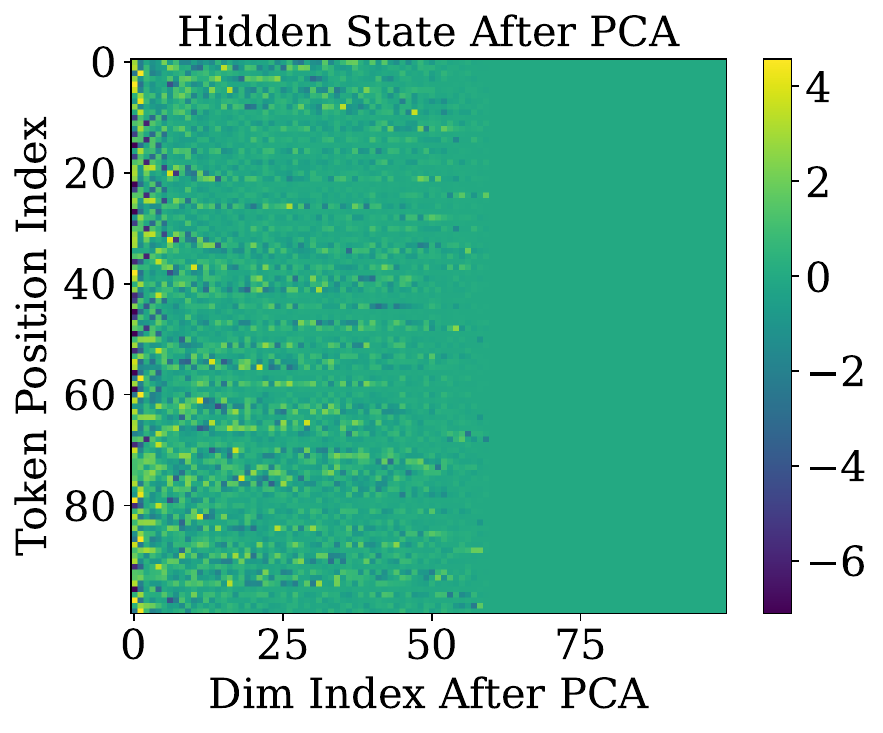}}
\hspace{0.08em}
\subfigure[Queries $Q$ of $\rmL 0\rmH 0$.]{\includegraphics[width=0.23\textwidth]{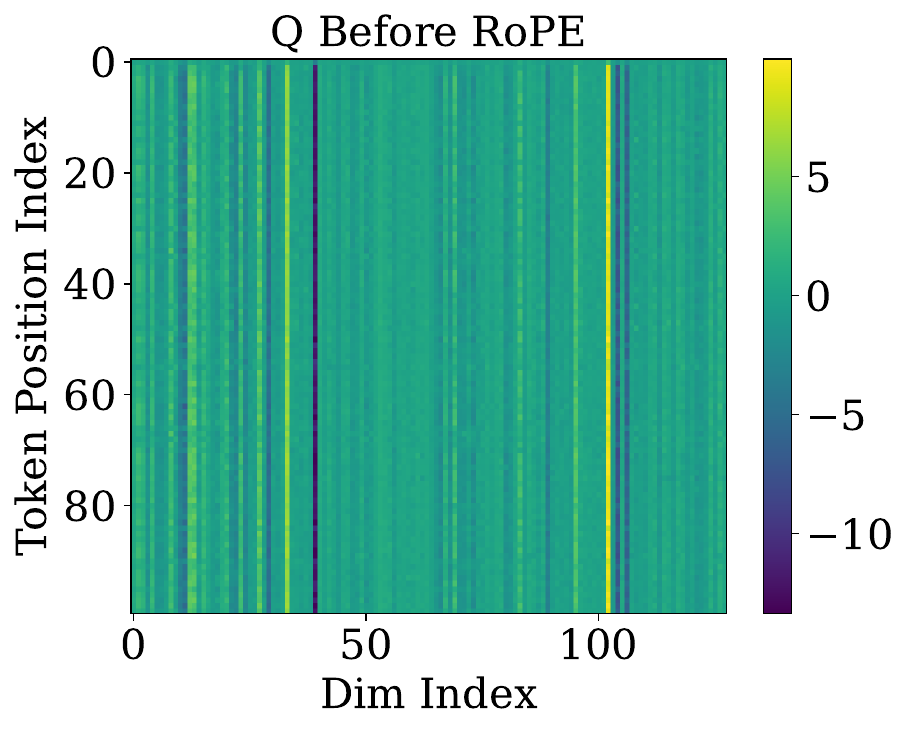}}
\hspace{0.08em}
\subfigure[Keys $K$ of $\rmL 0\rmH 0$.]{\includegraphics[width=0.23\textwidth]{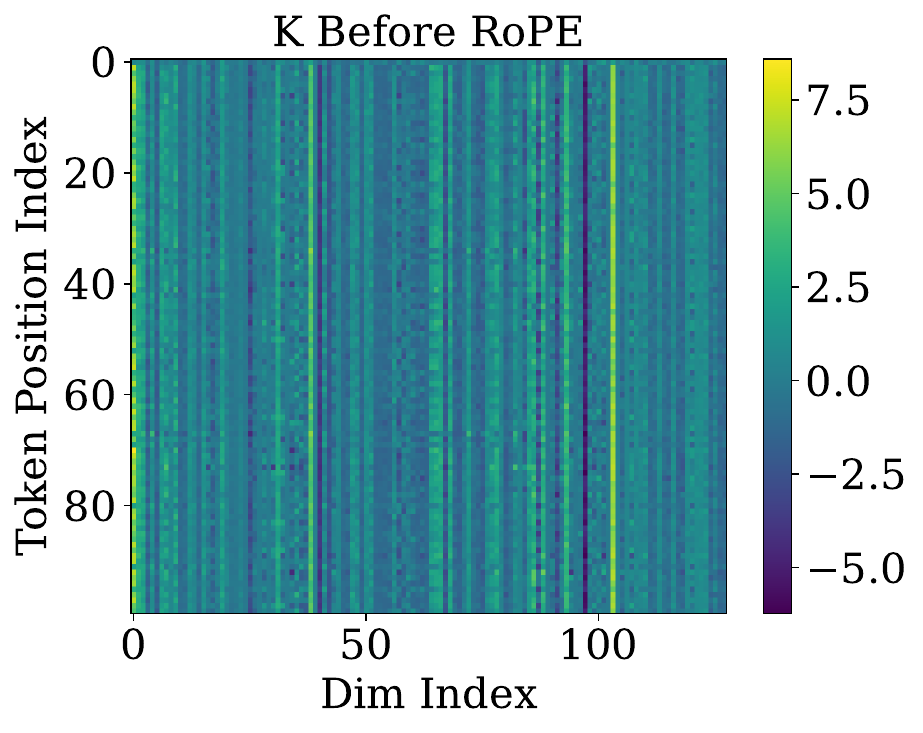}}
\hspace{0.08em}
\subfigure[$\InP(100,j,l)$ of $\rmL 0\rmH 1$.]{\includegraphics[width=0.26\textwidth]{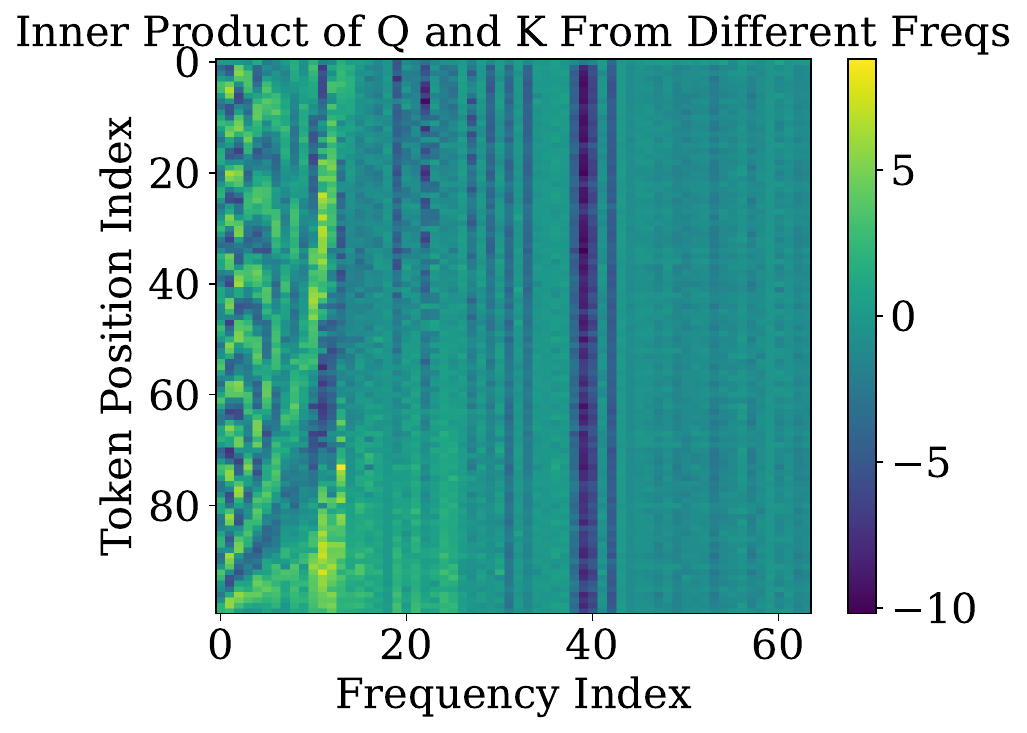}}
\caption{This figure shows the hidden states, queries, keys, and $\InP(100,j,l)$ for $j\in[100],l\in[64]$ for Llama3-8B-Instruct. Here, the example prompt contains $100$ tokens. The dimensions of queries and keys are both $128$.}
\label{fig:llama_qk}
\vspace{-1em}
\end{figure}

\subsubsection{Results of Gemma-7B}
\vspace{-1em}
\begin{figure}[H]
\centering
\subfigure[Average attention score matrix of $\rmL 0\rmH 4$.]{\includegraphics[width=0.3\textwidth]{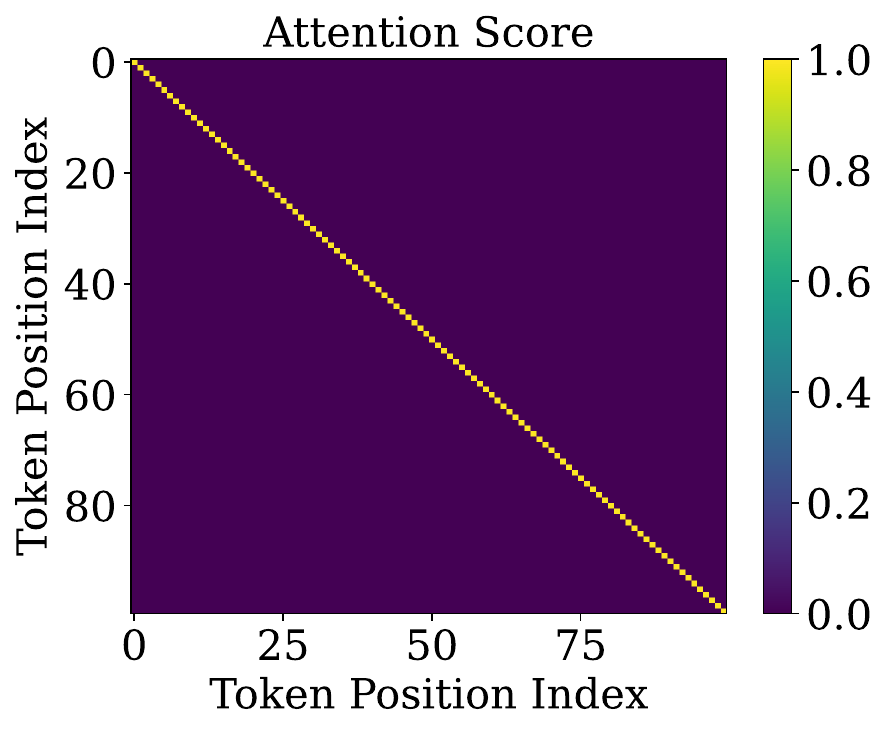}}
\quad
\subfigure[Average attention score matrix of $\rmL 0\rmH 7$.]{\includegraphics[width=0.3\textwidth]{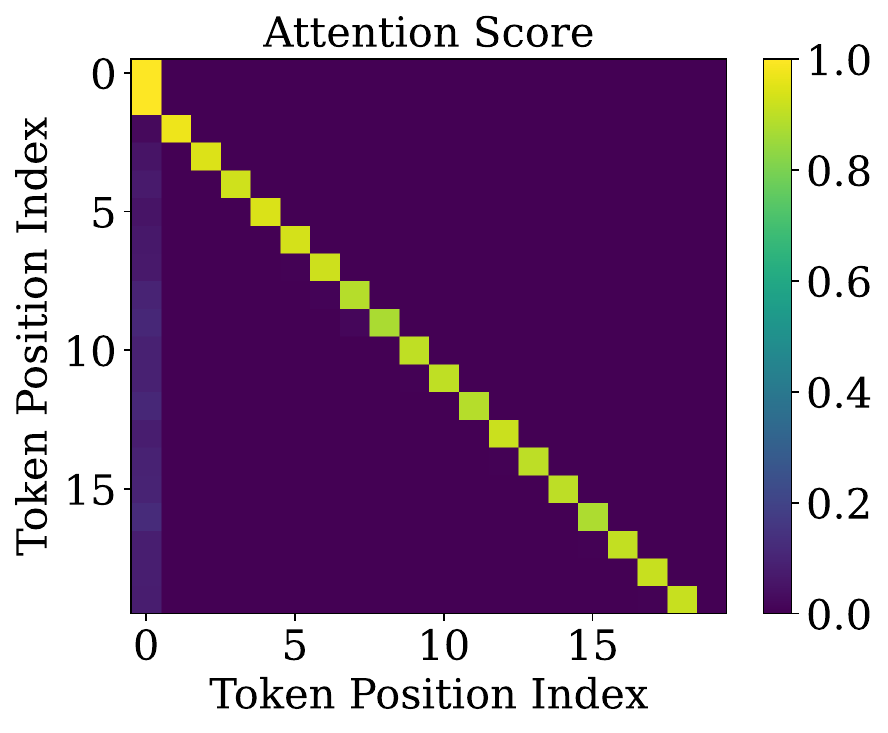}}
\quad
\subfigure[Average attention score matrix of $\rmL 0\rmH 1$.]{\includegraphics[width=0.3\textwidth]{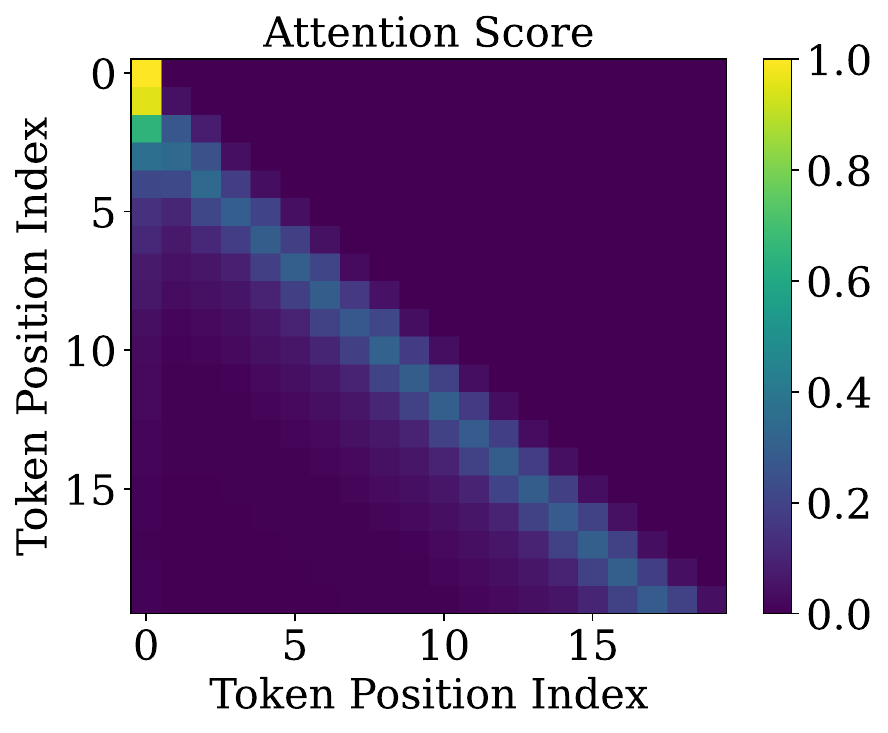}}
\caption{Average of attention score matrices in Gemma-7B with prompts from LongBench V2.}
\label{fig:gemma_small_longb}
\end{figure}

\vspace{-1.5em}

\begin{figure}[H]
\centering
\subfigure[Average attention score matrix of $\rmL 0\rmH 4$.]{\includegraphics[width=0.3\textwidth]{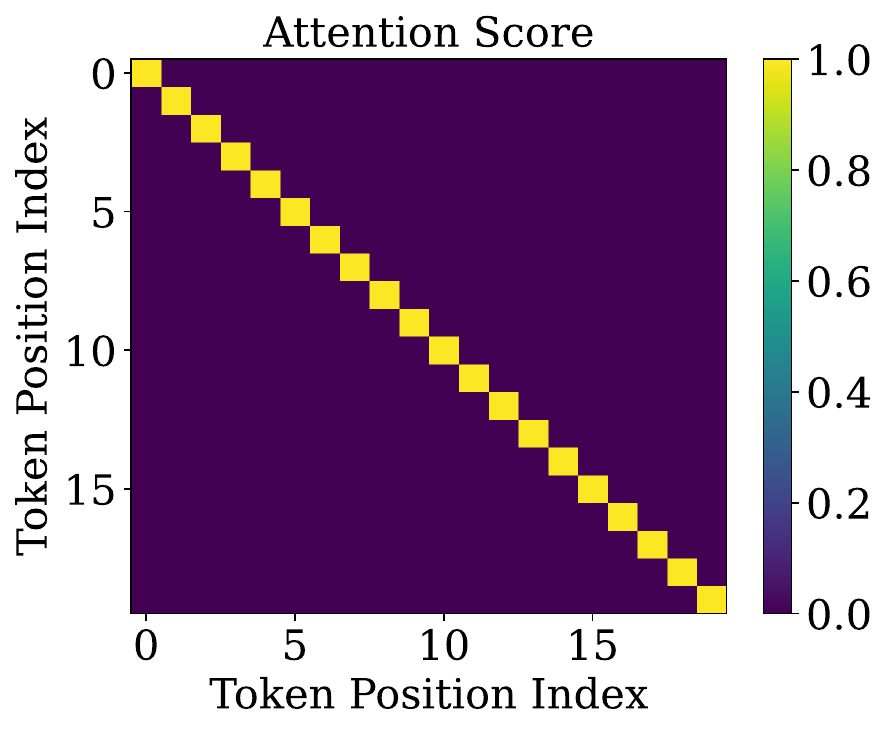}}
\quad
\subfigure[Average attention score matrix of $\rmL 0\rmH 7$.]{\includegraphics[width=0.3\textwidth]{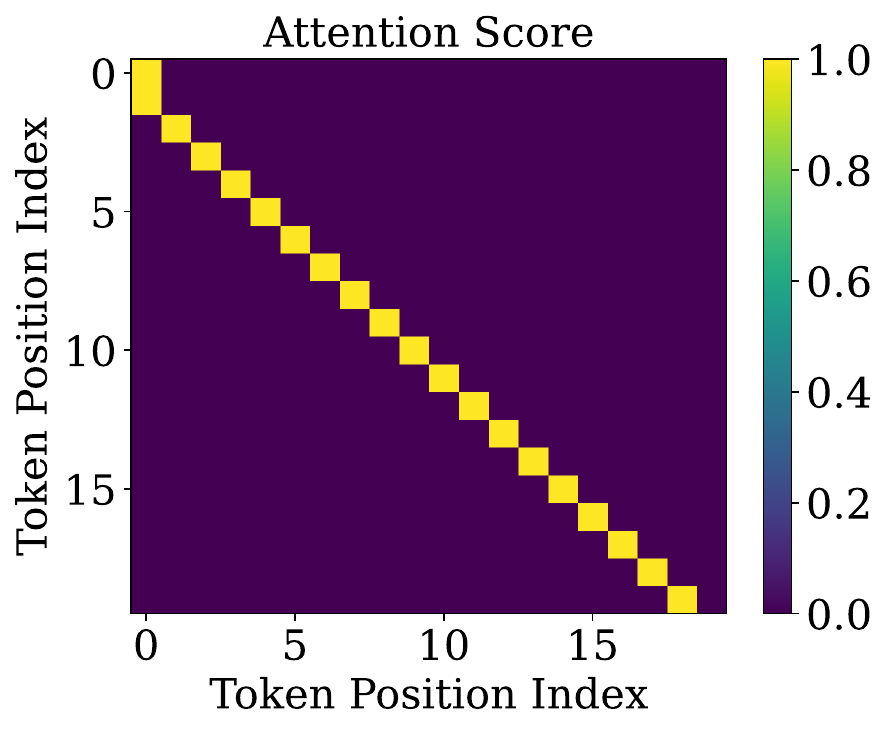}}
\quad
\subfigure[Average attention score matrix of $\rmL 0\rmH 1$.]{\includegraphics[width=0.3\textwidth]{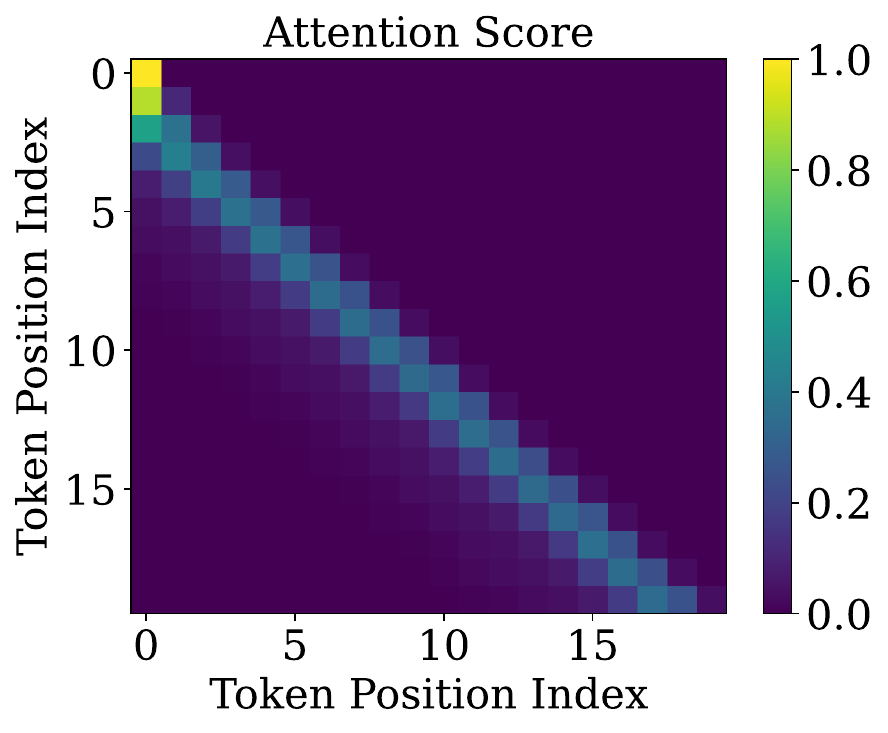}}
\caption{Average of attention score matrices in Gemma-7B with prompts whose tokens are i.i.d.\ sampled from the uniform distribution on the alphabet.}
\label{fig:gemma_small_ood}
\vspace{-1em}
\end{figure}

\begin{figure}[H]
\vspace{-4.5em}
\centering
\subfigure[Hidden state $H$ of $\rmL 0\rmH 4$ after PCA.]{\includegraphics[width=0.23\textwidth]{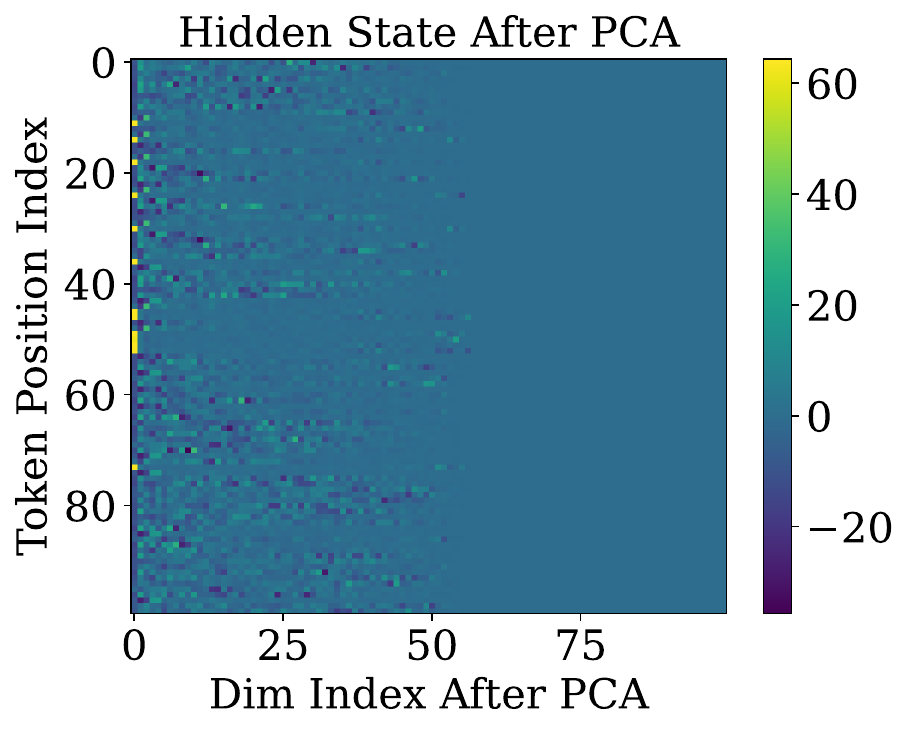}}
\hspace{0.08em}
\subfigure[Queries $Q$ of $\rmL 0\rmH 4$.]{\includegraphics[width=0.23\textwidth]{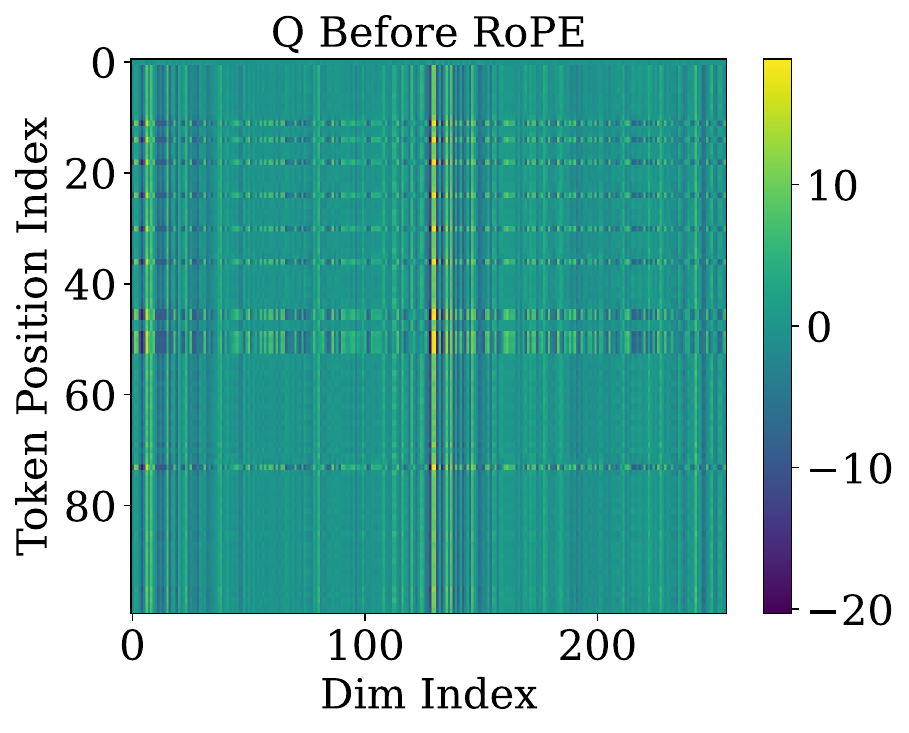}}
\hspace{0.08em}
\subfigure[Keys $K$ of $\rmL 0\rmH 4$.]{\includegraphics[width=0.23\textwidth]{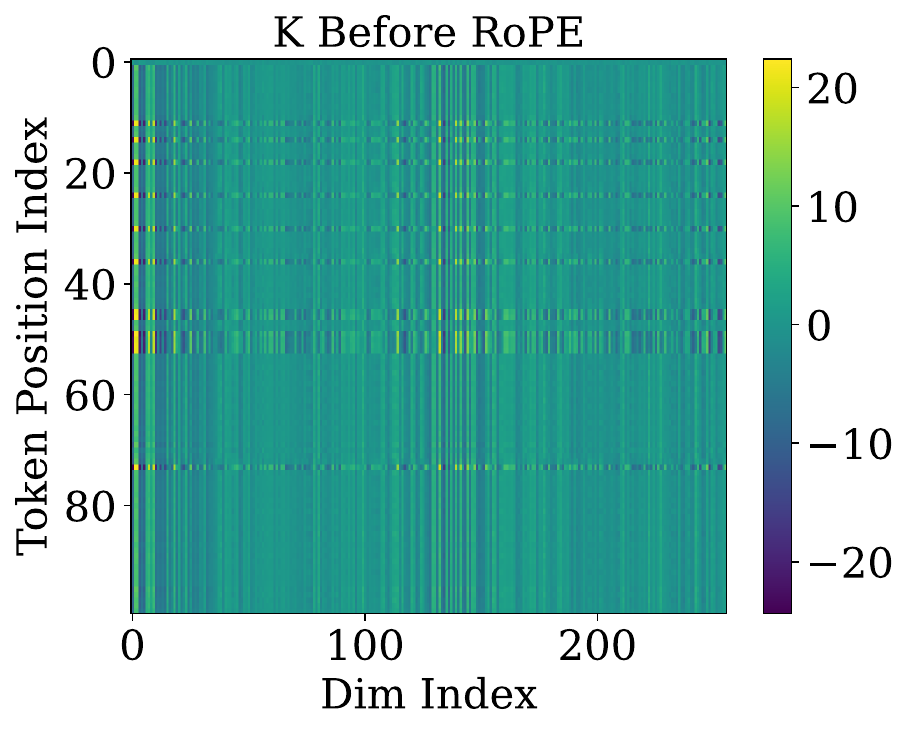}}
\hspace{0.08em}
\subfigure[$\InP(100,j,l)$ of $\rmL 0\rmH 4$.\label{fig:ip_1}]{\includegraphics[width=0.25\textwidth]{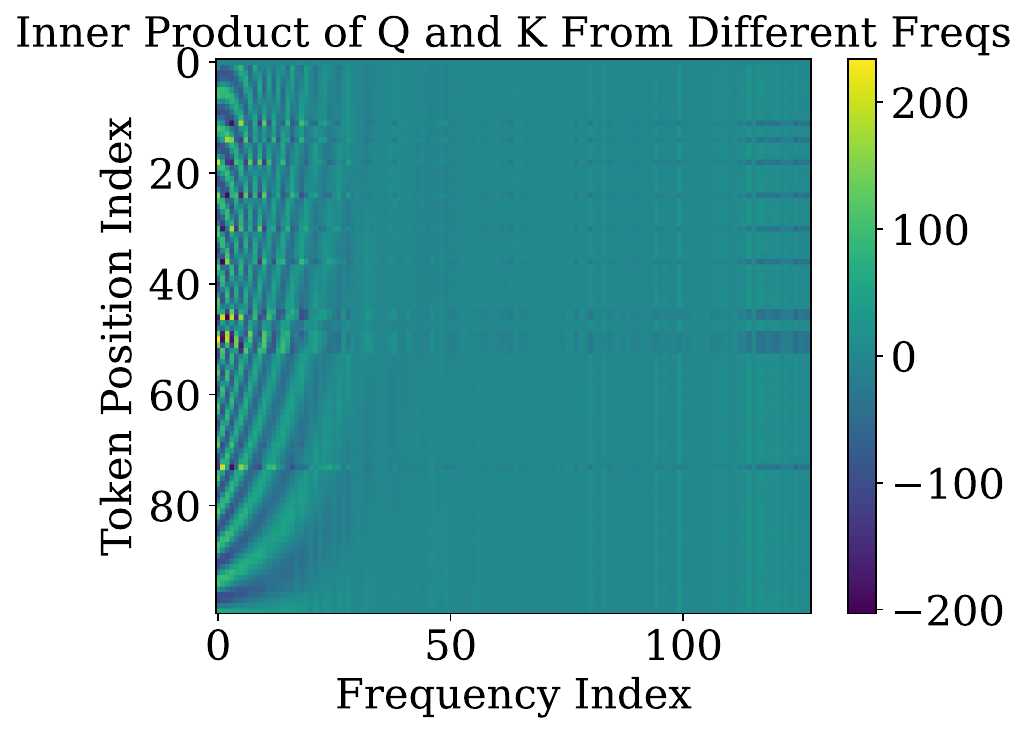}}

\subfigure[Hidden state $H$ of $\rmL 0\rmH 7$ after PCA.]{\includegraphics[width=0.23\textwidth]{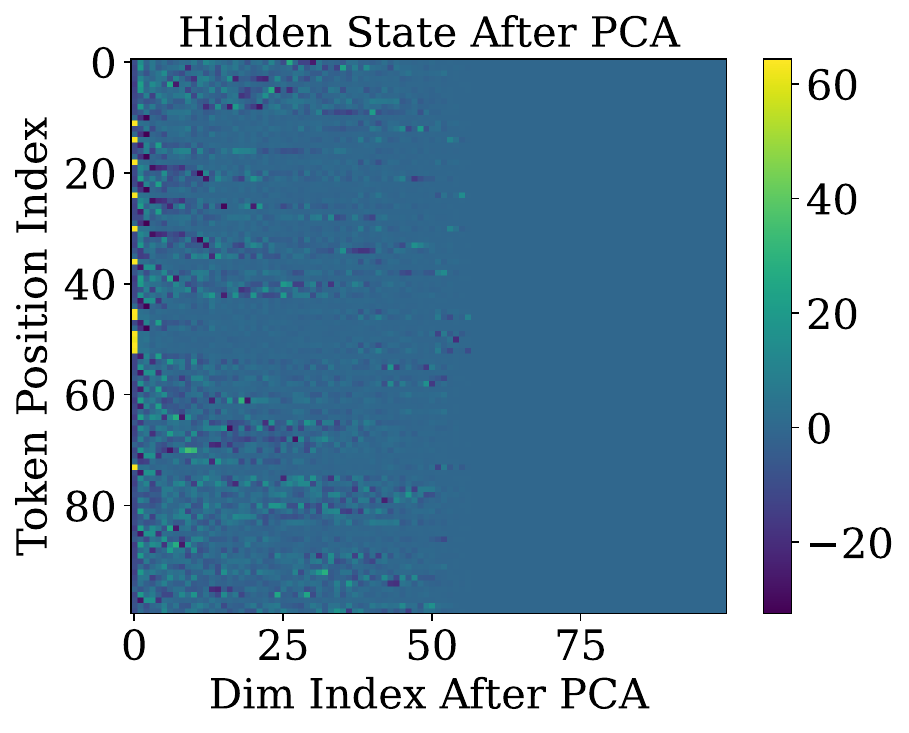}}
\hspace{0.08em}
\subfigure[Queries $Q$ of $\rmL 0\rmH 7$.]{\includegraphics[width=0.23\textwidth]{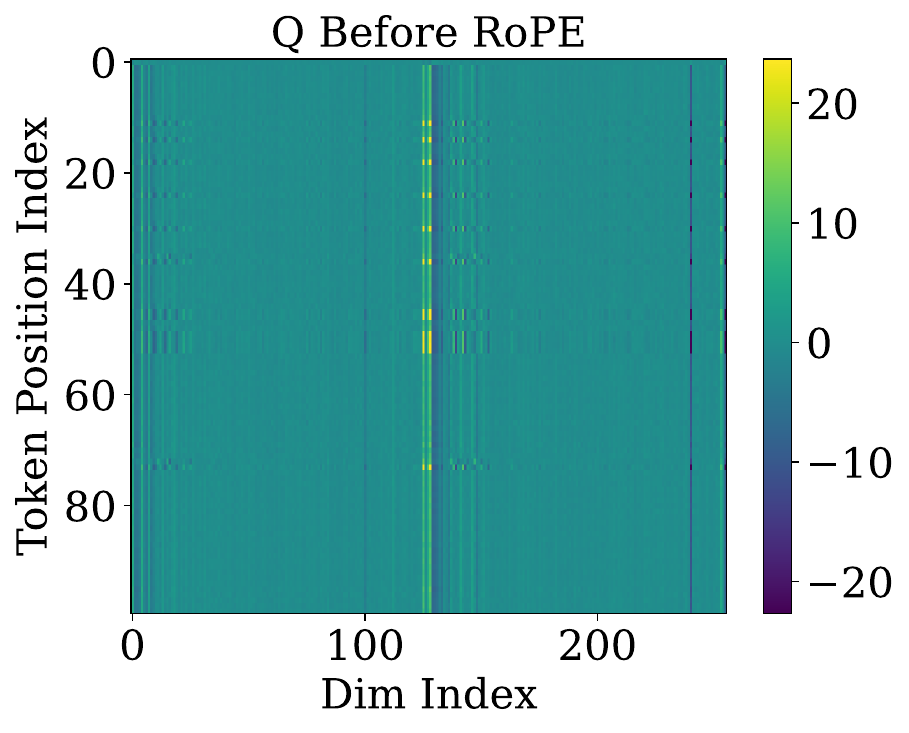}}
\hspace{0.08em}
\subfigure[Keys $K$ of $\rmL 0\rmH 7$.]{\includegraphics[width=0.23\textwidth]{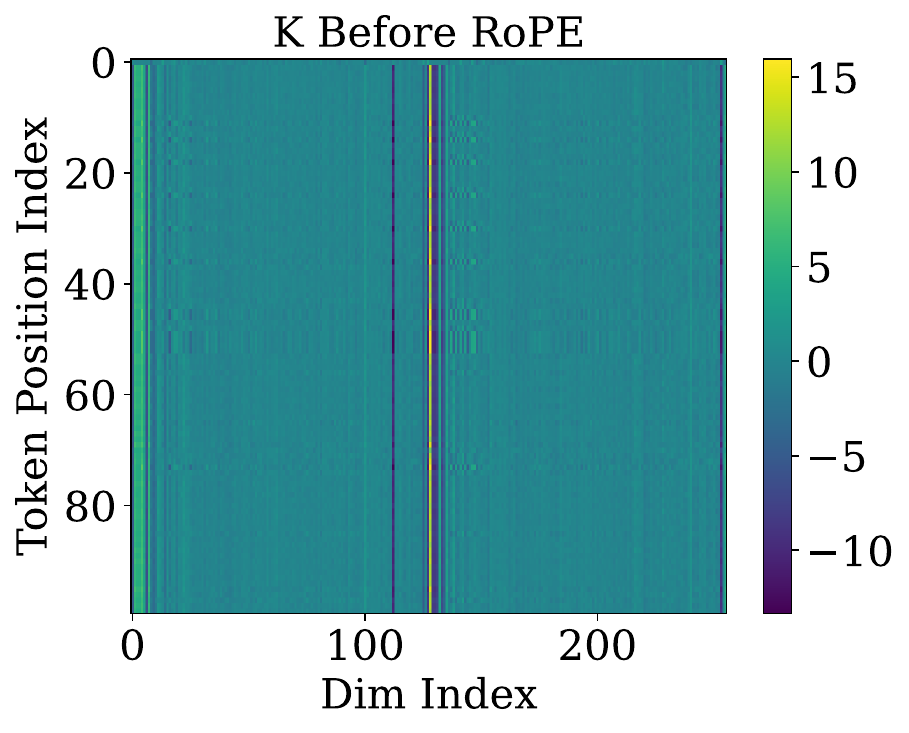}}
\hspace{0.08em}
\subfigure[$\InP(100,j,l)$ of $\rmL 0\rmH 7$.\label{fig:ip_2}]{\includegraphics[width=0.25\textwidth]{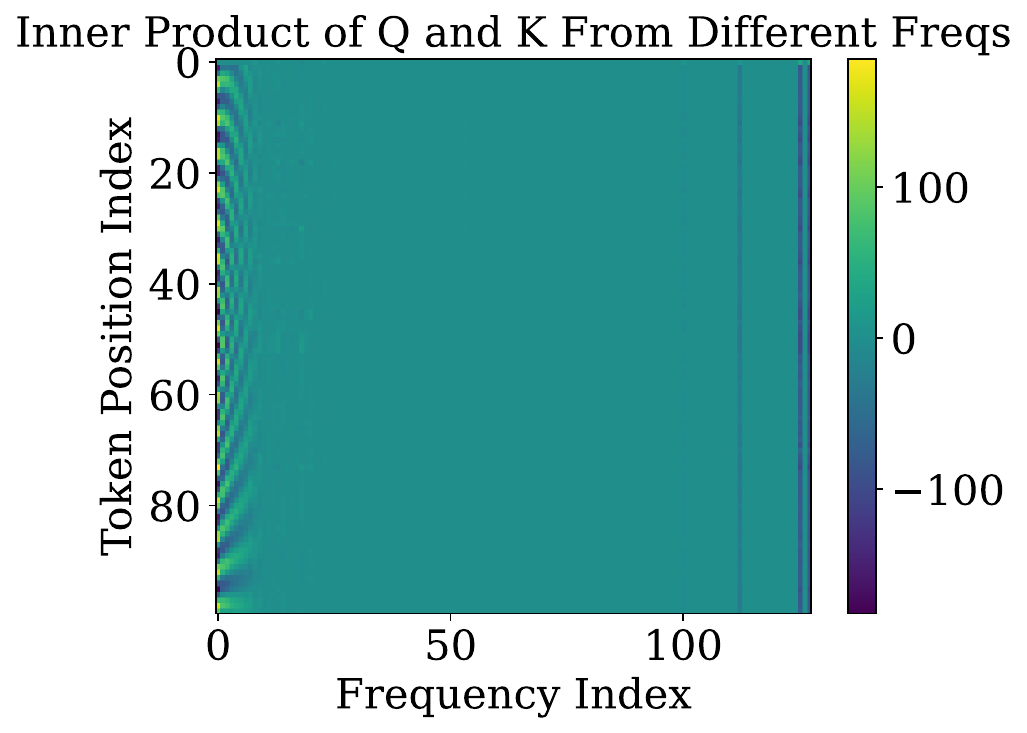}}

\subfigure[Hidden state $H$ of $\rmL 0\rmH 1$ after PCA.]{\includegraphics[width=0.23\textwidth]{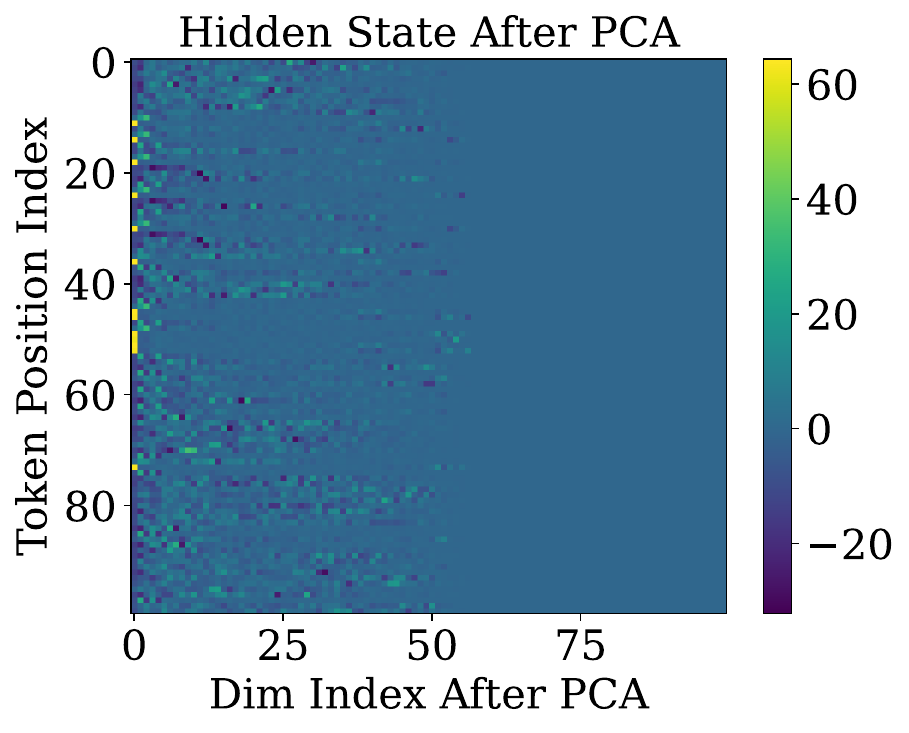}}
\hspace{0.08em}
\subfigure[Queries $Q$ of $\rmL 0\rmH 1$.]{\includegraphics[width=0.23\textwidth]{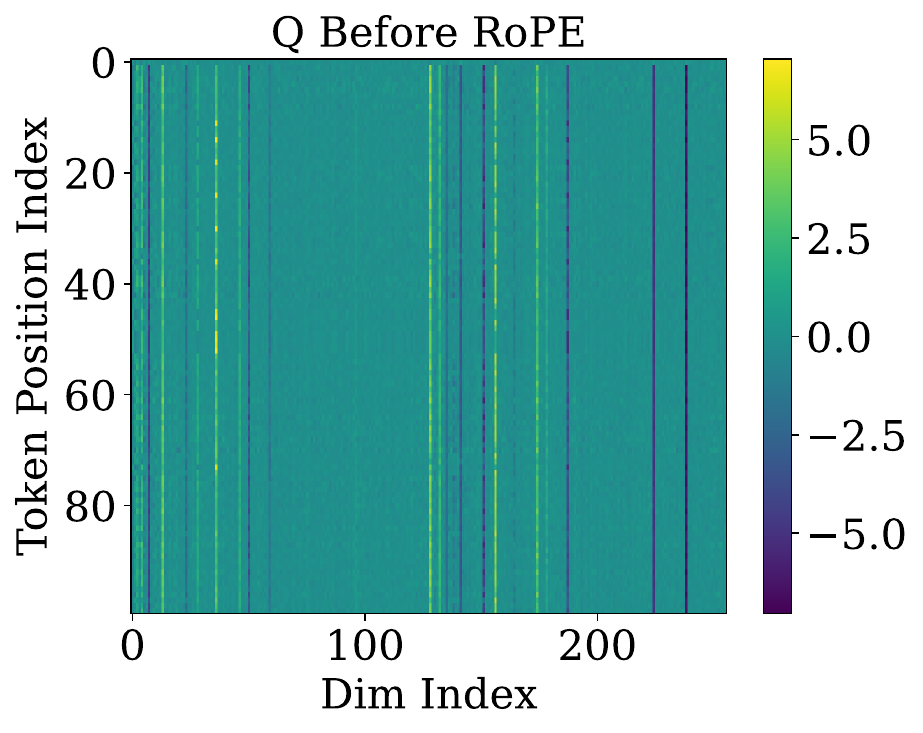}}
\hspace{0.08em}
\subfigure[Keys $K$ of $\rmL 0\rmH 1$.]{\includegraphics[width=0.23\textwidth]{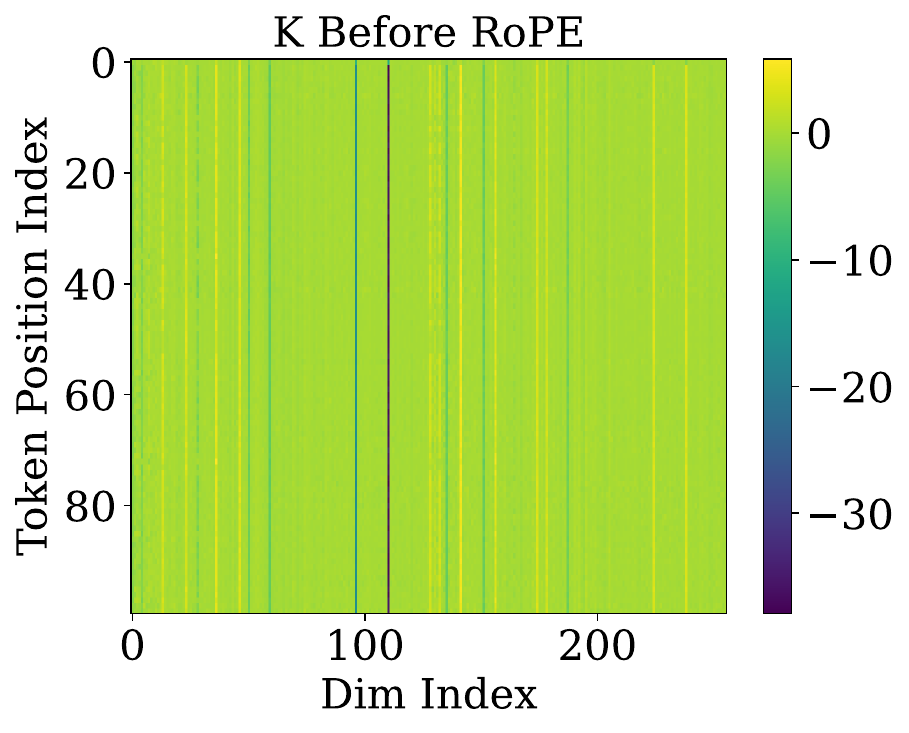}}
\hspace{0.08em}
\subfigure[$\InP(100,j,l)$ of $\rmL 0\rmH 1$.\label{fig:ip_3}]{\includegraphics[width=0.25\textwidth]{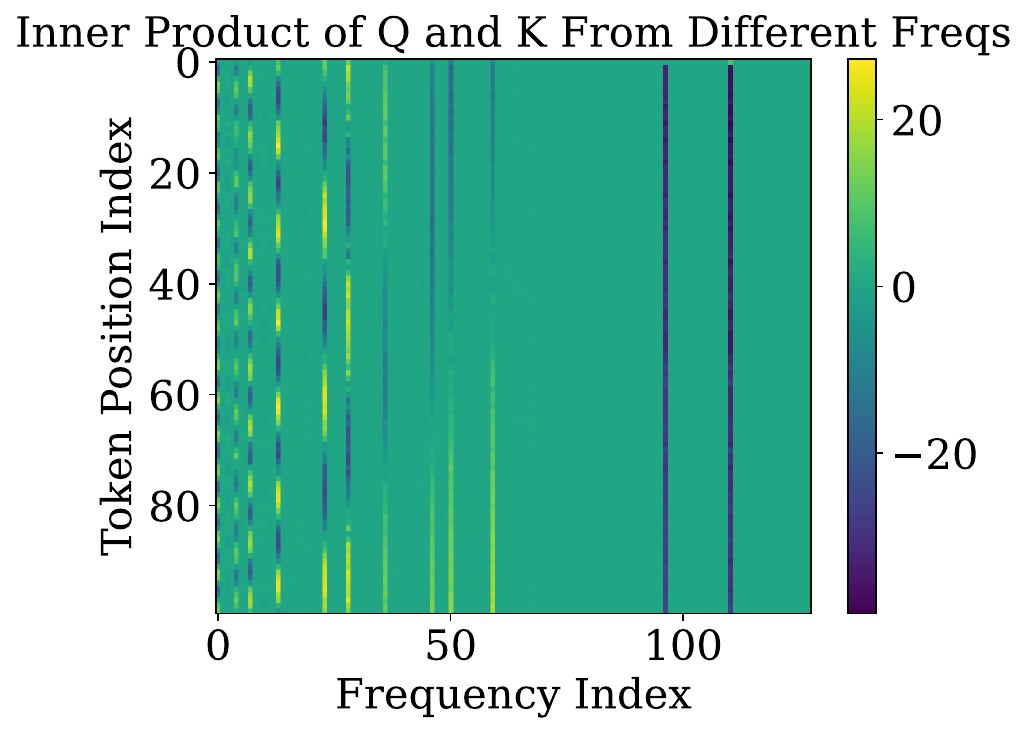}}
\vspace{-1em}
\caption{This figure shows the hidden states, queries, keys, and $\InP(100,j,l)$ for $j\in[100],l\in[128]$ for Gemma-7B. Here, the example prompt contains $100$ tokens. The dimensions of queries and keys are both $256$. We implement PCA for the Hidden state $H$ to reduce its dimension to $100$.}
\label{fig:gemma_qk}
\vspace{-1.5em}
\end{figure}

\subsection{More Results of Slash-dominant Heads with Large $\Delta$}\label{app:slarge_delta_list}
In this section, we present additional results on \acp{sdh} with large $\Delta$. For notational convenience, we slightly abuse $\bbE[\attn_{i,i-\Delta}]$ to represent the average slash score reported in the tables. In the following, we will abbreviate Llama3-8B-Instruct, and Qwen2.5-7B-Instruct as Llama3, and Qwen2.5, respectively.

\subsubsection{Full List of \acp{sdh} with Large $\Delta$}
In the following, we report each head and its corresponding $\Delta$ in the form $(\rmL a\rmH b, \Delta)$. As mentioned in Section~\ref{sec:long_range_sdh}, we set $\kappa=10^{-3}$ and do not aim to find all the slash patterns for large $\Delta$ but select a subset of them.

\noindent {\bf Llama3-8B-Instruct Model.}
\begin{itemize}
    \item [(Layer $0$)] $(\rmL0\rmH29,563)$, $(\rmL0\rmH29,564)$, $(\rmL0\rmH29,565)$, $(\rmL0\rmH29,566)$, $(\rmL0\rmH29,619)$, $(\rmL0\rmH29,620)$, $(\rmL0\rmH29,621)$, $(\rmL0\rmH29,689)$, $(\rmL0\rmH29,690)$, $(\rmL0\rmH29,849)$, $(\rmL0\rmH29,850)$, $(\rmL0\rmH29,851)$, $(\rmL0\rmH29,852)$, $(\rmL0\rmH29,853)$, $(\rmL0\rmH29,920)$, $(\rmL0\rmH29,921)$, $(\rmL0\rmH29,922)$, $(\rmL0\rmH29,934)$, $(\rmL0\rmH29,935)$, $(\rmL0\rmH29,1037)$, $(\rmL0\rmH29,1044)$, $(\rmL0\rmH29,1045)$, $(\rmL0\rmH29,1046)$, $(\rmL0\rmH29,1283)$, $(\rmL0\rmH29,1575)$, $(\rmL0\rmH29,1576)$, $(\rmL0\rmH29,2117)$, $(\rmL0\rmH29,2247)$, $(\rmL0\rmH29,2248)$, $(\rmL0\rmH29,2354)$, $(\rmL0\rmH29,2355)$, $(\rmL0\rmH29,2362)$, $(\rmL0\rmH29,2363)$, $(\rmL0\rmH29,2364)$, $(\rmL0\rmH29,2371)$, $(\rmL0\rmH29,2372)$, $(\rmL0\rmH29,2373)$, $(\rmL0\rmH29,2374)$, $(\rmL0\rmH29,2409)$, $(\rmL0\rmH29,2410)$, $(\rmL0\rmH29,2902)$, $(\rmL0\rmH29,2910)$, $(\rmL0\rmH29,2911)$, $(\rmL0\rmH29,2912)$, $(\rmL0\rmH29,3073)$, $(\rmL0\rmH29,3574)$, $(\rmL0\rmH30,599)$, $(\rmL0\rmH30,600)$, $(\rmL0\rmH30,601)$, $(\rmL0\rmH30,685)$, $(\rmL0\rmH30,686)$, $(\rmL0\rmH30,687)$, $(\rmL0\rmH30,692)$, $(\rmL0\rmH30,693)$, $(\rmL0\rmH30,1270)$, $(\rmL0\rmH30,1271)$, $(\rmL0\rmH30,1272)$, $(\rmL0\rmH30,1273)$, $(\rmL0\rmH30,2112)$, $(\rmL0\rmH30,2113)$, $(\rmL0\rmH30,2114)$, $(\rmL0\rmH30,2899)$, $(\rmL0\rmH30,2900)$, $(\rmL0\rmH30,3909)$, $(\rmL0\rmH30,3910)$, $(\rmL0\rmH30,3911)$, $(\rmL0\rmH30,3912)$, $(\rmL0\rmH30,3913)$, $(\rmL0\rmH30,3914)$, $(\rmL0\rmH30,3915)$.
    \item [(Layer $1$)] $(\rmL1\rmH16,3827)$, $(\rmL1\rmH18,3827)$.
\end{itemize}

\noindent {\bf Qwen2.5-7B-Instruct Model.}
\begin{itemize}
    \item [(Layer $0$)] $(\rmL0\rmH0,683)$, $(\rmL0\rmH0,684)$, $(\rmL0\rmH0,731)$, $(\rmL0\rmH0,732)$, $(\rmL0\rmH0,733)$, $(\rmL0\rmH0,740)$, $(\rmL0\rmH0,741)$, $(\rmL0\rmH0,1487)$, $(\rmL0\rmH0,1488)$, $(\rmL0\rmH1,673)$, $(\rmL0\rmH2,657)$, $(\rmL0\rmH2,658)$, $(\rmL0\rmH2,659)$, $(\rmL0\rmH2,660)$, $(\rmL0\rmH2,731)$, $(\rmL0\rmH2,732)$, $(\rmL0\rmH2,733)$, $(\rmL0\rmH2,734)$, $(\rmL0\rmH2,739)$, $(\rmL0\rmH2,740)$, $(\rmL0\rmH2,741)$, $(\rmL0\rmH2,742)$, $(\rmL0\rmH2,743)$, $(\rmL0\rmH2,744)$, $(\rmL0\rmH2,745)$, $(\rmL0\rmH2,746)$, $(\rmL0\rmH2,753)$, $(\rmL0\rmH2,754)$, $(\rmL0\rmH2,755)$, $(\rmL0\rmH7,912)$, $(\rmL0\rmH7,924)$, $(\rmL0\rmH7,925)$, $(\rmL0\rmH7,936)$, $(\rmL0\rmH7,937)$, $(\rmL0\rmH7,938)$, $(\rmL0\rmH7,939)$, $(\rmL0\rmH8,922)$, $(\rmL0\rmH8,923)$, $(\rmL0\rmH8,924)$, $(\rmL0\rmH8,925)$, $(\rmL0\rmH8,935)$, $(\rmL0\rmH8,936)$, $(\rmL0\rmH8,937)$, $(\rmL0\rmH8,938)$, $(\rmL0\rmH8,939)$, $(\rmL0\rmH8,940)$, $(\rmL0\rmH8,1030)$, $(\rmL0\rmH8,1031)$, $(\rmL0\rmH9,912)$, $(\rmL0\rmH9,913)$, $(\rmL0\rmH9,914)$, $(\rmL0\rmH9,915)$, $(\rmL0\rmH9,916)$, $(\rmL0\rmH9,917)$, $(\rmL0\rmH9,918)$, $(\rmL0\rmH9,919)$, $(\rmL0\rmH9,920)$, $(\rmL0\rmH9,921)$, $(\rmL0\rmH9,922)$, $(\rmL0\rmH9,923)$, $(\rmL0\rmH9,924)$, $(\rmL0\rmH9,925)$, $(\rmL0\rmH9,926)$, $(\rmL0\rmH9,927)$, $(\rmL0\rmH9,928)$, $(\rmL0\rmH9,929)$, $(\rmL0\rmH9,930)$, $(\rmL0\rmH9,936)$, $(\rmL0\rmH9,937)$, $(\rmL0\rmH11,923)$, $(\rmL0\rmH11,924)$, $(\rmL0\rmH11,925)$, $(\rmL0\rmH11,926)$, $(\rmL0\rmH11,927)$, $(\rmL0\rmH11,928)$, $(\rmL0\rmH11,929)$, $(\rmL0\rmH11,930)$, $(\rmL0\rmH11,937)$, $(\rmL0\rmH11,938)$, $(\rmL0\rmH11,939)$, $(\rmL0\rmH11,940)$, $(\rmL0\rmH11,941)$, $(\rmL0\rmH15,917)$, $(\rmL0\rmH15,918)$, $(\rmL0\rmH15,936)$, $(\rmL0\rmH15,1513)$, $(\rmL0\rmH15,1514)$, $(\rmL0\rmH15,1515)$, $(\rmL0\rmH15,1531)$, $(\rmL0\rmH15,1532)$, $(\rmL0\rmH15,1533)$, $(\rmL0\rmH15,1534)$, $(\rmL0\rmH15,1535)$, $(\rmL0\rmH15,3094)$, $(\rmL0\rmH15,3095)$, $(\rmL0\rmH15,3096)$, $(\rmL0\rmH15,3097)$, $(\rmL0\rmH15,3098)$, $(\rmL0\rmH15,3140)$, $(\rmL0\rmH15,3141)$, $(\rmL0\rmH15,3142)$, $(\rmL0\rmH15,3143)$, $(\rmL0\rmH15,3303)$, $(\rmL0\rmH15,3304)$, $(\rmL0\rmH15,3305)$, $(\rmL0\rmH15,3306)$, $(\rmL0\rmH15,3307)$, $(\rmL0\rmH15,3323)$, $(\rmL0\rmH15,3324)$, $(\rmL0\rmH15,3362)$, $(\rmL0\rmH15,3363)$, $(\rmL0\rmH15,3364)$, $(\rmL0\rmH15,3365)$, $(\rmL0\rmH15,3366)$, $(\rmL0\rmH15,3367)$, $(\rmL0\rmH15,3368)$, $(\rmL0\rmH15,3412)$, $(\rmL0\rmH15,3413)$, $(\rmL0\rmH15,3586)$, $(\rmL0\rmH15,3587)$, $(\rmL0\rmH15,3588)$, $(\rmL0\rmH15,3617)$, $(\rmL0\rmH15,3618)$, $(\rmL0\rmH15,3619)$, $(\rmL0\rmH15,3620)$, $(\rmL0\rmH15,3644)$, $(\rmL0\rmH15,3645)$, $(\rmL0\rmH15,3646)$, $(\rmL0\rmH15,3647)$, $(\rmL0\rmH15,3648)$, $(\rmL0\rmH15,3693)$, $(\rmL0\rmH15,3694)$, $(\rmL0\rmH15,3695)$, $(\rmL0\rmH16,1521)$, $(\rmL0\rmH16,1538)$, $(\rmL0\rmH16,1539)$, $(\rmL0\rmH16,1540)$,  $(\rmL0\rmH23,500)$, $(\rmL0\rmH23,501)$, $(\rmL0\rmH23,502)$, $(\rmL0\rmH23,503)$, $(\rmL0\rmH23,504)$, $(\rmL0\rmH23,505)$, $(\rmL0\rmH24,500)$, $(\rmL0\rmH24,501)$, $(\rmL0\rmH24,502)$, $(\rmL0\rmH24,503)$, $(\rmL0\rmH24,504)$, $(\rmL0\rmH24,505)$, $(\rmL0\rmH25,500)$, $(\rmL0\rmH25,501)$, $(\rmL0\rmH25,503)$, $(\rmL0\rmH25,504)$, $(\rmL0\rmH25,505)$, $(\rmL0\rmH25,506)$, $(\rmL0\rmH25,507)$, $(\rmL0\rmH25,508)$, $(\rmL0\rmH25,509)$, $(\rmL0\rmH25,510)$, $(\rmL0\rmH25,511)$, $(\rmL0\rmH26,500)$, $(\rmL0\rmH26,501)$, $(\rmL0\rmH26,502)$, $(\rmL0\rmH26,503)$, $(\rmL0\rmH26,504)$, $(\rmL0\rmH26,505)$.
    
    \item [(Layer $1$)] $(\rmL1\rmH6,502)$, $(\rmL1\rmH6,503)$, $(\rmL1\rmH6,504)$, $(\rmL1\rmH6,505)$, $(\rmL1\rmH6,506)$, $(\rmL1\rmH6,507)$, $(\rmL1\rmH6,508)$, $(\rmL1\rmH6,514)$, $(\rmL1\rmH6,515)$, $(\rmL1\rmH6,516)$, $(\rmL1\rmH6,517)$, $(\rmL1\rmH6,520)$, $(\rmL1\rmH6,521)$, $(\rmL1\rmH6,522)$, $(\rmL1\rmH6,523)$, $(\rmL1\rmH6,524)$, $(\rmL1\rmH6,525)$, $(\rmL1\rmH6,526)$, $(\rmL1\rmH6,527)$, $(\rmL1\rmH6,528)$, $(\rmL1\rmH6,529)$, $(\rmL1\rmH6,533)$, $(\rmL1\rmH6,534)$, $(\rmL1\rmH6,535)$, $(\rmL1\rmH6,536)$, $(\rmL1\rmH6,537)$, $(\rmL1\rmH6,538)$, $(\rmL1\rmH6,539)$, $(\rmL1\rmH6,540)$, $(\rmL1\rmH6,541)$, $(\rmL1\rmH6,542)$, $(\rmL1\rmH6,543)$, $(\rmL1\rmH6,553)$, $(\rmL1\rmH11,553)$, $(\rmL1\rmH11,554)$, $(\rmL1\rmH11,609)$, $(\rmL1\rmH11,610)$, $(\rmL1\rmH11,641)$, $(\rmL1\rmH11,710)$, $(\rmL1\rmH11,717)$, $(\rmL1\rmH11,718)$, $(\rmL1\rmH11,804)$, $(\rmL1\rmH11,805)$, $(\rmL1\rmH11,835)$, $(\rmL1\rmH11,1239)$, $(\rmL1\rmH11,1521)$, $(\rmL1\rmH11,3532)$, $(\rmL1\rmH12,546)$, $(\rmL1\rmH12,602)$, $(\rmL1\rmH12,640)$, $(\rmL1\rmH12,641)$, $(\rmL1\rmH13,640)$, $(\rmL1\rmH13,641)$, $(\rmL1\rmH13,647)$, $(\rmL1\rmH13,648)$, $(\rmL1\rmH15,516)$, $(\rmL1\rmH15,523)$, $(\rmL1\rmH15,554)$, $(\rmL1\rmH15,2294)$, $(\rmL1\rmH16,541)$, $(\rmL1\rmH16,554)$, $(\rmL1\rmH16,555)$, $(\rmL1\rmH16,556)$, $(\rmL1\rmH16,611)$, $(\rmL1\rmH16,671)$, $(\rmL1\rmH16,672)$, $(\rmL1\rmH16,673)$, $(\rmL1\rmH16,2778)$, $(\rmL1\rmH19,553)$, $(\rmL1\rmH19,554)$, $(\rmL1\rmH19,610)$, $(\rmL1\rmH19,611)$, $(\rmL1\rmH19,648)$, $(\rmL1\rmH20,554)$, $(\rmL1\rmH23,553)$, $(\rmL1\rmH23,554)$, $(\rmL1\rmH23,678)$, $(\rmL1\rmH23,679)$, $(\rmL1\rmH23,685)$, $(\rmL1\rmH23,686)$, $(\rmL1\rmH23,687)$, $(\rmL1\rmH23,709)$, $(\rmL1\rmH23,710)$, $(\rmL1\rmH23,717)$, $(\rmL1\rmH23,718)$, $(\rmL1\rmH23,998)$, $(\rmL1\rmH23,999)$, $(\rmL1\rmH23,1036)$, $(\rmL1\rmH23,1037)$, $(\rmL1\rmH23,1162)$, $(\rmL1\rmH23,1163)$, $(\rmL1\rmH23,1201)$, $(\rmL1\rmH23,1848)$, $(\rmL1\rmH23,1849)$, $(\rmL1\rmH23,2129)$, $(\rmL1\rmH23,2130)$, $(\rmL1\rmH23,2293)$, $(\rmL1\rmH23,2294)$, $(\rmL1\rmH23,2295)$, $(\rmL1\rmH23,2613)$, $(\rmL1\rmH23,2620)$, $(\rmL1\rmH23,2621)$, $(\rmL1\rmH23,2622)$, $(\rmL1\rmH23,2777)$, $(\rmL1\rmH23,2932)$, $(\rmL1\rmH23,2933)$, $(\rmL1\rmH23,2934)$, $(\rmL1\rmH23,3097)$, $(\rmL1\rmH23,3438)$, $(\rmL1\rmH24,671)$, $(\rmL1\rmH24,834)$, $(\rmL1\rmH24,835)$, $(\rmL1\rmH24,836)$, $(\rmL1\rmH24,1145)$, $(\rmL1\rmH24,1146)$, $(\rmL1\rmH24,1147)$, $(\rmL1\rmH24,1148)$, $(\rmL1\rmH24,1149)$, $(\rmL1\rmH24,2238)$, $(\rmL1\rmH24,2683)$, $(\rmL1\rmH24,2684)$, $(\rmL1\rmH24,2721)$, $(\rmL1\rmH24,3368)$, $(\rmL1\rmH24,3369)$.
    
    \item [(Layer $2$)] $(\rmL2\rmH6,514)$, $(\rmL2\rmH6,515)$, $(\rmL2\rmH6,741)$, $(\rmL2\rmH6,936)$, $(\rmL2\rmH6,998)$, $(\rmL2\rmH6,1933)$, $(\rmL2\rmH6,1934)$, $(\rmL2\rmH6,1935)$, $(\rmL2\rmH6,2159)$, $(\rmL2\rmH6,2160)$, $(\rmL2\rmH6,2161)$, $(\rmL2\rmH6,2185)$, $(\rmL2\rmH6,2186)$, $(\rmL2\rmH6,2198)$, $(\rmL2\rmH6,2273)$, $(\rmL2\rmH6,2274)$, $(\rmL2\rmH6,2275)$, $(\rmL2\rmH6,2362)$, $(\rmL2\rmH6,2363)$, $(\rmL2\rmH6,2758)$, $(\rmL2\rmH6,2845)$, $(\rmL2\rmH6,2846)$, $(\rmL2\rmH6,3695)$, $(\rmL2\rmH7,1277)$, $(\rmL2\rmH7,1278)$, $(\rmL2\rmH7,1551)$, $(\rmL2\rmH7,1552)$, $(\rmL2\rmH7,1553)$, $(\rmL2\rmH7,1554)$, $(\rmL2\rmH7,1557)$, $(\rmL2\rmH7,1558)$, $(\rmL2\rmH7,1559)$, $(\rmL2\rmH7,1560)$, $(\rmL2\rmH7,1561)$, $(\rmL2\rmH7,2439)$, $(\rmL2\rmH7,2440)$, $(\rmL2\rmH7,2721)$, $(\rmL2\rmH8,1275)$, $(\rmL2\rmH8,1276)$, $(\rmL2\rmH8,1277)$, $(\rmL2\rmH8,1278)$, $(\rmL2\rmH8,1279)$, $(\rmL2\rmH8,1551)$, $(\rmL2\rmH8,1552)$, $(\rmL2\rmH8,1553)$, $(\rmL2\rmH8,1554)$, $(\rmL2\rmH8,1555)$, $(\rmL2\rmH8,1556)$, $(\rmL2\rmH8,1557)$, $(\rmL2\rmH8,1558)$, $(\rmL2\rmH8,1559)$, $(\rmL2\rmH8,1560)$, $(\rmL2\rmH8,1561)$, $(\rmL2\rmH8,2439)$, $(\rmL2\rmH8,2440)$, $(\rmL2\rmH8,2441)$, $(\rmL2\rmH8,2716)$, $(\rmL2\rmH8,2717)$, $(\rmL2\rmH8,2718)$, $(\rmL2\rmH8,2720)$, $(\rmL2\rmH8,2721)$, $(\rmL2\rmH9,1561)$, $(\rmL2\rmH10,1552)$, $(\rmL2\rmH10,1553)$, $(\rmL2\rmH10,1554)$, $(\rmL2\rmH10,1555)$, $(\rmL2\rmH10,1556)$, $(\rmL2\rmH10,1557)$, $(\rmL2\rmH10,1558)$, $(\rmL2\rmH10,1559)$, $(\rmL2\rmH10,1560)$, $(\rmL2\rmH10,1561)$, $(\rmL2\rmH11,1559)$, $(\rmL2\rmH11,1560)$, $(\rmL2\rmH12,1559)$, $(\rmL2\rmH12,1560)$, $(\rmL2\rmH13,1552)$, $(\rmL2\rmH13,1553)$, $(\rmL2\rmH13,1554)$, $(\rmL2\rmH13,1558)$, $(\rmL2\rmH13,1559)$, $(\rmL2\rmH13,1560)$, $(\rmL2\rmH13,1561)$.
    \item [(Layer $3$)] $(\rmL3\rmH1,534)$, $(\rmL3\rmH1,535)$, $(\rmL3\rmH1,666)$, $(\rmL3\rmH1,667)$, $(\rmL3\rmH1,668)$, $(\rmL3\rmH1,669)$, $(\rmL3\rmH1,736)$, $(\rmL3\rmH1,737)$, $(\rmL3\rmH1,738)$, $(\rmL3\rmH1,739)$, $(\rmL3\rmH1,740)$, $(\rmL3\rmH1,741)$, $(\rmL3\rmH1,742)$, $(\rmL3\rmH1,743)$, $(\rmL3\rmH1,744)$, $(\rmL3\rmH2,736)$, $(\rmL3\rmH2,737)$, $(\rmL3\rmH2,738)$, $(\rmL3\rmH2,739)$, $(\rmL3\rmH2,740)$, $(\rmL3\rmH2,741)$, $(\rmL3\rmH2,742)$, $(\rmL3\rmH2,743)$, $(\rmL3\rmH2,744)$, $(\rmL3\rmH2,745)$, $(\rmL3\rmH2,746)$, $(\rmL3\rmH2,747)$, $(\rmL3\rmH3,531)$, $(\rmL3\rmH3,532)$, $(\rmL3\rmH3,533)$, $(\rmL3\rmH3,534)$, $(\rmL3\rmH3,665)$, $(\rmL3\rmH3,666)$, $(\rmL3\rmH3,667)$, $(\rmL3\rmH3,734)$, $(\rmL3\rmH3,735)$, $(\rmL3\rmH3,736)$, $(\rmL3\rmH3,737)$, $(\rmL3\rmH3,738)$, $(\rmL3\rmH3,739)$, $(\rmL3\rmH3,740)$, $(\rmL3\rmH3,741)$, $(\rmL3\rmH3,742)$, $(\rmL3\rmH3,743)$, $(\rmL3\rmH3,744)$, $(\rmL3\rmH3,745)$, $(\rmL3\rmH3,2217)$, $(\rmL3\rmH4,532)$, $(\rmL3\rmH4,533)$, $(\rmL3\rmH4,534)$, $(\rmL3\rmH4,535)$, $(\rmL3\rmH4,667)$, $(\rmL3\rmH4,740)$, $(\rmL3\rmH4,741)$, $(\rmL3\rmH4,742)$, $(\rmL3\rmH4,743)$, $(\rmL3\rmH5,668)$, $(\rmL3\rmH5,737)$, $(\rmL3\rmH5,738)$, $(\rmL3\rmH5,739)$, $(\rmL3\rmH5,740)$, $(\rmL3\rmH5,741)$, $(\rmL3\rmH5,742)$, $(\rmL3\rmH5,743)$, $(\rmL3\rmH5,744)$, $(\rmL3\rmH5,745)$, $(\rmL3\rmH5,746)$, $(\rmL3\rmH5,747)$, $(\rmL3\rmH5,748)$, $(\rmL3\rmH6,736)$, $(\rmL3\rmH6,737)$, $(\rmL3\rmH6,738)$, $(\rmL3\rmH6,739)$, $(\rmL3\rmH6,740)$, $(\rmL3\rmH6,741)$, $(\rmL3\rmH6,742)$, $(\rmL3\rmH6,743)$, $(\rmL3\rmH6,744)$, $(\rmL3\rmH6,745)$, $(\rmL3\rmH6,746)$, $(\rmL3\rmH6,747)$, $(\rmL3\rmH6,748)$, $(\rmL3\rmH6,749)$, $(\rmL3\rmH6,2218)$, $(\rmL3\rmH6,2219)$, $(\rmL3\rmH6,2220)$, $(\rmL3\rmH6,2221)$, $(\rmL3\rmH6,2222)$, $(\rmL3\rmH6,2223)$, $(\rmL3\rmH6,2224)$, $(\rmL3\rmH6,2225)$, $(\rmL3\rmH6,2226)$, $(\rmL3\rmH15,3368)$, $(\rmL3\rmH15,3369)$, $(\rmL3\rmH15,3370)$, $(\rmL3\rmH21,922)$, $(\rmL3\rmH21,923)$, $(\rmL3\rmH22,1030)$, $(\rmL3\rmH22,1031)$, $(\rmL3\rmH22,1032)$, $(\rmL3\rmH23,920)$, $(\rmL3\rmH23,921)$, $(\rmL3\rmH23,922)$, $(\rmL3\rmH23,923)$, $(\rmL3\rmH23,924)$, $(\rmL3\rmH23,925)$, $(\rmL3\rmH23,1016)$, $(\rmL3\rmH23,1017)$, $(\rmL3\rmH23,1025)$, $(\rmL3\rmH23,1026)$, $(\rmL3\rmH23,1027)$, $(\rmL3\rmH23,1028)$, $(\rmL3\rmH23,1029)$, $(\rmL3\rmH23,1030)$, $(\rmL3\rmH23,1031)$, $(\rmL3\rmH24,735)$, $(\rmL3\rmH24,736)$, $(\rmL3\rmH24,742)$, $(\rmL3\rmH24,743)$, $(\rmL3\rmH24,804)$, $(\rmL3\rmH24,805)$, $(\rmL3\rmH24,806)$, $(\rmL3\rmH24,943)$, $(\rmL3\rmH24,944)$, $(\rmL3\rmH24,954)$, $(\rmL3\rmH24,955)$, $(\rmL3\rmH24,986)$, $(\rmL3\rmH24,987)$, $(\rmL3\rmH24,1014)$, $(\rmL3\rmH24,1015)$, $(\rmL3\rmH24,2684)$, $(\rmL3\rmH24,3712)$, $(\rmL3\rmH24,3713)$, $(\rmL3\rmH24,3714)$, $(\rmL3\rmH24,3715)$, $(\rmL3\rmH25,3714)$, $(\rmL3\rmH25,3715)$, $(\rmL3\rmH25,3716)$, $(\rmL3\rmH26,743)$, $(\rmL3\rmH26,922)$, $(\rmL3\rmH26,923)$, $(\rmL3\rmH26,924)$, $(\rmL3\rmH26,1013)$, $(\rmL3\rmH26,1014)$, $(\rmL3\rmH26,1015)$, $(\rmL3\rmH26,1016)$, $(\rmL3\rmH26,1017)$, $(\rmL3\rmH26,3713)$, $(\rmL3\rmH26,3714)$, $(\rmL3\rmH27,1015)$, $(\rmL3\rmH27,1016)$, $(\rmL3\rmH27,1017)$, $(\rmL3\rmH27,1028)$, $(\rmL3\rmH27,1029)$, $(\rmL3\rmH27,1030)$, $(\rmL3\rmH27,1031)$, $(\rmL3\rmH27,3712)$, $(\rmL3\rmH27,3713)$, $(\rmL3\rmH27,3714)$, $(\rmL3\rmH27,3715)$, 
    \item [(Layer $4$)] $(\rmL4\rmH10,3839)$, $(\rmL4\rmH10,3923)$, $(\rmL4\rmH10,3956)$, $(\rmL4\rmH10,3957)$, $(\rmL4\rmH10,3980)$, $(\rmL4\rmH10,3987)$, $(\rmL4\rmH18,1094)$, $(\rmL4\rmH18,1754)$, $(\rmL4\rmH18,1755)$, $(\rmL4\rmH23,1258)$, $(\rmL4\rmH23,1684)$, $(\rmL4\rmH23,1685)$, $(\rmL4\rmH23,1741)$, $(\rmL4\rmH23,1742)$, $(\rmL4\rmH23,2018)$, $(\rmL4\rmH23,2276)$, $(\rmL4\rmH23,3532)$, $(\rmL4\rmH23,3960)$, $(\rmL4\rmH25,1740)$.
    \item [(Layer $6$)] $(\rmL6\rmH0,923)$, $(\rmL6\rmH0,1037)$, $(\rmL6\rmH0,1038)$, $(\rmL6\rmH0,2294)$, $(\rmL6\rmH0,2295)$, $(\rmL6\rmH0,2828)$, $(\rmL6\rmH0,3329)$, $(\rmL6\rmH0,3330)$, $(\rmL6\rmH0,3331)$, $(\rmL6\rmH0,3332)$, $(\rmL6\rmH0,3650)$, $(\rmL6\rmH0,3651)$, $(\rmL6\rmH0,3745)$, $(\rmL6\rmH0,3814)$, $(\rmL6\rmH1,1037)$, $(\rmL6\rmH1,1038)$, $(\rmL6\rmH1,2294)$, $(\rmL6\rmH1,2295)$, $(\rmL6\rmH1,3330)$, $(\rmL6\rmH1,3331)$, $(\rmL6\rmH2,922)$, $(\rmL6\rmH2,923)$, $(\rmL6\rmH2,924)$, $(\rmL6\rmH2,1037)$, $(\rmL6\rmH2,1038)$, $(\rmL6\rmH2,1452)$, $(\rmL6\rmH2,1905)$, $(\rmL6\rmH2,1974)$, $(\rmL6\rmH2,2293)$, $(\rmL6\rmH2,2294)$, $(\rmL6\rmH2,2295)$, $(\rmL6\rmH2,2827)$, $(\rmL6\rmH2,2828)$, $(\rmL6\rmH2,2941)$, $(\rmL6\rmH2,3330)$, $(\rmL6\rmH2,3331)$, $(\rmL6\rmH2,3425)$, $(\rmL6\rmH2,3745)$, $(\rmL6\rmH2,3746)$, $(\rmL6\rmH2,3814)$, $(\rmL6\rmH2,3853)$, $(\rmL6\rmH3,923)$, $(\rmL6\rmH3,924)$, $(\rmL6\rmH3,925)$, $(\rmL6\rmH3,931)$, $(\rmL6\rmH3,1037)$, $(\rmL6\rmH3,1038)$, $(\rmL6\rmH3,1039)$, $(\rmL6\rmH3,1358)$, $(\rmL6\rmH3,1359)$, $(\rmL6\rmH3,1910)$, $(\rmL6\rmH3,1911)$, $(\rmL6\rmH3,2295)$, $(\rmL6\rmH3,2296)$, $(\rmL6\rmH3,2828)$, $(\rmL6\rmH3,2829)$, $(\rmL6\rmH3,2942)$, $(\rmL6\rmH3,2943)$, $(\rmL6\rmH3,3330)$, $(\rmL6\rmH3,3331)$, $(\rmL6\rmH3,3332)$, $(\rmL6\rmH3,3650)$, $(\rmL6\rmH3,3651)$, $(\rmL6\rmH3,3652)$, $(\rmL6\rmH4,1037)$, $(\rmL6\rmH4,2293)$, $(\rmL6\rmH4,2294)$, $(\rmL6\rmH4,2295)$, $(\rmL6\rmH4,2828)$, $(\rmL6\rmH4,2941)$, $(\rmL6\rmH4,3329)$, $(\rmL6\rmH4,3330)$, $(\rmL6\rmH4,3331)$, $(\rmL6\rmH4,3425)$, $(\rmL6\rmH4,3532)$, $(\rmL6\rmH4,3533)$, $(\rmL6\rmH5,922)$, $(\rmL6\rmH5,923)$, $(\rmL6\rmH5,2294)$, $(\rmL6\rmH5,3329)$, $(\rmL6\rmH5,3330)$, $(\rmL6\rmH5,3331)$, $(\rmL6\rmH6,919)$, $(\rmL6\rmH6,1038)$, $(\rmL6\rmH6,1239)$, $(\rmL6\rmH6,1974)$, $(\rmL6\rmH6,2294)$, $(\rmL6\rmH6,2295)$, $(\rmL6\rmH6,3212)$, $(\rmL6\rmH6,3331)$, $(\rmL6\rmH6,3425)$, $(\rmL6\rmH6,3426)$, $(\rmL6\rmH6,3532)$, $(\rmL6\rmH6,3533)$, $(\rmL6\rmH6,3746)$, $(\rmL6\rmH6,3853)$, $(\rmL6\rmH25,1239)$, $(\rmL6\rmH25,1268)$, $(\rmL6\rmH25,1269)$, $(\rmL6\rmH25,1270)$, $(\rmL6\rmH25,1445)$, $(\rmL6\rmH25,1446)$, $(\rmL6\rmH26,1269)$, $(\rmL6\rmH26,1270)$, 
    \item [(Layers $10-11$)] $(\rmL10\rmH1,967)$, $(\rmL10\rmH1,3204)$, $(\rmL10\rmH1,3217)$, $(\rmL10\rmH1,3694)$, $(\rmL10\rmH1,3695)$, $(\rmL10\rmH22,1094)$, $(\rmL10\rmH22,1095)$, $(\rmL11\rmH3,2161)$, $(\rmL11\rmH3,2162)$, $(\rmL11\rmH14,2089)$, $(\rmL11\rmH14,2090)$, $(\rmL11\rmH14,2091)$, $(\rmL11\rmH18,2090)$, $(\rmL11\rmH18,2091)$, 
    \item [(Layers $21-14$)]$(\rmL21\rmH15,2175)$, $(\rmL24\rmH0,880)$, $(\rmL24\rmH0,881)$, $(\rmL24\rmH0,1025)$, $(\rmL24\rmH0,1026)$, $(\rmL24\rmH0,1414)$, $(\rmL24\rmH0,1415)$, $(\rmL24\rmH0,2294)$, $(\rmL24\rmH0,2295)$, $(\rmL24\rmH0,3249)$, $(\rmL24\rmH0,3250)$, $(\rmL24\rmH0,3368)$, $(\rmL24\rmH0,3369)$, $(\rmL24\rmH1,880)$, $(\rmL24\rmH1,881)$, $(\rmL24\rmH1,1163)$, $(\rmL24\rmH1,1415)$, $(\rmL24\rmH1,2294)$, $(\rmL24\rmH1,2295)$, $(\rmL24\rmH1,3249)$, $(\rmL24\rmH1,3250)$, $(\rmL24\rmH1,3356)$, $(\rmL24\rmH1,3368)$, $(\rmL24\rmH1,3369)$, $(\rmL24\rmH2,879)$, $(\rmL24\rmH2,880)$, $(\rmL24\rmH2,1413)$, $(\rmL24\rmH2,1414)$, $(\rmL24\rmH2,2293)$, $(\rmL24\rmH2,2294)$, $(\rmL24\rmH2,3248)$, $(\rmL24\rmH2,3367)$, $(\rmL24\rmH2,3368)$, $(\rmL24\rmH4,880)$, $(\rmL24\rmH4,2294)$, $(\rmL24\rmH4,3249)$, $(\rmL24\rmH4,3250)$, $(\rmL24\rmH5,880)$, $(\rmL24\rmH5,881)$, $(\rmL24\rmH5,2294)$, $(\rmL24\rmH5,2295)$, $(\rmL24\rmH5,3248)$, $(\rmL24\rmH5,3249)$, $(\rmL24\rmH5,3250)$, $(\rmL24\rmH6,880)$, $(\rmL24\rmH6,881)$, $(\rmL24\rmH6,2294)$, $(\rmL24\rmH6,3249)$, $(\rmL24\rmH6,3250)$.
    \item [(Layers $25-26$)]$(\rmL25\rmH14,2701)$, $(\rmL25\rmH14,2702)$, $(\rmL25\rmH18,2701)$, $(\rmL25\rmH18,2702)$, $(\rmL25\rmH19,2701)$, $(\rmL25\rmH19,2702)$, $(\rmL26\rmH8,610)$, $(\rmL26\rmH8,2294)$, $(\rmL26\rmH8,3494)$, $(\rmL26\rmH8,3532)$
\end{itemize}

\subsubsection{Attention Scores and Ranks of $Q$, $K$, and $H$}\label{app:large_delta_list}
In this section, we present per-head statistics for \acp{sdh} with large $\Delta$. The previous section shows that a single head can exhibit multiple slash patterns. For conciseness, we report at most one values of $\Delta$ for each head in each model, since slash patterns within the same head share the same $Q$, $K$, and $H$. Specifically, for each head, we rank all $\Delta$ by their average slash score and list the largest values in the following table.
\vspace{-1em}
\begin{table}[H]
    \centering
    \resizebox{0.99\textwidth}{!}{
    \begin{tabular}{ccccccccccccc}
        \hline
        Model & $\Delta$ & Head & $\bbE[\attn_{i,i-\Delta}](\times 10^{-3})$ & \makecell{\ac{ood} \\ $\bbE[\attn_{i,i-\Delta}](\times 10^{-3})$} & $r_1(Q)$ & $R_{0.95}(Q)$ & $r_1(K)$ & $R_{0.95}(K)$ & $r_1(H)$ & $R_{0.95}(H)$ \\
        \hline 
        Llama3 & $563$ & $\rmL0\rmH29$ & $1.321$  & $0.725$ & $0.837$ & $3$ & $0.719$ & $4$ & $0.304$ & $294$ \\
        Llama3 & $685$ & $\rmL0\rmH30$ & $1.161$  & $0.561$ & $0.939$ & $2$ & $0.719$ & $4$ & $0.304$ & $294$ \\
        Llama3 & $3827$ & $\rmL1\rmH16$ & $1.021$  & $4.490$ & $0.922$ & $3$ & $0.847$ & $11$ & $0.386$ & $227$ \\
        Llama3 & $3827$ & $\rmL1\rmH18$ & $1.002$  & $7.999$ & $0.895$ & $6$ & $0.847$ & $11$ & $0.386$ & $227$ \\
        \hline\hline
        Qwen2.5 & $937$ & $\rmL 0 \rmH 7$ & $3.064$ & $6.664$ & $0.748$ & $46$ & $0.999$ & $1$ & $0.001$ & $3584$ \\
        Qwen2.5 & $503$ & $\rmL 0 \rmH 23$ & $2.049$ & $1.326$ & $0.941$ & $3$ & $0.996$ & $1$ & $0.160$ & $590$ \\
        Qwen2.5 & $503$ & $\rmL 0 \rmH 24$ & $2.035$ & $1.398$ & $0.913$ & $13$ & $0.996$ & $1$ & $0.160$ & $590$ \\
        Qwen2.5 & $503$ & $\rmL 0 \rmH 26$ & $2.077$ & $1.436$ & $0.910$ & $14$ & $0.996$ & $1$ & $0.160$ & $590$ \\
        Qwen2.5 & $804$ & $\rmL 1 \rmH 11$ & $2.644$ & $2.873$ & $0.824$ & $6$ & $0.946$ & $2$ & $0.507$ & $227$ \\
        Qwen2.5 & $554$ & $\rmL 1 \rmH 15$ & $2.207$ & $2.106$ & $0.875$ & $6$ & $0.975$ & $1$ & $0.002$ & $3584$ \\
        Qwen2.5 & $554$ & $\rmL 1 \rmH 16$ & $2.969$ & $3.222$ & $0.778$ & $9$ & $0.975$ & $1$ & $0.507$ & $227$ \\
        Qwen2.5 & $610$ & $\rmL 1 \rmH 19$ & $2.206$ & $2.868$ & $0.741$ & $10$ & $0.975$ & $1$ & $0.507$ & $227$ \\
        Qwen2.5 & $2293$ & $\rmL 1 \rmH 23$ & $3.866$ & $10.523$ & $0.909$ & $3$ & $0.934$ & $2$ & $0.002$ & $3584$ \\
        Qwen2.5 & $835$ & $\rmL 1 \rmH 24$ & $2.599$ & $12.972$ & $0.812$ & $7$ & $0.934$ & $2$ & $0.002$ & $3584$ \\
        Qwen2.5 & $1934$ & $\rmL 2 \rmH 6$ & $2.651$ & $8.729$ & $0.906$ & $9$ & $0.645$ & $51$ & $0.001$ & $3584$ \\
        Qwen2.5 & $3713$ & $\rmL 3 \rmH 24$ & $2.940$ & $1.070$ & $0.770$ & $27$ & $0.959$ & $1$ & $0.187$ & $554$ \\
        Qwen2.5 & $1685$ & $\rmL 4 \rmH 23$ & $3.636$ & $2.229$ & $0.950$ & $1$ & $0.680$ & $44$ & $0.163$ & $382$ \\
        Qwen2.5 & $3331$ & $\rmL 6 \rmH 1$ & $2.217$ & $0.992$ & $0.961$ & $1$ & $0.721$ & $51$ & $0.333$ & $669$ \\
        Qwen2.5 & $2294$ & $\rmL 6 \rmH 2$ & $7.480$ & $23.225$ & $0.961$ & $1$ & $0.721$ & $51$ & $0.333$ & $669$ \\
        Qwen2.5 & $3651$ & $\rmL 6 \rmH 3$ & $2.573$ & $2.864$ & $0.933$ & $3$ & $0.721$ & $51$ & $0.333$ & $669$ \\
        Qwen2.5 & $2294$ & $\rmL 6 \rmH 4$ & $5.650$ & $12.231$ & $0.942$ & $2$ & $0.721$ & $51$ & $0.333$ & $669$ \\
        Qwen2.5 & $1974$ & $\rmL 6 \rmH 6$ & $2.241$ & $2.678$ & $0.977$ & $1$ & $0.721$ & $51$ & $0.333$ & $669$ \\
        Qwen2.5 & $3694$ & $\rmL 10 \rmH 1$ & $2.275$ & $4.549$ & $0.943$ & $2$ & $0.761$ & $27$ & $0.317$ & $716$ \\
        Qwen2.5 & $3369$ & $\rmL 24 \rmH 1$ & $2.650$ & $2.249$ & $0.921$ & $7$ & $0.755$ & $51$ & $0.240$ & $645$ \\        
        \hline
    \end{tabular}
    }
    \caption{This table lists the average attention scores of prompts in LongBench and \ac{ood} prompts for Llama3-8B-Instruct and Qwen2.5-7B-Instruct, the rank information of $Q$, $K$, and $H$. The unit of the attention scores is $10^{-3}$. }
    \label{table:head_eg_large_full}
\end{table}

\begin{table}[H]
    \centering
    \begin{tabular}{cccccccccc}
        \hline
        Model & Head & $r_1(W_Q)$ & $R_{0.95}(W_Q)$ & $r_1(W_K)$ & $R_{0.95}(W_K)$ \\
        \hline
        Llama3  & $\rmL0\rmH29$ & $0.339$  &$24$ & $0.344$  &$25$ \\
        Llama3  & $\rmL0\rmH30$ & $0.487$  &$26$ & $0.344$  &$25$ \\
        Llama3  & $\rmL1\rmH16$ & $0.044$  &$69$ & $0.033$  &$72$ \\
        Llama3  & $\rmL1\rmH18$ & $0.035$  &$66$ & $0.033$  &$72$ \\
         \hline\hline
        Qwen2.5 & $\rmL 0 \rmH 7$ & $0.018$ & $98$ & $0.016$ & $100$ \\
        Qwen2.5 & $\rmL 0 \rmH 23$ & $0.031$ & $118$ & $0.022$ & $116$ \\
        Qwen2.5 & $\rmL 0 \rmH 24$ & $0.027$ & $117$ & $0.022$ & $116$ \\
        Qwen2.5 & $\rmL 0 \rmH 26$ & $0.025$ & $117$ & $0.022$ & $116$ \\
        Qwen2.5 & $\rmL 1 \rmH 11$ & $0.066$ & $80$ & $0.042$ & $86$ \\
        Qwen2.5 & $\rmL 1 \rmH 15$ & $0.058$ & $73$ & $0.042$ & $80$ \\
        Qwen2.5 & $\rmL 1 \rmH 16$ & $0.087$ & $75$ & $0.042$ & $80$ \\
        Qwen2.5 & $\rmL 1 \rmH 19$ & $0.080$ & $73$ & $0.042$ & $80$ \\
        Qwen2.5 & $\rmL 1 \rmH 23$ & $0.066$ & $66$ & $0.049$ & $74$ \\
        Qwen2.5 & $\rmL 1 \rmH 24$ & $0.095$ & $65$ & $0.049$ & $74$ \\
        Qwen2.5 & $\rmL 2 \rmH 6$ & $0.043$ & $99$ & $0.019$ & $97$ \\
        Qwen2.5 & $\rmL 3 \rmH 24$ & $0.030$ & $105$ & $0.021$ & $101$ \\
        Qwen2.5 & $\rmL 4 \rmH 23$ & $0.050$ & $99$ & $0.018$ & $101$ \\
        Qwen2.5 & $\rmL 6 \rmH 1$ & $0.021$ & $108$ & $0.015$ & $103$ \\
        Qwen2.5 & $\rmL 6 \rmH 2$ & $0.033$ & $108$ & $0.015$ & $103$ \\
        Qwen2.5 & $\rmL 6 \rmH 3$ & $0.017$ & $108$ & $0.015$ & $103$ \\
        Qwen2.5 & $\rmL 6 \rmH 4$ & $0.034$ & $108$ & $0.015$ & $103$ \\
        Qwen2.5 & $\rmL 6 \rmH 6$ & $0.023$ & $107$ & $0.015$ & $103$ \\
        Qwen2.5 & $\rmL 10 \rmH 1$ & $0.028$ & $101$ & $0.025$ & $102$ \\
        Qwen2.5 & $\rmL 24 \rmH 1$ & $0.032$ & $112$ & $0.015$ & $111$ \\
        \hline
    \end{tabular}
    \caption{The table reports $r_1$ and $R_{0.95}$ for the query and key projection matrices, $W_Q$ and $W_K$, across attention heads in \acp{llm}.}
    \label{table:w_qk_rank_large_full}
\end{table}

\begin{table}[H]
    \centering
    \begin{tabular}{cccccccccc}
        \hline
        Model & Head & $\tilde{r}_1(Q)$ & $\tilde{R}_{0.95}(W_Q)$ & $\tilde{r}_1(K)$ & $\tilde{R}_{0.95}(K)$ \\
        \hline
        Llama3  & $\rmL0\rmH29$ & $0.904$  &$3$ & $0.877$  &$4$ \\
        Llama3  & $\rmL0\rmH30$ & $.943$  &$2$ & $0.877$  &$4$ \\
        \hline\hline
        Qwen2.5 & $\rmL 0 \rmH 7$ & $0.221$ & $68$ & $0.079$ & $73$  \\
        Qwen2.5 & $\rmL 0 \rmH 23$ & $0.250$ & $72$ & $0.155$ & $75$ \\
        Qwen2.5 & $\rmL 0 \rmH 24$ & $0.218$ & $73$ & $0.155$ & $75$ \\
        Qwen2.5 & $\rmL 0 \rmH 26$ & $0.215$ & $73$ & $0.155$ & $75$  \\
        \hline
    \end{tabular}
    \caption{The values of $\tilde{r}_1$ and $\tilde{R}_{0.95}$ for the \acp{sdh} located in the $0$-th layer of Llama3-8B-Instruct, and Qwen2.5-7B-Instruct.}
    \label{table:qk_rank_large_full}
\end{table}


\begin{table}[H]
    \centering
    \begin{tabular}{ccccccc}
        \hline
        Model & $\Delta$ & Head & $\bbE[\attn_{i,i-\Delta}]$ & \makecell{$\bbE[\attn_{i,i-\Delta}]$ w/o \\ high freqs} & \makecell{$\bbE[\attn_{i,i-\Delta}]$ w/o \\ med freqs} & \makecell{$\bbE[\attn_{i,i-\Delta}]$ w/o \\ low freqs} \\
        \hline 
        Llama3 & $563$ & $\rmL0\rmH29$ & $1.321$ & $0.598$ & $1.742$ & $1.657$ \\
        Llama3 & $685$ & $\rmL0\rmH30$ & $1.161$ & $0.872$ & $1.417$ & $1.223$ \\
        Llama3 & $3827$ & $\rmL1\rmH16$ & $1.021$ & $0.710$ & $0.974$ & $2.466$ \\
        Llama3 & $3827$ & $\rmL1\rmH18$ & $1.002$ & $0.774$ & $0.970$ & $2.764$ \\
        \hline\hline
        Qwen2.5 & $937$ & $\rmL0\rmH7$ & $3.093$ & $1.523$ & $1.522$ & $3.164$ \\
        Qwen2.5 & $503$ & $\rmL0\rmH23$ & $2.063$ & $1.446$ & $1.263$ & $2.073$ \\
        Qwen2.5 & $503$ & $\rmL0\rmH24$ & $2.063$ & $1.456$ & $1.230$ & $2.060$ \\
        Qwen2.5 & $503$ & $\rmL0\rmH26$ & $2.079$ & $1.409$ & $1.241$ & $2.104$ \\
        Qwen2.5 & $804$ & $\rmL1\rmH11$ & $2.717$ & $0.831$ & $2.808$ & $2.800$ \\
        Qwen2.5 & $554$ & $\rmL1\rmH15$ & $2.340$ & $0.422$ & $1.924$ & $2.220$ \\
        Qwen2.5 & $554$ & $\rmL1\rmH16$ & $3.018$ & $0.975$ & $2.712$ & $2.975$ \\
        Qwen2.5 & $610$ & $\rmL1\rmH19$ & $2.245$ & $0.605$ & $2.407$ & $2.207$ \\
        Qwen2.5 & $2293$ & $\rmL1\rmH23$ & $3.884$ & $0.257$ & $4.143$ & $5.779$ \\
        Qwen2.5 & $835$ & $\rmL1\rmH24$ & $2.773$ & $0.169$ & $6.733$ & $2.754$ \\
        Qwen2.5 & $1934$ & $\rmL2\rmH6$ & $2.617$ & $1.348$ & $2.842$ & $2.670$ \\
        Qwen2.5 & $3713$ & $\rmL3\rmH24$ & $2.957$ & $1.448$ & $1.974$ & $2.916$ \\
        Qwen2.5 & $1685$ & $\rmL4\rmH23$ & $3.638$ & $0.651$ & $10.284$ & $3.637$ \\
        Qwen2.5 & $3331$ & $\rmL6\rmH1$ & $2.253$ & $0.710$ & $2.934$ & $2.226$ \\
        Qwen2.5 & $2294$ & $\rmL6\rmH2$ & $7.417$ & $1.360$ & $13.771$ & $7.481$ \\
        Qwen2.5 & $3651$ & $\rmL6\rmH3$ & $2.666$ & $0.832$ & $7.652$ & $2.583$ \\
        Qwen2.5 & $2294$ & $\rmL6\rmH4$ & $5.698$ & $1.466$ & $9.087$ & $5.638$ \\
        Qwen2.5 & $1974$ & $\rmL6\rmH6$ & $2.264$ & $0.422$ & $4.368$ & $2.241$ \\
        Qwen2.5 & $3694$ & $\rmL10\rmH1$ & $2.179$ & $0.555$ & $4.799$ & $4.768$ \\
        Qwen2.5 & $3369$ & $\rmL24\rmH1$ & $2.566$ & $0.613$ & $2.914$ & $2.655$ \\
        \hline
    \end{tabular}
    \caption{This table quantifies the effect of low-, medium-, and high-frequency components on \acp{sdh} by reporting, for each band, the average slash score after removing that band. The unit of attention scores is $10^{-3}$.}
    \label{table:head_freq_large_full}
\end{table}

\begin{table}[H]
    \centering
    \begin{tabular}{cccccccccc}
        \hline
         Head & Avg. $\|W_Q^{\top}\xb\|$ & $\|\bb_Q\|$ & Avg. $\|W_K^{\top}\xb\|$ & $\|\bb_K\|$\\
        \hline\
         $\rmL 0 \rmH 7$ & $7.403$ & $9.343$ & $10.039$ & $281.761$ \\
         $\rmL 0 \rmH 23$ & $5.632$ & $17.083$ & $9.173$ & $108.957$ \\
         $\rmL 0 \rmH 24$ & $5.782$ & $16.097$ & $9.173$ & $108.957$ \\
         $\rmL 0 \rmH 26$ & $5.810$ & $15.828$ & $9.173$ & $108.957$ \\
        \hline
    \end{tabular}
    \caption{This table lists the average norms of $W_Q^{\top}\xb$ and $W_K^{\top}\xb$ across all the token embeddings and the norms of biases $\bb_Q$ and $\bb_K$ }
    \label{table:sigma_qk_large_full}
\end{table}
\subsection{Figures of Slash-dominant Heads with Large $\Delta$ and Small Average Slash Score}\label{app:large_delta_fig}
\subsubsection{Results of Qwen2.5}

\begin{figure}[H]
\centering
\subfigure[Hidden state $H$ of $\rmL 0\rmH 7$ after PCA.]{\includegraphics[width=0.23\textwidth]{figures/qwen25/hidden_state_aft_attn_ln_pca_L0H7.pdf}}
\hspace{0.08em}
\subfigure[Queries $Q$ of $\rmL 0\rmH 7$.]{\includegraphics[width=0.23\textwidth]{figures/qwen25/query_state_pre_pos_L0H7.pdf}}
\hspace{0.08em}
\subfigure[Keys $K$ of $\rmL 0\rmH 7$.]{\includegraphics[width=0.23\textwidth]{figures/qwen25/key_state_pre_pos_L0H7.pdf}}
\hspace{0.08em}
\subfigure[$\InP(5000,j,l)$ of $\rmL 0\rmH 7$.]{\includegraphics[width=0.25\textwidth]{figures/qwen25/summed_prod_L0H7.pdf}}

\subfigure[Hidden state $H$ of $\rmL 1\rmH 11$ after PCA.]{\includegraphics[width=0.23\textwidth]{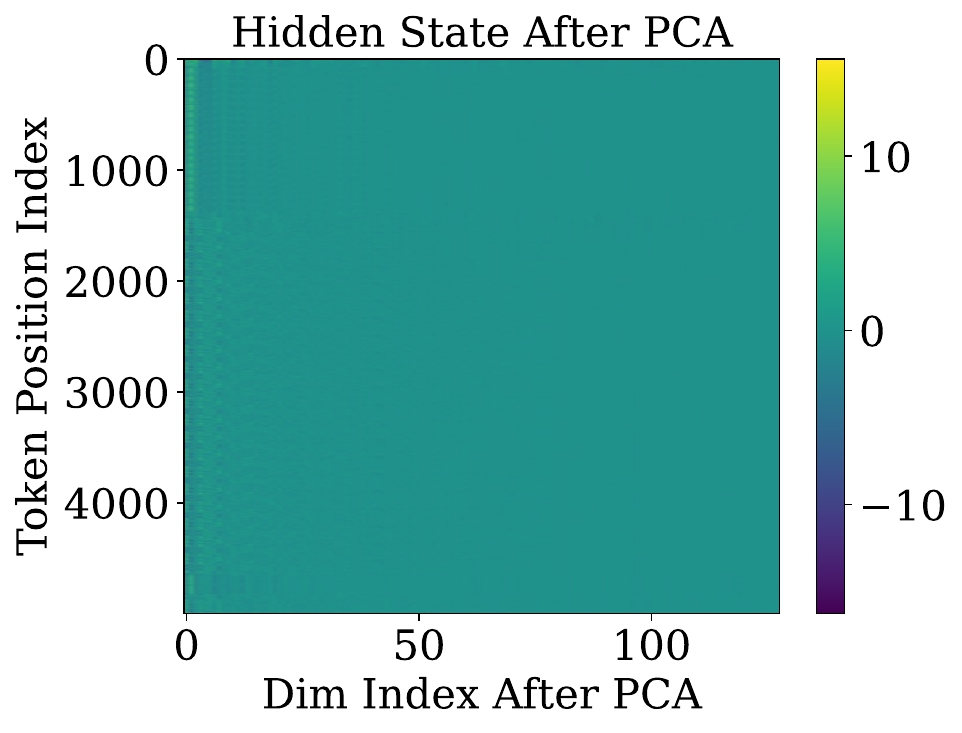}}
\hspace{0.08em}
\subfigure[Queries $Q$ of $\rmL 1\rmH 11$.]{\includegraphics[width=0.23\textwidth]{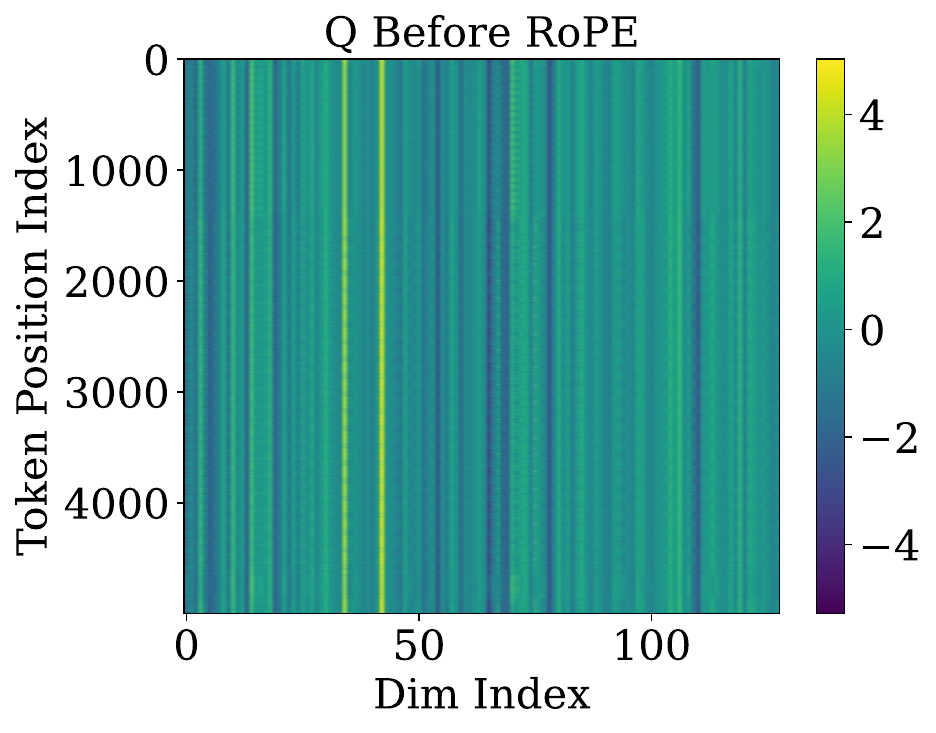}}
\hspace{0.08em}
\subfigure[Keys $K$ of $\rmL 1\rmH 11$.]{\includegraphics[width=0.23\textwidth]{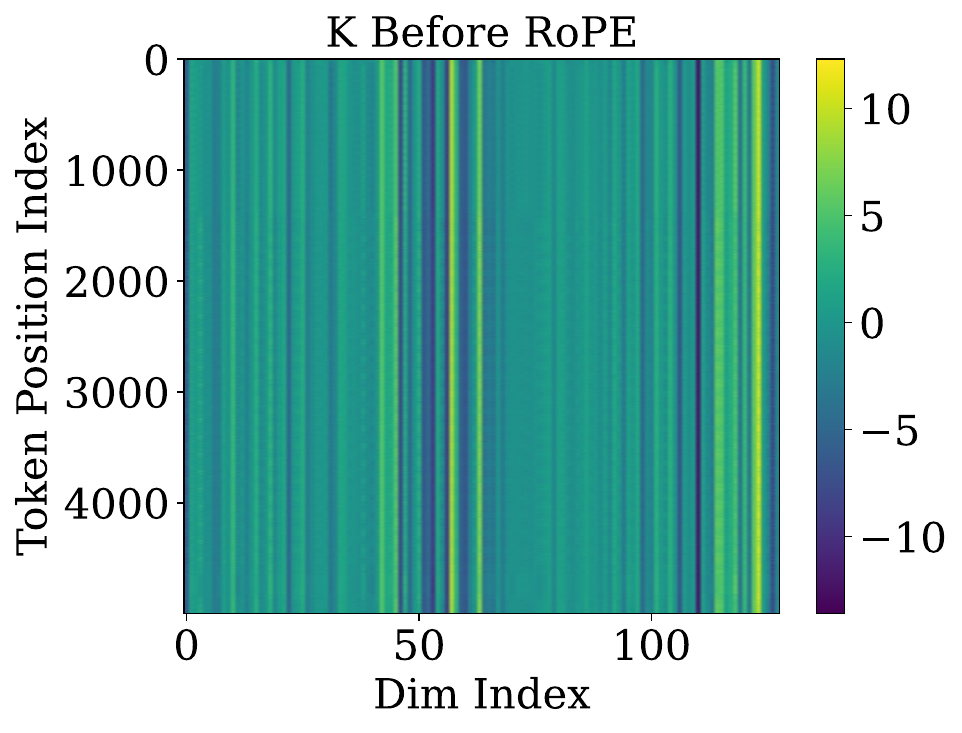}}
\hspace{0.08em}
\subfigure[$\InP(5000,j,l)$ of $\rmL 1\rmH 11$.]{\includegraphics[width=0.25\textwidth]{figures/qwen25/summed_prod_L1H11.pdf}}

\subfigure[Hidden state $H$ of $\rmL 2\rmH 6$ after PCA.]{\includegraphics[width=0.23\textwidth]{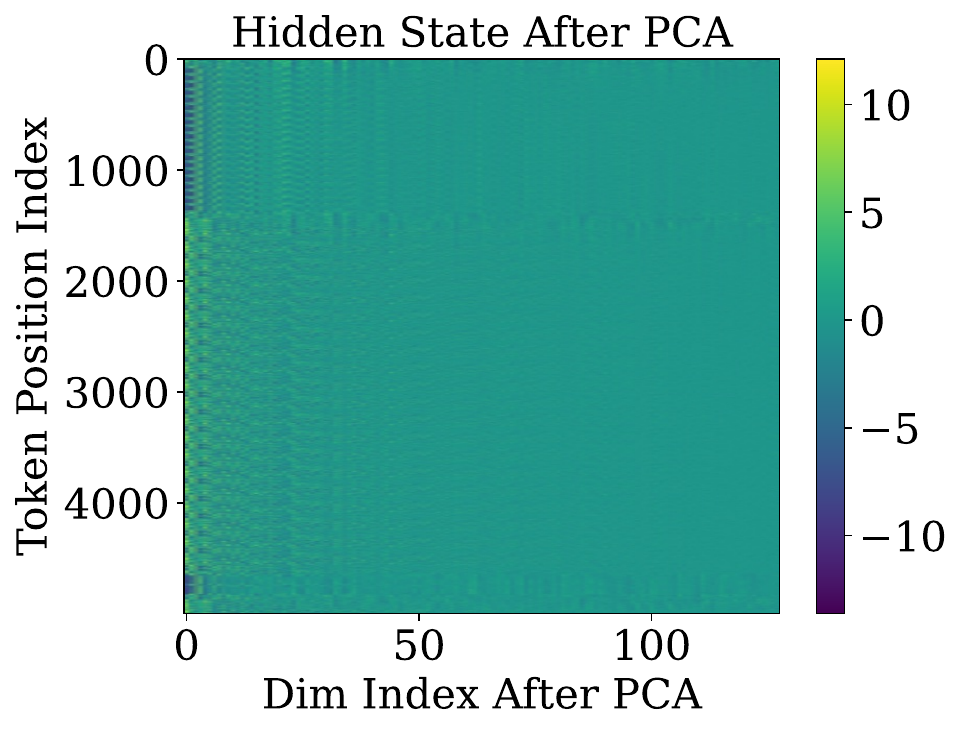}}
\hspace{0.08em}
\subfigure[Queries $Q$ of $\rmL 2\rmH 6$.]{\includegraphics[width=0.23\textwidth]{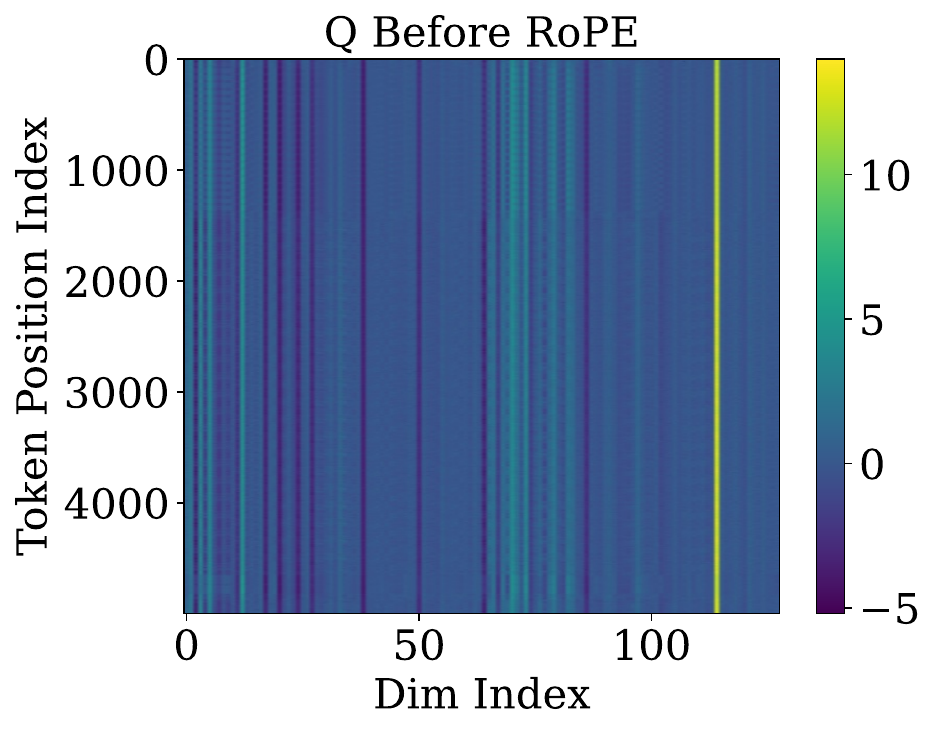}}
\hspace{0.08em}
\subfigure[Keys $K$ of $\rmL 2\rmH 6$.]{\includegraphics[width=0.23\textwidth]{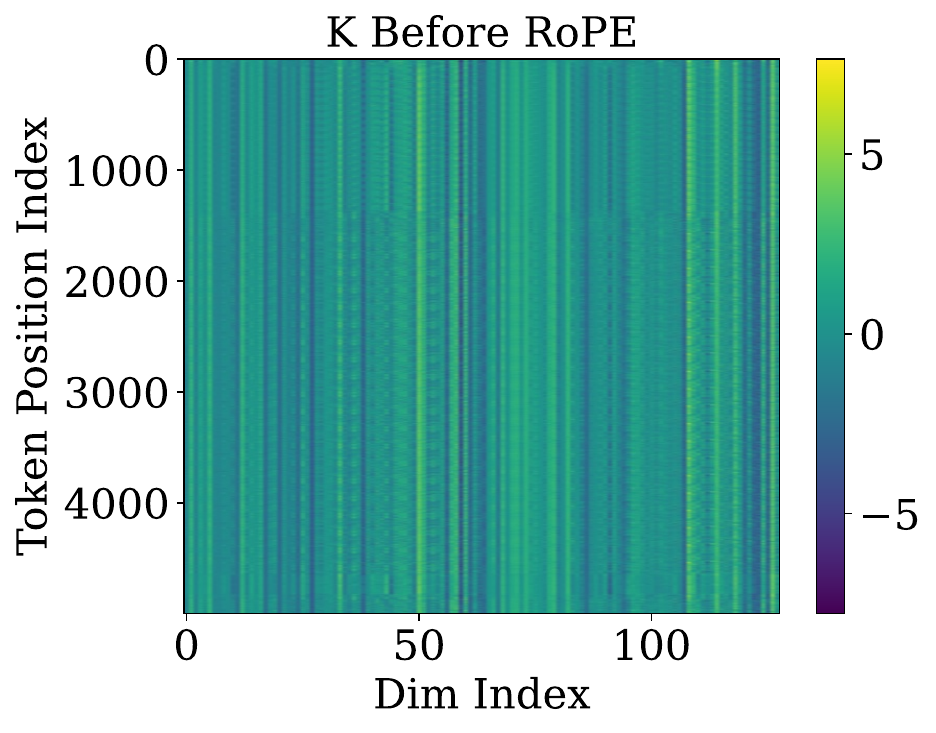}}
\hspace{0.08em}
\subfigure[$\InP(5000,j,l)$ of $\rmL 2\rmH 6$.]{\includegraphics[width=0.25\textwidth]{figures/qwen25/summed_prod_L2H6.pdf}}
\caption{This figure shows the hidden states, queries, keys, and $\InP(5000,j,l)$ for $j\in[5000],l\in[64]$ for Qwen2.5-7B-Instruct. Here, the example prompt contains $5000$ tokens. The dimensions of queries and keys are both $128$.}
\label{fig:qwen_qk}
\vspace{-1em}
\end{figure}

\subsubsection{Results of Llama3-8B-Instruct}
\begin{figure}[H]
\centering
\subfigure[Average attention score matrix of $\rmL 0\rmH 29$.]{\includegraphics[width=0.3\textwidth]{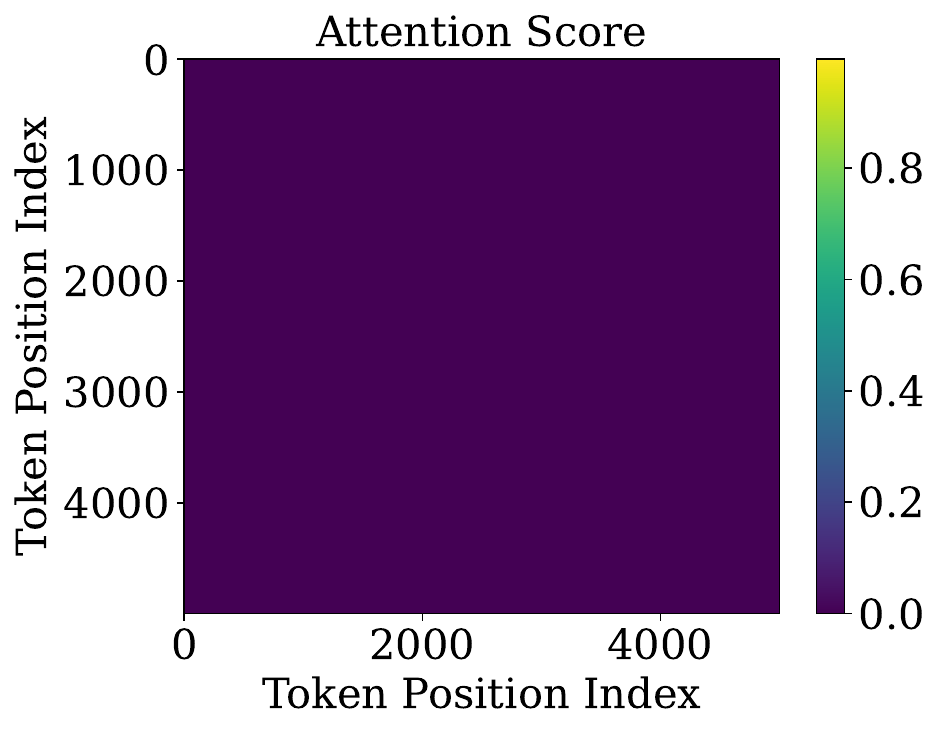}}
\quad
\subfigure[Average attention score matrix of $\rmL 0\rmH 30$.]{\includegraphics[width=0.3\textwidth]{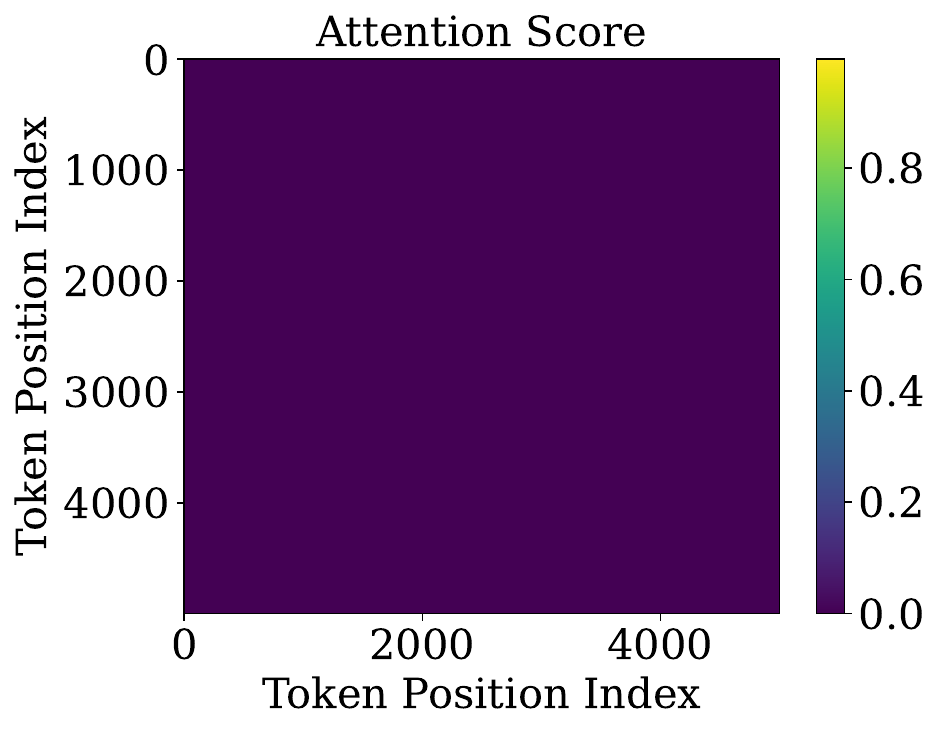}}
\quad
\subfigure[Average attention score matrix of $\rmL 1\rmH 16$.]{\includegraphics[width=0.3\textwidth]{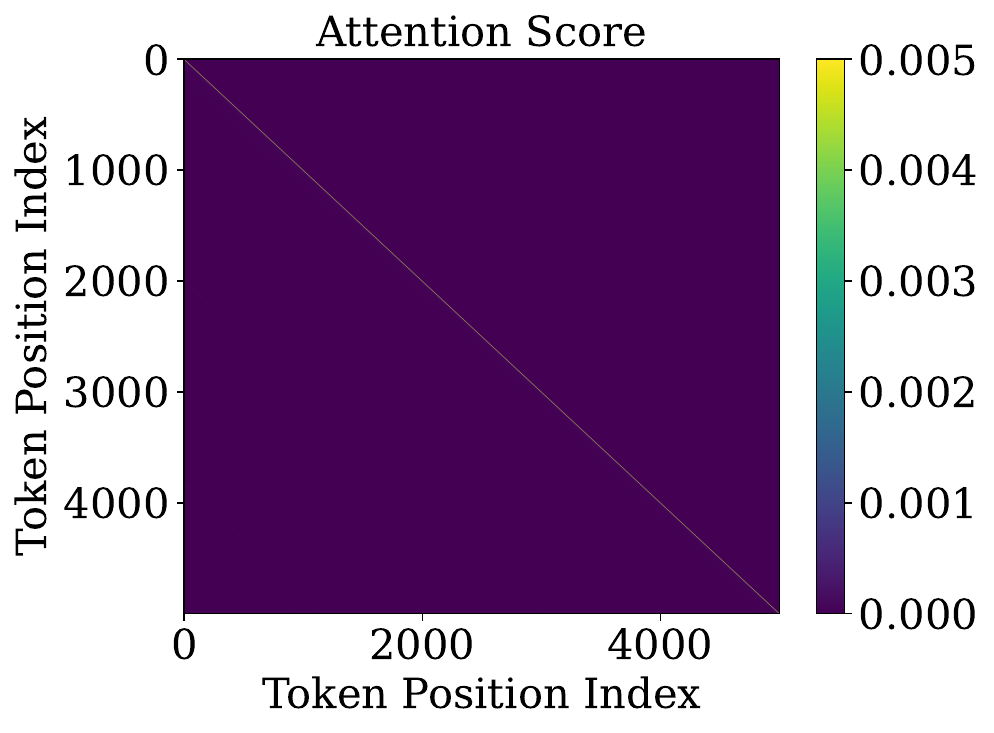}}
\caption{Average of attention score matrices in Llama3-8B-Instruct with prompts from LongBench V2.}
\label{fig:llama_large_longb}
\end{figure}

\begin{figure}[H]
\centering
\subfigure[Average attention score matrix of $\rmL 0\rmH 29$.]{\includegraphics[width=0.3\textwidth]{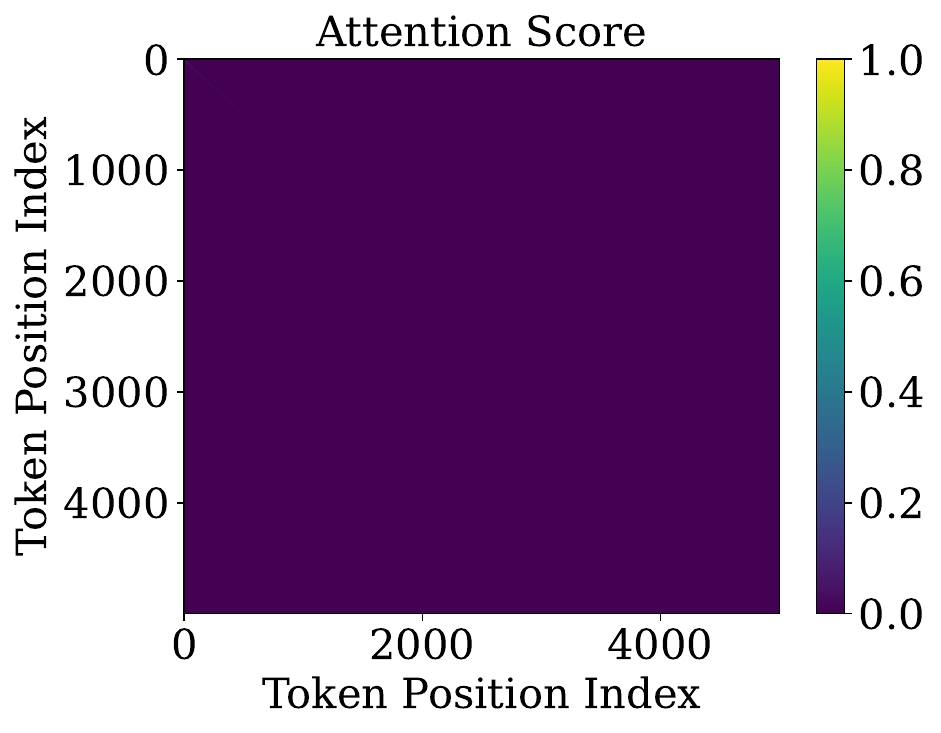}}
\quad
\subfigure[Average attention score matrix of $\rmL 0\rmH 30$.]{\includegraphics[width=0.3\textwidth]{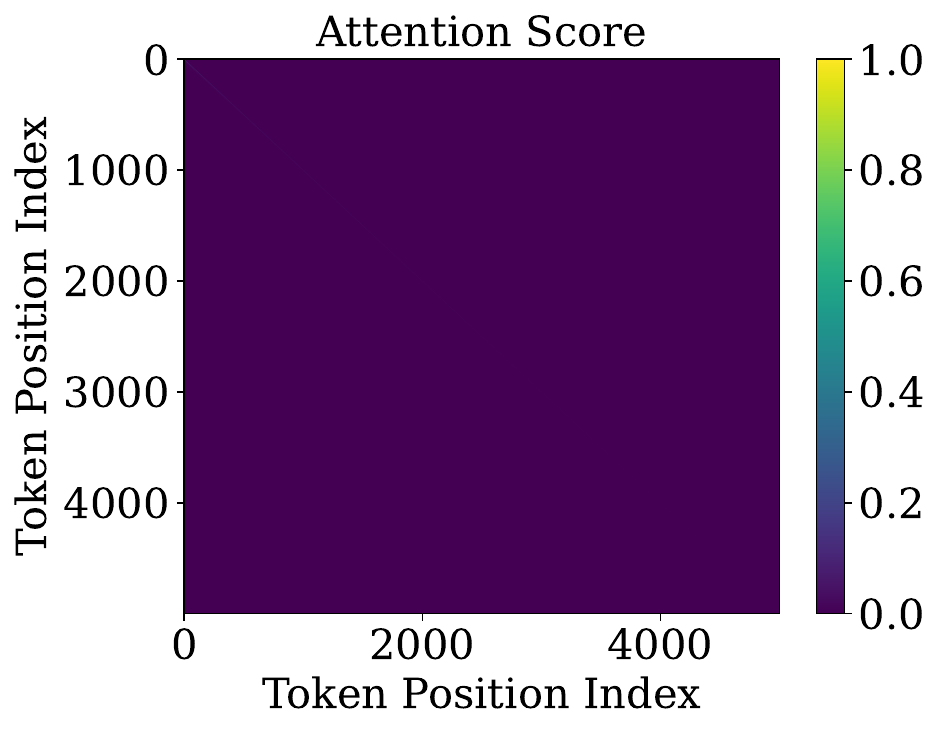}}
\quad
\subfigure[Average attention score matrix of $\rmL 1\rmH 16$.]{\includegraphics[width=0.3\textwidth]{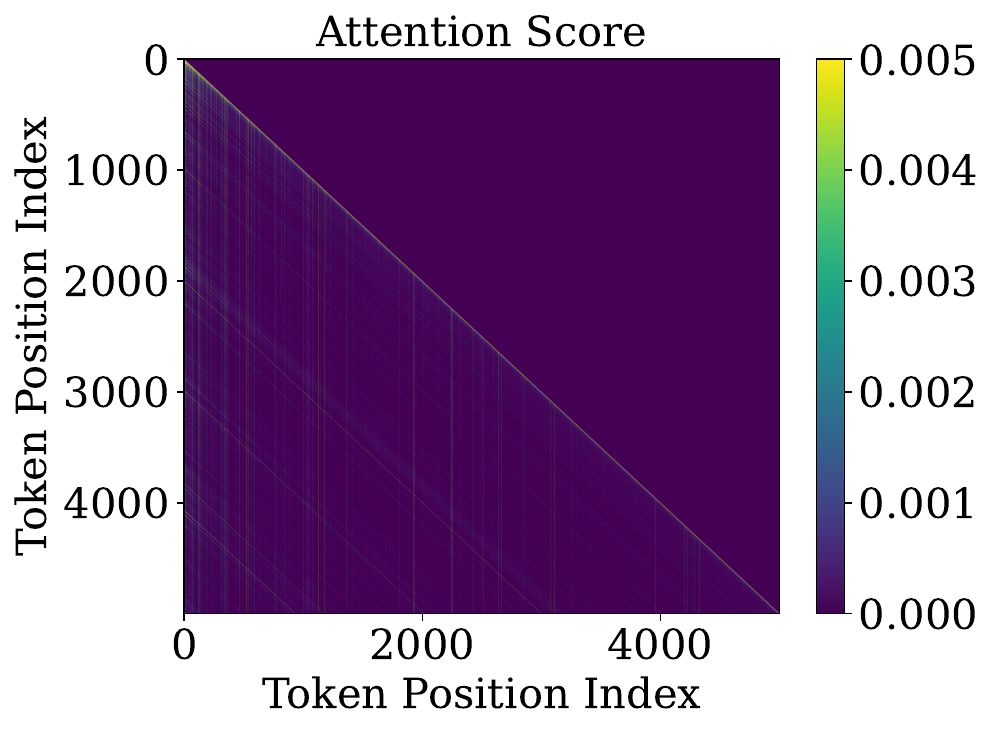}}
\caption{Average of attention score matrices in Llama3-8B-Instruct with prompts whose tokens are i.i.d.\ sampled from the uniform distribution on the alphabet.}
\label{fig:llama_large_ood}
\end{figure}

\begin{figure}[H]
\centering
\subfigure[Hidden state $H$ of $\rmL 0\rmH 29$ after PCA.]{\includegraphics[width=0.23\textwidth]{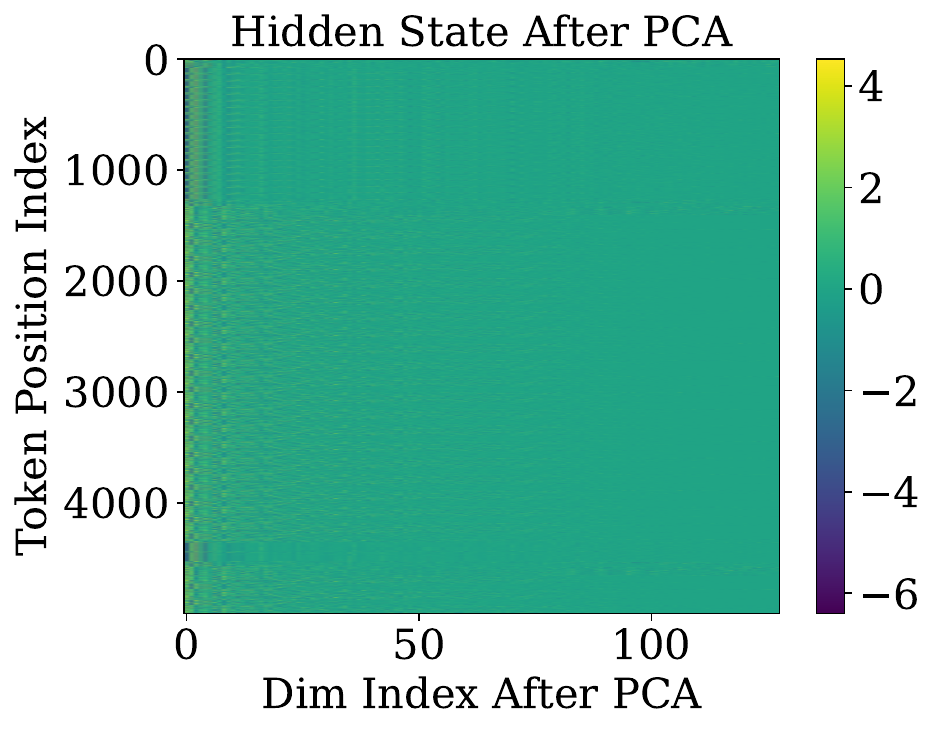}}
\hspace{0.08em}
\subfigure[Queries $Q$ of $\rmL 0\rmH 29$.]{\includegraphics[width=0.23\textwidth]{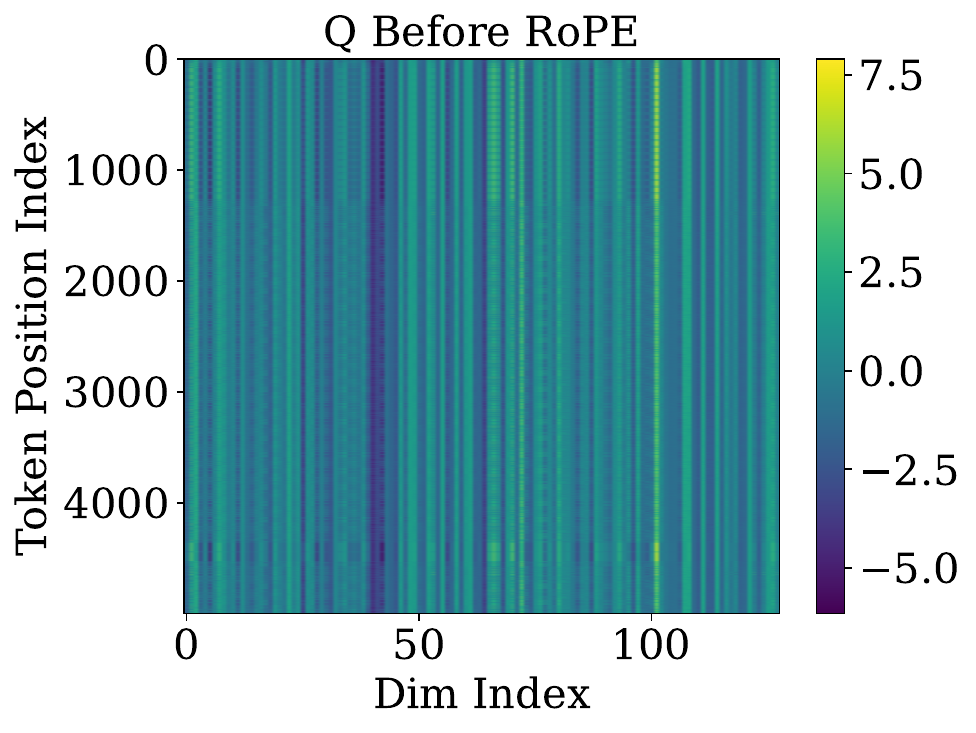}}
\hspace{0.08em}
\subfigure[Keys $K$ of $\rmL 0\rmH 29$.]{\includegraphics[width=0.23\textwidth]{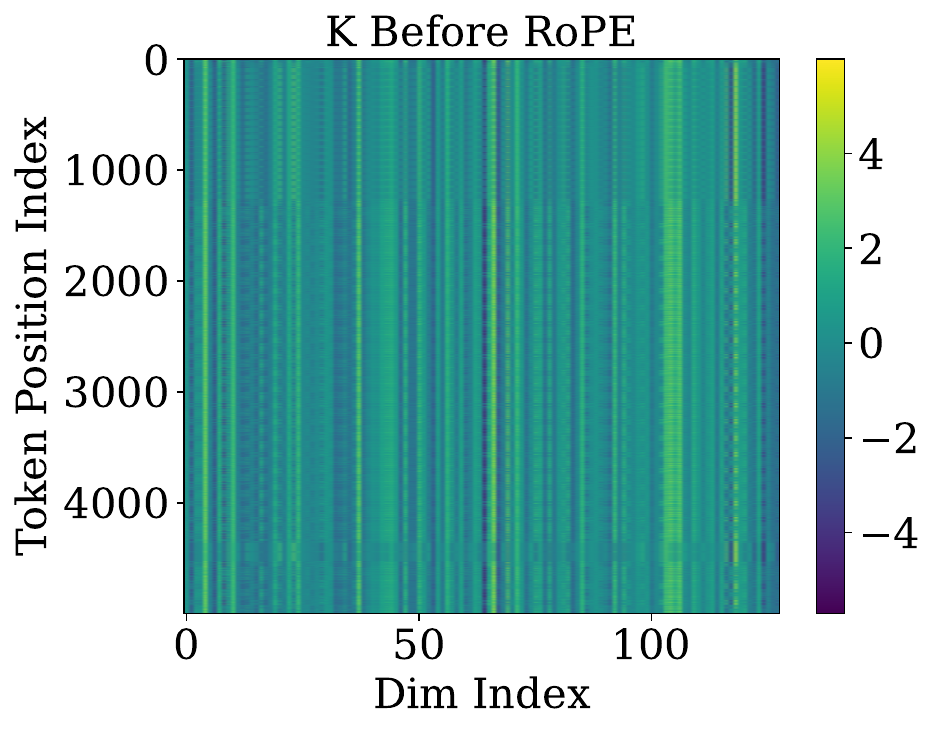}}
\hspace{0.08em}
\subfigure[$\InP(100,j,l)$ of $\rmL 0\rmH 29$.]{\includegraphics[width=0.25\textwidth]{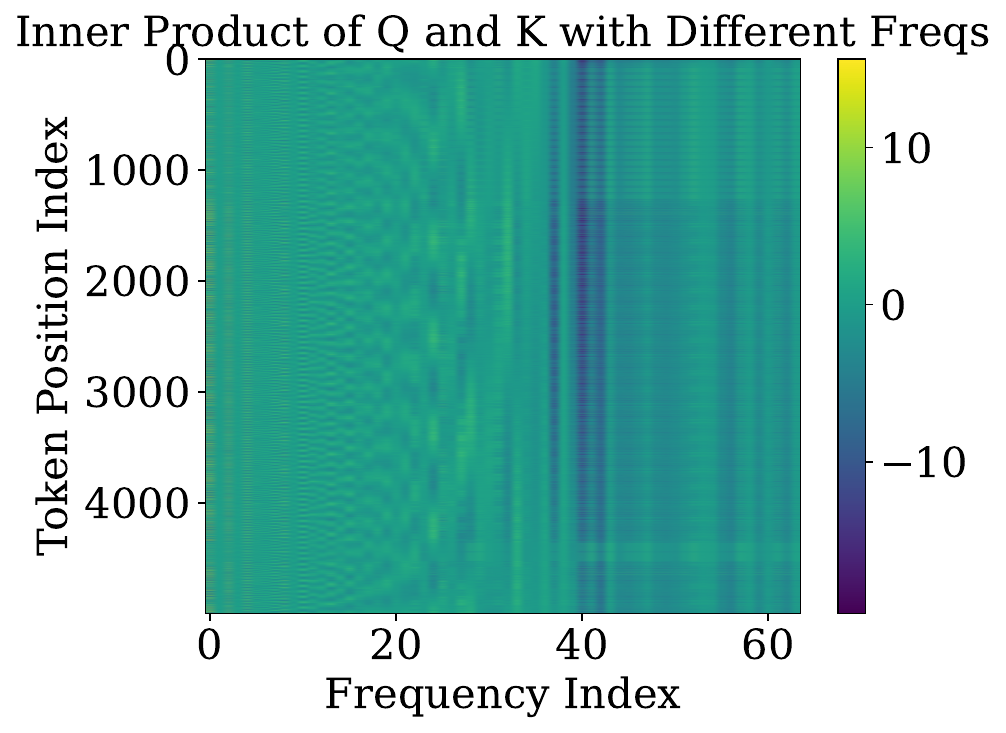}}

\subfigure[Hidden state $H$ of $\rmL 0\rmH 30$ after PCA.]{\includegraphics[width=0.23\textwidth]{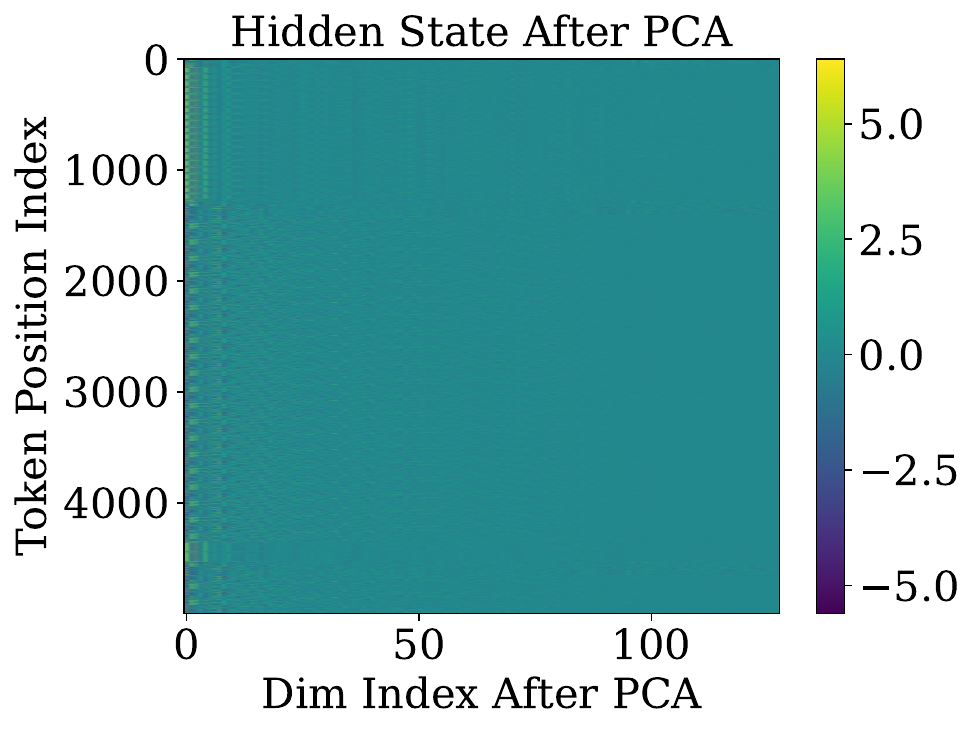}}
\hspace{0.08em}
\subfigure[Queries $Q$ of $\rmL 0\rmH 30$.]{\includegraphics[width=0.23\textwidth]{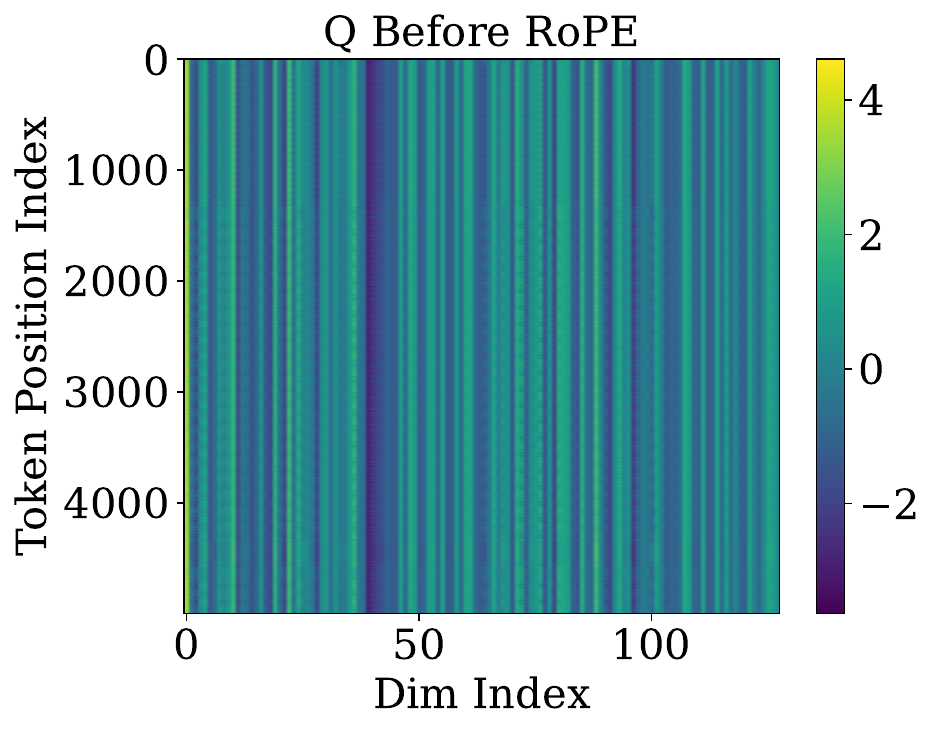}}
\hspace{0.08em}
\subfigure[Keys $K$ of $\rmL 0\rmH 30$.]{\includegraphics[width=0.23\textwidth]{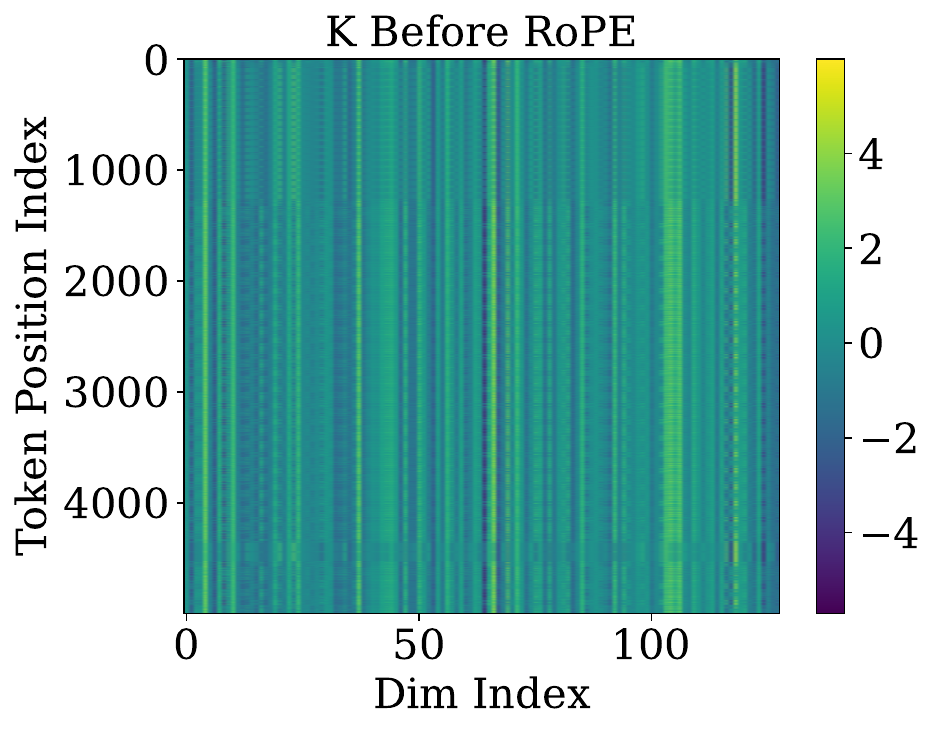}}
\hspace{0.08em}
\subfigure[$\InP(100,j,l)$ of $\rmL 0\rmH 30$.]{\includegraphics[width=0.25\textwidth]{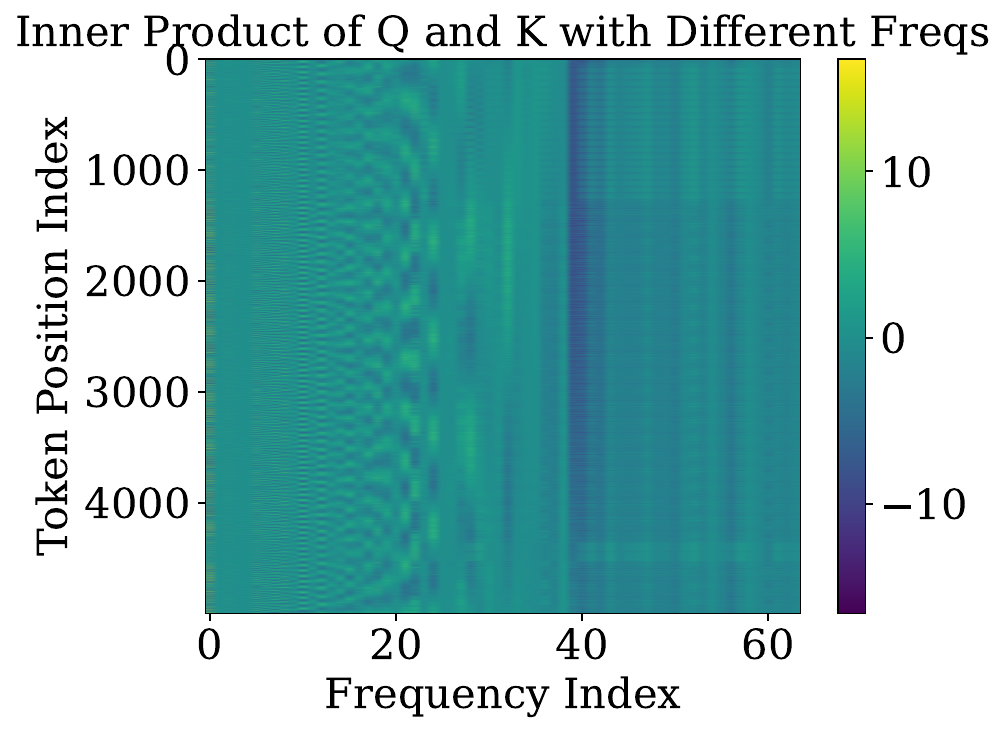}}

\subfigure[Hidden state $H$ of $\rmL 1\rmH 16$ after PCA.]{\includegraphics[width=0.23\textwidth]{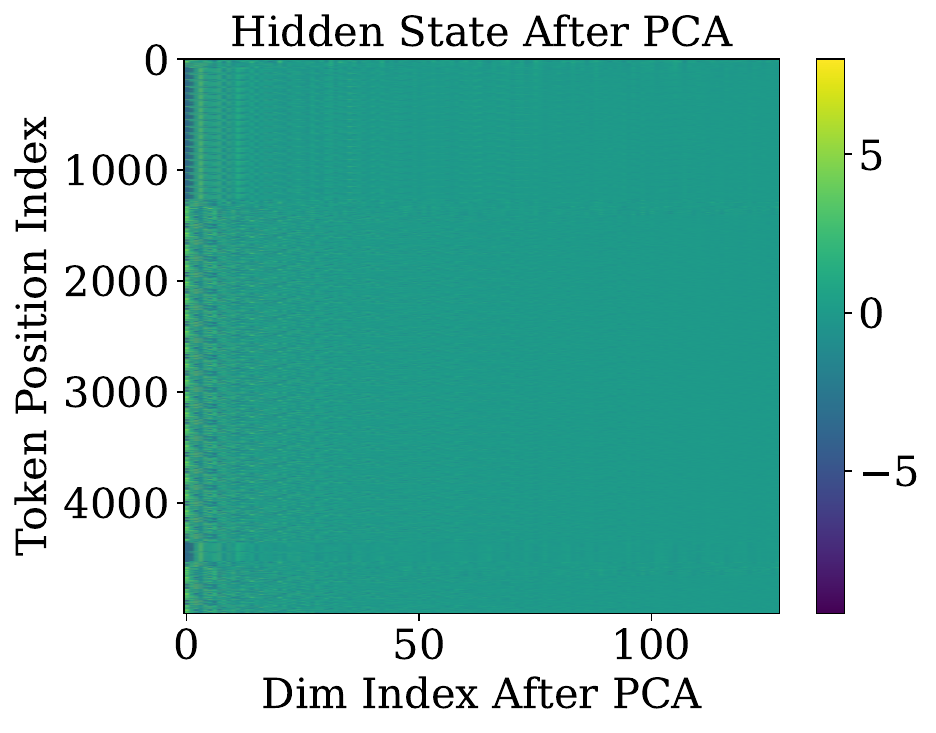}}
\hspace{0.08em}
\subfigure[Queries $Q$ of $\rmL 1\rmH 16$.]{\includegraphics[width=0.23\textwidth]{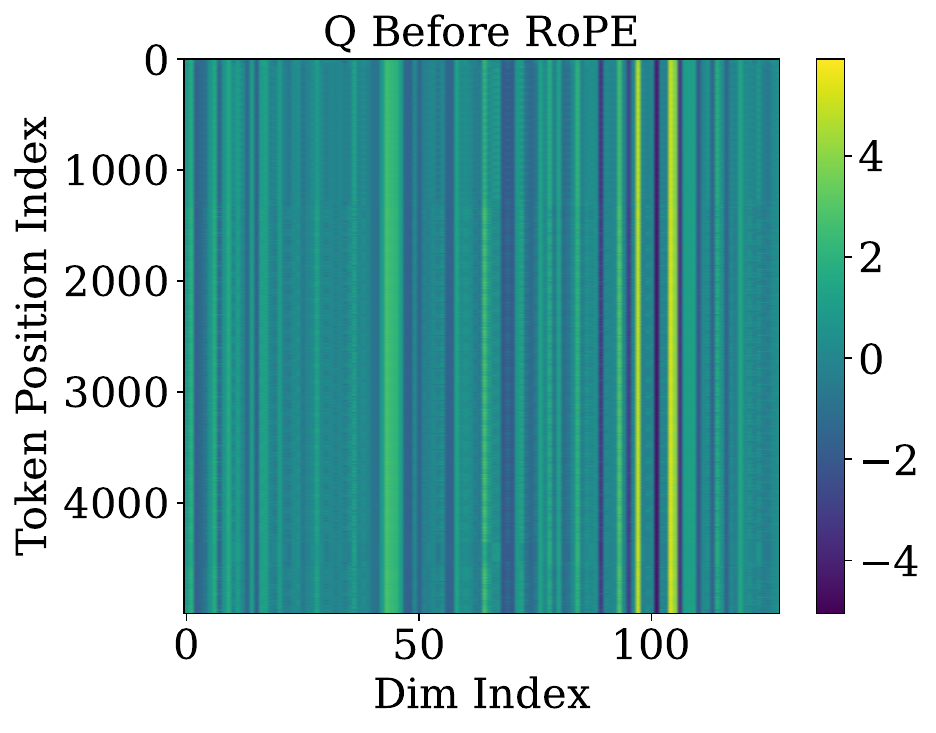}}
\hspace{0.08em}
\subfigure[Keys $K$ of $\rmL 1\rmH 16$.]{\includegraphics[width=0.23\textwidth]{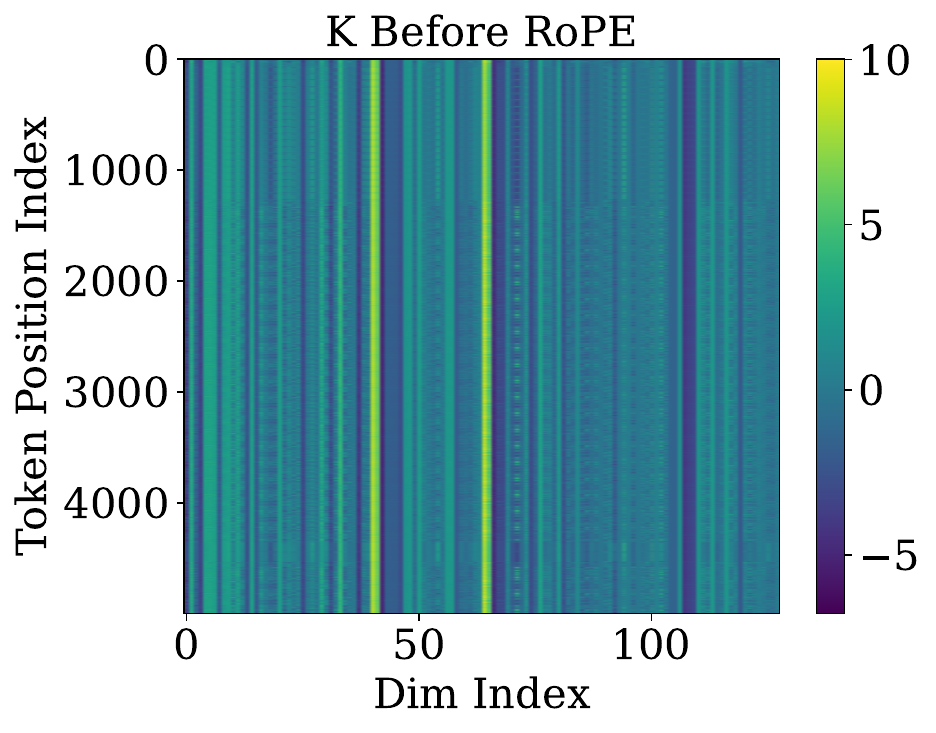}}
\hspace{0.08em}
\subfigure[$\InP(100,j,l)$ of $\rmL 0\rmH 1$.]{\includegraphics[width=0.25\textwidth]{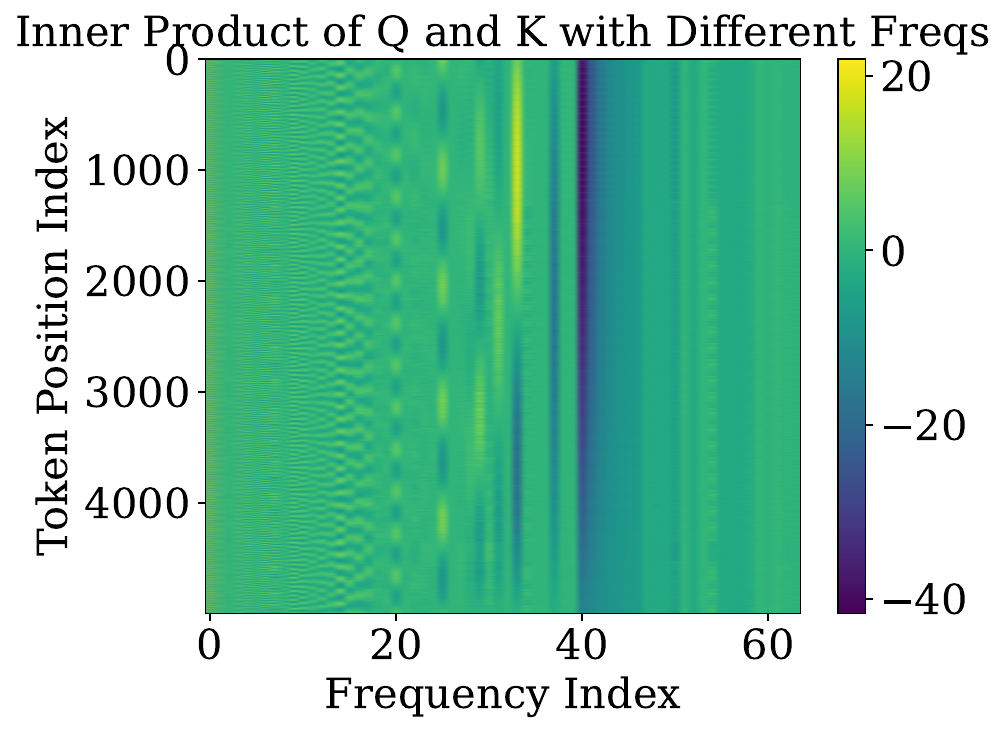}}
\caption{This figure shows the hidden states, queries, keys, and $\InP(100,j,l)$ for $j\in[100],l\in[64]$ for Llama3-8B-Instruct. Here, the example prompt contains $100$ tokens. The dimensions of queries and keys are both $128$.}
\label{fig:llama_qk_large}
\vspace{-1em}
\end{figure}

\subsection{Tables related to token embeddings in the 0-th layer}\label{app:tables_in_L0}
\begin{table}[h]
\centering
\begin{minipage}{0.48\textwidth}
\centering
\setlength{\tabcolsep}{0.9pt}
\begin{tabular}{ccccc}
        \hline
        Model & Head & $\text{RV}(\mathbf{v}_{Q})$ & $\text{RV}(\mathbf{v}_{K})$ & $\text{RV}(\mathbf{v}_{\text{rand}})$\\
        \hline
        Gemma & $\rmL 0\rmH 4$  & $8.0\%$ & $6.6\%$ & $953\%$ \\
        Gemma & $\rmL 0\rmH 7$  & $6.9\%$ & $6.2\%$ & $953\%$ \\
        Gemma & $\rmL 0\rmH 1$  & $4.3\%$ & $6.0\%$ & $953\%$ \\
        Llama3 & $\rmL 0\rmH 0$  & $34.1\%$ & $24.8\%$ & $1312\%$ \\
        Llama3 & $\rmL 0\rmH 2$  & $20.1\%$ & $24.8\%$ & $1312\%$ \\
        Llama3 & $\rmL 0\rmH 29$  & $9.8\%$ & $14.6\%$ & $1312\%$ \\
        Llama3 & $\rmL 0\rmH 30$  & $8.2\%$ & $14.6\%$ & $1312\%$ \\
        \hline
    \end{tabular}
    \caption{This table reports the relative variation of token embeddings projected onto the dominant subspace of the \acp{sdh} in the 0-th layer of Gemma-7B and Llama3-8B-Instruct. }
    \label{table:angle_qk_full}
\end{minipage}
\hspace{0.02\textwidth}
\begin{minipage}{0.48\textwidth}
\centering
\setlength{\tabcolsep}{0.9pt}
\begin{tabular}{ccccc}
        \hline
        Head & $\|W_Q^{\top}\hb_i\|$ & $\|\mathbf{b}_Q\|$ & $\|W_K^{\top}\hb_i\|$ & $\|\mathbf{b}_K\|$\\
        \hline
        $\rmL 0\rmH 5$ & $6.324$ & $8.518$ & $7.092$ & $466.607$ \\
        $\rmL 0\rmH 6$ & $6.439$ & $13.522$ & $7.092$ & $466.607$  \\
        $\rmL 0 \rmH 7$ & $7.403$ & $9.343$ & $10.039$ & $281.761$ \\
        $\rmL 0 \rmH 15$ & $12.112$ & $160.007$ & $7.469$ & $237.963$ \\
        $\rmL 0 \rmH 23$ & $5.632$ & $17.083$ & $9.173$ & $108.957$ \\
        $\rmL 0 \rmH 24$ & $5.782$ & $16.097$ & $9.173$ & $108.957$ \\
        $\rmL 0 \rmH 26$ & $5.810$ & $15.828$ & $9.173$ & $108.957$ \\
        \hline
    \end{tabular}
    \caption{This table lists the average norms of $\|W_Q^{\top}\hb_i\|$ and $\|W_K^{\top}\hb_i\|$ over the alphabet, together with the norms of $\|\mathbf{b}_Q\|$ and $\|\mathbf{b}_K\|$ for \ac{sdh} in the 0-th layer
of  Qwen2.5-7B-Instruct.}
    \label{table:sigma_qk_partial}
\end{minipage}
\end{table}

\section{Notations in Theory Sections and Expression of Reduced Model}\label{app: Notation}
\begin{table}[H]
    \centering
    \begin{tabularx}{\textwidth}{ll}
        \toprule
        \rowcolor{gray}
        \textbf{Notations} & \textbf{Descriptions} \\
        \midrule
        $A_{i,j}$ & The logit from position $j$ to $i$ before softmax in \textit{Layer 1} (\eqref{Eq: reduced model}).\\
        \midrule
        $\attn_{i,j}^{(1)}$ & The attention score from position $j$ to $i$ in \textit{Layer 1}.\\
        \midrule
         $A,\attn^{(1)} \in \RR^{N \times N}$ & The attention logit and attention score matrix in \textit{Layer 1}.\\
        \midrule
         $\attn^{(1)}_{i,\xb}$& Attention scores on $\xb$, i.e., odd position: $\sum_{j \leq i}\mathbbm{1}{\{j\equiv 1 \!\!\!\pmod{2} \}}\attn^{(1)}_{i,j}$.  
        \\
        \midrule
        $\attn^{(1)}_{i,y}$ & Attention scores on $y$, i.e., even position: $\sum_{j \leq i}\mathbbm{1}{\{j\equiv 0 \!\!\!\pmod{2} \}}\attn^{(1)}_{i,j}$. \\
        \midrule
        $\attn_{i}^{(2)}$, $\Attn_{k}^{(2)}$ & Attention scores to $i$-token and $k$-th feature in \textit{Layer 2} (\Cref{Attn_k}). \\
        \midrule
        $\overline{\attn}_i^{(2)},\overline{\Attn}_k^{(2)}$ & Approximate expectation version of $\attn_{i}^{(2)}$, $\Attn_{k}^{(2)}$ in Stage I (\Cref{lem: st1-aux2-S-concen}). \\
        \midrule
        $\ub$, $\attn^{(2)} \in \RR^{1 \times N}$ & The attention logit and score vector of the question $\Eq$ in \textit{Layer 2}.\\
        \midrule
        $\ub_i$ & The attention logit from $\Eq$ to $E_i$ in \textit{Layer 2} (\eqref{Eq: layer1-gdc-2} and \eqref{Eq: layer2-gdc-1}).\\
        \midrule
        $\tilde\cbb$ &
        Constant vector $(1,0,\cdots,1,0)^\top$.\\
        \bottomrule
    \end{tabularx}
    \caption{Summary of frequently used notations in proof. We omit $(t)$ here for simplicity.}
    \label{Table: Summary of Notations}
\end{table}
\noindent\textbf{Notations.} For two functions $f$ and $g$ of $n$, we write $f(n)=O(g(n))$ or $f(n) \lesssim g(n)$ if $g(n)\geq 0$, and there exist constants $C,n_0>0$ such that for all $n\ge n_0$, $f(n)\le C\,g(n)$, and $f(n)= -O(g(n))$ or $f(n) \gtrsim - g(n)$ if $g(n) \geq 0$, and there exist constants $C>0$ and $n_0$ such that for all $n\ge n_0$, $f(n)\ge -C\,g(n)$. Similarly, we write $f(n)=\Omega(g(n))$ or $f(n) \gtrsim g(n)$, if $g(n)\geq 0$, and there exist constants $C>0$ and $n_0$ such that for all $n\ge n_0$, $f(n)\ge C\,g(n)$, and we write $f(n)= - \Omega(g(n))$ or $f(n) \lesssim -g(n)$ if $g(n)\leq 0$, and there exist constants $C>0$ and $n_0$ such that for all $n\ge n_0$, $f(n)\le - C\,g(n)$. We write $f(n)=\Theta(g(n))$ or $f(n) \simeq g(n)$ if $f(n)=O(g(n))$ and $f(n)=\Omega(g(n))$.     

In addition, we abuse the notation $\mathbbm{1}{\{j=2r-1\}}$ to indicate that for a given $j$, if there exists $r \in \ZZ$, s.t. $j=2r-1$, then it takes 1, and 0 otherwise. For simplicity, we further denote different sub-matrices of $E$ as $E^\cbb=E_{:,1:d_\cbb} \in \RR^{N\times d_\cbb},E^{\xb,y}=E_{:,d_\cbb+1:d}\in \RR^{N\times (d_\cX+2)},E^y=E_{:,d} \in \RR^{N}$. For each $i$, we write $E^{\xb,y}_{i}=E^{\xb,y}_{i,:} \in \RR^{1 \times (d+2)}$, and we denote the last row by $\Eq^{\xb,y}=E^{\xb,y}_{N,:} \in \RR^{1 \times (d+2)}$. We summarize additional frequently used notations related to transformers in \Cref{Table: Summary of Notations}, and we omit the timestep $(t)$ for simplicity when there is no ambiguity.

\noindent\textbf{Expression of Reduced Model.} We first visualize the reduction of Layer 1 and 2 weights in \Cref{fig:layer1weight,fig:layer2weight}.

\begin{figure}[ht]
    \centering
    \includegraphics[width=\linewidth]{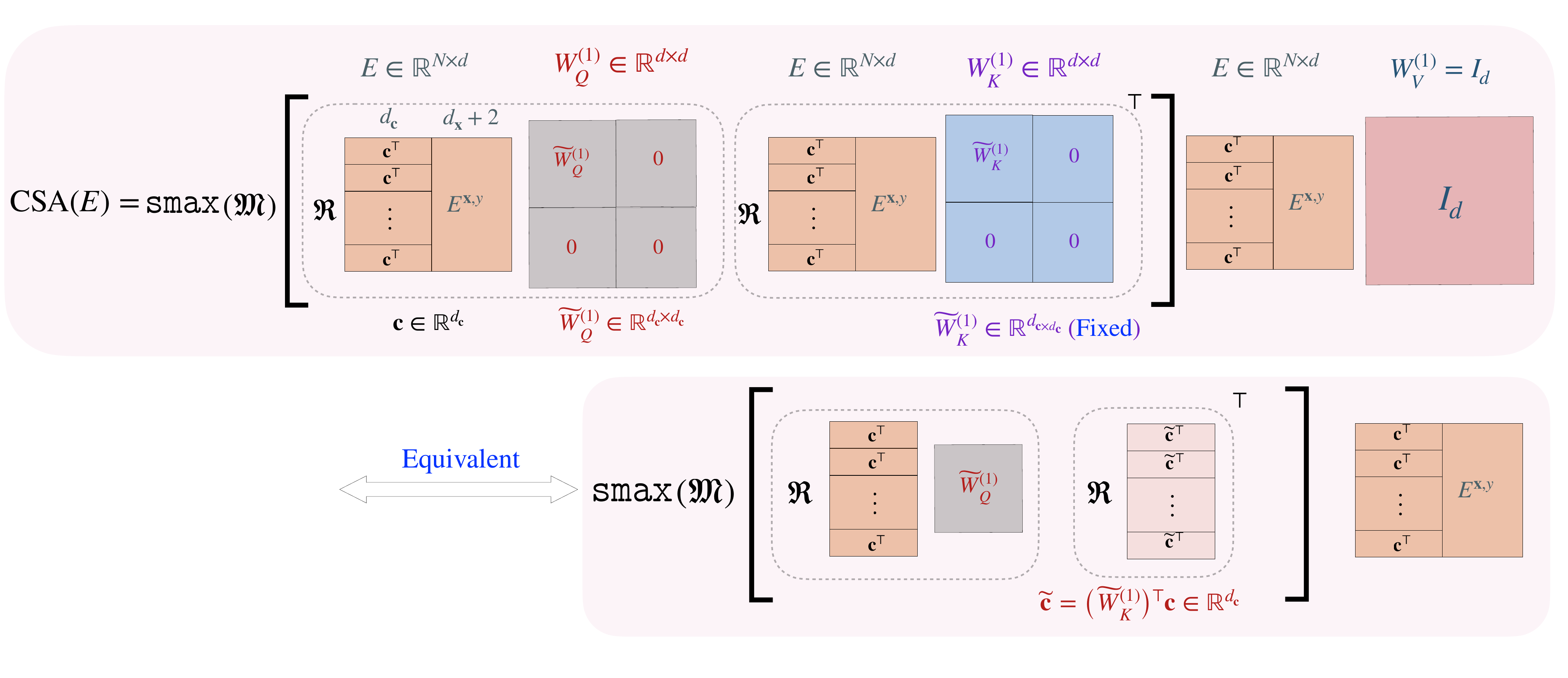}
    \caption{Illustration of the Reduction of Layer 1. In \Cref{fig:layer1weight,fig:layer2weight}, the softmax operator is abbreviated as sf. The output of Layer 1 is $H^{(1)}=[E,\CSA(E)]$. In the figure, we show the reduction of Layer 1 parameters $W_{\{Q,K,V\}}^{(1)}$ and how $\CSA(E)$ is generated by the reduced parameters. In particular, the weight matrices are only non-zero in the semantically independent subspace associated with $\cbb$, and the key and value weight matrices $W_{\{K,V\}}^{(1)}$ are fixed during training. Hence, $\CSA(E)$ only depends on $\widetilde{W}_Q^{(1)}$.}
    \label{fig:layer1weight}
\end{figure}

\begin{figure}[ht]
    \centering
    \includegraphics[width=\linewidth]{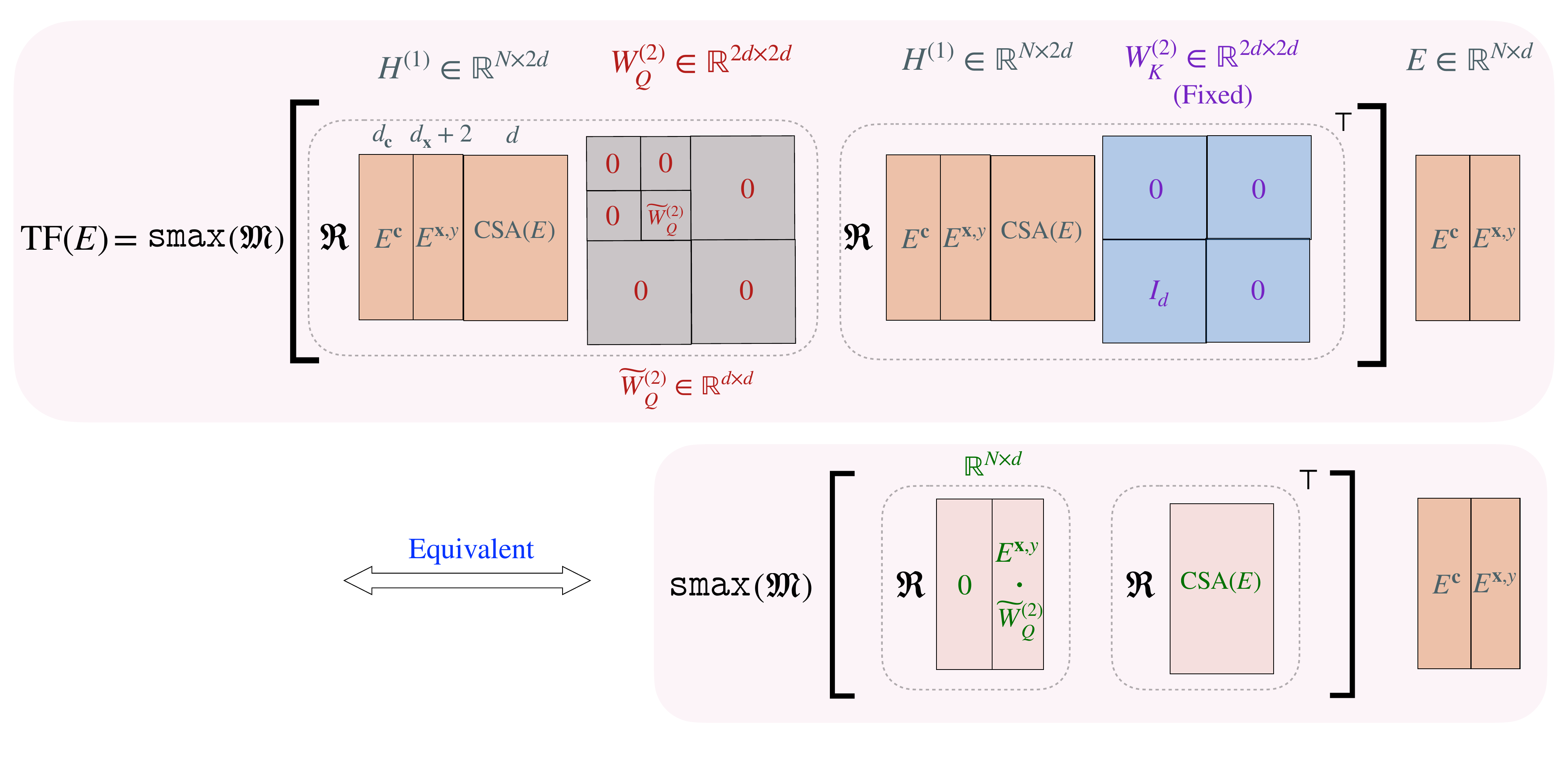}
    \caption{Illustration of the Reduction of Layer 2 and Transformer Output. First, with reduced $W_V^{(2)}=I_{2d}$  and $W_O=[0_{d\times d}\;\:0_{d\times d}\;\:I_{d}\;\:0_{d\times d}]^\top$, the CSA output $\CSA(H^{(1)})=S^{(2)}H^{(1)}W_V^{(2)}=S^{(2)}H^{(1)}\in\RR^{N\times 2d}$,  Layer 2 output $H^{(2)}=[E,\CSA(E),\CSA(H^{(1)})]\in\RR^{N\times 4d}$, and the transformer output ${{\mathrm{TF}}}_{\theta}(E) = H^{(2)}  W_O = S^{(2)}E$. Here, $S^{(2)}$ is the  Layer 2 attention score matrix. In this figure, we show the reduction of $S^{(2)}$ under sparse trainable query weight $W_{Q}^{(2)}$ and fixed key weight $W_{K}^{(2)}$. Hence, the output ${{\mathrm{TF}}}_{\theta}(E)$ only depends on $\widetilde{W}_Q^{(1)},\widetilde{W}_Q^{(2)}$.}
    \label{fig:layer2weight}
\end{figure}

Then we finally write the expression of the  prediction under the reduced model as follows:
\begin{align}
    \yq(E;\tilde\theta)= \softmax\left( \Re_{\bvartheta_{d_\cbb+1:d}}(\Eq^{\xb,y}\widetilde{W}_Q^{(2)})\Re_{\bvartheta_{d_\cbb+1:d}}\left(\text{softmax} \left(A\right)E^{\xb,y}\right)^\top \right)E^{y},\label{Eq: reduced model}
\end{align}
where $A_{i,j}=\cbb^\top \widetilde{W}_Q^{(1)}\RMat{\bvartheta_{1:d_\cbb}}{j-i}{d_\cbb}\tilde\cbb$ for $j \leq i$, and $-\infty$ otherwise, and the \ac{rope} operator $\Re_{\bvartheta_{d_\cbb+1:d}}$ and $\RMat{\bvartheta_{1:d_\cbb}}{j-i}{d_\cbb}$ are defined in  \eqref{Eq: rope}. 
\allowdisplaybreaks


\section{Proof of Stage I of \Cref{Thm: stage1 & 2}}\label{app: proof-st1}
\subsection{Roadmap of the Proof}
We analyze the emergence of slash-dominance in Stage I via three phases of dynamics (\Cref{app: s1p1,app: s1p2,app: s1p3}). In each phase, we formulate an induction hypothesis and derive several lemmas describing the values and evolution of key statistics, which collectively govern the attention scores. We prove the induction hypothesis and the lemmas by induction.  

The main idea of the proof lies in tracking the update dynamics of the attention logit in layer 1 $A_{l,r}{(t)}$, which decides the Layer 1 attention score $\attn_{l,r}^{(1)}(t)$. From \Cref{lem: st1-upA-com}, we find that for any $0\leq r \leq l \leq N$, the update of $A_{l,r}{(t)}: \Delta A_{l,r}{(t)}=A_{l,r}{(t+1)}-A_{l,r}{(t)}$ can be written as:
    \begin{align*}
        \Delta A_{l,r}{(t)}
    &
    = \eta_1 \cdot \Big(C_1\sum_{i=l-r+1}^{N}a_{i,i+r-l}(t)+C_2\sum_{1 \leq j\leq i \leq N}a_{i,j}(t)\pm \sum_{1 \leq j\leq i \leq N} a_{i,j}(t) O(\epsFN)\Big),
    \end{align*}
     where we denote
    \begin{align*}
    a_{i,j}(t)= \EE\Big[\left(\yq-\left\langle \wb, \xbq \right\rangle\right) \attn_{i}^{(2)}(t)(\yq-E_i^y)\attn_{i,j}^{(1)}(t)(I(i)_j-\sum_{\ell \leq i}\attn_{i,\ell}^{(1)}(t)I(i)_\ell) \Big].
\end{align*}
The value of $ \Delta A_{l,r}{(t)}$ largely depends on $a_{i,j}(t)$ for all $i,j \in [N]$. The expression of $a_{i,j}(t)$ can be further simplified by considering different cases of $(i,j)$, as shown in \Cref{lem: st1-upA-sim}, where we demonstrate that $a_{i,j}(t)$ is determined by the attention scores in Layer~1 and Layer~2, $\attn_{i,j}^{(1)}(t)$ and $\Attn^{(2)}_k(t)$. In addition, from \Cref{coro: stage1-aux2-S2-bd}, $\Attn^{(2)}_k(t)=\Theta(1/K)$ in stage I. A schematic of the interaction of statistics is shown in \Cref{Fig: Stage I sketch}. In addition, from the interaction above, we observe that for different $i, j \in [N]$, the attention scores $\attn_{i,j}^{(1)}(t)$ with the same offset $i - j$ are of the same order. The similar conclusion also holds for $\attn_{l,r}^{(1)}(t)$, $A_{l,r}^{(1)}(t)$ and the increments $\Delta A_{l,r}^{(1)}(t)$.  
\begin{figure}[t]
\centering
\begin{tikzpicture}[
  >=Latex,
  node distance = 3.0cm and 4.2cm,
  every node/.style = {font=\large, draw, rounded corners, inner sep=6pt, align=center},
  arrow/.style   = {->, line width=0.9pt},
  ann/.style     = {font=\tiny, draw=none},
  dashedbox/.style = {draw, dashed, rounded corners, inner sep=10pt, fit=#1}
]
\node (S1)  {$\attn^{(1)}_{\ell,r}(t),\; \ell,r\in [N]$};

\node (S2) [right=0.5cm of S1] {$\Attn^{(2)}_{k}(t),\; k\in [K]$};

\node (Aij) [below=3.0cm of S1] {$a_{i,j}(t),\; i,j\in[N]$};
\node (Alr) [right=7.0cm of Aij] {$\Delta A_{\ell,r}(t),\; \ell,r\in[N]$};

\node (S1') [above=3.0 of Alr] {$\attn^{(1)}_{\ell,r}(t+1),\; \ell,r\in [N]$};

\node (txt) at ($(S2)+(0,2)$) {$\Theta(\frac{1}{K})$ in Stage I};


\draw[arrow] (S1) -- node[midway, right=12pt, font=\tiny, align=left]
  {\Cref{lem: st1-upA-sim}} (Aij);
\draw[arrow] (Aij) -- node[midway, above=3pt, font=\tiny, align=center]
  {\Cref{lem: st1-upA-com}} (Alr);

\draw[arrow] (Alr) -- node[midway, right=12pt, font=\tiny, align=center]
  {$A_{\ell,r}(t+1)$} (S1');

\draw[arrow]  (txt) -- node[midway, right=12pt, font=\tiny, align=center]
  {\Cref{coro: stage1-aux2-S2-bd}} (S2);

\node[dashedbox=(S2)(S1)] (group) {};

\end{tikzpicture}
\caption{Schematic of dependencies among key statistics and attention scores at timestep $t$ and $t+1$ in Stage I.}
\label{Fig: Stage I sketch}
\end{figure}
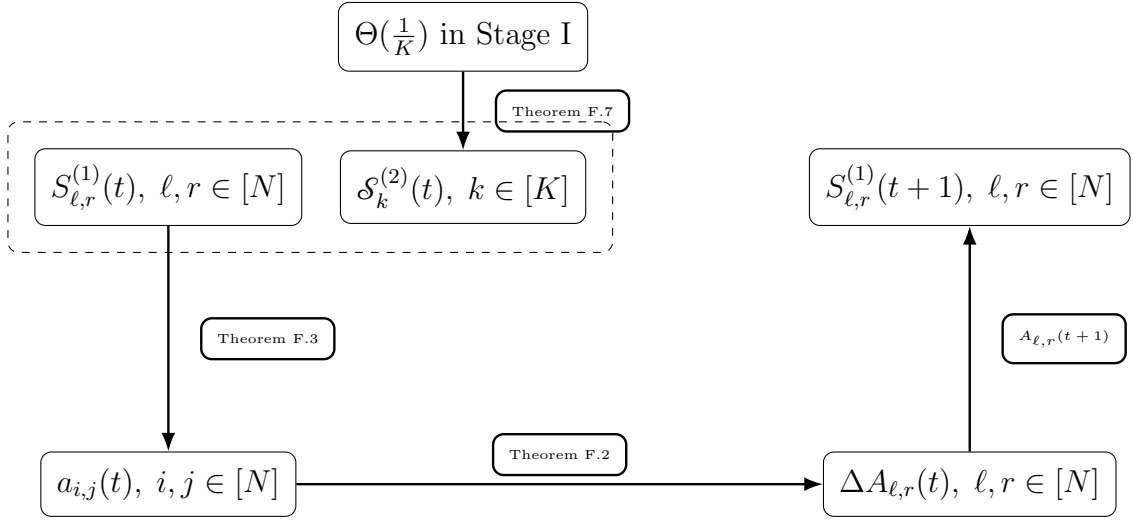

The learning process can be divided into three phases. Throughout all three phases, for any $l \in [N]$, the \emph{attention logit} corresponding to the immediately preceding token, $A_{l,l-1}{(t)}$, keeps growing, although its growth rate, denoted by $\Delta A_{l,l-1}{(t)}$, varies across the phases. In contrast, other logits $A_{l,r}{(t)}, r\neq l-1$ oscillate, typically with much smaller rate than the growth rate of $A_{l,l-1}{(t)}$. 
As a result, $S_{l,l-1}^{(1)}{(t)}$ keeps increasing and exhibits different orders of magnitude across the three phases as described below.
\begin{itemize}
    \item \textbf{Phase I: Emergence of Slash-Dominance.} ($t\in[0,T_{1}^{(1)}]$, \Cref{app: s1p1}). During phase I, for any $l \in [N]$, the \emph{attention scores} $\attn^{(1)}_{l,r}$ are relatively uniform across $r$: $\attn^{(1)}_{l,l-1}{(t)}=\Omega\left(l^{-1}\right)$, and
    $\attn^{(1)}_{l,r}{(t)}=O\left(l^{-1}\right), r \neq l-1$. As for the \emph{attention logits}, $A_{l,l-1}{(t)}$ keeps growing, but $A_{l,r}{(t)},r\neq l-1$ may oscillate. In addition, the growth rate $\Delta A_{l,l-1}{(t)}$ is much higher than that of other $A_{l,r}{(t)}$ with $r\neq l-1$, since for any $i \in [N]$, $a_{i,i-1}$ is always positive and has a much larger order of magnitude than $a_{i,j}$ with $j \neq i$. Specifically, from \Cref{lem: st1-ph1-A}, $\Delta A_{l,l-1}{(t)}-\max_{r \leq l, r \neq l-1} \Delta A_{l,r}{(t)}
    \gtrsim
     { \eta_1C_1 }{K^{-1}N^{-1}}$. Therefore, the increase in $A_{l,l-1}^{(t)}$ dominates the learning dynamics during phase I.
     \item \textbf{Phase II: Rapip Growth of Slash-Dominance.} ($t\in(T_{1}^{(1)},T_{2}^{(1)}]$, \Cref{app: s1p2}). After rapid growth of $A_{l,l-1}{(t)}$ in phase I, $\attn^{(1)}_{l,l-1}{(t)}=\Omega(1)$ grows to a constant order. However, for other $r \neq l-1$, $\attn^{(1)}_{l,r}{(t)} \leq O\left({N}^{-1}\right)$. The \emph{attention scores} $\attn^{(1)}_{l,r}$ is no longer uniform. The increase in $A_{l,l-1}^{(t)}$ still dominates the learning dynamics during phase II, and the growth rate gets much larger than Phase I. Specifically, from \Cref{lem: st1-ph2-A}, $
    \Delta A_{l,l-1}{(t)}-\max_{r \leq l, r \neq l-1} \Delta A_{l,r}{(t)}
    \gtrsim
     { \eta_1 C_1}{K^{-2}}$.
     \item \textbf{Phase III: Convergence} ($t\in(T_{2}^{(1)},T_{3}^{(1)}]$, \Cref{app: s1p3}). for any $l \in [N]$, $A_{l,l-1}{(t)}$ keep growing but in a smaller rate. As a result,  $\attn^{(1)}_{l,l-1}{(t)}$ keeps growing but can not exceed $1-\epsI$. Finally, after the end of phase III, at step $t=T_{3}^{(1)}+1$, $\attn^{(1)}_{l,l-1}{(t)}$ finally exceeds $1-\epsI$.  
\end{itemize}
We summarize the upcoming sections as follows: In \Cref{app: st1-gd}, we compute and simplify the gradients to identify the key update variables. In \Cref{app: stage1- aux lem}, we introduce several useful auxiliary lemmas for Stage I. In \Cref{app: s1p1,app: s1p2,app: s1p3}, we analyze the three phases of the dynamics.

\subsection{Stage I: Preliminary Development}\label{app: st1-gd}
\paragraph{Computations of Gradients.}
We first calculate Stage I gradient with respect to $\widetilde{W}_Q^{(1)}$. We omit $(t)$ in this section when there is no ambiguity and write $L(\tilde\theta)$ as $L$ here for simplicity.
\begin{lemma}[Layer 1 Gradient]\label{lem: st1-gdc} In Stage I, where $\widetilde{W}_Q^{(2)}$ is kept as $I_d$. The gradient of the loss function with respect to $\widetilde{W}_Q^{(1)}$ is given by
    \begin{align*}
        \nabla_{\widetilde{W}_Q^{(1)}}L=\EE\Bigg[\left(\yq-\left\langle \wb, \xbq \right\rangle\right)\sum_{j \leq i \leq N} \attn_{i}^{(2)}\left(E_i^y-\yq\right)\attn_{i,j}^{(1)}\Big(I(i)_j-\sum_{\ell \leq i}\attn_{i,\ell}^{(1)}I(i)_\ell\Big)\cbb\tilde\cbb^\top \RMat{\bvartheta}{i-j}{d_\cbb}\Bigg],
    \end{align*}
where $I(i)=E^{\xb,y} R_{\bvartheta,N-i} (\Eq^{\xb,y})^\top \in \RR^{N}$, capturing the correlation of the question embedding with the prompt embeddings. The notation $I(i)_\ell$ refers to the $\ell$-th entry of $I(i)$, and $E^{\xb,y},\Eq^{\xb,y}$ correspond to the semantic dependent subspaces of the token embeddings $E$ and $\Eq$, respectively.
\end{lemma}
\begin{proof}
We first write out the expression of $\yq$. In Stage I, $\widetilde{W}_Q^{(2)}$ is fixed as $I_d$, then
\begin{align}
     \yq(\widetilde{W}_Q^{(1)},\widetilde{W}_Q^{(2)})
    &  
    = \sum_{i=1}^{\Nit} \attn_{2i}^{(2)} \cdot y_i,\label{Eq: layer1-gdc-1}
\end{align}
where $\attn^{(2)}=\softmax(\ub) \in \RR^{1\times N}$, and
\begin{align}
    \ub_i
    &
    =  \attn^{(1)}_{i,:} I(i)
    =  (\attn^{(1)}_{i,:})\Eq^{\xb,y} R_{\bvartheta,i-N} (E^{\xb,y})^\top. 
   \label{Eq: layer1-gdc-2}
\end{align}
Here $\attn^{(1)}_{i,:}$ is the $i$-th row of the attention score matrix. Then we can compute the gradient. We first obtain:
    \begin{align}
    \nabla_{\widetilde{W}_Q^{(1)}} L= \mathbb{E}[\left(\yq\!-\!\left\langle \wb, \xbq \right\rangle\right)\nabla_{\widetilde{W}_Q^{(1)}}\yq ]
    =\mathbb{E}\Big[\left(\yq\!-\!\left\langle \wb, \xbq \right\rangle\right)\!\!\sum_{\ell\in[\Nit]}\!\!(\nabla_{{\widetilde{W}_Q^{(1)}}}\attn_{2\ell}^{(2)})y_\ell\Big]\!,\label{Eq: layer1-gdc-3}
\end{align}
where the last equation follows from \eqref{Eq: layer1-gdc-1}. We next compute $\nabla_{{\widetilde{W}_Q^{(1)}}}\attn_{2\ell}^{(2)}$ by chain rule. Recall that for any $U \in \RR^{1 \times d}$, $\nabla_U\softmax(U)=\diag(\softmax(U))-\softmax(U)^\top\softmax(U)$, we have 
\begin{align}
    \frac{\partial \attn^{(2)}_{2\ell} }{ \partial \attn^{(1)}_{i,n}}
    &
    =\sum_{m=1}^{N}\frac{\partial \attn^{(2)}_{2\ell} }{\partial\ub_m} \frac{\partial\ub_m}{ \partial \attn^{(1)}_{i,n}}
    \overset{ \RM{1}}{=}
    \frac{\partial \attn^{(2)}_{2\ell} }{\partial\ub_i} \frac{\partial\ub_i}{ \partial \attn^{(1)}_{i,n}}
    \overset{ \RM{2}}{=}(\attn_{i}^{(2)}\mathbbm{1}{\{2\ell=i\}}-\attn_{i}^{(2)}\attn_{2\ell}^{(2)}) \cdot I(i)_n, \label{Eq: layer1-gdc-4}
\end{align}
where $\RM{1}$ follows from  that $ \partial\ub_m /  \partial \attn^{(1)}_{i,n}=0$ if $m \neq i$, and $\RM{2}$ follows from the definition of $\ub_i$ in \eqref{Eq: layer1-gdc-2}. In addition,
\begin{align}
    \frac{\partial\attn^{(1)}_{i,n}}{\partial A_{i,j}}=\attn_{i,j}^{(1)}\mathbbm{1}{\{n=j\}}-\attn_{i,j}^{(1)}\attn_{i,n}^{(1)}. \label{Eq: layer1-gdc-5}
\end{align}
Then combine \eqref{Eq: layer1-gdc-4} and \eqref{Eq: layer1-gdc-5}, we have
\begin{align}
    \frac{\partial \attn^{(2)}_{2\ell} }{ \partial A_{i,j}}
    &
    \overset{\RM{1}}{=} \sum_{n\leq i}\frac{\partial \attn^{(2)}_{2\ell} }{ \partial \attn^{(1)}_{i,n}}\cdot \frac{\partial\attn^{(1)}_{i,n}}{\partial A_{i,j}}
    =\attn_{i}^{(2)}(\mathbbm{1}{\{2\ell=i\}}-\attn_{2\ell}^{(2)})\attn_{i,j}^{(1)}\Big(I(i)_j-\sum_{\ell \leq i}\attn_{i,\ell}^{(1)}I(i)_\ell\Big),\label{Eq: layer1-gdc-6}
\end{align}
where $\RM{1}$ follows from the chain rule and the definition of causal mask. In addition, recall, if $j\leq i$, then $A_{i,j} = \cbb^\top \widetilde{W}_Q^{(1)}\RMat{\bvartheta}{j-i}{d_\cbb}\tilde\cbb$, and we have
\begin{align}
    \nabla_{\widetilde{W}_Q^{(1)}}{A_{i,j}} = \cbb\tilde\cbb^\top (\RMat{\bvartheta}{j-i}{d_\cbb})^\top =  \cbb\tilde\cbb^\top \RMat{\bvartheta}{i-j}{d_\cbb}. \label{Eq: layer1-gdc-7}
\end{align}
As a result, combine \eqref{Eq: layer1-gdc-6} and \eqref{Eq: layer1-gdc-7}, we have 
\begin{align*}
    \nabla_{\widetilde{W}_Q^{(1)}}\attn^{(2)}_{2\ell} 
    & =  \sum_{j \leq i \leq N} \attn_{i}^{(2)}(\mathbbm{1}{\{2\ell=i\}}-\attn_{2\ell}^{(2)})\attn_{i,j}^{(1)}\Big(I(i)_j-\sum_{\ell \leq i}\attn_{i,\ell}^{(1)}I(i)_\ell\Big)\cbb\tilde\cbb^\top \RMat{\bvartheta}{i-j}{d_\cbb}.
\end{align*}
Then, by direct calculation, we have 
\begin{align}
    &\sum_{\ell\in[\Nit]}(\nabla_{\widetilde{W}_Q^{(1)}}\attn_{2\ell}^{(2)}) \cdot y_\ell  
    =\sum_{j \leq i \leq N} \attn_{i}^{(2)}(E_i^y-\yq)\attn_{i,j}^{(1)}\Big(I(i)_j-\sum_{\ell \leq i}\attn_{i,\ell}^{(1)}I(i)_\ell\Big)\cbb\tilde\cbb^\top \RMat{\bvartheta}{i-j}{d_\cbb}. \label{Eq: layer1-gdc-8}
\end{align}
As a result, plugging \eqref{Eq: layer1-gdc-8} into \eqref{Eq: layer1-gdc-3} proves \Cref{lem: st1-gdc}. Thus, we conclude the proof of Lemma~\ref{lem: st1-gdc}.
\end{proof}

\paragraph{Computations of Logits Update.} In stage I, the logits $A_{l,r}(t)$ determine the attention scores. Hence, to get the dynamics of attention scores, we only need to track the update of $A_{l,r}(t): \Delta A_{l,r}{(t)}=A_{l,r}{(t+1)}-A_{l,r}{(t)}$ as follows.
\begin{lemma}[Track $A_{l,r}(t)$ Update]\label{lem: st1-upA-com} Suppose Assumption \ref{Assp 1: Frequency seq} holds. In Stage I, for any $0\leq r \leq l \leq N$, we have 
    \begin{align*}
        \Delta A_{l,r}{(t)}
    &
    = \eta_1\Big(C_1\sum_{i=l-r+1}^{N}a_{i,i+r-l}(t)+C_2\sum_{1 \leq j\leq i \leq N}a_{i,j}(t)+  \sum_{1 \leq j\leq i \leq N} (\pm a_{i,j}(t)) O(\epsFN)\Big),
    \end{align*}
    where $C_1,C_2,\epsFN$ are parameters in Assumption \ref{Assp 1: Frequency seq}, and 
    \begin{align*}
    a_{i,j}(t)= \EE\Big[\left(\yq-\left\langle \wb, \xbq \right\rangle\right) \attn_{i}^{(2)}(t)(\yq-E_i^y)\attn_{i,j}^{(1)}(t)(I(i)_j-\sum_{\ell \leq i}\attn_{i,\ell}^{(1)}(t)I(i)_\ell) \Big].
\end{align*}
\end{lemma}
\begin{proof}
For any $1 \leq r\leq l \leq N$, $A_{l,r}(t) =\cbb^\top \widetilde{W}_Q^{(1)}(t)\RMat{\bvartheta}{r-l}{d_\cbb}\tilde\cbb$. Considering the GD update rule and applying \Cref{lem: st1-gdc}, we have
    \begin{align}
        & \Delta A_{l,r}{(t)}
         = -\eta_1 \cbb^\top(\nabla_{\widetilde{W}_Q^{(1)}}L(t)) \RMat{\bvartheta}{r-l}{d_\cbb}\tilde\cbb
        = - \eta_1\!\!\! \sum_{j \leq i \leq N}\!\!\!  a_{i,j}(t) \left(\tilde\cbb^\top \RMat{\bvartheta}{i-j+r-l}{d_\cbb}\tilde\cbb\right).\label{Eq: layer1-varup-1}
    \end{align}
Following from Assumption \ref{Assp 1: Frequency seq}, we have
\begin{align}
    \tilde\cbb^\top \RMat{\bvartheta}{i-j+r-l}{d_\cbb}\tilde\cbb =\sum_{s=1}^{d_{\cbb}/2} \cos{(\theta_s(i-j+r-l))} = C_1 \delta_{0}(i-j+r-l)+C_2\pm O(\epsFN). \label{Eq: layer1-varup-2}
\end{align}
By plugging \eqref{Eq: layer1-varup-2} into \eqref{Eq: layer1-varup-1}, we conclude the proof of \Cref{lem: st1-upA-com}.
\end{proof}

\paragraph{Simplification of Logits Update.}
From \Cref{lem: st1-upA-com}, we find that the update $\Delta A_{r,l}(t)$ depends significantly on $a_{i,j}(t)$. However, the 
expression of $a_{i,j}(t)$ is quite complicated and depends on $i,j$. To simplify it, we next characterize $a_{i,j}(t)$ by cases in the next lemma. We omit $(t)$ in \Cref{lem: st1-upA-sim} for abbreviation. The key auxiliary lemmas used in the characterization of $a_{i,j}(t)$ are deferred to \Cref{app: stage1- aux lem}.

\begin{lemma}[Characterization of $a_{i,j}$]\label{lem: st1-upA-sim}
In Stage I, consider $\overline{\attn}_i^{(2)}$ and $\overline{\Attn}_k^{(2)}$ are approximate expectation versions of $\attn_{i}^{(2)}$ and $\Attn_{k}^{(2)}$ for any $i \in [N]$ and $k \in [K]$. For different cases of $i \in [N],j \leq i$, the following holds. 
    \begin{enumerate}[leftmargin=*]
        \item If $i=2n$ and $j=2n-1$ for some $n \in [\Nit]$, then we have 
        \begin{align}
            a_{i,j}
            &
            \!\simeq\!
            \frac{\attn_{i,j}^{(1)}}{KN}\sum_{k=1}^K\Big\{
            \EE\Big[\mathbbm{1}{\{\xbq=\vb_k\}}\Big\{\Big(\frac{K-1}{K}(1-\overline{\Attn}_k^{(2)})+\frac{1}{K}\sum_{o\neq k}^K\overline{\Attn}_{o}^{(2)} \Big) \Big(\attn_{i,\xb}^{(1)}-\attn_{i,i-1}^{(1)}\Big)\nonumber\\
            & \qquad
           +\Big(\sum_{o=1}^K (\overline{\Attn}_o^{(2)})^2-2\overline{\Attn}_k^{(2)}+1 \Big)\attn_{i,y}^{(1)}\Big\}\Big]\pm O\bigg(\frac{1-\attn_{i,i-1}^{(1)}}{K}\Big(\frac{\log N}{N^2}+\frac{1}{N^{\alpha-1}}\Big)\bigg)\bigg\}\pm O\!\Big(\frac{1}{N^3}\!\Big)\!.\nonumber
        \end{align}
        \item If $i=2n$ and $j=2m-1$ for some $m < n \leq \Nit$, then we have
        \begin{align}
            a_{i,j}
            &
            \!\simeq\!
            \frac{\attn_{i,j}^{(1)}}{KN}\sum_{k=1}^K\Big\{\EE\Big[\mathbbm{1}{\{\xbq=\vb_k\}}\Big\{\Big(\sum_{o=1}^K (\overline{\Attn}_o^{(2)})^2-\overline{\Attn}_k^{(2)}-\frac{\sum_{o=1}^K\overline{\Attn}_{o}^{(2)}}{K}+\frac{1}{K} \Big) \Big(\attn_{i,y}^{(1)}+\attn_{i,i-1}^{(1)}\Big)\nonumber\\
            & \qquad
             -
             \!\!\Big(\sum_{o=1}^K (\overline{\Attn}_o^{(2)})^2\!-\!2\overline{\Attn}_k^{(2)} \!+1\Big)\attn_{i,i-1}^{(1)}
            \Big\}\Big]\!\pm\! O\!\bigg(\frac{(1\!-\!\attn_{i,j}^{(1)})\log(N)}{N^2}\bigg)\!\pm \!O\!\bigg(\frac{1\!-\!\attn_{i,j}^{(1)}}{K N^{\alpha-1}}\bigg)\bigg\}\!\pm\! O\!\Big(\!\frac{1}{N^3}\!\Big)\!.\nonumber
        \end{align}
        \item If $i=2n$ and $j=2m$ for some $m \leq n \leq \Nit$, then we have
        \begin{align}
            a_{i,j} 
            & 
            \!\simeq\! \frac{\attn_{i,j}^{(1)}}{KN}\sum_{k=1}^K\Big\{\EE\Big[\mathbbm{1}{\{\xbq=\vb_k\}}\Big\{\Big(\sum_{o=1}^K (\overline{\Attn}_o^{(2)})^2-\overline{\Attn}_k^{(2)}\!-\!\frac{\sum_{o=1}^K\overline{\Attn}_{o}^{(2)}}{K}\!+\!\frac{1}{K} \Big) 
            \Big(\!-\!\Big(\attn_{i,\xb}^{(1)}\!-\!\attn_{i,i-1}^{(1)}\Big)\Big)\nonumber\\
            & \qquad
            \Big(\sum_{o=1}^K (\overline{\Attn}_o^{(2)})^2-2\overline{\Attn}_k^{(2)}+1\Big)\!\!\Big(\!-\!\attn_{i,i-1}^{(1)}\Big) \Big\}\Big]
            \!\pm\! O\bigg(\frac{\attn_{i,\xb}^{(1)}\log(N)}{N^2}\bigg)\!\pm\! O\bigg(\frac{\attn_{i,\xb}^{(1)}}{K N^{\alpha-1}}\bigg)\bigg\}\pm O\!\Big(\!\frac{1}{N^3}\!\Big)\!.\nonumber
        \end{align}
        \item If $i=2n-1$ and $j=2m$ for some $m \leq n-1 \leq \Nit -1$, then we have
        \begin{align}
             a_{i,j} 
            & 
            \!\simeq\! \frac{\attn_{i,j}^{(1)}}{KN}\sum_{k=1}^K \bigg\{\!\EE\!\Big[\mathbbm{1}{\{\xbq=\vb_k\}}\Big(\sum_{o=1}^K (\overline{\Attn}_o^{(2)})^2\!-\!\overline{\Attn}_k^{(2)} \Big)\!\!\Big(\!-\!{\attn_{i,\xb}^{(1)}}\Big) \Big]\!\pm\! O\bigg(\!\frac{\attn_{i,\xb}^{(1)}\sqrt{\log N}}{N}\!\bigg)\!\pm\! O\!\bigg(\!\frac{\attn_{i,\xb}^{(1)}}{KN^{\alpha-1}}\!\bigg)\!\bigg\}\nonumber\\
            &
            \qquad\pm O\!\Big(\!\frac{1}{N^3}\!\Big)\!.\nonumber
        \end{align}
        \item If $i=2n-1$ and $j=2m-1$ for some $m \leq n \leq \Nit$, then we have
        \begin{align}
             a_{i,j} 
            & 
            \!\simeq\! \frac{\attn_{i,j}^{(1)}}{KN}\!\sum_{k=1}^K\!\bigg\{\!\EE\!\Big[\!\mathbbm{1}\!{\{\xbq=\vb_k\}}\!\Big(\sum_{o=1}^K (\overline{\Attn}_o^{(2)})^2\!-\!\overline{\Attn}_k^{(2)} \!\Big)\attn_{i,y}^{(1)} \Big]\!\pm\! O\!\bigg(\!\frac{(1\!-\!\attn_{i,j}^{(1)})\sqrt{\log N}}{N}\!\bigg)\!\pm\! O\!\bigg(\!\frac{1\!-\!\attn_{i,j}^{(1)}}{N^{\alpha-1}}\!\bigg)\!\bigg\}\nonumber\\ 
            &
            \qquad
            \pm O\!\Big(\!\frac{1}{N^3}\!\Big).\nonumber
        \end{align}
    \end{enumerate}
\end{lemma}

\begin{proof}
First, we notice that
\begin{align*}
     a_{i,j}
    &
    \!= \!\EE\Big[\Big(\sum_{o=1}^K \Attn_o^{(2)} \vb_o\!-\!\xbq \Big)^{\!\!\top}\!\!\!\wb \wb^\top \!\Big(\sum_{o=1}^K \Attn_o^{(2)} \vb_o\!-\!\mathbbm{1}{\{i\equiv 0 \!\!\!\!\pmod{2}\}}\xb_{\frac{i}{2}}\Big)\attn_{i}^{(2)}\attn_{i,j}^{(1)}(I(i)_j\!-\!\sum_{\ell \leq i}\attn_{i,\ell}^{(1)}I(i)_\ell)\Big]\\
    &
    \!\overset{\RM{1}}{=} \!\EE\Big[\!\Big(\sum_{o=1}^K \Attn_o^{(2)} \vb_o-\xbq \Big)^{\!\!\top} \! \bigg(\sum_{o=1}^K \Attn_o^{(2)} \vb_o-\mathbbm{1}{\{i\equiv 0 \!\!\!\!\pmod{2}\}}\xb_{\frac{i}{2}}\bigg)\attn_{i}^{(2)}\attn_{i,j}^{(1)}(I(i)_j-\sum_{\ell \leq i}\attn_{i,\ell}^{(1)}I(i)_\ell) \Big]\\ 
    &
    \!\overset{\RM{2}}{\simeq}\! \EE\Big[\Big(\sum_{o=1}^K \Attn_o^{(2)} \vb_o-\xbq\Big)^{\!\!\top} \!\!\! \Big(\sum_{o=1}^K \Attn_o^{(2)} \vb_o\!-\!\mathbbm{1}{\{i\equiv 0 \!\!\!\!\pmod{2}\}}\xb_{\frac{i}{2}}\Big)\!\attn_{i}^{(2)}\attn_{i,j}^{(1)}\Big(I(i)_j\!-\!\sum_{\ell \leq i}\attn_{i,\ell}^{(1)}I(i)_\ell\Big)\mathbbm{1}\{\Es\}\Big]\\ 
    &
    \quad
    + 2\PP({\Es}^c),
\end{align*}
where $\RM{1}$ follows from taking the conditional expectation over $\wb$, and $\RM{2}$ follows from the fact that $\left|\left(\sum_{o=1}^K \Attn_o^{(2)} \vb_o\!-\!\xbq \right)^\top\!\!  \left(\sum_{o=1}^K \Attn_o^{(2)} \vb_o\!-\!\mathbbm{1}{\{i\equiv 0 \!\!\pmod{2}\}}\xb_{\frac{i}{2}}\right)\attn_{i}^{(2)}\attn_{i,j}^{(1)}(I(i)_j\!-\!\sum_{\ell \leq i}\attn_{i,\ell}^{(1)}I(i)_\ell)\right| \!\leq\! 2$.

By \Cref{lem: st1-aux2-S-concen}, and taking $\delta=N^{-3}$, we can replace the random variables $\Attn_o^{(2)}$ and $\attn_i^{(2)}$ with their deterministic approximation,  which enables a simpler form of $a_{i,j}$. Then we have
\begin{align}
    a_{i,j}
    &
    \!\simeq\!  \EE\!\bigg[\bigg(\sum_{o=1}^K \Big(\!\overline{\Attn}_o^{(2)}\!\pm\! O\bigg(\!\frac{\sqrt{\log(N)}}{N}\bigg)\bigg) \!\vb_o\!-\!\xbq \!\bigg)^{\!\!\top}\! \bigg(\sum_{o=1}^K\! \Big(\overline{\Attn}_o^{(2)}\!\pm\! O\bigg(\!\frac{\sqrt{\log(N)}}{N}\bigg)\!\bigg) \vb_o\!-\!\mathbbm{1}{\{i\!\equiv \!0\! \!\!\!\!\pmod{2}\}}\xb_{\frac{i}{2}}\!\bigg)\nonumber\\
    & \quad \overline{\attn}_i^{(2)}\big(1\pm O\big(\min\big\{1,{i}^{-1}{\sqrt{\log({2N}/{\delta})}}\big\}\big)\big)\attn_{i,j}^{(1)}\Big(I(i)_j-\sum_{\ell \leq i}\attn_{i,\ell}^{(1)}I(i)_\ell\Big) \Big]\pm O\!\Big(\!\frac{1}{N^3}\!\Big)\nonumber\\
    & 
    \!\overset{\RM{1}}{\simeq}\!
    \EE\bigg[\bigg\{\!\Big(\sum_{o=1}^K \overline{\Attn}_o^{(2)}\vb_o\!-\!\xbq\!\Big)^{\!\!\top}\! \Big(\sum_{o=1}^K\overline{\Attn}_o^{(2)}\vb_o\!-\!\mathbbm{1}{\{i\!\equiv \!0\! \!\!\!\!\pmod{2}\}}\xb_{\frac{i}{2}}\!\Big)\!\nonumber\\
    &  \quad
   \pm O\bigg(\bigg(\!1\!-\! \mathbbm{1}{\{i\!\equiv \!0\! \!\!\!\!\pmod{2}\}}\!\!\frac{\sqrt{\log(N)}}{N}\!\bigg)\!\frac{\sqrt{\log(N)}}{N}\!\bigg)\bigg\}\Big(\!I(i)_j\!-\!\sum_{\ell \leq i}\attn_{i,\ell}^{(1)}I(i)_\ell\!\Big) \bigg] \frac{\attn_{i,j}^{(1)}}{N}\pm O\Big(\!\frac{1}{N^3}\!\Big),\label{Eq: layer1-gdsim-1}
\end{align}
where $\RM{1}$ follows from that we only care about the order, and $\overline{\attn}^{(2)}_i=\Theta({1}/{N})$. Notice that the choice of $N^{-3}$ here ensures that it remains a negligible term in the later analysis. Starting from \eqref{Eq: layer1-gdsim-1}, we consider by cases. We only prove Case (1), and the other cases follow by the same argument as Case (1).

{Case (1): $i=2n,j=2n-1$.} 
Because $i=2n$ is even, by direct calculation, we have
\begin{align}
    a_{i,j}
    &
    = \EE \bigg[\bigg(\Big(\sum_{o=1}^K \overline{\Attn}_o^{(2)} \vb_o-\xbq \Big)^{\!\!\top}\!\! \Big(\sum_{o=1}^K \overline{\Attn}_o^{(2)} \vb_o-\xb_n\Big)\pm O\Big(\frac{{\log(N)}}{N^2}\Big)\bigg)\!\Big(I(i)_j-\sum_{\ell \leq i}\attn_{i,\ell}^{(1)}I(i)_\ell\Big) \bigg]\cdot \frac{\attn_{i,j}^{(1)}}{N}\nonumber\\
    & \quad
     \pm O\!\bigg(\!\frac{1}{N^3}\!\bigg)\nonumber\\
    &
    =\sum_{k=1}^K\EE\bigg[\mathbbm{1}{\{\xbq=\vb_k\}}\bigg(\sum_{o=1}^K (\overline{\Attn}_o^{(2)})^2-\overline{\Attn}_k^{(2)}-\sum_{o=1}^{K}\overline{\Attn}_o^{(2)}\mathbbm{1}{\{\xb_{n}=\vb_o\}}+\mathbbm{1}{\{\xb_{n}=\vb_k\}} \pm O\bigg(\frac{{\log(N)}}{N^2}\bigg)\bigg)\nonumber\\
    & \qquad  
    \bigg(I(i)_j-\sum_{\ell \leq i}\attn_{i,\ell}^{(1)}I(i)_\ell\bigg) \bigg]\frac{\attn_{i,j}^{(1)}}{N}\pm O\!\bigg(\!\frac{1}{N^3}\!\bigg)\nonumber\\
    &
    =\frac{\attn_{i,j}^{(1)}}{KN}\sum_{k=1}^K (\URM{1})+O\!\bigg(\!\frac{1}{N^3}\!\bigg). \label{Eq: layer1-aijsim-1}
\end{align}
Here, the dominant term $(I)$ is defined and further simplified as follows
\begin{align}
 (\URM{1})
 &
 =\!\sum_{k'=1}^K\EE\bigg[\mathbbm{1}{\{\xbq=\vb_k\}}\!\bigg(\sum_{o=1}^K (\overline{\Attn}_o^{(2)})^2\!-\!\overline{\Attn}_k^{(2)}\!-\!\overline{\Attn}_{k'}^{(2)}\!+\!\mathbbm{1}{\{k'=k\}}\pm O\bigg(\frac{{\log(N)}}{N^2}\bigg)\!\bigg)\Big(\!I(i)_j\!-\!\sum_{\ell \leq i}\attn_{i,\ell}^{(1)}I(i)_\ell\Big) \bigg]\nonumber\\   
    &
    \overset{\RM{1}}{=}  \EE\bigg[\mathbbm{1}{\{\xbq=\vb_k\}}\sum_{k'=1}^K\bigg(\sum_{o=1}^K (\overline{\Attn}_o^{(2)})^2-\overline{\Attn}_k^{(2)}-\overline{\Attn}_{k'}^{(2)} \bigg)\!\bigg(-\frac{1}{K}\sum_{\ell \leq i-2}\attn_{i,\ell}^{(1)}\mathbbm{1}{\{\ell=2r-1\}}\bigg) \bigg]\nonumber\\
    & \quad+
     \EE\bigg[\mathbbm{1}{\{\xbq=\vb_k\}}\bigg(-\sum_{\ell \leq i-2}\frac{1}{K}\attn_{i,\ell}^{(1)}\mathbbm{1}{\{\ell=2r-1\}}+\bigg(\sum_{o=1}^K (\overline{\Attn}_o^{(2)})^2-2\overline{\Attn}_k^{(2)}+1 \bigg)\!\big(1-\attn_{i,i-1}^{(1)}\big)\bigg)\bigg]\nonumber\\
     & \quad
    \pm O\bigg(\frac{(1-\attn_{i,i-1}^{(1)})\log(N)}{KN^2}\bigg)\pm O\bigg(\frac{1-\attn_{i,i-1}^{(1)}}{K N^{\alpha-1}}\bigg)\nonumber\\
     & 
     = \EE\bigg[\mathbbm{1}{\{\xbq=\vb_k\}}\bigg\{\bigg(\frac{K-1}{K}(1-\overline{\Attn}_k^{(2)})+\frac{1}{K}\sum_{o\neq k}^K\overline{\Attn}_{o}^{(2)} \bigg)\!\big(\attn_{i,\xb}^{(1)}-\attn_{i,i-1}^{(1)}\big)\nonumber\\ 
     & \quad
     + \bigg(\sum_{o=1}^K (\overline{\Attn}_o^{(2)})^2-2\overline{\Attn}_k^{(2)}+1 \bigg)\attn_{i,y}^{(1)}\bigg\}\bigg]\pm O\bigg(\frac{(1-\attn_{i,i-1}^{(1)})}{K}\bigg(\frac{\log N}{N^2}+\frac{1}{N^{\alpha-1}}\bigg)\bigg), \label{Eq: layer1-aijsim-2}
\end{align}
where $\RM{1}$ follows from \Cref{coro: stage1-aux1-u-bd}, which states that $\left|I(i)_j - \mathbbm{1}{\{j=2r-1,\xb_{r} = \vb_k\}}\right| \leq {5}N^{1-\alpha}$ when $\xbq=\vb_k$. It also uses the fact that $\{\xb_{r}\}_{r=1}^{\lfloor\frac{i-1}{2}\rfloor}$ is independent with $\{\overline{\Attn}^{(2)}_{o}\}_{o=1}^K,\{\attn^{(1)}_{i,j}\}_{j=1}^{i-1}$, and taking partial expectation over $\{\xb_{r}\}_{r=1}^{\lfloor\frac{i-1}{2}\rfloor}$.

Plugging \eqref{Eq: layer1-aijsim-2} into \eqref{Eq: layer1-aijsim-1} proves the result of Case (1). 
The analyses for the other cases proceed analogously to Case (1), and we therefore omit the details. This concludes the proof of \Cref{lem: st1-upA-sim}.

\end{proof}



\subsection{Stage I: Auxiliary Lemmas}\label{app: stage1- aux lem}
The next lemma characterizes the influence of the low-frequency components of RoPE on the tokens, which will be useful for simplifying $I(i)=E^{\xb,y} R_{\bvartheta,N-i}(\Eq^{\xb,y})^\top$.
\begin{lemma}\label{lem: st1-aux1}
    Suppose Assumption \ref{Assp 1: Frequency seq} holds. Let $E^{\xb,y},\Eq^{\xb,y}$ denote the semantic dependent subspaces of the token embeddings $E$ and $\Eq$. Then, for any $1 \leq j\leq i \leq N$, the partial token embedding $E^{\xb,y}_j$ at position $j$ satisfies:
    \begin{enumerate}[leftmargin=*]
        \item If $j$ is odd, and $E^{\xb,y}_j$ has the same feature with $\Eq^{\xb,y}$, i.e. $E_j^{\xb,y}=\Eq^{\xb,y}$, then $$1-E^{\xb,y}_j R_{\bvartheta,N-i} (\Eq^{\xb,y})^\top\leq \frac{2}{N^{2(\alpha-1)}}.$$ 
        \item If $j$ is odd, and $E_j^{\xb,y}$ has a different feature with $\Eq^{\xb,y}$, i.e. $E_j^{\xb,y} \perp \Eq^{\xb,y}$, then  $$\left|E^{\xb,y}_j R_{\bvartheta,N-i}(\Eq^{\xb,y})^\top\right| \leq \frac{5}{N^{\alpha-1}}.$$
        \item If $j$ is even, then $E^{\xb,y}_j R_{\bvartheta,N-i}(\Eq^{\xb,y})^\top=0$.
    \end{enumerate}
\end{lemma}
\begin{proof}
We prove by cases.

{Case (1).} If $j$ is odd, and $E_j^{\xb,y}$ has the same feature with $\Eq^{\xb,y}$, i.e. $E_j^{\xb,y}=\Eq^{\xb,y}$, then 
        \begin{align*}
        1
        &-E^{\xb,y}_j R_{\bvartheta,N-i} (\Eq^{\xb,y})^\top\\
        &
        =1-\sum_{\ell=d_\cbb/2+1}^{d/2}\left\{(\Eq^{\xb,y})_{2\ell-1}^2+(\Eq^{\xb,y})_{2 \ell}^2\right\}\cos(\theta_{\ell}(N-i))\\
        &
        \leq
        1-\sum_{\ell=d_\cbb/2+1}^{d/2}\left\{(\Eq^{\xb,y})_{2\ell-1}^2+(\Eq^{\xb,y})_{2 \ell}^2\right\}\Big\{1-\frac{1}{2}(\theta_{\ell}(N-i))^2\Big\}\\
        &
        \leq \frac{2}{N^{2(\alpha-1)}} \\
        &\leq \frac{2}{N^{{\alpha-1}}},
    \end{align*}
    where the first inequality follows from that $1-x^2 \leq \cos(x)$, and the second inequality follows from the that $\|\Eq^{\xb,y}\|_2=1$.
    
{Case (2).} If $j$ is odd, and $E_j^{\xb,y}$ has a different feature with $\Eq^{\xb,y}$, i.e. $E_j^{\xb,y} \perp \Eq^{\xb,y}$,
    \begin{align*}
    &
    E^{\xb,y}_j R_{\bvartheta,N-i} (\Eq^{\xb,y})^\top\\
    &
        \overset{\RM{1}}{\leq}
        \left|(\Eq^{\xb,y})^\top E_j^{\xb,y}\cos(\theta_{d/2}(N-i))\right|\\
        & \quad
        +
        \bigg|\sum_{\ell=1}^{\frac{d}{2}}\Big\{(\Eq^{\xb,y})_{2\ell-1}^2+(E_j^{\xb,y})_{2\ell-1}^2+(\Eq^{\xb,y})_{2 \ell}^2+(E_j^{\xb,y})_{2 \ell}^2\Big\}\bigg\{(\theta_{\ell}-\theta_{d/2})(N-i)\cdot\frac{2}{N^{{\alpha-1}}}\bigg\}\bigg|\\
        & \quad
        + \bigg|\sum_{\ell=1}^{\frac{d}{2}}\left\{(\Eq^{\xb,y})_{2\ell-1}^2+(E_j^{\xb,y})_{2\ell-1}^2+(\Eq^{\xb,y})_{2 \ell}^2+(E_j^{\xb,y})_{2 \ell}^2\right\}\bigg|\left|\theta_{\ell}(N-i)\right|\\
        & 
        \leq 
        \frac{8}{N^{2(\alpha-1)}}+\frac{4}{N^{\alpha-1}} \leq \frac{5}{N^{\alpha-1}},
    \end{align*}
    where $\RM{1}$ follows from adding and subtracting $\cos(\theta_{d/2}(N-i))$, the inequality $\sin(x)\leq x$ and the fact that for $0\leq x_1 \leq x_2 \leq 1$, $|\cos(x_1)-\cos(x_2)|\leq |x_1-x_2|\sin(x_2)$.
    
{Case (3).} If $j$ is even, then the sub-vector of $E_j^{\xb,y}$ corresponding to $\xb$ is zero. Consider the sub-vector of $\Eq^{\xb,y}$ corresponding to $y$ is also zero, we immediately have $E^{\xb,y}_j R_{\bvartheta,N-i} (\Eq^{\xb,y})^\top=0$ for all $i$ is even. Hence, we finish the proof of \Cref{lem: st1-aux1}.
\end{proof}

As a direct result of \Cref{lem: st1-aux1}, we have the following corollary to control $I(i)$. 
\begin{corollary}[Bounds of $I(i)$]\label{coro: stage1-aux1-u-bd}
     Suppose Assumption \ref{Assp 1: Frequency seq} holds, for any $1 \leq j \leq i \leq N$,
\begin{align*}
    \left|I(i)_j - \mathbbm{1}{\{j \equiv 1 \!\!\!\!\pmod{2}\}}\mathbbm{1}{\{E_j = \Eq\}}\right| \leq 5N^{1-\alpha}.
\end{align*}
In addition, conditioned on the event $\{\xbq=\vb_k\}$, we have
\begin{align*}
    \left|I(i)_j - \mathbbm{1}{\{j \equiv 1 \!\!\!\!\pmod{2}\}}\mathbbm{1}{\{\xb_{(j+1)/2} = \vb_k\}}\right| \leq
    5N^{1-\alpha}.
\end{align*}
\end{corollary}

We omit the proof here, as it follows directly from \Cref{lem: st1-aux1}. We next show concentration results of $\attn^{(2)}$ and $\Attn^{(2)}$, which is important in simplifying $a_{i,j}$. We define their approximate expectation versions $\overline{\attn}^{(2)}$ and $\overline{\Attn}^{(2)}$, then control the approximation errors. 
\begin{lemma}[Concentration of $\attn^{(2)}$,$\Attn^{(2)}$]\label{lem: st1-aux2-S-concen}
Suppose the prompt $P$ is randomly sampled according to the data model in \Cref{subsec: data model}. Define $\overline{\attn}_i^{(2)}={\exp{(\EE[\ub_i])} }/{( \sum_{j=1}^{N}\exp{(\EE[\ub_j])})}$ and $\overline{\Attn}_k^{(2)}=K^{-1}\sum_{i=1}^{\Nit}\overline{\attn}_{2i}^{(2)}$ as the approximate expectation versions of $\attn^{(2)}$ and $\Attn^{(2)}$. For any $i \in [N]$, with probability at least $1-\delta$, we have 
\begin{align*}
    \attn_i^{(2)} 
    & =\overline{\attn}_i^{(2)}\big(1\pm O\big(\min\big\{1,{i}^{-1}{\sqrt{\log({2N}/{\delta})}}\big\}\big)\big), i \in [N],\\
    \Attn_m^{(2)}
    &=\overline{\Attn}_m^{(2)}\pm O\big({N}^{-1}\sqrt{\log({2K}/{\delta})}\big), m \in [K].   
\end{align*}
Furthermore, denote the event above by $\Es$, we have $\PP(\Es) \geq 1-\delta$.
\end{lemma}
\begin{proof}
We first prove the concentration result of $\attn^{(2)}$. Recall that from \eqref{Eq: layer1-gdc-2}, we have
\begin{align*}
    \attn_{i}^{(2)}  
    &
    = \frac{\exp{(\ub_i)}}{\sum_{j=1}^{N}\exp{(\ub_j)}},
    \ub_i
    = \sum_{j=1}^{i}\attn^{(1)}_{i,j} E^{\xb,y}_{j} R_{\bvartheta,N-i} (\Eq^{\xb,y})^\top.
\end{align*}
From the definition, $\attn^{(1)}$ is independent with prompt $P$. Consequently, when $j$ is odd, the terms  $\attn^{(1)}_{i,j} E^{\xb,y}_{j} R_{\bvartheta,N-i} (\Eq^{\xb,y})^\top$ are conditionally independent with each other for different $j \leq i$ given $\Eq^{\xb,y}$. Furthermore, for each $j \leq i$, the term  $\attn^{(1)}_{i,j} E^{\xb,y}_{j} R_{\bvartheta,N-i} (\Eq^{\xb,y})^\top$ is bounded between $-\attn^{(1)}_{i,j}$ and $\attn^{(1)}_{i,j}$. Then for all $i \in [N]$, following from the standard Hoeffding inequality techniques~\citep{hoeffding1963probability,mohri2018foundations}, with probability at least $1-\delta/2$, we have
\begin{align}
    |\ub_i - \EE[\ub_i]| \leq \sqrt{\frac{\sum_{j=1}^i 4(\attn_{i,j}^{(1)})^2}{2i^2}\cdot\log({2N}/{\delta})} \leq \frac{\sqrt{2\log({2N}/{\delta})}}{i} \label{Eq:layer1-aux2-con1}.
\end{align}
By \Cref{coro: stage1-aux1-u-bd}, which shows $\big|\EE[\ub_i]-\frac{1}{K}\sum_{j \leq i,j \equiv 1 \!\!\pmod{2}}\attn_{i,j}^{(1)}\big| \leq {5}{N^{1-\alpha}}$, and \eqref{Eq:layer1-aux2-con1}, with probability at least $1-\delta/2$, we have
\begin{align}
    \attn_{i}^{(2)}
    &
    =\frac{\exp{(\EE[\ub_i])}  (1+(\ub_i - \EE[\ub_i])+o(\ub_i - \EE[\ub_i]))}{\sum_{j=1}^{N}\exp{(\EE[\ub_j])  (1+(\ub_j- \EE[\ub_j])+o(\ub_j - \EE[\ub_j]))}}\nonumber\\
    &
    =\overline{\attn}_i^{(2)}\left(1\pm O\left({i}^{-1}{\sqrt{\log({2N}/{\delta})}}\right)\right), \label{Eq:layer1-aux2-con2}
\end{align}
In addition, by noticing that $\overline{\attn}_i^{(2)}$ and ${\attn}_i^{(2)}$ are in the same order, i.e., ${\attn}_i^{(2)}=\Theta(\overline{\attn}_i^{(2)})$, we have 
\begin{align*}
    \attn_{i}^{(2)} = \overline{\attn}_i^{(2)}\big(1\pm O\big(\min\big\{1,{i}^{-1}{\sqrt{\log({2N}/{\delta})}}\big\}\big)\big).
\end{align*}

We next prove the concentration of $\Attn^{(2)}$.
With probability at least $1-\delta$, for any $m \in [K]$, we have

\begin{align*}
     \Attn_{m}^{(2)}
     &
     \overset{\RM{1}}{=} \sum_{i=1}^{\Nit}\mathbbm{1}{\{\xb_i=\vb_m\}}\overline{\attn}_{2i}^{(2)}\bigg(1\pm O\bigg(\frac{\sqrt{\log({2N}/{\delta})}}{2i}\bigg)\bigg)\\
     &
     \overset{\RM{2}}{=} 
     \frac{1}{K}\sum_{i=1}^{\Nit}\overline{\attn}_{2i}^{(2)}\bigg(1\pm O\bigg(\frac{\sqrt{\log({2N}/{\delta})}}{2i}\bigg)\bigg)\pm O\bigg(\frac{\sqrt{\log\left({2K}/{\delta}\right)}}{N}\bigg)\\
     &
     =\overline{\Attn}_m^{(2)}\pm O\bigg(\frac{\sqrt{\log({2K}/{\delta})}}{N}\bigg),
\end{align*}
where $\RM{1}$ follows from the definitions of ${\Attn}_m^{(2)}$, $\overline{\Attn}_m^{(2)}$ and \eqref{Eq:layer1-aux2-con2}, and $\RM{2}$ follows from the standard Hoeffding inequality and that $p_k=1/K$. Hence, we finish the proof of \Cref{lem: st1-aux2-S-concen}.
\end{proof}

Because $-1 \leq \ub_j \leq 1$ for all $j \in [N]$, by the definition of $\overline{\attn}_i^{(2)}$ and $\overline{\Attn}_k^{(2)}$, we obtain the following corollary via simple calculations based on \Cref{lem: st1-aux2-S-concen}, which provides the order of $\overline{\Attn}^{(2)}$ and $\overline{\attn}^{(2)}$.
\begin{corollary}\label{coro: stage1-aux2-S2-bd}
    Through Stage I, for all $i \in [N]$ and $m \in [K]$, we have $\overline{\attn}_i^{(2)}=\Theta\left({1}/{N}\right)$ and  $\overline{\Attn}_k^{(2)}=\Theta\left({1}/{K}\right)$.   
\end{corollary}

In the next section, we will analyze the training dynamics of Stage I. 
\subsection{Stage I: Phase I}\label{app: s1p1}
In this section, we shall study the initial phase of learning the relationship between any position $i$ and its previous tokens. We define the \textbf{Phase I} as all iterations $0 \leq t\leq T^{(1)}_{1}$, where
$$
T^{(1)}_{1} \triangleq \max \Big\{t: \min_{i \in [N]} \Big(A_{i,i-1}{(t)}-\max_{j \leq i, j \neq i-1}A_{i,j}{(t)}\Big) \leq \log(N)\Big\}.
$$
We then state the following induction hypothesis, which holds throughout Phase I. This hypothesis is proved by induction together with the technical lemmas in \Cref{app: st1-ph1-techlem}. 

\begin{hypothesis}\label{hp1}
For each $0 \leq t \leq T^{(1)}_{1}$ and any $i \in [N]$, the following holds:  
\begin{enumerate}[leftmargin=*]
    \item $A_{i,i-1}{(t)}-\max_{j \leq i, j \neq i-1}A_{i,j}{(t)}$ is monotonically increasing and for all $i \in [N]$, $0 \leq A_{i,i-1}{(t)}-\max_{j \leq i, j \neq i-1}A_{i,j}{(t)} \leq \log N\left(1+O\left(K^{-1}{\log N}\right)\right)$.
    \item $\max_{j \leq i, j \neq i-1}A_{i,j}{(t)}-\min_{l \leq i, l \neq i-1}A_{i,l}{(t)}\leq O\left(K^{-1}{(\log N)^2}\right)$.
\end{enumerate}
\end{hypothesis}
\begin{proof}
We first prove Claim (1), as $A_{i,j}(0)=0$ for all $j \leq i \leq N$, we only need to prove that $\Delta A_{i,i-1}(t-1)-\Delta A_{i,j}(t-1) \geq $ for all $j \leq i, j \neq i-1$.
By \Cref{lem: st1-ph1-A}, we have 
\begin{align*}
    \Delta A_{i,i-1}{(t)}-\max_{j \leq i, j \neq i-1} \Delta A_{i,j}{(t)}
    \gtrsim \frac{\eta_1 C_1}{KN}, 
\end{align*}
which is larger than 0. As a result,
\begin{align*}
     A_{i,i-1}{(t)}-\max_{j \leq i, j \neq i-1} A_{i,j}{(t)} 
     & 
     \leq 
     \sum_{\tau=0}^{t-1}\Delta A_{i,i-1}{(\tau)}-\min_{j \leq i, j \neq i-1} \Delta A_{i,j}{(\tau)}\\
     & 
     \leq 
     \sum_{\tau=0}^{t-1}\left(1+O\left({K}^{-1}{\log N}\right)\right)\min_{l \in [N]}\left(\Delta A_{l,l-1}{(\tau)}-\max_{r \leq l, r \neq l-1} \Delta A_{l,r}{(\tau)}\right)\\
     & 
     = 
     \left(1+O\left({K}^{-1}{\log N}\right)\right)\min_{l \in [N]}\left( A_{l,l-1}{(t)}-\max_{r \leq l, r \neq l-1}  A_{l,r}{(t)}\right)\\
     & \leq \left(1+O\left({K}^{-1}{\log N}\right)\right)\log N,
\end{align*}
where the second inequality follows from that for each $i$,
$\Delta A_{i,i-1}{(t)}-\min_{j \leq i, j \neq i-1} \Delta A_{i,j}{(t)}=\min_{i\in[N]}\{\Delta A_{i,i-1}{(t)}-\max_{j \leq i, j \neq i-1} \Delta A_{i,j}{(t)}\}(1+O(\log N/K))$ as in \Cref{coro: st1-ph1-A-ratio}, and the last inequality follows from the definition of Phase I.  Hence, we prove Claim (1).

Then we proceed to prove Claim (2) as follows.
\begin{align*}
    \max_{j \leq i, j \neq i-1}A_{i,j}{(t)}-\min_{l \leq i, l \neq i-1}A_{i,l}{(t)} 
    & \leq \sum_{\tau=0}^{t-1} \left(\max_{j \leq i, j \neq i-1}\Delta A_{i,j}{(\tau)}-\min_{l \leq i, l \neq i-1}\Delta A_{i,l}{(\tau)}\right)\\
    & \overset{\RM{1}}{\leq} \sum_{\tau=0}^{t-1}O\Big(\frac{\log N}{K}\Big)\!\!\left(\Delta A_{i,i-1}{(\tau)}-\max_{j \leq i, j \neq i-1} \Delta A_{i,j}{(\tau)}\right)\\
    & \leq O\Big(\frac{\log N}{K}\Big)\!\!\left( A_{i,i-1}{(t)}-\max_{j \leq i, j \neq i-1}  A_{i,j}{(t)}\right)\\
    & \overset{\RM{2}}{=} O\Big(\frac{(\log N)^2}{K}\Big),
\end{align*}
where $\RM{1}$ follows from \Cref{coro: st1-ph1-A-ratio}, and $\RM{2}$ follows from Claim (1). Hence, we finish the proof of \Cref{hp1}.
\end{proof}
\subsubsection{Technical Lemmas}\label{app: st1-ph1-techlem}
We first introduce several useful technical lemmas, which characterize the important values including $\attn^{(1)},\Attn^{(2)},a_{i,j},\Delta A$ across phase I. By induction, for all these lemmas, we only need to show that (1) for $t=0$, the lemmas hold, and (2) if we assume that for $\attn^{(1)}(t-1),\Attn^{(2)}(t-1), A(t-1), a_{i,j}(t-1),\Delta A(t-1)$, the lemmas hold, then they still hold for $\attn^{(1)}(t),\Attn^{(2)}(t), A(t), a_{i,j}(t),\Delta A(t)$. 

\begin{lemma}\label{lem: st1-ph1-S}
    For iteration $0 \leq t \leq T_1^{(1)}$, if \Cref{hp1} holds at iteration $t$, and \Cref{lem: st1-ph1-a,lem: st1-ph1-A} hold at iteration $t-1$, then  
    \begin{enumerate}[leftmargin=*]
    \item  For any $i \in [N]$, we have $\attn^{(1)}_{i,i-1}{(t)}=\Omega\left(1/i\right)$ and 
    $\attn^{(1)}_{i,j}{(t)}=O\left(1/i\right)$ for $ j \neq i-1$.
    \item For any $i \in [ N]$, 
    \begin{enumerate}
        \item $\attn^{(1)}_{i,\xb}(t)-\attn^{(1)}_{i,i-1}(t)\geq \Omega(i/N)$.
        \item $\attn^{(1)}_{i,y}(t) \geq \Omega(i/N)$.
        \item $1-\attn^{(1)}_{i,i-1}(t)\geq \Omega(i/N)$.
    \end{enumerate}
    \end{enumerate}
\end{lemma}
\begin{proof}
We first prove Claim (1). For any $i \in [N]$
    \begin{align*}
        \attn_{i,i-1}^{(1)}(t)
        &
        \geq \frac{1}{1+(i-1)\exp{(\max_{j \leq i, j \neq i-1}A_{i,j}(t)-A_{i,i-1}(t))}}\\
        &
        \overset{\RM{1}}{\geq} 
        \frac{1}{1+(i-1)\exp{(\max_{j \leq i, j \neq i-1}A_{i,j}(0)-A_{i,i-1}(0))}} \\
        &= \frac{1}{i},
    \end{align*}
where $\RM{1}$ follows from that ${1}/({1+(i-1)e^x})$ is monotonically decreasing with $x$ and \Cref{hp1}.  
Furthermore, we have that
\begin{align*}
    \attn_{i,j}^{(1)}(t)
        &
        \overset{\RM{1}}{\leq} \frac{1}{1+\sum_{l \leq i, l \neq j} \exp{(-\left(\max_{j \leq i, j \neq i-1}A_{i,j}{(t)}-\min_{l \leq i, l \neq i-1}A_{i,l}{(t)}\right)})}\\
        & 
        \overset{\RM{2}}{\leq} \frac{1}{1+(i-1) \exp{\left(-O\left({K}^{-1}{(\log N)^2}\right)\right)}}\\
        &
        = \frac{1}{1+(i-1)\Omega(1)} \\
        &= O\Big(\frac{1}{i}\Big),
\end{align*}
where $\RM{1}$ follows from that ${1}/({1+e^x})$ decrease with of $x$, and \Cref{hp1} that $A_{i,i-1}{(t)}-\max_{j \leq i, j \neq i-1}A_{i,j}{(t)} \geq 0$, and $\RM{2}$ follows from \Cref{hp1}. Hence, we have proved Claim (1).

We then proceed to prove Claim (2) as follows.
\begin{align*}
    1-\attn_{i,i-1}^{(1)}(t)
        &
        \overset{\RM{1}}{\geq}
        \frac{\sum_{j \leq i, j \neq i-1} \exp{(\min_{j \leq i, j\neq i-1 }A_{i,j}(t)-A_{i,i-1}(t))}}{1+\sum_{j \leq i, j \neq i-1} \exp{(\min_{j \leq i, j\neq i-1 } A_{i,j}(t)-A_{i,i-1}(t))}}\\
        &
        \overset{\RM{2}}{\geq} \frac{\sum_{j \leq i, j \neq i-1} \exp{\left(-\log(N)\left(1+O\left({K}^{-1}{\log N}\right)\right)\right)}\exp{\left(-O\left({K}^{-1}{(\log N)^2}\right)\right)}}{1+\sum_{j \leq i, j \neq i-1} \exp{\left(-\log(N)\left(1+O\left({K}^{-1}{\log N}\right)\right)\right)}\exp{\left(-O\left({K}^{-1}{(\log N)^2}\right)\right)}}\\
        &
        \geq \Omega\Big(\frac{i-1}{N}\Big),
\end{align*}
where $\RM{1}$ follows from that ${e^x}/{(1+e^x)}$ increase with of $x$, and $\RM{2}$ follows from \Cref{coro: st1-ph1-A-ratio} and \Cref{hp1}. Similarly, we can prove 
\begin{align*}
    &\attn^{(1)}_{i,y}(t) \geq \Omega\Big(\frac{i-1}{N}\Big), \attn^{(1)}_{i,\xb}(t)-\attn^{(1)}_{i,i-1}(t)\geq \Omega\Big(\frac{i-1}{N}\Big).
\end{align*}
Hence, we finish the proof of \Cref{lem: st1-ph1-S}.
\end{proof}

In the following lemma, we control the order of $a_{i,j}(t)$.
\begin{lemma}[Order of $a_{i,j}(t)$]\label{lem: st1-ph1-a}
    During stage I, suppose \Cref{hp1} and \Cref{lem: st1-ph1-S} hold at iteration $0\leq t\leq T^{(1)}_{1}$. For different cases of $i,j$, we have the following results.
    \begin{enumerate}[leftmargin=*]
    \item For $i=2n,j=2n-1$, $a_{i,i-1}{(t)}\simeq (KN^2)^{-1}$.
    \item For $i=2n,j=2m-1,m \leq n$, $-({iKN})^{-1}\left({K}^{-1}+{i}^{-1}\right) \lesssim a_{i,j}{(t)} \lesssim 
    ({iKN})^{-1}\left({K}^{-1}-{i}^{-1}\right)$.
    \item For $i=2n,j=2m, m \leq n$,  $-({iKN})^{-1}\left({K}^{-1}+{i}^{-1}\right) \lesssim a_{i,j}{(t)} \lesssim 
    ({iKN})^{-1}\left({K}^{-1}-{i}^{-1}\right)$.
    \item For $i=2n-1,j=2m, m \leq n-1$,
    $|a_{i,j}{(t)}| \lesssim (iK^2N)^{-1}$.
    \item For $i=2n-1,j=2m-1, m \leq n$, $|a_{i,j}{(t)}| \lesssim (iK^2N)^{-1}$.
    \end{enumerate}
\end{lemma}
\begin{proof}
We only need to apply \Cref{lem: st1-upA-sim} and evaluate all $a_{i,j}(t)$s' order for the timestep $t$ given $\attn^{(1)}(t),\attn^{(2)}(t),\Attn^{(2)}(t),a_{i,j}(t-1)$ by cases. 

{ Case (1): $i=2n,j=2n-1$.} We have
\begin{align*}
    a_{i,i-1}{(t)}
        & 
        \overset{\RM{1}}{\gtrsim} \frac{1}{iKN}\left\{\sum_{k=1}^K\EE\left[\mathbbm{1}{\{\xbq=\vb_k\}}\left(\!\left(1\!-\!\frac{1}{K}\right)\!\!\left(1-\frac{1}{2K}\right) \frac{i-1}{N}\!+\!\left(\!\frac{1}{4K}\!-\!\frac{1}{K}\!+\!1\right) \frac{i-1}{N}\right)\!\right]\right.\\
        &
        \quad
        \left.
        \pm O\left(\frac{\log N}{N^2}\right) \pm O\left(\frac{1}{ N^{\alpha-1}}\right)\right\}\pm O\!\left(\!\frac{1}{N^3}\!\right)\\
        &
        \simeq \frac{1}{KN^2}>0,
\end{align*}
where $\RM{1}$ follows from expression of $a_{i,j}(t)$ in \Cref{lem: st1-upA-sim}, that $\attn_{i,i-1}^{(1)}(t)\geq\Omega({1}/{i})$ as in  \Cref{lem: st1-ph1-S}, and $\overline{\attn}_i^{(2)}=\Theta\left({1}/{N}\right)$ and  $\overline{\Attn}_k^{(2)}=\Theta\left({1}/{K}\right)$ as in \Cref{coro: stage1-aux2-S2-bd}. Hence, we prove Claim (1).

Similar to Case (1), following \Cref{lem: st1-upA-sim,lem: st1-ph1-S,coro: stage1-aux2-S2-bd}, we can compute for the other cases and finish the proof of \Cref{lem: st1-ph1-a}.
\end{proof}
As a result of \Cref{lem: st1-ph1-a}, for any $l \in [N]$, we can characterize the relative increments gap of $A$ across different tokens as follows.
\begin{lemma}\label{lem: st1-ph1-A}
    Suppose \Cref{hp1}, \Cref{lem: st1-ph1-S,lem: st1-ph1-a} hold at iteration $0\leq t\leq T^{(1)}_{1}$, for any $l \in [N]$, we have the following results.
    \begin{enumerate}[leftmargin=*]
        \item \textbf{Upper Bound:} $
    \Delta A_{l,l-1}{(t)}-\max_{r \leq l, r \neq l-1} \Delta A_{l,r}{(t)}
    \gtrsim
      \eta_1C_1/{(KN)}.$
    \item \textbf{Lower Bound:} 
    $
    \max_{j \leq l, j \neq l-1}\Delta A_{l,j}{(t)}-\min_{r \leq l, r \neq l-1} \Delta A_{l,r}{(t)}
    \lesssim
     {\eta_1 C_1\log (N)}/{(N K^2)}.$
    \end{enumerate}

\end{lemma}
\begin{proof}
From \Cref{lem: st1-upA-com}, the increment of $A_{l,r}(t)$ can be decomposed as follows 
\begin{align*}
        \Delta A_{l,r}{(t)}
    &
    = \eta_1\bigg(\underbrace{C_1\sum_{i=l-r+1}^{N}a_{i,i+r-l}(t)}_{J_{l,r}^1(t)}+\underbrace{C_2\sum_{1 \leq j\leq i \leq N}a_{i,j}(t)}_{J_{l,r}^2(t)}+\underbrace{\sum_{1 \leq j\leq i \leq N}\left(\pm a_{i,j}(t) \Theta(\epsFN)\right)}_{J_{l,r}^3(t)}\bigg),
\end{align*}
which has three parts: $J_{l,r}^1(t)$, $J_{l,r}^2(t)$ and the error $J_{l,r}^3(t)$. The second part $J_{l,r}^2(t)$ is independent with $r$, so we only need to discuss the first part $J_{l,r}^1(t)$ and the third part $J_{l,r}^3(t)$.

\textbf{First Part.} We discuss $J_{l,r}^1(t)$ by cases,

For Case (1): $r=l-1$ and $J_{l,l-1}^1(t)$, we have
\begin{align}
    J_{l,l-1}^1(t)
    &
    =  C_1\sum_{n=1}^{\Nit}a_{2n,2n-1}(t)+\sum_{m=1}^{\Nit}a_{2m+1,2m}(t)\nonumber\\
    &
    \overset{\RM{1}}{\gtrsim}
     C_1 \left(\sum_{n=1}^{\Nit}\frac{1}{KN^2}-\sum_{m=1}^{\Nit}\frac{1}{(2m+1)K^2 N}\right)\nonumber\\
    &
    \gtrsim
    \frac{ C_1}{KN} >0, \label{Eq: lay1-ph1-A-1}
\end{align} 
where $\RM{1}$ follows from \Cref{lem: st1-ph1-a}. A direct result from the equation above is that for any $l$, $J_{l,l-1}^1(t)$ is the same, and is lower bounded by the same order $\Omega({ C_1}/{(KN)})$.

Similar to Case (1), for Case (2) $r=l$ and  $J_{l,l}^1(t)$, we have
\begin{align}
    J_{l,l}^1(t) 
    &
    =  C \Big(\sum_{n=1}^{\Nit}a_{2n,2n}(t)+\sum_{m=1}^{\Nit+1}a_{2m-1,2m-1}(t)\Big)
    \lesssim
    \frac{ C_1\log N}{K^2N}.\label{Eq: lay1-ph1-A-2}
\end{align} 

For Case (3), $l - r \geq 2$ and $l-r \equiv 0 \pmod{2}$, we have
\begin{align}
    J^1_{l,r}(t) 
     &
    =  C_1 \Big(\sum_{n=\lfloor(l-r+1)/2\rfloor}^{\Nit}a_{2n,2n+r-l}(t)+\sum_{m=\lfloor(l-r+1)/2\rfloor}^{\Nit}a_{2m+1,2m+1+r-l}(t)\Big)
    \lesssim
    \frac{ C_1\log N}{2K^2N}. \label{Eq: lay1-ph1-A-3}
\end{align} 

For Case (4), $l - r \geq 2$ and $l-r \equiv 0 \pmod{2}$, we have
\begin{align}
    J^1_{l,r}{(t)} 
     &
    =  C_1 \Big(\sum_{n=\lfloor(l-r+1)/2\rfloor}^{\Nit}a_{2n,2n+r-l}(t)+\sum_{m=\lfloor(l-r+1)/2\rfloor}^{\Nit}a_{2m+1,2m+1+r-l}(t)\Big)
    \lesssim
    \frac{ C_1 \log N}{2K^2N}. \label{Eq: lay1-ph1-A-4}
\end{align} 

Combine \eqref{Eq: lay1-ph1-A-1} to \eqref{Eq: lay1-ph1-A-4}, we have
\begin{align}
    &
    J_{l,l-1}^1(t)-\max_{r \leq l, r \neq l-1} J_{l,r}^1(t) \gtrsim
    \frac{C_1}{KN}. \label{Eq: lay1-ph1-A-5}
\end{align}
\textbf{Third Part.} Then for $J_{l,r}^3(t)$, and any $l \in [N]$, we have
\begin{align}
    &
    J_{l,l-1}^3(t)-\max_{r \leq l, r \neq l-1} J_{l,r}^3(t)\nonumber\\
    & \qquad
    \gtrsim - O\left(\epsFN\right)\sum_{n=0,n\neq 1}^{N}\sum_{i=n+1}^{N}\left|a_{i,i-n}(t)\right|\nonumber\\
    & \qquad
    \overset{\RM{1}}{\gtrsim}
    - O\left(\epsFN\right) \sum_{n=0,n\neq 1}^{N}O\left(\frac{\log N}{K^2 N}\right)\nonumber\\
    & \qquad
    \gtrsim
    - O\left(\frac{ \epsFN \log N}{K^2}\right), \label{Eq: lay1-ph1-A-6}
\end{align}  
where $\RM{1}$ follows from \Cref{lem: st1-ph1-a}.

Finally, combine \eqref{Eq: lay1-ph1-A-5} and \eqref{Eq: lay1-ph1-A-6}, we have 
\begin{align*}
   & 
    \Delta A_{l,l-1}{(t)}-\max_{r \leq l, r \neq l-1} \Delta A_{l,r}{(t)}
    \gtrsim
     \frac{ \eta_1 C_1 }{KN}-\frac{ \eta_1 \epsFN \log N}{K^2} \gtrsim \frac{ \eta_1 C_1 }{KN}, 
\end{align*}
where the last equation follows from Assumption \ref{Assp 1: Frequency seq} that $\epsFN \leq \Omega(C_1/N)$.  Hence, we prove Claim (1).

Then we prove Claim (2). Similar to Claim (1), combine \eqref{Eq: lay1-ph1-A-5} and \eqref{Eq: lay1-ph1-A-6}, we have 
\begin{align*}
    & 
    \max_{j \leq l, j \neq l-1}\Delta A_{l,j}{(t)}-\min_{r \leq l, r \neq l-1} \Delta A_{l,r}{(t)}
    \lesssim
     \frac{\eta_1 \log N}{K^2}\left(\frac{C_1}{N}+{\epsFN}\right)\lesssim
     \frac{\eta_1 C_1\log N}{N K^2}.
    \end{align*}
Hence, we finish the proof of \Cref{lem: st1-ph1-A}.
\end{proof}

A direct corollary of \Cref{lem: st1-ph1-A} is as follows. 
\begin{corollary}\label{coro: st1-ph1-A-ratio}
    Under the same conditions of \Cref{lem: st1-ph1-A}, for any $l \in [N]$, we have 
    \begin{align*}
        & \frac{\max_{j \leq l, j \neq l-1}\Delta A_{l,j}{(t)}-\min_{r \leq l, r \neq l-1} \Delta A_{l,r}{(t)}}{\min_{l \in [N]}\{\Delta A_{l,l-1}{(t)}-\max_{r \leq l, r \neq l-1} \Delta A_{l,r}{(t)}
        \}} = O\left(\!\frac{\log N}{K}\!\right),\\
        & 
        \frac{\Delta A_{l,l-1}{(t)}-\min_{r \leq l, r \neq l-1} \Delta A_{l,r}{(t)}}{\min_{l \in [N]}\{\Delta A_{l,l-1}{(t)}-\max_{r \leq l, r \neq l-1} \Delta A_{l,r}{(t)}\}} =  1+O\left(\!\frac{\log N}{K}\!\right),\\
        &
        \frac{\Delta A_{l,l-1}{(t)}-\max_{r \leq l, r \neq l-1} \Delta A_{l,r}{(t)}}{\min_{l \in [N]}\{\Delta A_{l,l-1}{(t)}-\max_{r \leq l, r \neq l-1} \Delta A_{l,r}{(t)}\}} =  1+O\left(\!\frac{\log N}{K}\!\right).
    \end{align*}
\end{corollary}
\subsubsection{End of Phase I}
\begin{lemma}\label{lem: st1-ph1-end}
    With $T^{(1)}_1$ at most $O\left({\eta_1^{-1}C_1^{-1} K N\log N}\right)$, for any $i \in [N]$, at iteration $t= T^{(1)}_1+1$, we have
    \begin{enumerate}
        \item $A_{i,i-1}{(T^{(1)}_1+1)}-\max_{j \leq i, j \neq i-1}A_{i,j}{(T^{(1)}_1+1)} \geq \log(N)$,
        \item $\attn^{(1)}_{i,i-1}(T^{(1)}_1+1)=\Omega(1)$.
    \end{enumerate}
\end{lemma}
\begin{proof}
    Claim (1) holds because of the definition of Phase I. Then, we only need to show the upper bound of $T^{(1)}_1$.
    As \Cref{hp1} holds, there exists an $i$ for $T^{(1)}_1$ such that 
    \begin{align}
        \log(N) 
        &
        \geq A_{i,i-1}{(T^{(1)}_1)}-\max_{j \leq i, j \neq i-1}A_{i,j}{(T^{(1)}_1)}  \nonumber\\   
        & 
        {\geq}
        \sum_{\tau=0}^{T^{(1)}_1-1} \Delta A_{i,i-1}{(\tau)}-\max_{j \leq i, j \neq i-1}\Delta A_{i,j}{(\tau)}  \nonumber\\
        &
        \overset{\RM{1}}{\geq}
        \frac{T^{(1)}_1  \eta_1 C_1}{KN},\label{Eq:st1-ph1-end1}
    \end{align}
    where $\RM{1}$ follows from \Cref{lem: st1-ph1-A}.
    As a result, we must have $T^{(1)}_1 = O
    \left(\eta_1^{-1} C_1^{-1}K N\log N  \right)$, otherwise the inequality would fail to hold. Hence, we prove Claim (1).
    
    We then prove Claim (2), for all $i  \in [N]$,
     \begin{align*}
        \attn_{i,i-1}^{(1)}(T^{(1)}_1+1)
        &
        = \frac{\exp{(A_{i,i-1}(T^{(1)}_1+1))}}{\exp{(A_{i,i-1}(T^{(1)}_1+1))}+\sum_{j \leq i, j \neq i-1} \exp{(A_{i,j}(T^{(1)}_1+1))}}\\
        &
        \geq \frac{1}{1+(i-1)\exp{(\max_{j \leq i, j \neq i-1}A_{i,j}(T^{(1)}_1+1)-A_{i,i-1}(T^{(1)}_1+1))}}\\
        &
        \overset{\RM{1}}{\geq} 
        \frac{1}{1+(i-1)e^{-\log(N)}} = \Omega(1),
    \end{align*}
    where $\RM{1}$ follows from Claim (1) of this lemma. Hence, we finish the proof of this lemma.
\end{proof}

\subsection{Stage I: Phase II}\label{app: s1p2}
In this section, we study the rapid growth phase of learning the relationship between any position $i$ and its previous tokens. We define the \textbf{Phase II} as all iterations $T^{(1)}_1+1 \leq t\leq T^{(1)}_{2}$, where
$$
T^{(1)}_{2} \triangleq \max \Big\{t \geq T^{(1)}_1: \min_{i \in [N]} \Big(A_{i,i-1}{(t)}-\max_{j \leq i, j \neq i-1}A_{i,j}{(t)}\Big) \leq \log(KN)\Big\}.
$$
Then, we state the following induction hypothesis, which holds throughout Phase II. This hypothesis is proved by induction with the technical lemmas in \Cref{app: st1-ph2-techlem}.

\begin{hypothesis}\label{hp2}
For each iteration $T^{(1)}_1+1 \leq t \leq T^{(1)}_{2}$ and any $i \in [N]$, the followings hold:  
\begin{enumerate}[leftmargin=*]
    \item $A_{i,i-1}{(t)}-\max_{j \leq i, j \neq i-1}A_{i,j}{(t)}$ is monotonically increasing and $A_{i,i-1}{(t)}-\max_{j \leq i, j \neq i-1}A_{i,j}{(t)}\in \left[\log(N),\left(1+O\left(K^{-1}{\log N}\right)\right)\log\left({KN}\right)\right]$.
    \item $\max_{j \leq i, j \neq i-1}A_{i,j}{(t)}-\min_{l \leq i, l \neq i-1}A_{i,l}{(t)}\leq O\left({K}^{-1}{(\log N)^2}\right)$.
\end{enumerate}
\end{hypothesis}
\begin{proof}
We first prove {Claim (1)}, the monotonicity can be proved similarly to \Cref{hp1}. As in \Cref{lem: st1-ph2-A},
\begin{align*}
    \Delta A_{i,i-1}{(t)}-\max_{j \leq i, j \neq i-1} \Delta A_{i,j}{(t)}
    \gtrsim \frac{\eta_1 C_1}{K^2} \geq 0.
\end{align*}

By direct calculation, we have
\begin{align}
     & A_{i,i-1}{(t)}-\max_{j \leq i, j \neq i-1} A_{i,j}{(t)} 
     - \Big(A_{i,i-1}{(T^{(1)}_1)}-\max_{j \leq i, j \neq i-1} A_{i,j}{(T^{(1)}_1)}\Big)\nonumber\\
     & \quad
     \overset{\RM{1}}{\leq} 
     \sum_{\tau=T^{(1)}_1}^{t-1}\left(1+O\left(N^{-1}K\right)\right)\min_{l \in [N]}\Big(\Delta A_{l,l-1}{(\tau)}-\max_{r \leq l, r \neq l-1} \Delta A_{l,r}{(\tau)}\Big)\nonumber\\
     &
     \quad
     \leq  
     \left(\!1+O\left(N^{-1}K\right)\!\right)\min_{l \in [N]}\Big( A_{l,l-1}{(t)}- A_{l,l-1}{(T^{(1)}_1)}-\max_{r \leq l, r \neq l-1}  \left(A_{l,r}{(t)}-A_{l,r}{(T^{(1)}_1)}\right)\Big)\nonumber\\
      & \quad
     \overset{\RM{2}}{\leq}  
     (1+O(N^{-1}K))\Big\{\min_{l \in [N]}\Big( \!A_{l,l-1}{(t)}\!-\!\!\max_{r \leq l, r \neq l-1}\!\!\!\! A_{l,r}{(t)}\!\Big)
     \!-\!\!\min_{l \in [N]}\Big( \!  A_{l,l-1}{(T^{(1)}_1)}\!-\!\!\!\max_{r \leq l, r \neq l-1} \!\!\!\! A_{l,r}{(T^{(1)}_1)}\Big)\Big\}
     \!,\label{Eq: hp2-0}   
\end{align}
where $\RM{1}$ follows from \Cref{lem: st1-ph2-A-ratio}, and $\RM{2}$ follows from $-\max(A-B) \leq - \max A +\max B$, and $\min (A-B) \leq \min A- \min B$. As a result of \eqref{Eq: hp2-0}, we have
\begin{align}
    A_{i,i-1}{(t)}-\max_{j \leq i, j \neq i-1} A_{i,j}{(t)}
    & 
    \overset{\RM{1}}{\leq} 
     \!\left(1+O\left(N^{-1}K+K^{-1}{\log N}\right)\right)\min_{l \in [N]}\Big( A_{l,l-1}{(t)}\!-\!\max_{r \leq l, r \neq l-1} A_{l,r}{(t)} \!\Big)\label{Eq: hp2-1}\\
    &
    \leq \left(1+O\left(\!K^{-1}{\log N}\right)\right)\log\left({KN}\right),\nonumber
\end{align}
where $\RM{1}$ follows from \Cref{hp1}, and that $A_{l,l-1}{(t)}-\max_{r \leq l, r \neq l-1} A_{l,r}{(t)} $ is monotonically increasing.  Hence, we prove Claim (1).

We next prove Claim (2), 
\begin{align*}
    & \max_{j \leq i, j \neq i-1}A_{i,j}{(t)}-\min_{l \leq i, l \neq i-1}A_{i,l}{(t)} 
    -\Big(\max_{j \leq i, j \neq i-1}A_{i,j}{(T^{(1)}_1)}-\min_{l \leq i, l \neq i-1}A_{i,l}{(T^{(1)}_1)}\Big)\\
    & 
    \quad 
    \leq \sum_{\tau=T^{(1)}_1}^{t-1}O\left(N^{-1}K\right)
    \Big(\Delta A_{i,i-1}{(\tau)}-\max_{j \leq i, j \neq i-1} \Delta A_{i,j}{(\tau)}\Big)\\
    & \quad
    = O\left(N^{-1}K\right)\!\Big(A_{i,i-1}{(t)}-\max_{j \leq i, j \neq i-1}A_{i,j}{(t)}-
    \Big(A_{i,i-1}{(T^{(1)}_1)}-\max_{j \leq i, j \neq i-1}A_{i,j}{(T^{(1)}_1)}\Big)\Big),
\end{align*}
where the first inequality follows from \Cref{lem: st1-ph2-A-ratio}. By direct calculation, we have
\begin{align*}
      \max_{j \leq i, j \neq i-1}A_{i,j}{(t)}-\min_{l \leq i, l \neq i-1}A_{i,l}{(t)}
     & 
     \overset{\RM{1}}{\leq} 
     O\left({K}^{-1}{(\log N)^2}\right)+O\left(N^{-1}K\right)\!\!\left(1+O\left(K^{-1}{\log N}\right)\right)\log\left({KN}\right)\\
     & 
     \leq O\left({K}^{-1}{(\log N)^2}\right),
\end{align*}
where $\RM{1}$ follows from \Cref{hp1}. Hence, we finish the proof of \Cref{hp2}.
\end{proof}
\subsubsection{Technical Lemmas}\label{app: st1-ph2-techlem}
Similar to phase I, we next introduce several useful technical lemmas. The proofs follow arguments analogous to those in Phase I and are omitted for brevity.
\begin{lemma}\label{lem: st1-ph2-S}
    For iteration $T^{(1)}_{1}+1 \leq t\leq T^{(1)}_{2}$, if \Cref{hp2} holds at iteration $t$, and \Cref{lem: st1-ph2-a,lem: st1-ph2-A} hold at iteration $t-1$, then the followings hold: 
    \begin{enumerate}[leftmargin=*]
    \item  For any $i \in [N]$,
    \begin{enumerate}
    \item $\attn^{(1)}_{i,i-1}{(t)}=\Omega\left(1\right)$,
    \item $\Omega\left({1}/(KN)\right)\leq\attn^{(1)}_{i,j}{(t)} \leq O\left({1}/{N}\right), j \neq i-1$.
    \end{enumerate}
    \item For any $i \in [N]$, 
    \begin{enumerate}
        \item $\attn^{(1)}_{i,\xb}(t)-\attn^{(1)}_{i,i-1}(t)\geq \Omega({i}/{(KN)})$,
        \item $\attn^{(1)}_{i,y}(t) \geq \Omega({i}/{(2KN)})$,
        \item $1-\attn^{(1)}_{i,i-1}(t)\geq \Omega({i}/{(2KN)})$.
    \end{enumerate}
    \end{enumerate}
\end{lemma}
Similar to \Cref{lem: st1-ph1-a}, and following from the expression of $a_{i,j}(t)$ in \Cref{lem: st1-upA-sim}, control of attention scores in \Cref{lem: st1-ph2-S,coro: stage1-aux2-S2-bd}, we can control the order of $a_{i,j}(t)$ as follows,
\begin{lemma}[Order of $a_{i,j}(t)$]\label{lem: st1-ph2-a}
    Suppose \Cref{hp2,lem: st1-ph2-S} hold at iteration $T^{(1)}_{1}+1\leq t\leq T^{(1)}_{2}$. For different cases of $i,j$, we have the following results.
    \begin{enumerate}[leftmargin=*]
    \item $i=2n,j=2n-1$, $a_{i,i-1}{(t)}\gtrsim {i}{K^{-2}N^{-2}}$. 
    \item $i=2n,j=2m-1,m \leq n$, $- {K^{-1}N^{-2}} \lesssim a_{i,j}{(t)} \lesssim - {K^{-2}N^{-2}}$.
    \item $i=2n,j=2m, m \leq n$, $- {K^{-1}N^{-2}} \lesssim a_{i,j}{(t)} \lesssim - {K^{-2}N^{-2}}$.
    \item $i=2n-1,j=2m, m \leq n-1$,
    $|a_{i,j}{(t)}| \lesssim K^{-2}N^{-2}$.
    \item $i=2n-1,j=2m-1, m \leq n$, $|a_{i,j}{(t)}| \lesssim K^{-2}N^{-2}$.
    \end{enumerate}
\end{lemma}

As a result of \Cref{lem: st1-ph2-a}, for any $l \in [N]$, we can characterize the relative increments gap of $A$ across different tokens as follows.
\begin{lemma}\label{lem: st1-ph2-A}
    Suppose \Cref{hp2,lem: st1-ph2-S,lem: st1-ph2-a} hold at iteration $T^{(1)}_1+1\leq t\leq T^{(1)}_{2}$. For any $l \in [N]$, we have the following results.
    \begin{enumerate}
        \item \textbf{Upper Bound:}
        $
    \Delta A_{l,l-1}{(t)}-\max_{r \leq l, r \neq l-1} \Delta A_{l,r}{(t)}
    \gtrsim
     \eta_1 C_1K^{-2}.$
    \item \textbf{Lower Bound:} 
    $\max_{j \leq l, j \neq l-1}\Delta A_{l,j}{(t)}-\min_{r \leq l, r \neq l-1} \Delta A_{l,r}{(t)}
    \lesssim
      {  \eta_1 C_1 }(KN)^{-1}.$
    \end{enumerate}

\end{lemma}


As a direct corollary \Cref{lem: st1-ph2-A}, we can bound the increment differences' ratios as follows.
\begin{corollary}\label{lem: st1-ph2-A-ratio}
    Under the same condition of \Cref{lem: st1-ph2-A}, for any $l \in [N]$, we have 
    \begin{align*}
        &
        \frac{\max_{j \leq l, j \neq l-1}\Delta A_{l,j}{(t)}-\min_{r \leq l, r \neq l-1} \Delta A_{l,r}{(t)}}{\min_{l \in [N]}\{\Delta A_{l,l-1}{(t)}-\max_{r \leq l, r \neq l-1} \Delta A_{l,r}{(t)}\}}= O\Big(\frac{K }{N }\Big),\\
        &
        \frac{\Delta A_{l,l-1}{(t)}-\min_{r \leq l, r \neq l-1} \Delta A_{l,r}{(t)}}{\min_{l \in [N]}\{\Delta A_{l,l-1}{(t)}-\max_{r \leq l, r \neq l-1} \Delta A_{l,r}{(t)}\}} =  1+O\Big(\frac{K }{N }\Big),\\
        &
        \frac{\Delta A_{l,l-1}{(t)}-\max_{r \leq l, r \neq l-1} \Delta A_{l,r}{(t)}}{\min_{l \in [N]}\{\Delta A_{l,l-1}{(t)}-\max_{r \leq l, r \neq l-1} \Delta A_{l,r}{(t)}\}} =  1+O\Big(\frac{K }{N }\Big).
    \end{align*}
\end{corollary}
We omit the proof here, as it follows directly from \Cref{lem: st1-ph2-A}.
\subsubsection{End of Phase II}
\begin{lemma}\label{lem: st1-ph2-end}
    With $T^{(1)}_2-T^{(1)}_1$ at most $O
    \left(\eta_1^{-1} C_1^{-1} K^2\log K \right)$, and at iteration $t= T^{(1)}_2+1$, for any $i \in [N]$, we have
    \begin{enumerate}
        \item $A_{i,i-1}{(T^{(1)}_2+1)}-\max_{j \leq i, j \neq i-1}A_{i,j}{(T^{(1)}_2+1)}\ge \log(KN)$,
        \item $1-\attn^{(1)}_{i,i-1}(T^{(1)}_2+1)=O({1}/{K})$.
    \end{enumerate}
\end{lemma}
\begin{proof}
    Claim (1) holds from the definition of Phase II. We only need to show the upper bound of $T^{(1)}_2$.
    As \Cref{hp2} holds, there exists an $i$ for $T^{(1)}_2$ such that 
    \begin{align*}
        \log(KN) 
        &
        \geq A_{i,i-1}{(T^{(1)}_2)}-\max_{j \leq i, j \neq i-1}A_{i,j}{(T^{(1)}_2)}  \\   
        & 
        {\geq}
        \sum_{\tau=T^{(1)}_1+1}^{T^{(1)}_2-1} \Delta A_{i,i-1}{(\tau)}-\max_{j \leq i, j \neq i-1}\Delta A_{i,j}{(\tau)}+A_{i,i-1}{(T^{(1)}_1+1)}-\max_{j \leq i, j \neq i-1}A_{i,j}{(T^{(1)}_1+1)}  \\
        &
        \overset{\RM{1}}{\geq}
        (T^{(1)}_2-T^{(1)}_1-1)\frac{ \eta_1 C_1}{K^2}+\log(N),
    \end{align*}
    where $\RM{1}$ follows from \Cref{lem: st1-ph2-A,hp1}.
    By direct calculation, we have $T^{(1)}_2-T^{(1)}_1 = O
    \left(\eta_1^{-1} C_1^{-1} K^2\log K \right)$. Hence, we prove Claim (1).

    We then prove Claim (2), for all $i \in [N]$, we have
     \begin{align*}
        1-\attn_{i,i-1}^{(1)}(T^{(1)}_2+1)
        &
        = \frac{\sum_{j \leq i, j \neq i-1} \exp{(A_{i,j}(T^{(1)}_2+1)-A_{i,i-1}(T^{(1)}_2+1))}}{1+\sum_{j \leq i, j \neq i-1} \exp{(A_{i,j}(T^{(1)}_2+1)-A_{i,i-1}(T^{(1)}_2+1))}}\\
        &
        \overset{\RM{1}}{\leq} 
        \frac{(i-1)\exp{(-\log(KN))}}{1+(i-1)\exp{(-\log(KN))}} \leq O\Big(\frac{i}{KN}\Big),
    \end{align*}
    where $\RM{1}$ follows from that ${x}/({1+x})$ increase with $x$ and $\max_{j \leq i, j \neq i-1}A_{i,j}(T^{(1)}_2+1)-A_{i,i-1}(T^{(1)}_2+1) \leq - \log(KN)$. Hence, we finish the proof of this lemma.
\end{proof}

\subsection{Stage I: Phase III}\label{app: s1p3}
In this section, we study the convergence phase. For $\epsI=O(N^{-\frac{1}{2}})$, we define the \textbf{Phase III} as all iterations $T^{(1)}_2+1 \leq t\leq T^{(1)}_{3}$, where
$$
T^{(1)}_{3} \triangleq \max \Big\{t \geq T^{(1)}_2: \min_{i \in [N]} \Big(A_{i,i-1}{(t)}-\max_{j \leq i, j \neq i-1}A_{i,j}{(t)}\Big) \leq \log\left(N{\epsI}^{-1}\right)\Big\}.
$$
We then state the following induction hypothesis, which holds throughout Phase III. This hypothesis is proved by induction with the technical lemmas in \Cref{app: st1-ph3-techlem}. 

\begin{hypothesis}\label{hp3}
For each $T^{(1)}_2+1 \leq t \leq T^{(1)}_{3}$ and any $i \in [N]$, the following holds:   
\begin{enumerate}
    \item $A_{i,i-1}{(t)}-\max_{j \leq i, j \neq i-1}A_{i,j}{(t)}$ is monotonically increasing and $A_{i,i-1}{(t)}-\max_{j \leq i, j \neq i-1}A_{i,j}{(t)}\in \left[\log(N),\left(1+O((KN \epsI)^{-1}+K^{-1}\log N)\right)\log\left(N{\epsI}^{-1}\right)\right]$,
    \item $\max_{j \leq i, j \neq i-1}A_{i,j}{(t)}-\min_{l \leq i, l \neq i-1}A_{i,l}{(t)}\leq O\left({K}^{-1}{(\log N)^2}+(KN\epsI)^{-1}{\log\left(N{\epsI}^{-1}\right)}\right)$.
\end{enumerate}
\end{hypothesis}
\begin{proof}
We first prove Claim (1), and the monotonicity can be proved similarly to \Cref{hp2}. Following from \Cref{lem: st1-ph3-A}, we have
\begin{align*}
    \Delta A_{i,i-1}{(t)}-\max_{j \leq i, j \neq i-1} \Delta A_{i,j}{(t)}
    \gtrsim \frac{\eta_1 C_1\epsI }{K} \geq 0.
\end{align*}

Furthermore, following from \Cref{coro: st1-ph3-A-ratio}, and following the same steps as in \eqref{Eq: hp2-0}, except for the change in coefficients, we obtain
\begin{align*}
     & A_{i,i-1}{(t)}-\max_{j \leq i, j \neq i-1} A_{i,j}{(t)} 
     - \big(A_{i,i-1}{(T^{(1)}_2)}-\max_{j \leq i, j \neq i-1} A_{i,j}{(T^{(1)}_2)}\big)\\
     & \quad
     {\leq}  
    \Big(1+O\Big(\frac{K}{N }\!\Big)\Big)\Big\{\min_{l \in [N]}\Big(A_{l,l-1}{(t)}\!-\!\!\max_{r \leq l, r \neq l-1}\!\! A_{l,r}{(t)}\Big)
     \!-\!\!\min_{l \in [N]}\Big( A_{l,l-1}{(T^{(1)}_2)}\!-\!\!\!\max_{r \leq l, r \neq l-1} \!\!\!\! A_{l,r}{(T^{(1)}_2)}\Big)\!\Big\}.  
\end{align*}

By direct calculation, we have
\begin{align*}
    & A_{i,i-1}{(t)}-\max_{j \leq i, j \neq i-1} A_{i,j}{(t)}\\
    & \quad 
    \overset{\RM{1}}{\leq} \Big(1+O\Big(\frac{1}{KN \epsI}\Big)+O\Big(\frac{\log N}{K}\Big)\Big)\min_{l \in [N]}\Big( A_{l,l-1}{(t)}-\max_{r \leq l, r \neq l-1} A_{l,r}{(t)} \Big)\\
    & \quad
    \leq \Big(1+O\Big(\frac{1}{KN \epsI}\Big)+O\Big(\frac{\log N}{K}\Big)\Big)\log\left(N{\epsI}^{-1}\right),
\end{align*}
where $\RM{1}$ follows from \Cref{hp2,Eq: hp2-1}, and that $A_{l,l-1}{(t)}-\max_{r \leq l, r \neq l-1} A_{l,r}{(t)} $ is monotonically increasing with $t$. Hence, we prove Claim (1).

We next prove {Claim (2)},
\begin{align}
    & \max_{j \leq i, j \neq i-1}A_{i,j}{(t)}-\min_{l \leq i, l \neq i-1}A_{i,l}{(t)} 
    - \Big(\max_{j \leq i, j \neq i-1}A_{i,j}{(T^{(1)}_2)}-\min_{l \leq i, l \neq i-1}A_{i,l}{(T^{(1)}_2)}\Big)\nonumber\\
    & \qquad
    \leq \sum_{\tau=T^{(1)}_2}^{t-1} \Big(\max_{j \leq i, j \neq i-1}\Delta A_{i,j}{(\tau)}-\min_{l \leq i, l \neq i-1}\Delta A_{i,l}{(\tau)}\Big)\nonumber\\
    & 
    \qquad 
    \leq \sum_{\tau=T^{(1)}_2}^{t-1}O\Big(\frac{1}{KN\epsI}\Big)
    \Big(\Delta A_{i,i-1}{(\tau)}-\max_{j \leq i, j \neq i-1} \Delta A_{i,j}{(\tau)}\Big)\nonumber\\
    & \qquad
    = O\Big(\frac{1}{KN\epsI}\Big)\!\Big(A_{i,i-1}{(t)}-\max_{j \leq i, j \neq i-1}A_{i,j}{(t)}-
    \Big(A_{i,i-1}{(T^{(1)}_2)}-\max_{j \leq i, j \neq i-1}A_{i,j}{(T^{(1)}_2)}\Big)\Big),\label{Eq: hp3-2}
\end{align}
where the last inequality follows from \Cref{coro: st1-ph3-A-ratio}. By direct calculation, we have
\begin{align*}
     & \max_{j \leq i, j \neq i-1}A_{i,j}{(t)}-\min_{l \leq i, l \neq i-1}A_{i,l}{(t)}\\
     & \qquad 
     \overset{\RM{1}}{\leq} 
     O\Big(\frac{(\log N)^2}{K}\Big)+O\Big(\frac{1}{KN\epsI}\Big)\Big(1+O\Big(\frac{1}{KN \epsI}\Big)+O\Big(\frac{\log N}{K}\Big)\Big)\log\left(N{\epsI}^{-1}\right)\\
     & \qquad
     \leq O\Big(\frac{(\log N)^2}{K}+\frac{\log(N{\epsI}^{-1})}{KN\epsI}\Big),
\end{align*}
where $\RM{1}$ follows from \Cref{hp2,Eq: hp3-2}. Hence, we finish the proof of \Cref{hp3}.
\end{proof}
\subsubsection{Technical Lemmas}\label{app: st1-ph3-techlem}
Similar to Phases I and II, we next introduce several useful technical lemmas. The proofs follow arguments analogous to those in Phase I and are omitted for brevity.
\begin{lemma}\label{lem: st1-ph3-S}
   Suppose \Cref{hp3} holds at iteration $t$, and \Cref{lem: st1-ph3-a,lem: st1-ph3-A} hold at iteration $t-1$. For iteration $T^{(1)}_{2}+1 \leq t\leq T^{(1)}_{3}$, we have the following results.
    \begin{enumerate}
    \item  For any $i \in [N]$,
    \begin{enumerate}
    \item $\attn^{(1)}_{i,i-1}{(t)}=\Omega\left(1\right)$,
    \item $\Omega\left({\epsI}/{N}\right)\leq\attn^{(1)}_{i,j}{(t)} \leq O\left({1}/{(KN)}\right), j \neq i-1$.
    \end{enumerate}
    \item For any $i \in [N]$, 
    \begin{enumerate}
        \item $\attn^{(1)}_{i,\xb}(t)-\attn^{(1)}_{i,i-1}(t)\geq \Omega({i \epsI}/{N})$,
        \item $\attn^{(1)}_{i,y}(t) \geq \Omega({i \epsI}/{N})$,
        \item $1-\attn^{(1)}_{i,i-1}(t)\geq \Omega({i \epsI}/{N})$.
    \end{enumerate}
    \end{enumerate}
\end{lemma}

Similar to \Cref{lem: st1-ph1-a,lem: st1-ph2-a}, and following from the expression of $a_{i,j}(t)$ in \Cref{lem: st1-upA-sim}, control of attention scores in \Cref{lem: st1-ph3-S,coro: stage1-aux2-S2-bd}, we can control the order of $a_{i,j}(t)$.
\begin{lemma}[Order of $a_{i,j}(t)$]\label{lem: st1-ph3-a}
    Suppose \Cref{hp3,lem: st1-ph3-S} hold at iteration $T^{(1)}_{2}+1\leq t\leq T^{(1)}_{3}$. For different cases of $i,j$, we have the following results.
    \begin{enumerate}
    \item For $i=2n,j=2n-1$, $a_{i,i-1}{(t)}\gtrsim {i\epsI}/{(KN^2)}$. 
    \item For $i=2n,j=2m-1,m \leq n$, $- {1}/{(K^2N^2)} \lesssim a_{i,j}{(t)} \lesssim - {\epsI}/{(KN^2)}$.
    \item For $i=2n,j=2m, m \leq n$, $- {1}/{(K^2N^2)} \lesssim a_{i,j}{(t)} \lesssim - {\epsI}/{(KN^2)}$.
    \item For $i=2n-1,j=2m, m \leq n-1$,
    $|a_{i,j}{(t)}| \lesssim {1}/{(K^3N^2)}$.
    \item For $i=2n-1,j=2m-1, m \leq n$, $|a_{i,j}{(t)}| \lesssim {1}/{(K^3N^2)}$.
    \end{enumerate}
\end{lemma}

As a result of \Cref{lem: st1-ph3-a}, for any $l \in [N]$, we can characterize the relative increments gap of $A$ across different tokens as follows.
\begin{lemma}\label{lem: st1-ph3-A}
    Suppose \Cref{hp3,lem: st1-ph3-S,lem: st1-ph3-a} hold at iteration $T^{(1)}_2+1\leq t\leq T^{(1)}_{3}$, for any $l \in [N]$, the following holds
    \begin{enumerate}
        \item \textbf{Upper Bound:}
    $
    \Delta A_{l,l-1}{(t)}-\max_{r \leq l, r \neq l-1} \Delta A_{l,r}{(t)}
    \gtrsim
     K^{-1}\eta_1 C_1 \epsI.$
    \item \textbf{Lower Bound:} 
    $ 
    \max_{j \leq l, j \neq l-1}\Delta A_{l,j}{(t)}-\min_{r \leq l, r \neq l-1} \Delta A_{l,r}{(t)}
    \lesssim
      K^{-2}N^{-1}\eta_1 C_1 \log N.$
    \end{enumerate}

\end{lemma}

As a direct corollary of \Cref{lem: st1-ph3-A}, we can bound the increment differences' ratios as follows. 
\begin{corollary}\label{coro: st1-ph3-A-ratio}
    Under the same condition of \Cref{lem: st1-ph3-A}, for any $l \in [N]$, we have 
    \begin{align*}
        &
        \frac{\max_{j \leq l, j \neq l-1}\Delta A_{l,j}{(t)}-\min_{r \leq l, r \neq l-1} \Delta A_{l,r}{(t)}}{\min_{l \in [N]}\{\Delta A_{l,l-1}{(t)}-\max_{r \leq l, r \neq l-1} \Delta A_{l,r}{(t)}\}}= O\Big(\frac{1}{K N\epsI}\Big),\\
        &
        \frac{\Delta A_{l,l-1}{(t)}-\min_{r \leq l, r \neq l-1} \Delta A_{l,r}{(t)}}{\min_{l \in [N]}\{\Delta A_{l,l-1}{(t)}-\max_{r \leq l, r \neq l-1} \Delta A_{l,r}{(t)}\}} =  1+O\Big(\frac{1}{K N\epsI}\Big),\\
        &
        \frac{\Delta A_{l,l-1}{(t)}-\max_{r \leq l, r \neq l-1} \Delta A_{l,r}{(t)}}{\min_{l \in [N]}\{\Delta A_{l,l-1}{(t)}-\max_{r \leq l, r \neq l-1} \Delta A_{l,r}{(t)}\}} =  1+O\Big(\frac{1}{K N\epsI}\Big).
    \end{align*}
\end{corollary}
We omit the proof here, as it follows directly from \Cref{lem: st1-ph3-A}.

\subsubsection{End of Phase III}
\begin{lemma}\label{lem: st1-ph3-end}
    With $T^{(1)}_3-T^{(1)}_2$ at most $O
    \left(\eta_1^{-1} C_1^{-1}\epsI^{-1}{ K\log(K^{-1}\epsI^{-1})}\right)$, and at iteration $t= T^{(1)}_3+1$, for any $i \in [N]$, we have
    \begin{enumerate}
        \item $A_{i,i-1}{(T^{(1)}_3+1)}-\max_{j \leq i, j \neq i-1}A_{i,j}{(t)}
        \ge \log(N\epsI^{-1})$,
        \item $1-\attn^{(1)}_{i,i-1}(T^{(1)}_3+1) \leq \epsI$.
    \end{enumerate}
\end{lemma}
\begin{proof}
    Claim (1) holds from the definition of Phase III, so we only need to show the upper bound of $T^{(1)}_3$.
    As \Cref{hp3} holds, there exists an $i$ for $T^{(1)}_3$ such that,
    \begin{align*}
        \log\left(N{\epsI}^{-1}\right) 
        &
        \geq A_{i,i-1}{(T^{(1)}_3)}-\max_{j \leq i, j \neq i-1}A_{i,j}{(T^{(1)}_3)}  \\   
        & 
        {\geq}
        \sum_{\tau=T^{(1)}_2+1}^{T^{(1)}_3-1} \Delta A_{i,i-1}{(\tau)}-\max_{j \leq i, j \neq i-1}\Delta A_{i,j}{(\tau)}+A_{i,i-1}{(T^{(1)}_2+1)}-\max_{j \leq i, j \neq i-1}A_{i,j}{(T^{(1)}_2+1)}  \\
        &
        \overset{\RM{1}}{\geq}
        (T^{(1)}_3-T^{(1)}_2-1)\frac{ \eta_1 C_1 \epsI}{K}+\log(KN),
    \end{align*}
    where $\RM{1}$ follows from \Cref{lem: st1-ph3-A,hp2}. As a result, $T^{(1)}_3-T^{(1)}_2 = O
    \left(\eta_1^{-1} C_1^{-1}\epsI^{-1}{ K\log(K^{-1}\epsI^{-1})}\right)$.
    
    We then prove Claim (2), for all $i \in [N]$, we have
     \begin{align*}
        1-\attn_{i,i-1}^{(1)}(T^{(1)}_3+1)
        &
        = \frac{\sum_{j \leq i, j \neq i-1} \exp{(A_{i,j}(T^{(1)}_3+1)-A_{i,i-1}(T^{(1)}_3+1))}}{1+\sum_{j \leq i, j \neq i-1} \exp{(A_{i,j}(T^{(1)}_3+1)-A_{i,i-1}(T^{(1)}_3+1))}}\\
        &
        \leq \frac{(i-1)\exp{(\max_{j \leq i, j \neq i-1}A_{i,j}(T^{(1)}_3+1)-A_{i,i-1}(T^{(1)}_3+1))}}{1+(i-1)\exp{(\max_{j \leq i, j \neq i-1}A_{i,j}(T^{(1)}_3+1)-A_{i,i-1}(T^{(1)}_3+1))}}\\
        &
        \overset{\RM{1}}{\leq} 
        \frac{(i-1)e^{-\log(N\epsI^{-1})}}{1+(i-1)e^{-\log(N\epsI^{-1})}} \\
        &\leq \epsI,
    \end{align*}
    where $\RM{1}$ follows from that $x/(1+x)$ increase with $x$ and $\max_{j \leq i, j \neq i-1}A_{i,j}(T^{(1)}_2+1)-A_{i,i-1}(T^{(1)}_2+1) \leq - \log(N\epsI^{-1})$. Hence, we finish the proof of this lemma.
\end{proof}

\subsection{Proof of Stage I of \Cref{Thm: stage1 & 2}}
\begin{proof}
    As a direct result of \Cref{lem: st1-ph1-end,lem: st1-ph2-end,lem: st1-ph3-end}, with at most \begin{align*}
    \tau_1
    &
    =T^{(1)}_3+1\\
    &
    =
    O\Big(\frac{K N\log N}{ \eta_1C_1}\Big)+O
    \Big(\frac{ K^2\log K}{ \eta_1 C_1 }\Big)+O
    \Big(\frac{ K\log(K^{-1}\epsI^{-1})}{\eta_1 C_1 \epsI  }\Big)\\
    &
    =
    O\Big(\frac{ K N\log N}{ \eta_1  C_1}+\frac{ K\log(K^{-1}\epsI^{-1})}{  \eta_1 C_1 \epsI}\Big),
    \end{align*}
    we have $1-\attn^{(1)}_{i,i-1}(\tau_1) \leq \epsI$. Hence, we finish the proof.
\end{proof}

\section{Proof of Stage II of \Cref{Thm: stage1 & 2}} \label{app: proof-st2}
In Stage II, $\widetilde{W}_Q^{(1)}$ is fixed as $\widetilde{W}_Q^{(1)}(\tau_1)$. As a result, $\attn_{i,j}^{(1)}$ is fixed as $\attn_{i,j}^{(1)}(\tau_1)$. We omit $\tau_1$ for simplicity. We also omit $(t)$ and write $L(\tilde\theta)$ as $L$ here when there is no ambiguity. 
\subsection{Roadmap of the Proof}
We analyze the feature matching in stage II via two phases of dynamics (\Cref{app: s2p1,app: s2p2}). In each phase, we formulate an induction hypothesis and derive several lemmas describing the values and evolution of key statistics, which collectively govern the attention scores. We prove the induction hypothesis and the lemmas by induction.  

We begin by introducing the key statistic that is tracked throughout the different phases. Notice that after stage I training, Layer 1 attention scores almost concentrate on the immediate prefix, i.e., for each $l$, $1-\attn^{(1)}_{l,l-1}(t) \leq \epsI$. Under this condition, and by \Cref{lem: st2-var_upd}, we observe that Layer 2 attention logit  $\ub_i{(t)}$ depends strongly on the features associated with $\xbq$ and $\xb_i$. Since the frequencies are sufficiently small, given a fixed $\xbq$, both $\ub_i{(t)}$ and its update take nearly the same value across all positions $i$ where $\xb_i$ share the same feature $\vb_m,m\in[K]$. Consequently, we denote this common value of $\ub_i(t)$ by $B_m(t)$ and that of its update by $b_m(t)$ for feature $\vb_m$. Both quantities are random, as they depend on the random variable $\xbq$.

The main idea of the proof lies in analyzing the GD dynamics of $B_{m}(t)$, which decides Layer 2 attention scores $\attn_{i}^{(2)}(t)$ and  $\Attn_{m}^{(2)}(t)$. By tracking $B_{m}(t)$, we obtain the attention scores on the features, $\Attn_{m}^{(2)}(t)$. Note that when $\xbq=\vb_k$, achieving a good prediction only requires that $\Attn_k^{(2)}{(t)} \approx 1$.

For any $k\in[K]$, conditioned on $\xbq=\vb_k$, the learning process can be divided into two phases. 

\begin{itemize}
    \item \textbf{Phase I: Growth} ($t\in[\tau_1+1,\tau_1+T_{1,k}^{(2)}]$, \Cref{app: s2p1}). During phase I, $B_k{(t)}$ keeps growing at a rate of $b_k(t)=\Omega({\eta}{K^{-2}})$, while for ${k'}\neq k$, $B_{k'}{(t)}$ oscillates with a smaller rate of $b_{k'}(t)$, which satisfying $|b_{k'}(t)|=O(K^{-1}b_k(t))$. Therefore, the increase in $B_k{(t)}$ will dominate the learning dynamics during phase I. The attention score $\Attn_k^{(2)}(t)$ keeps increasing, while still satisfying $1 - \Attn_k^{(2)}(t) = \Omega(1)$.
     \item \textbf{Phase II: Convergence} ($t\in(\tau_1+T_{1,k}^{(2)}+1,\tau_1+T_{2,k}^{(2)}]$, \Cref{app: s2p2}).
     After the rapid growth of self-attention module parameters in phase I, the question token featuring $\vb_k$ is aligned with these input tokens also featuring $\vb_k$ effectively and disregards other features. Then the process proceeds to the convergence phase, where $B_k{(t)}$ monotonically increases and for ${k'}\neq k$, $B_{k'}^{(t)}$ monotonically decreases, which finally contributes to the convergence of the loss. As a result, $B_k(t)$ keeps increasing and for $k' \neq k$, $B_{k'}(t)$ keeps deceasing. After the end of phase II,  $B_k(T_{1,k}^{(2)}+1)\geq \log(K\epsII^{-1})$,
        and  $1-\Attn^{(2)}_{k}(T_{1,k}^{(2)}+1)\leq (\epsII/2)^{\frac{1}{2}}$.
\end{itemize}

We summarize the upcoming sections as follows: In \Cref{app: st2-gd}, we compute and simplify the gradients to identify the key update variables. In \Cref{app: stage2- aux lem}, we introduce several useful auxiliary lemmas for Stage II. In \Cref{app: s2p1,app: s2p2}, we analyze the two phases of the dynamics.
\subsection{Stage II: Preliminary Development}\label{app: st2-gd}

\paragraph{Computations of Gradients.} We first calculate Stage II gradient with respect to $\widetilde{W}_Q^{(2)}$. 
\begin{lemma}[Layer 2 Gradient]\label{lem: st2-gd} In Stage II, the gradient of the loss function with respect to $\widetilde{W}_Q^{(2)}$ is given by
    \begin{align*}
        \nabla_{\widetilde{W}_Q^{(2)}}L
        &
        =\EE\Big[\left(\yq-\left\langle \wb, \xbq \right\rangle\right)\sum_{\ell\in[\Nit]}y_\ell\attn_{2\ell}^{(2)}\\
        & 
        \quad \cdot \sum_{i=1}^{N}(\mathbbm{1}{\{2\ell=i\}}-\attn_{i}^{(2)})\Big( R_{\bvartheta,N-i}\Big(\attn^{(1)}_{i,i-1}(E^{\xb,y}_{i-1})^\top+\sum_{j \neq i-1}\attn^{(1)}_{i,j}(E^{\xb,y}_{j})^\top\Big)\Eq^{\xb,y}\Big)^\top\Big],
    \end{align*}
    where $E_j^{\xb,y}$ and $\Eq^{\xb,y}$ correspond to the semantic dependent subspaces of the token embeddings $E_j$ and $\Eq$, respectively.
\end{lemma}
\begin{proof}
We first write out the expression of $\yq$. In Stage II, $\widetilde{W}_Q^{(1)}$ is fixed, then
\begin{align*}
     \yq(W_\qb{(\tau_1)},\widetilde{W}_Q^{(2)}) 
    &
    = \sum_{i=1}^N \attn_{2i}^{(2)} y_i
    = \sum_{k=1}^K \Attn_{k}^{(2)} \langle\wb,\vb_k\rangle.
\end{align*}
Consider$\attn^{(2)}=\softmax(\ub) \in \RR^{1\times N}$, and
\begin{align}
    \ub_i
    & =
    \Eq^{\xb,y}\widetilde{W}_Q^{(2)}R_{\bvartheta,N-i} (\attn^{(1)}E^{\xb,y})^\top_{i,:}\nonumber\\
    &
    =
    \Eq^{\xb,y}\widetilde{W}_Q^{(2)}(t)R_{\bvartheta,N-i}\Big(\attn^{(1)}_{i,i-1}(E^{\xb,y}_{i-1})^\top+\sum_{j \neq i-1}\attn^{(1)}_{i,j}(E^{\xb,y}_{j})^\top\Big), \label{Eq: layer2-gdc-1}
\end{align}
where the last inequality follows from that $(\attn^{(1)}E^{\xb,y})_{i,:}= 
    \attn^{(1)}_{i,:}(t)E^{\xb,y}
    =\attn^{(1)}_{i,i-1}E^{\xb,y}_{i-1,:}+ \sum_{j \neq i-1}\attn^{(1)}_{i,j}E^{\xb,y}_{j,:}$. Then we can calculate the gradient. We first obtain:
\begin{align}
    \nabla_{\widetilde{W}_Q^{(2)}} L= \mathbb{E}[\left(\yq-\left\langle \wb, \xbq \right\rangle\right)\nabla_{\widetilde{W}_Q^{(2)}}\yq ]
    =\mathbb{E}\Big[\left(\yq-\left\langle \wb, \xbq \right\rangle\right)\sum_{\ell\in[\Nit]}(\nabla_{\widetilde{W}_Q^{(2)}}\attn_{2\ell}^{(2)})y_\ell\Big].\label{Eq: layer2-gdc-2}
\end{align}
Recall that for any $U \in \RR^{1 \times d}$, $\nabla_U\softmax(U)=\diag(\softmax(U))-\softmax(U)^\top\softmax(U)$, we have
\begin{align}
    \nabla_{\widetilde{W}_Q^{(2)}} \attn^{(2)}_{2\ell}
    &
    =
    \sum_{i=1}^{N}\frac{\partial \attn^{(2)}_{2\ell} }{\partial\ub_i} \nabla_{\widetilde{W}_Q^{(2)}}\ub_i\nonumber\\
    &
    =
    \attn_{2\ell}^{(2)}\sum_{i=1}^{N}(\mathbbm{1}{\{2\ell=i\}}-\attn_{i}^{(2)})\Big( R_{\bvartheta,N-i}\Big(\attn^{(1)}_{i,i-1}(E^{\xb,y}_{i-1})^\top+\sum_{j \neq i-1}\attn^{(1)}_{i,j}(E^{\xb,y}_{j})^\top\Big)\Eq^{\xb,y}\Big)^\top.\label{Eq: layer2-gdc-3}
\end{align}
Plug \eqref{Eq: layer2-gdc-3} into \eqref{Eq: layer2-gdc-2}, we finish the proof of \Cref{lem: st2-gd}.
\end{proof}

We find that for stage II, $\ub_{i}(t)$ is important in determining the attention scores. As a result, we track the update of $\ub_{i}(t)$ as follows.   
\begin{lemma}[Track variable update]\label{lem: st2-var_upd} In Stage II, the attention logit at position $i$ can be expressed as 
\begin{align*}
    \ub_{i}(t)
    &
    =
    \Eq^{\xb,y}\widetilde{W}_Q^{(2)}(t)R_{\bvartheta,N-i}\Big(\attn^{(1)}_{i,i-1}(t)(E^{\xb,y}_{i-1})^\top+\sum_{j \neq i-1}\attn^{(1)}_{i,j}(t)(E^{\xb,y}_{j})^\top\Big).
\end{align*}
Conditioned on the event that $\xbq=\vb_k$, we can characterize the increments of $\ub_i(t)$ over $t$ for even $i$ and upper bound $\ub_i(t)$ for odd $i$ as follows.
\begin{enumerate}
    \item  For any $m\in[K]$, if $i$ is even and $\xb_{i/2}=\vb_{m}$, we define the following variables conditioned on the features,
\begin{align*}
    B_{i,m}(t)
    &
     =\ub_{i}(t), 
\end{align*}
where the subscript $m$ specifies the particular feature $\vb_{m}$ taken by $\xb_{i/2}$. Then for each $m \in [K]$, the update of $B_{i,m}(t)$ denoted by $ b_{i,m}(t)=B_{i,m}(t+1)-B_{i,m}(t)$ can be characterized as follows.
\begin{enumerate}[label=(\alph*)]
    \item If $m=k$, then
    \begin{align*}
        b_{i,k}(t)& =\!\eta_2 \EE\Big[\mathbbm{1}{\{\xbq=\vb_k\}}\Attn_k^{(2)}(t)\!\Big((\Attn_k^{(2)}(t)-1)^2+\sum_{m \neq k}(\Attn_{m}^{(2)}(t))^2\Big)\Big] \pm O\Big(\frac{\eta_2\epsI d_\cX}{K}\Big).
    \end{align*}
    \item For $m\neq k$, we have
    \begin{align*}
        b_{i,m}(t) 
    &
    =\!
    \eta_2\EE\Big[\mathbbm{1}{\{\xbq=\vb_k\}}\Attn_{m}^{(2)}(t)\!\Big(\!\!\sum_{m\in [K]}(\Attn_{m}^{(2)}(t))^2\!-\!\Attn_{m}^{(2)}(t)-\Attn_k^{(2)}(t)\Big)\Big]\pm O\Big(\frac{\eta_2\epsI d_\cX}{K}\Big).
    \end{align*}
\end{enumerate}
    
    Since the dependence on $i$
    is negligible from RHS above, we omit the position subscripts $i$ in $B_{i,m}(t)$ and $b_{i,m}(t)$ 
    when no ambiguity arises.  
 \item If $i$ is odd, then $\ub_{i}(t)$ is always relatively small and can be controlled
    \begin{align*}
        \ub_{i}(t)
        \leq \epsI \Theta\Big(\max_{m\in [K]} B_m(t)\Big).
    \end{align*}
\end{enumerate}
\end{lemma}
\begin{proof}
   We first prove Claim (1), where $i$ is even, and $i-1$ is odd. In the following, we omit $(t)$ of $\attn^{(1)}(t),\attn^{(2)}(t)$ and $\Attn^{(2)}(t)$ for abbreviation. From \eqref{Eq: layer2-gdc-1}, we have
    \begin{align}
    &\ub_{i}(t+1)-\ub_{i}(t)\nonumber\\
    & \quad
    =
    -\eta_2 \Eq^{\xb,y}\nabla_{\widetilde{W}_Q^{(2)}}L(t) R_{\bvartheta,N-i}\Big(\attn^{(1)}_{i,i-1}(E^{\xb,y}_{i-1})^\top+\sum_{j \neq i-1}\attn^{(1)}_{i,j}(E^{\xb,y}_{j})^\top\Big) \nonumber\\
    &
    \quad
    \overset{\RM{1}}{=}
    -\eta_2 \Eq^{\xb,y} \EE\Big[\left(\yq-\left\langle \wb, \xbq \right\rangle\right)\sum_{\ell\in[\Nit]}y_\ell\attn_{2\ell}^{(2)}\sum_{m=1}^{N}(\mathbbm{1}{\{2\ell=m\}}-\attn_{m}^{(2)})\nonumber\\
        & 
        \qquad \!\cdot\! \Big((\Eq^{\xb,y})^{\!\top}\! \Big(\attn^{(1)}_{m,m-1}E^{\xb,y}_{m-1}\!+\!\sum_{j \neq m-1}\attn^{(1)}_{m,j}E^{\xb,y}_{j} \Big)R_{\bvartheta,m-i}\Big)\Big]\!\!\Big(\attn^{(1)}_{i,i-1}(E^{\xb,y}_{i-1})^\top\!+\!\sum_{j \neq i-1}\attn^{(1)}_{i,j}(E^{\xb,y}_{j})^\top\Big)\nonumber\\
    &
    \quad
    \overset{\RM{2}}{=}
    \!-\eta_2\EE\!\Big[\mathbbm{1}{\{\xbq=\vb_k\}}\left(\yq\!-\!\left\langle \wb, \xbq \right\rangle\right)\!\sum_{\ell\in[\Nit]}y_\ell\attn_{2\ell}^{(2)}\!\sum_{m=1}^{N}(\mathbbm{1}{\{2\ell=m\}}\!-\!\attn_{m}^{(2)})\left(E^{\xb,y}_{m-1} R_{\bvartheta,m-i}\right)\Big]\nonumber\\
    &
    \qquad 
    \!\cdot\! \left(E^{\xb,y}_{i-1}\right)^\top
    \pm O\Big(\eta_2\epsI  \EE\Big[\mathbbm{1}{\{\xbq=\vb_k\}}\Big|\left(\yq-\left\langle \wb, \xbq \right\rangle\right)\sum_{\ell\in[\Nit]}y_\ell\attn_{2\ell}^{(2)}\Big|\Big]\Big)\nonumber\\
    &
    \quad
    \overset{\RM{3}}{=}
    \!-\eta_2
     \EE\Big[\!\mathbbm{1}{\{\xbq=\vb_k\}}\!\left(\yq\!-\!\left\langle \wb, \xbq \right\rangle\right)\!\!\sum_{\ell\in[\Nit]}y_\ell\attn_{2\ell}^{(2)}\!\sum_{m=1}^{N}\!\attn_{m}^{(2)}\!\left(E^{\xb,y}_{2\ell-1} R_{\bvartheta,2\ell-i}\!-\!E^{\xb,y}_{m-1} R_{\bvartheta,m-i}\right)\!\Big]\!\!\nonumber\\
    &
    \qquad 
    \!\cdot\! \left(E^{\xb,y}_{i-1}\right)^\top
    \pm O\left(K^{-1}\eta_2\epsI d_\cX\right),\label{Eq: layer2-varup-0}
    \end{align}
    where $\RM{1}$ follows from \Cref{lem: st2-gd} and the property of \ac{rope} operator, $\RM{2}$ follows from $\Eq^{\xb,y}(\Eq^{\xb,y})^\top=1$, and the property that $1-\attn^{(1)}_{i,i-1}(\tau_1) \leq \epsI$ for all $i \in [N]$ at the end of Stage I, as in \Cref{lem: st1-ph3-end}, and $\RM{3}$ follows from that $\|\wb\|_2 \leq \sqrt{d_\cX}$ and $\|\vb_k\|_2=1$ for all $k \in [K]$.

For the dominant term (first term) in \eqref{Eq: layer2-varup-0}, conditioned on the event $\xbq=\vb_k$, we can further divide it into two terms and discuss by cases.

Case (a): If the corresponding $\xb_{i/2}$ of $E_{i-1}^{\xb,y}$ is $\vb_k$, then with $\yq-\left\langle \wb, \xbq \right\rangle=\wb^\top (\sum_{\ell\in[\Nit]}\attn_{2\ell}^{(2)}\xb_\ell- \xbq)$, we have the first term as
\begin{align}
    & \EE\Big[\mathbbm{1}{\{\xbq=\vb_k\}}\left(\yq-\left\langle \wb, \xbq \right\rangle\right)\sum_{\ell\in[\Nit]}y_\ell\attn_{2\ell}^{(2)}\sum_{m=1}^{N}\attn_{m}^{(2)}E^{\xb,y}_{2\ell-1} R_{\bvartheta,2\ell-i}\Big] \left(E^{\xb,y}_{i-1}\right)^\top\nonumber\\
    &
    \quad
    \overset{\RM{1}}{=}
    \EE\Big[\mathbbm{1}{\{\xbq=\vb_k\}}\Big(\sum_{r\in[\Nit]}\attn_{2r}^{(2)}\xb_r- \xbq\Big)^\top\sum_{\ell\in[\Nit]}\xb_\ell\attn_{2\ell}^{(2)}\left(\xb_\ell\right)^\top R_{\bvartheta,2\ell-i}\Big]\vb_k \nonumber\\
    &
    \quad
    \overset{\RM{2}}{=}\!
    \EE\!\Big[\mathbbm{1}{\{\xbq=\vb_k\}}\!\Big(\sum_{r\in[\Nit]}\attn_{2r}^{(2)}\xb_r\!-\! \xbq\!\Big)^{\!\!\top}\!\!\!\sum_{\ell\in[\Nit]}\xb_\ell\attn_{2\ell}^{(2)}\!\!\left(\mathbbm{1}{\{\xb_\ell\!=\!\vb_k\}}\pm\mathbbm{1}{\{\xb_\ell\!\neq\!\vb_k\!\}}O\left({N^{-\alpha+1}}\right)\right)\!\Big]\nonumber\\
    &
    \quad
    =
    \EE\Big[\mathbbm{1}{\{\xbq=\vb_k\}}\big(\Attn_k^{(2)}\big(\Attn_k^{(2)}-1\big)\pm \sum_{m\neq k}(\Attn^{(2)}_m)^2 O\left({N^{-\alpha+1}}\right)\big)\Big],\label{Eq: layer2-varup-1}
\end{align}
where $\RM{1}$ follows from taking the conditional expectation over $\wb$ and that the last 2-dimensional subspace of embedding of the input is a null space, and $\RM{2}$ follows similarly from \Cref{lem: st1-aux1}. And the second term is 
\begin{align}
    &
    \EE\Big[\mathbbm{1}{\{\xbq=\vb_k\}}\left(\yq-\left\langle \wb, \xbq \right\rangle\right)\sum_{\ell\in[\Nit]}y_\ell\attn_{2\ell}^{(2)}\sum_{m=1}^{N}\attn_{m}^{(2)}\left(-E^{\xb,y}_{m-1} R_{\bvartheta,m-i}\right)\Big] \left(E^{\xb,y}_{i-1}\right)^\top\nonumber\\
    &
    \quad
    \overset{\RM{1}}{=}
    \EE\Big[\mathbbm{1}{\{\xbq=\vb_k\}}\Big(\sum_{r\in[\Nit]}\attn_{2r}^{(2)}\xb_r- \xbq\Big)^\top\Big(\sum_{\ell\in[\Nit]}\xb_\ell\attn_{2\ell}^{(2)}\Big) \sum_{m=1}^{N}\attn_{m}^{(2)}\left(-E^{\xb,y}_{m-1} R_{\bvartheta,m-i}\right)\Big] \left(E^{\xb,y}_{i-1}\right)^\top\nonumber\\
    &
    \quad
    \overset{\RM{2}}{=}
    -\EE\Big[\mathbbm{1}{\{\xbq=\vb_k\}}\Big(\sum_{m\in[K]}(\Attn_{m}^{(2)})^2- \Attn_{k}^{(2)}\Big)\Big(\Attn_k^{(2)}(1-\Omega\left({N^{-\alpha+1}}\right))\pm\sum_{m\neq k }\Attn_m^{(2)}O(N^{-\alpha+1})\Big)\Big],\label{Eq: layer2-varup-2}
\end{align}
where $\RM{1}$ follows from taking the conditional expectation over $\wb$ and that the last 2-dimensional subspace of embedding of the input is a null space, and $\RM{2}$ follows from the fact that the summation over $\ell$ and $r$ are independent of the summation over $m$, and \Cref{lem: st1-aux1}. Sum \eqref{Eq: layer2-varup-1} and \eqref{Eq: layer2-varup-2}, we get the dominant term in \eqref{Eq: layer2-varup-0} as follow  
\begin{align*}
    -\eta_2\EE\Big[\mathbbm{1}{\{\xbq=\vb_k\}}\Big(\Attn_k^{(2)}\Big(-(\Attn_k^{(2)}-1)^2-\sum_{m \neq k}(\Attn_{m}^{(2)})^2\Big)\pm  O\left({N^{-\alpha+1}}\right)\Big)\Big].
\end{align*}
Plugging into \eqref{Eq: layer2-varup-0}, we have
\begin{align*}
    b_{i,k}(t)& =\!\eta_2 \EE\Big[\mathbbm{1}{\{\xbq=\vb_k\}}\Attn_k^{(2)}\!\Big((\Attn_k^{(2)}-1)^2+\sum_{m \neq k}(\Attn_{m}^{(2)})^2\Big)\Big] \pm O\left(\frac{\eta_2\epsI d_\cX}{K}\right).
\end{align*}
Case (b): If the corresponding $\xb_{i/2}$ of $E_{i-1}^{\xb,y}$ is $\vb_{m},m \neq k$, similar to Case (a), we can get 
\begin{align*}
    b_{i,m}(t)& =\eta_2 \EE\Big[\mathbbm{1}{\{\xbq=\vb_k\}}\Attn_{m}^{(2)}\Big(\Attn_{m}^{(2)}+\Attn_k^{(2)}-\sum_{m\in [K]}(\Attn_{m}^{(2)})^2\Big)\Big]\pm O\left(\frac{\eta_2\epsI d_\cX}{K}\right).
\end{align*}
We note that with the low frequency small enough, the dependence of the increments of $\ub_i(t)$ on the position $i$ is dominated and can be controlled by a minor term. Since the dependence on $i$
is negligible, we omit the position subscripts $i$ when no ambiguity arises. Hence we prove Claim (1).

We next prove Claim (2), where $i$ is odd, and $i-1$ is even. Then, $\Eq^{\xb,y}\widetilde{W}_Q^{(2)}R_{\bvartheta,N-i}\attn^{(1)}_{i,i-1}(E^{\xb,y}_{i-1})^\top=0$, as a result,
    \begin{align*}
        \ub_{i}(t)
    &
    =
    \Eq^{\xb,y}\widetilde{W}_Q^{(2)}(t)R_{\bvartheta,N-i}\Big(\sum_{j \neq i-1}\attn^{(1)}_{i,j}(E^{\xb,y}_{j})^\top\Big)\\
    &
    \overset{\RM{1}}{\leq}
    \epsI \max_{m\in [K]}{\vb_k^\top\widetilde{W}_Q^{(2)}(t)R_{\bvartheta,N-i} \vb_{m}}\\
    &
    \overset{\RM{2}}{\leq}
    \epsI \Theta\Big(\max_{m\in [K]} B_m(t)\Big),
    \end{align*}
    where $\RM{1}$ follows from the property that $\sum_{j \neq i-1}\attn^{(1)}_{i,j}(\tau_1)=1-\attn^{(1)}_{i,i-1}(\tau_1) \leq \epsI$ for all $i \in [N]$ at the end of Stage I, as in \Cref{lem: st1-ph3-end}, and that the last 2-dimensional subspace of embedding of the input is a null space, and $\RM{2}$ follows from the definition of $B_m(t)$. Hence, we prove \Cref{lem: st2-var_upd}.
\end{proof}

\subsection{Stage II: Auxiliary Lemmas}\label{app: stage2- aux lem}
 Recall that given a prompt $P = (\xb_1, y_1, \ldots, \xb_{\Nit}, y_{\Nit}, \xb_{\text{q}})$, we denote $\pit$ as the collection of input tokens in the example, i.e., $\{\xb_i\}_{i=1}^{\Nit}$. Similar to \cite{huang2024context}, it is worth noting that, based on our data distribution, the occurrence count of the $k$-th feature in $\pit$, denoted as $|\cV_k|$, follows a multinomial distribution. Leveraging the concentration property inherent to multinomial distributions, we can identify a high-probability event to which $\pit$ belongs. As $N=2\Nit+1$, we can use $N$ instead of $\Nit$ in the following lemma. 
\begin{lemma}[High-probability Event for $\pit$]\label{lem: st2-aux-p}
  Suppose that $p_k=\Theta\left({1}/{K}\right)$ for any $k\in[K]$ and  $K^3\ll  N$. For some constant  $c\geqslant \sqrt{{20 K^3}/{N}} $, define
$$
\esb:=\Big\{\pit: |\cV_k|\in \Big[p_k N-\frac{c N}{K}, p_k N+\frac{c N}{K}\Big]\text{ for }k\in[K]\Big\}.
$$Then 
, we have
\begin{align*}
    \mathbb{P}(\pit\in\esb )\geq 1-3 \exp \left(-\frac{c^2 N}{25 K^2}\right).
\end{align*}
Let us denote $\Lba_k=p_k K-c$ and $\Uba_k= p_k K+c$. Note that $\Lba_k, \Uba_k$ are at the order of  the constant level since $p_k=\Theta\left({1}/{K}\right)$. Then for any $\pit$ belonging to $\esb$, $|\cV_k|\in[{\Lba_kN}/{K},{\Uba_kN}/{K}]=\Theta({N}/{K})$. Note that  we can properly choose $c$ to guarantee $\Lba_k> 0$ for $k\in[K]$.

\end{lemma}

\subsection{Stage II: Phase I}\label{app: s2p1}
In this section, we study the initial phase of learning the relationship between the question token $\Eq$ and its previous tokens with different features. We define the \textbf{Phase I} as all iterations $\tau_1+1 \leq t\leq \tau_1+T^{(2)}_{1,k}$, where
$$
T^{(2)}_{1,k} \triangleq \max \left\{t:  B_{k}(\tau_1+t)\leq \log K \right\}.
$$
Then, we state the following induction hypothesis, which holds throughout Phase I. This hypothesis is proved by induction with the technical lemmas in \Cref{app: st2-ph1-techlem}. 

\begin{hypothesis}\label{hp2-1}
Suppose \Cref{lem: st2-ph1-AttnS,lem: st2-ph1-b,hp2-1} hold at iteration $t-1$, where $\tau_1+1 \leq t\leq \tau_1+T^{(2)}_{1,k}$. Given any $k \in [K]$, conditioned on the events that $\xbq=\vb_k$ and $\pit \in \esb$, the following holds for iteration $t$:
\begin{enumerate}
    \item $B_k(t)$ is monotonically increasing, and $B_k(t) \in [(1-\Omega(\epsI+N^{1-\alpha})),\log (K)]$.
    \item For any $k' \neq k$, $
    |B_{k'}(t)|=O\left({B_{k}(t)}/{K}\right)=O((\log K)/K)$.
    \item Specially, for the initialization of Stage II, $t=\tau_1+1$, 
    \begin{align*}
       B_k(\tau_1+1) &
       \in \left[1-\Omega(\epsI+{N^{1-\alpha}}),1+O(\epsI)\right],
       |B_{k'}(\tau_1+1)|
    \leq {  O\left(\epsI\right)}.
    \end{align*}
\end{enumerate}
\end{hypothesis}
\begin{proof}
    We first prove Claim (3). By the definition of $B_{m}$ in \Cref{lem: st2-var_upd}, we only need to consider odd positions. For any $i \in [\Nit]$, by definition of $\ub_{2i}(\tau_1+1)$ and $\widetilde{W}_Q^{(2)}(\tau_1+1)=I_d$, we have
    \begin{align*}
        \ub_{2i}(\tau_1+1)
    &
    =
     \Eq^{\xb,y}R_{\bvartheta,N-2i}\big(\attn^{(1)}_{2i,2i-1}(E^{\xb,y}_{2i-1})^\top+\sum_{j \neq i}\attn^{(1)}_{2i,j}(E^{\xb,y}_{j})^\top\big).
    \end{align*}
    Then applying \Cref{lem: st2-var_upd}, together with $\xbq=\xb_i=\vb_k$ and $1-\attn_{i,i-1}^{(1)}(\tau_1+1) \leq \epsI$ for all $i \in \Nit$, we have
    \begin{align*}
         B_{2i,k}(\tau_1+1) 
        &
        {=}
        { \big(S_{i,i-1}^{(1)}(1-\Omega(N^{1-\alpha}))\pm O(\epsI)\big)}
        \in
        \left[(1-\Omega(\epsI+{N^{1-\alpha}})),1+O(\epsI)\right].
    \end{align*}
    Hence, $B_{k}(\tau_1+1)\in
        \left[(1-\Omega(\epsI+{N^{1-\alpha}})),1+O(\epsI)\right]$. And for $k' \neq k$, we have
    \begin{align*}
         |B_{2i,k'}(\tau_1+1)| 
        &
        =
        \left|{\ub_{2i}(\tau+1)}|{\{\xb_i=\vb_{k'}\}}\right|
        = { O \left(\epsI\right)}.
    \end{align*}
 Hence $B_{k'}(\tau_1+1)= { O \left(\epsI\right)}$, and we finish the proof of Claim (3). 
 
    We next prove Claim (1). For any $\tau_1+1 \leq t\leq \tau_1+T^{(2)}_{1,k}$, we first need to show $b_k(t)\geq 0$, 
    which is obvious from \Cref{lem: st2-ph1-b}. The lower bound follows from Claim (3), and the upper bound follows from the definition of Phase I.

    We finally prove Claim (2). By \Cref{lem: st2-ph1-b}, we have for any $k' \neq k$, $|b_{k'}(t-1)| \leq O\left({b_{k}(t-1)}/{K}\right)$ at time $t-1$. Then, 
\begin{align*}
    |B_{k'}(t)| 
    &
    \leq |B_{k'}(t-1)| + |b_{k'}(t-1)| 
    \leq
    O\left(\frac{B_{k}(t-1)}{K}\right)+O\left(\frac{b_{k}(t-1)}{K}\right)
    \leq
    O\left(\frac{B_{k}(t)}{K}\right).
\end{align*}
Hence, we finish the proof of \Cref{hp2-1}.

\end{proof}
\subsubsection{Technical Lemmas}\label{app: st2-ph1-techlem}
Similar to stage I, we next introduce several useful technical lemmas and prove them by induction.

\begin{lemma}\label{lem: st2-ph1-AttnS}
    Suppose \Cref{hp2-1} holds at iteration $\tau_1+1\leq t\leq \tau_1+T^{(2)}_{1,k}$. Conditioned on the events $\xbq=\vb_k$ and $\pit \in \esb$, the following holds for iteration $t$.  
    \begin{enumerate}
    \item For the attention scores related to feature $k$, 
    \begin{enumerate}
        \item $\Attn_k^{(2)}(t)=\Omega({1}/{K})$,
        \item $1-\Attn_k^{(2)}(t)=\Omega(1)$.
    \end{enumerate}
     \item  For the attention scores related to feature $k' \neq k$, $\Attn_{k'}^{(2)}(t)=\Theta\big(({1-\Attn_k^{(2)}(t)})/{K}\big)=\Theta({1}/{K})$.
    \end{enumerate}
\end{lemma}
\begin{proof}
We first prove {Claim (1)-(a)}.
\begin{align}
        \Attn_{k}^{(2)}(t)
        &
        =\frac{\sum_{i=2n,n\in[\Nit]}\exp{(\ub_i(t))}\mathbbm{1}{\{\xb_i=\vb_k\}}}
        {\sum_{m \in [K]}\sum_{i=2n,n\in[\Nit]}\exp{(\ub_i(t))}\mathbbm{1}{\{\xb_i=\vb_m\}}+\sum_{i=2n-1,n\in[\Nit+1]}\exp{(\ub_i(t))}}\nonumber\\
        &
        =\frac{|\cV_k|\exp{(B_{k}(t))}}{|\cV_k|\exp{(B_{k}(t))}+\sum_{m \neq k}|\cV_m|\exp{(B_{m}(t))}+\sum_{i=2n-1,n\in[\Nit+1]}\exp{(\ub_i(t))}}\nonumber\\
        &
        \overset{\RM{1}}{\geq}
        \frac{1}{1+\sum_{m \neq k}(|\cV_m|/|\cV_k|)\exp{(B_m(t)-B_k(t))}+({\Nit}/{|\cV_k|})\exp{(-(1-\Theta(\epsI))B_k(t))}}\label{Eq: st2-ph1-S1}\\
        &
        \overset{\RM{2}}{\geq}
        \frac{1}{1+\sum_{m \neq k}({|\cV_m|}/{|\cV_k|})\exp{(B_m(t)-B_k(t))}+{\Nit}/{|\cV_k|}},\nonumber
    \end{align}
where $\RM{1}$ follows from \Cref{lem: st2-var_upd}, which gives $\ub_{i}(t)
        \leq \epsI \Theta(\max_{m\in [K]} B_m(t))$ for odd $i$, together with the fact that $\max_{m \neq k}B_{m}(t) \leq B_k(t)$ in \Cref{hp2-1}, and $\RM{2}$ follows from that $-B_k(t)(1-\Theta(\epsI))\leq 0$. Furthermore, by \Cref{hp2-1}, for any $k' \neq k$
\begin{align*}
    \exp{(-\log K-O((\log K)/K))}\leq \exp(B^{(t)}_{k'}-B^{(t)}_k)\leq \exp{(O((\log K)/K))}.
\end{align*}
By direct calculation and \Cref{lem: st2-aux-p}, we have
\begin{align*}
    \Attn_k^{(2)}(t) 
    &
    \geq
    \frac{1}{e^{O\left((\log K)/K\right)}({\Nit}/{|\cV_k|}-1)+1+{\Nit}/{|\cV_k|}}\geq \frac{1}{e^{O\left((\log K)/K\right)}(K/\Lba_k-1)+1+K/\Lba_k}\!=\Omega\Big(\frac{1}{K}\Big).
\end{align*}
Hence, we prove Claim (1)-(a) and next prove Claim (1)-(b). Similar to Claim (1)-(a), we have 
\begin{align*}
        \Attn_{k}^{(2)}(t)
        &
        {\leq}
        \frac{1}{1+\sum_{m \neq k}({|\cV_m|}/{|\cV_k|})\exp{(B_m(t)-B_k(t))}}\leq \frac{1}{e^{-1}({1}/{\Uba_k}-{1}/{K})+1}.
    \end{align*}
Together with $\Uba_k= \Theta(1)$, we have
\begin{align*}
    1- \Attn^{(t)}_{k} \geq \frac{({1}/{\Uba_k}-{1}/{K})}{({1}/{\Uba_k}-{1}/{K})+e}\geq \Omega(1).
\end{align*}
Hence, we prove Claim (1)-(b). Finally, we prove Claim (2). For any $k' \neq k$, we have
    \begin{align*}
        \Attn_{k'}^{(2)}(t)
        &
        =\frac{|\cV_{k'}|\exp{(B_{k'}(t))}}{|\cV_{k'}|\exp{(B_{k'}(t))}+\sum_{m \neq k'}|\cV_m|\exp{(B_{m}(t))}+\sum_{i=2n-1,n\in[\Nit+1]}\exp{(\ub_i(t))}}.
    \end{align*}
    Because $\exp{(\ub_i(t))}\geq 0$ for all $i$, we have
    \begin{align*}
        \frac{\Attn^{(2)}_{k'}(t)}{1-\Attn_k^{(2)}(t)}
        &
        \leq
        \frac{1}{1+\sum_{m \neq k',k}(|\cV_m|/|\cV_{k'}|)\exp{(B_{m}(t)-B_{k'}(t))}}
        \leq
        O\left(\frac{1}{K}\right),
    \end{align*}
  where the last inequality follows from \Cref{lem: st2-aux-p} and ${\Uba_m}/{\Lba_{k'}}=\Theta(1)$. For another direction,  we have
  \begin{align}
     \frac{\Attn^{(2)}_{k'}(t)}{1-\Attn_k^{(2)}(t)}
     &
     \overset{\RM{1}}{\geq}
    \frac{1}{1+\sum_{m \neq k',k}(|\cV_m|/|\cV_{k'}|)\exp{(B_m(t)-B_{k'}(t))}+(\Nit/|\cV_{k'}|)\exp{(\epsI \Theta\left( B_k(t)\right)-B_{k'}(t))}}\nonumber\\
    &
    \overset{\RM{2}}{\geq}
    \frac{1}{1+\sum_{m \neq k',k}(|\cV_m|/|\cV_{k'}|)e^{O((\log K)/{K})}+(\Nit/{|\cV_{k'}|})e^{\epsI\log(K)+O((\log K)/{K})}}\nonumber\\
    &
    \overset{\RM{3}}{\geq}
    \frac{1}{1+(\sum_{m \neq k',k}\Uba_m e/\Lba_{k'})+K e/\Lba_{k'}}
    \geq
    \Omega\left(\frac{1}{K}\right), \label{Eq: st2-ph1-S2}
  \end{align}
where $\RM{1}$ follows from \Cref{lem: st2-var_upd}, $\RM{2}$ follows from \Cref{hp2-1}, and $\RM{3}$ follows from \Cref{lem: st2-aux-p}. Hence, we finish the proof of \Cref{lem: st2-ph1-AttnS}.
\end{proof}

In the following lemma, we control the order of $b_{k'}(t),k' \in [K]$.
\begin{lemma}[Order of updates $b(t)$]\label{lem: st2-ph1-b}
    Suppose \Cref{hp2-1,lem: st2-ph1-AttnS} hold at iteration $\tau_1+1\leq t\leq \tau_1+T^{(2)}_{1,k}$, under the same assumption of \Cref{lem: st2-ph1-AttnS}, we have   
    \begin{enumerate}
    \item For $k$, $b_k(t)\geq 0$ and $b_k(t)=\Omega\left(\eta_2/ K^{2}\right)$.
    \item For $k' \neq k$, $|b_{k'}(t)| \leq O({b_{k}(t)}/{K})$.
    \end{enumerate}
\end{lemma}
\begin{proof}
    We first prove Claim (1). By \Cref{lem: st2-var_upd}, we only need to consider the dominated term of $b_{k}(t)$ as follows:
    \begin{align*}
        &\eta_2 \EE\big[\mathbbm{1}{\{\xbq=\vb_k\}}\big(\Attn_k^{(2)}\big((\Attn_k^{(2)}-1)^2+\sum_{m \neq k}(\Attn_{m}^{(2)})^2\big)\big)\big]\\
        &
        \quad
        \geq
        \eta_2 p_k \PP(\pit \in \esb)
        \times \EE\big[\Attn_k^{(2)}\left(\Attn_k^{(2)}-1\right)^2|\{\xbq=\vb_k\}\cap\{\pit \in \esb\}\big]\!+3\eta_2 p_k\PP(\pit \notin \esb)\\
        &
        \quad
        \overset{\RM{1}}{\geq}
        \Omega\left(\frac{\eta_2}{K^2}\right), 
    \end{align*}
    where $\RM{1}$ follows from \Cref{lem: st2-aux-p,lem: st2-ph1-AttnS}. Hence, we prove Claim (1). We next prove Claim (2). By \Cref{lem: st2-var_upd}, for any $k' \neq k$, we denote the dominated term of $b_{k'}(t)$ as follows
    \begin{align*}
      (I) 
    &
    =
    \eta_2 \EE\Big[\mathbbm{1}{\{\xbq=\vb_k\}}\Attn_{k'}^{(2)}\Big(\sum_{m\in [K]}(\Attn_{m}^{(2)})^2-(\Attn_{k'}^{(2)}+\Attn_k^{(2)})\Big)\Big].
    \end{align*}
On one hand, we have
\begin{align}
    (I) 
        &
        \leq \eta_2 p_k \PP(\pit \in \esb) \EE\big[\Attn_{k'}^{(2)}\big(\max_{m\neq k,k'}\Attn_{m}^{(2)}\big)|\{\xbq=\vb_k\}\cap\{\pit \in \esb\}\big]
        +
        \eta_2 p_k \PP(\pit \notin \esb)\nonumber\\
        &
        \overset{\RM{1}}{\leq} O\left(\frac{\eta_2}{K^3}\right)\label{Eq: ph1-U2},
\end{align}
where $\RM{1}$ follows from \Cref{lem: st2-aux-p}, and \Cref{lem: st2-ph1-AttnS} and $N \gg K^3$. On the other hand, we have
\begin{align}
    - (I) 
    &
    \overset{\RM{1}}{\leq}
    \eta_2 \big\{2p_k\cdot\PP(\pit \notin \esb)+p_k\cdot\PP(\pit \in \esb) \nonumber\\
    &
     \quad \cdot \EE\Big[\Theta\Big(\frac{1-\Attn_{k}^{(2)}(t)}{K}\Big)\Big(\Theta\Big(\frac{1-\Attn_{k}^{(2)}(t)}{K}\Big)+\Attn_{k}^{(2)}(1-\Attn_{k}^{(2)}(t))\Big)\Big|\{\xbq=\vb_k\}\cap\{\pit \in \esb\}\Big]\Big\}\nonumber\\
    &
    \overset{\RM{2}}{\leq}
    O\Big(\frac{b_{k}(t)}{K}\Big) \label{Eq: ph1-L2},
\end{align}
where $\RM{1}$ follows from \Cref{lem: st2-ph1-AttnS}, $\RM{2}$ follows from \Cref{lem: st2-aux-p}, Claim (1), which states $b_k(t)=\Omega\left({\eta_2}/{K^2}\right)$ and the fact that ${\eta_2}/{K^2} \geq \exp \left(-{c^2 N}/(25 K^2)\right)$. As a result, combining \eqref{Eq: ph1-U2} and \eqref{Eq: ph1-L2}, we prove Claim (2) and this lemma.
\end{proof}

\subsubsection{End of Phase I}
\begin{lemma}\label{lem: st2-ph1-end}
    Conditioned on the events that $\xbq=\vb_k$ and $\pit \in \esb$, with $T_{1,k}^{(2)}$ at most $O\left({ (K^2 \log K)}/{\eta_2}\right)$, and at iteration $t= \tau_1+T_{1,k}^{(2)}+1$,  we have
    \begin{enumerate}
        \item $B_k(\tau_1+T_{1,k}^{(2)}+1)\geq \log(K)$,
        \item $\Attn^{(2)}_{k}(\tau_1+T_{1,k}^{(2)}+1)=\Omega(1)$.
    \end{enumerate}
\end{lemma}
\begin{proof}
Claim (1) holds from the definition of Phase I. We only need to show the upper bound of $T_{1,k}^{(2)}$. By direct calculation, we have 
    \begin{align*}
        \log (K) 
        & 
        \geq B_k(\tau_1+T_{1,k}^{(2)})\\
        &
        \geq
        B_k(\tau_1+1)+ \sum_{t=\tau_1+1}^{\tau_1+T_{1,k}^{(2)}-1}B_k(t+1)-B_k(t)\\
        &
        \geq
        1-\Omega(\epsI+{N^{1-\alpha}})+(T_{1,k}^{(2)}-2)\Omega({\eta_2}/{K^2}),
    \end{align*}
    where the last inequality follows from \Cref{lem: st2-ph1-b,hp2-1}. As a result, $T_{1,k}^{(2)} \leq O\left({ (K^2 \log K)}/{\eta_2}\right)$. We next prove Claim (2). Similar to \eqref{Eq: st2-ph1-S1} and by $B_k(\tau_1+T_{1,k}^{(2)}+1)\geq \log(K)$, we have
    \begin{align*}
        &
        \Attn_k^{(2)}(\tau_1+T_{1,k}^{(2)}+1) \gtrsim
        \frac{1}{1+\sum_{m \neq k}|\cV_m|(|\cV_k|K)+\Nit(|\cV_k|K)} \overset{\RM{1}}{\simeq} \Theta(1),
    \end{align*}
    where $\RM{1}$ follows from the fact $|\cV_k|=\Theta({N}/{K})$ in \Cref{lem: st2-aux-p}. Hence, we prove Claim (2).
\end{proof}
\subsection{Stage II: Phase II}\label{app: s2p2}
We define the \textbf{Phase II} as all iterations $\tau_1+T^{(2)}_{1,k}+1 \leq t\leq \tau_1+T^{(2)}_{2,k}$, where
$$
T^{(2)}_{2,k} \triangleq \max \Big\{t:  B_{k}(\tau_1+t)-\max_{m\neq k}B_m(\tau_1+t)\leq \log\big({K}{\epsII^{-\frac{1}{2}}}\big) \Big\}.
$$
We then state the following induction hypothesis, which holds throughout Phase II. This hypothesis is proved by induction with the technical lemmas in \Cref{app: st2-ph2-techlem}. 

\begin{hypothesis}\label{hp2-2}
Suppose \Cref{lem: st2-ph2-AttnS,lem: st2-ph2-b,hp2-1} hold at iteration $t-1$, where $\tau_1+T^{(2)}_{1,k}+1 \leq t\leq \tau_1+T^{(2)}_{2,k}$. Given any $k \in [K]$, conditioned on the event that $\xbq=\vb_k$ and $\pit \in \esb$, the following holds for iteration $t$:  
\begin{enumerate}
    \item $B_{k}{(t)}$ is monotonically increasing and $B_{k}{(t)}\in [\log(K), O(\log(K/\epsII))]$.
    \item For any $k'\not=k$, $B_{k'}^{(t)}$ is monotonically decreasing and $|B_{k'}{(t)}|=O(B_{k}{(t)}/{K})$.
\end{enumerate}
\end{hypothesis}
\begin{proof}
We first prove {Claim (2)}.  By \Cref{lem: st2-ph2-b}, we have for any $k' \neq k$, $|b_{k'}(t-1)| \leq O\left({b_{k}(t-1)}/{K}\right)$ at time $t-1$, then 
\begin{align*}
    |B_{k'}(t)| 
    &
    \leq |B_{k'}(t-1)| + |b_{k'}(t-1)|
    \leq
    O\left({B_{k}(t)}/{K}\right).
\end{align*}
Hence, we prove Claim (2). We next prove {Claim (1)}. For any $\tau_1+T^{(2)}_{1,k}+1 \leq t\leq \tau_1+T^{(2)}_{2,k}$, we first need to show $b_k(t)\geq 0$, which is obvious from \Cref{lem: st2-ph2-b}. 
    Then, by the monotonicity and \Cref{lem: st2-ph1-end}, we have $\log(K) \leq B_{k}{(t)}$. In addition, we have  
    \begin{align*}
        \left(1-O\left({1}/{K}\right)\right)B_k(t) \leq  B_{k}(t)-\max_{m\neq k}B_m(t)\leq \log({K}{\epsII^{-\frac{1}{2}}}) \leq O\left(\log({K}/{\epsII})\right),
    \end{align*}
    where the first inequality follows from Claim (2), and the second inequality follows from the definition of Phase II. Hence, we prove Claim (1).
\end{proof}
\subsubsection{Technical Lemmas}\label{app: st2-ph2-techlem}
Similar to phase I, we next introduce several useful technical lemmas and prove them by induction.

\begin{lemma}\label{lem: st2-ph2-AttnS}
    Suppose \Cref{hp2-2} holds at iteration $\tau_1+T^{(2)}_{1,k}+1\leq t\leq \tau_1+T^{(2)}_{2,k}$, conditioned on the events that $\xbq=\vb_k$ and $\pit \in \esb$, the following holds for iteration $t$:   
    \begin{enumerate}
    \item  For the attention scores related to feature $k$, we have that
    \begin{enumerate}
        \item $\Attn_k^{(2)}(t)=\Omega(1)$,
        \item $(1-\Attn_k^{(2)}(t))^2 \geq \Omega(\epsII)$.
    \end{enumerate}
     \item  For the attention scores related to feature $k' \neq k$, $\Attn_{k'}^{(2)}(t)=\Theta\big(\big({1-\Attn_k^{(2)}(t)}\big)/{K}\big)$.
    \end{enumerate}
\end{lemma}
\begin{proof}
We first prove {Claim (1)-(a)}. Similar to \Cref{lem: st2-ph1-AttnS}, we have
\begin{align*}
        \Attn_{k}^{(2)}(t)
        &
        \geq
        \frac{1}{\sum_{m \in [K]}(|\cV_m|/|\cV_k|)\exp{(B_m(t)-B_k(t))}+(\Nit/|\cV_{k}|)\exp{\left(\epsI \Theta\left(\max_{m\in [K]} B_m(t)\right)-B_{k}(t)\right)}}\\
        &
        \overset{\RM{1}}{\geq}
        \frac{1}{1+(\Nit/|\cV_{k}|-1)\exp{(\max_{m \neq k}B_m(t)-B_k(t))}+(\Nit/|\cV_{k}|)\exp{(-B_k(t)(1-\Theta(\epsI)))}}\\
        &
        \overset{\RM{2}}{\gtrsim}
        \frac{1}{1+({K}/{\Lba_k})\cdot({\epsII^{1/2}}/{K})+{K^{\Theta(\epsI)}}/{\Lba_k}}=\Omega(1),
    \end{align*}
where $\RM{1}$ follows from the fact $\max_{m \neq k}B_{m}(t) \leq B_k(t)$ in \Cref{hp2-2}, and $\RM{2}$ follows from the bound of $B_k(t)$ in \Cref{hp2-2} and \Cref{lem: st2-aux-p}. Hence, we prove Claim (1)-(a). Similarly, for Claim (1)-(b), we have 
\begin{align*}
        1-\Attn_{k}^{(2)}(t)
        &
        =\frac{\sum_{m \neq k}|\cV_m|\exp{(B_{m}(t))}+\sum_{i=2n-1,n\in[\Nit+1]}\exp{(\ub_i(t))}}{|\cV_k|\exp{(B_{k}(t))}+\sum_{m \neq k}|\cV_m|\exp{(B_{m}(t))}+\sum_{i=2n-1,n\in[\Nit+1]}\exp{(\ub_i(t))}}\\
        &
        \overset{\RM{1}}{\geq}
        \frac{\sum_{m \neq k}(|\cV_m|/|\cV_k|)\exp{(B_m(t)-B_k(t))}}{1+\sum_{m \neq k}(|\cV_m|/|\cV_k|)\exp{(B_m(t)-B_k(t))}+(N/|\cV_k|)\exp{(-B_k(t)(1-\Theta(\epsI)))}}\\
        &
        \overset{\RM{2}}{\geq}
        \frac{\exp{(\max_{m\neq k}B_m(t)-B_k(t)-\Delta B(t))}({K}/{\Uba_k}-1)}{1\!+\!\exp{(\max_{m\neq k}B_m(t)\!-\!B_k(t)\!-\!\Delta B(t))}({K}/{\Uba_k}\!-\!1)\!+\!({K}/{\Uba_k})\exp{(-B_k(t)(1-\Theta(\epsI)))}}\\
        &
        \overset{\RM{3}}{\geq}
        \frac{K^{-1}\epsII^{1/2}\exp{(-O(K^{-1}\log(K/\epsII))})\cdot({K}/{\Uba_k}-1)}{1+(K^{-1}\epsII^{1/2}\exp{-O(K^{-1}\log(K/\epsII))})({K}/{\Uba_k}-1)+K^{\Theta(\epsI)}/{\Lba_k}} \geq \Omega\left(\epsII^{\frac{1}{2}}\right),
    \end{align*}
where $\RM{1}$ follows from \Cref{lem: st2-var_upd} and \Cref{hp2-2}, $\RM{2}$ follows by defining $\Delta B(t)=\max_{m\neq k}B_m(t)-\min_{m\neq k}B_m(t)$, and $\RM{3}$ follows from \Cref{hp2-2} and the fact that $\Delta B(t)=O\left({B_k(t)}/{K}\right)=O\left(K^{-1}\log(K/\epsII)\right)$. Hence, we prove Claim (1)-(b).
    
We next prove Claim (2). For any $k' \neq k$, we have
    \begin{align*}
        \Attn_{k'}^{(2)}(t)
        &
        =\frac{|\cV_{k'}|\exp{(B_{k'}(t))}}{|\cV_{k'}|\exp{(B_{k'}(t))}+\sum_{m \neq k'}|\cV_m|\exp{(B_{m}(t))}+\sum_{i=2n-1,n\in[\Nit+1]}\exp{(\ub_i(t))}}.
    \end{align*}
    Then, similar to proof of Claim (2) in \Cref{lem: st2-ph1-AttnS}, we have 
    \begin{align*}
        \frac{\Attn^{(2)}_{k'}(t)}{1-\Attn_k^{(2)}(t)}
        &
        =
        \frac{|\cV_{k'}|\exp{(B_{k'}(t))}}{|\cV_{k'}|\exp{(B_{k'}(t))}+\sum_{m \neq k',k}|\cV_m|\exp{(B_{m}(t))}+\sum_{i=2n-1,n\in[\Nit+1]}\exp{(\ub_i(t))}}\\
        &
        \leq
        \frac{1}{1+\sum_{m \neq k',k}(|\cV_m|/|\cV_{k'}|)\exp{(B_{m}(t)-B_{k'}(t))}}\\
        &
        \leq
        \frac{1}{1+\sum_{m \neq k',k}(\Uba_m/\Lba_{k'})\exp{(-O({K}^{-1}\log(K/\epsII)))}}= O\left(\frac{1}{K}\right),
    \end{align*}
  where the last inequality follows from the fact ${\Uba_m}/{\Lba_{k'}}=\Theta(1)$ in \Cref{lem: st2-aux-p} and \Cref{hp2-2}.
  
  Similar to \eqref{Eq: st2-ph1-S2}, we have
  \begin{align*}
     \frac{\Attn^{(2)}_{k'}(t)}{1-\Attn_k^{(2)}(t)}
     &
     \geq
     \frac{1}{1+\sum_{m \neq k',k}(|\cV_m|/|\cV_{k'}|)\exp{(B_m(t)-B_{k'}(t))}+(\Nit/|\cV_{k'}|)\exp{(\epsI \Theta\left( B_k(t)\right)-B_{k'}(t))}}\\
    &
    \overset{\RM{1}}{\geq}
    \frac{1}{1+\sum_{m \neq k',k}\Uba_m e/{\Lba_{k'}}+Ke/\Lba_{k'}}
    \geq
    \Omega\left(\frac{1}{K}\right),
  \end{align*}
where $\RM{1}$ follows from \Cref{lem: st2-var_upd} and \Cref{hp2-2}. Hence, we prove \Cref{lem: st2-ph2-AttnS}.
\end{proof}

In the following lemma, we control the order of $b_{k'}(t),k' \in [K]$.
\begin{lemma}[Order of updates $b(t)$]\label{lem: st2-ph2-b}
    Suppose \Cref{hp2-2,lem: st2-ph2-AttnS} hold at iteration $\tau_1+T^{(2)}_{1,k}+1\leq t\leq \tau_1+T^{(2)}_{2,k}$, under the same assumption of \Cref{lem: st2-ph2-AttnS}, we have that
    \begin{enumerate}
    \item For $k$, $b_k(t)\geq 0$ and $b_k(t)=\Omega({\eta_2\epsII}/{K})$,
    \item For $k' \neq k$, $b_{k'}(t) \leq 0$, and $|b_{k'}(t)| \leq O({b_{k}(t)}/{K})$.
    \end{enumerate}
\end{lemma}
\begin{proof}
    We first prove Claim (1), by \Cref{lem: st2-var_upd}, and $\epsII \geq \epsI$, we only need to consider the dominated term of $b_{k}(t)$ as follows:
    \begin{align*}
        &\eta_2 \EE\Big[\mathbbm{1}{\{\xbq=\vb_k\}}\Big(\Attn_k^{(2)}\Big((\Attn_k^{(2)}-1)^2+\sum_{m \neq k}(\Attn_{m}^{(2)})^2\Big)\Big)\Big]\\
        & 
        \quad
        \geq
        \eta_2 p_k\PP(\pit \in \esb)
        \times \EE\big[\Attn_k^{(2)}\big(\Attn_k^{(2)}-1\big)^2|\{\xbq=\vb_k\}\cap\{\pit \in \esb\}\big]\!+3\eta_2 p_k\PP(\pit \notin \esb)\\
        &
        \quad
        \overset{\RM{1}}{\geq}
        \Omega\left(\frac{\eta_2\epsII}{K}\right), 
    \end{align*}
    where $\RM{1}$ follows from \Cref{lem: st2-aux-p,lem: st2-ph2-AttnS}. Hence, we prove Claim (1). We next prove Claim (2). By \Cref{lem: st2-var_upd}, for any $k' \neq k$, denote the dominated term of $b_{k'}(t)$ as follows
    \begin{align*}
       (I) 
    &
    =
    \eta_2 \EE\Big[\mathbbm{1}{\{\xbq=\vb_k\}}\Big(\Attn_{k'}^{(2)}\sum_{m\in [K]}(\Attn_{m}^{(2)})^2-(\Attn_{k'}^{(2)}+\Attn_k^{(2)})\Big)\Big].
    \end{align*}
On one hand, we have
\begin{align}
    (I) 
        &
        \leq \eta_2 \EE\Big[\mathbbm{1}{\{\xbq=\vb_k\}\cap\{\pit \in \esb\}}\Attn_{k'}^{(2)}\Big(\sum_{m\notin k,k'}(\Attn_{m}^{(2)})^2-\Attn_{k}^{(2)}(1-\Attn_{k}^{(2)})\Big)\Big]
        +
        \eta_2 p_k \PP(\pit \notin \esb)\nonumber\\
        &
        \overset{\RM{1}}{\lesssim}
        \eta_2\Big(-p_k\mathbb{E}\Big[\frac{(1-\Attn_k^{(2)})^2}{K}\Big|\{\xbq=\vb_k\} \cap \{\pit \in \esb\} \Big]+3 p_k\exp \Big(-\frac{c^2 N}{25 K^2}\Big)\Big)\nonumber\\
        &
        \overset{\RM{2}}{\leq} - \Omega\left(\frac{\eta_2\epsII}{K}\right) \leq 0\label{Eq: ph2-U2},
\end{align}
where $\RM{1}$ follows from \Cref{lem: st2-aux-p} and the fact that $\Attn_{k'}^{(2)}(t)=\Theta\big(\big({1-\Attn_k^{(2)}(t)}\big)/{K}\big)$ in \Cref{lem: st2-ph2-AttnS}, and $\RM{2}$ follows from the fact $(1-\Attn_k^{(2)}(t))^2 \geq \Omega(\epsII)$ in \Cref{lem: st2-ph1-AttnS} and $N \gg K^3$.

On the other hand, we have
\begin{align}
    - (I) 
    &
    \overset{\RM{1}}{\leq}
    \eta_2 \big\{2p_k\cdot\PP(\pit \notin \esb)+p_k\cdot\PP(\pit \in \esb) \times \nonumber\\
    &
   \quad \EE\Big[\Theta\Big(\frac{1-\Attn_{k}^{(2)}(t)}{K}\Big)\Big(\Theta\Big(\frac{1-\Attn_{k}^{(2)}(t)}{K}\Big)+\Attn_{k}^{(2)}(1-\Attn_{k}^{(2)}(t))\Big)\Big|\{\xbq=\vb_k\}\cap\{\pit \in \esb\}\Big]\Big\}\nonumber\\
    &
    \overset{\RM{2}}{\leq}
    \eta_2 \Big\{p_k  \EE\Big[O\Big(\frac{\Attn_{k}^{(2)}(t)(1-\Attn_{k}^{(2)}(t))^2}{K}\Big)|\{\xbq=\vb_k\}\cap\{\pit \in \esb\}\Big]+6 p_k\exp \Big(-\frac{c^2 N}{25 K^2}\Big)\Big\}\nonumber\\
    &
    \overset{\RM{3}}{\leq}
     O\Big(\frac{b_{k}(t)}{K}\Big) \label{Eq: ph2-L2},
\end{align}
where $\RM{1}$ follows from \Cref{lem: st2-ph2-AttnS}, $\RM{2}$ follows from \Cref{lem: st2-aux-p}, and $\RM{3}$ follows from Claim (1). Finally, combining \eqref{Eq: ph2-U2} and \eqref{Eq: ph2-L2}, we prove Claim (2).
\end{proof}
\subsubsection{End of Phase II}
\begin{lemma}\label{lem: st2-ph2: end}
    Conditioned on the events that $\xbq=\vb_k$ and $\pit \in \esb$, with $T_{2,k}^{(2)}-T_{1,k}^{(2)}$ at most $O({\eta_2^{-1}\epsII^{-1}}{K\log(K\epsII^{-\frac{1}{2}})})$, and at iteration $t= \tau_1+T_{2,k}^{(2)}+1$,  we have that
    \begin{enumerate}
        \item $B_k(\tau_1+T_{2,k}^{(2)}+1)-\max_{m\neq k}B_m(\tau_1+T_{2,k}^{(2)}+1)\geq \log\big({K}{\epsII^{-\frac{1}{2}}}\big)$,
        \item $1-\Attn^{(2)}_{k}(\tau_1+T_{2,k}^{(2)}+1)\leq \epsII^{\frac{1}{2}}$.
    \end{enumerate}
\end{lemma}
\begin{proof}
    Claim (1) holds from the definition of Phase II.
    We only need to show the upper bound of  $T_{2,k}^{(2)}$. By \Cref{hp2-2,lem: st2-ph2-b}, for any $\tau_1+T^{(2)}_{1,k}+1 \leq t\leq \tau_1+T^{(2)}_{2,k}$, we have 
    \begin{align*}
        & 
        \big(B_{k}(t+1)-\max_{m\neq k}B_m(t+1)\big) - \big(B_{k}(t)-\max_{m\neq k}B_m(t)\big)
        \geq
        (1-O({1}/{K}))b_k(t) = \Omega ({\eta_2\epsII}/{K}).
    \end{align*}
    By direct calculation, we have
    \begin{align*}
        &\big(B_{k}(\tau_1+T^{(2)}_{2,k})-\max_{m\neq k}B_m(\tau_1+T^{(2)}_{2,k})\big) - \big(B_{k}(\tau_1+T^{(2)}_{1,k}+1)-\max_{m\neq k}B_m(\tau_1+T^{(2)}_{1,k}+1)\big)\\
        &
        \qquad
        =\sum_{t=\tau_1+T^{(2)}_{1,k}+1}^{\tau_1+T^{(2)}_{2,k}-1} \big(B_{k}(t+1)-\max_{m\neq k}B_m(t+1)\big) - \big(B_{k}(t)-\max_{m\neq k}B_m(t)\big)\\
        &
        \qquad
        \geq
        \big(T^{(2)}_{2,k}-T^{(2)}_{1,k}-1\big)\Omega \big({\eta_2\epsII}/{K}\big).
    \end{align*}
    Dividing both side by $\Omega \big({\eta_2\epsII}/{K}\big)$, we have
    \begin{align*}
        T^{(2)}_{2,k}-T^{(2)}_{1,k} \leq 
        O\bigg(\frac{K\log(K\epsII^{-\frac{1}{2}})}{\eta_2\epsII}\bigg).
    \end{align*}
    Hence, we prove Claim (1). Then we prove Claim (2). For abbreviation, we omit the iteration $\tau_1+T_{1,k}^{(2)}+1$ in the following. By definition, we have 
\begin{align*}
        1
        -\Attn_{k}^{(2)}
        &
        =\frac{\sum_{m \neq k}|\cV_m|\exp{(B_{m})}+\sum_{i=2n-1,n\in[\Nit+1]}\exp{(\ub_i)}}{\sum_{m \in [K]}|\cV_m|\exp{(B_{m})}+\sum_{i=2n-1,n\in[\Nit+1]}\exp{(\ub_i)}}\\
        &
        \overset{\RM{1}}{\leq}
        \frac{\sum_{m \neq k}(|\cV_m|/|\cV_k|)\exp{(B_m-B_k)}+(\Nit/|\cV_k|)\exp{(-B_k(1-\Theta(\epsI)))}}{1+\sum_{m \neq k}(|\cV_m|/|\cV_k|)\exp{(B_m-B_k)}}\\
        &
        \overset{\RM{2}}{\leq}
        \frac{\exp{(\max_{m\neq k}B_m-B_k)}({K}/{\Lba_k}-1)+{K}/{\Lba_k}\exp{(-B_k(1-\Theta(\epsI)))}}{1+\exp{(\max_{m\neq k}B_m(T_{1,k}^{(2)}+1)-B_k(T_{1,k}^{(2)}+1))}({K}/{\Lba_k}-1)}\\
        &
        \overset{\RM{3}}{\leq}
        \frac{K^{-1}{\epsII^{1/2}}\cdot(K/{\Lba_k}-1)+(K/{\Lba_k})\cdot K^{-1}\epsII^{1/2}}{1+K^{-1}{\epsII^{1/2}}\cdot(K/{\Lba_k}-1)} \leq O\big(\epsII^{\frac{1}{2}}\big),
    \end{align*}
where $\RM{1}$ follows from the fact $\ub_i(t) \leq \epsI \Theta(\max_{m\in [K]} B_m(t))$ for all $t$ in \Cref{lem: st2-var_upd}, $\RM{2}$ follows from the bound of $|\cV_k|$ in \Cref{lem: st2-aux-p}, and $\RM{3}$ follows from Claim (1) and the definition of phase II. Hence, we finish the proof of \Cref{lem: st2-ph2: end}. 
\end{proof}
\section{Proof of \Cref{Thm: stage1 & 2}} \label{app: proof-converge}
\begin{proof}
    The Claims about Stage I and Stage II follow directly from \Cref{lem: st1-ph3-end,lem: st2-ph2: end}, respectively. Then we only need to show the loss convergence. The loss can be rewritten as
    \begin{align*}
        L(\theta(\tau_1+\tau_2+1))
        &
        =
        \frac{1}{2}\EE\Big[\Big(\sum_{i=1}^{N}\Attn_i^{(2)}(\tau_1+\tau_2+1))E^y-y_{N+1}\Big)^2\Big]\\
        &
        =
        \frac{1}{2}\sum_{k=1}^K\EE\Big[\mathbbm{1}{\{\xbq=\vb_k\}}\Big(\sum_{m=1}^K\Attn_m^{(2)}(\tau_1+\tau_2+1)\wb^\top\vb_m-\wb^\top\vb_{k}\Big)^2\Big]\\
        &
        \overset{\RM{1}}{=}
        \frac{1}{2}\sum_{k=1}^K\EE\Big[\mathbbm{1}{\{\xbq=\vb_k\}}\Big(\sum_{m\neq k}(\Attn_m^{(2)}(\tau_1+\tau_2+1))^2+
        (1-\Attn_k^{(2)}(\tau_1+\tau_2+1))^2\Big)\Big]\\
        &
        \leq
        \sum_{k=1}^K\EE\big[\mathbbm{1}{\{\xbq=\vb_k\}}\big(1-\Attn_k^{(2)}(\tau_1+\tau_2+1)\big)^2\big]\\
        & 
        \overset{\RM{2}}{\leq}
        \epsII,
    \end{align*}
    where $\RM{1}$ follows from taking the conditional expectation of $\wb$ conditioned on $P$ and direct calculation, and $\RM{2}$ follows from \Cref{lem: st2-ph2: end}. Hence, we finish the proof.
\end{proof}

\end{document}